%% file: main.tex
\pdfoutput=1
\documentclass[11pt,a4paper]{memoir}

\newbool{FINAL}
\setbool{FINAL}{true}

\input{setup/packages_thesis}
\input{setup/layout}
\input{setup/style}
\input{setup/macros}

\title{Near-optimal Active Reconstruction}
\author{Daniel Yang}
\thesistype{Bachelor Thesis}
\advisors{Advisors: Prof.\ Dr.\ Andreas Krause, Manish Prajapat, Jelena Trisovic}
\department{Department of Computer Science}
\date{April 12, 2023}

\begin{document}

\frontmatter

\ifbool{FINAL}{}{
    \listoftodos*
    \cleartorecto
}

\begin{titlingpage}
  \calccentering{\unitlength}
  \begin{adjustwidth*}{\unitlength-24pt}{-\unitlength-24pt}
    \maketitle
  \end{adjustwidth*}
\end{titlingpage}

\input{00_abstract}
\cleartorecto 
\input{00_acknowledgements}
\cleartorecto

\begingroup
    \hypersetup{linkcolor=black} 
    \tableofcontents
    \cleartorecto
\endgroup

\input{00_notations}

\mainmatter

\input{01_introduction}
\input{02_background}
\input{03_related_work}
\input{04_problem}
\input{05_design}
\input{06_analysis}
\input{07_experiments}
\input{08_conclusion}

\printbibliography

\appendix

\input{A_proofs}
\input{B_simulation}

\cleartorecto
\includepdf[pages={-}]{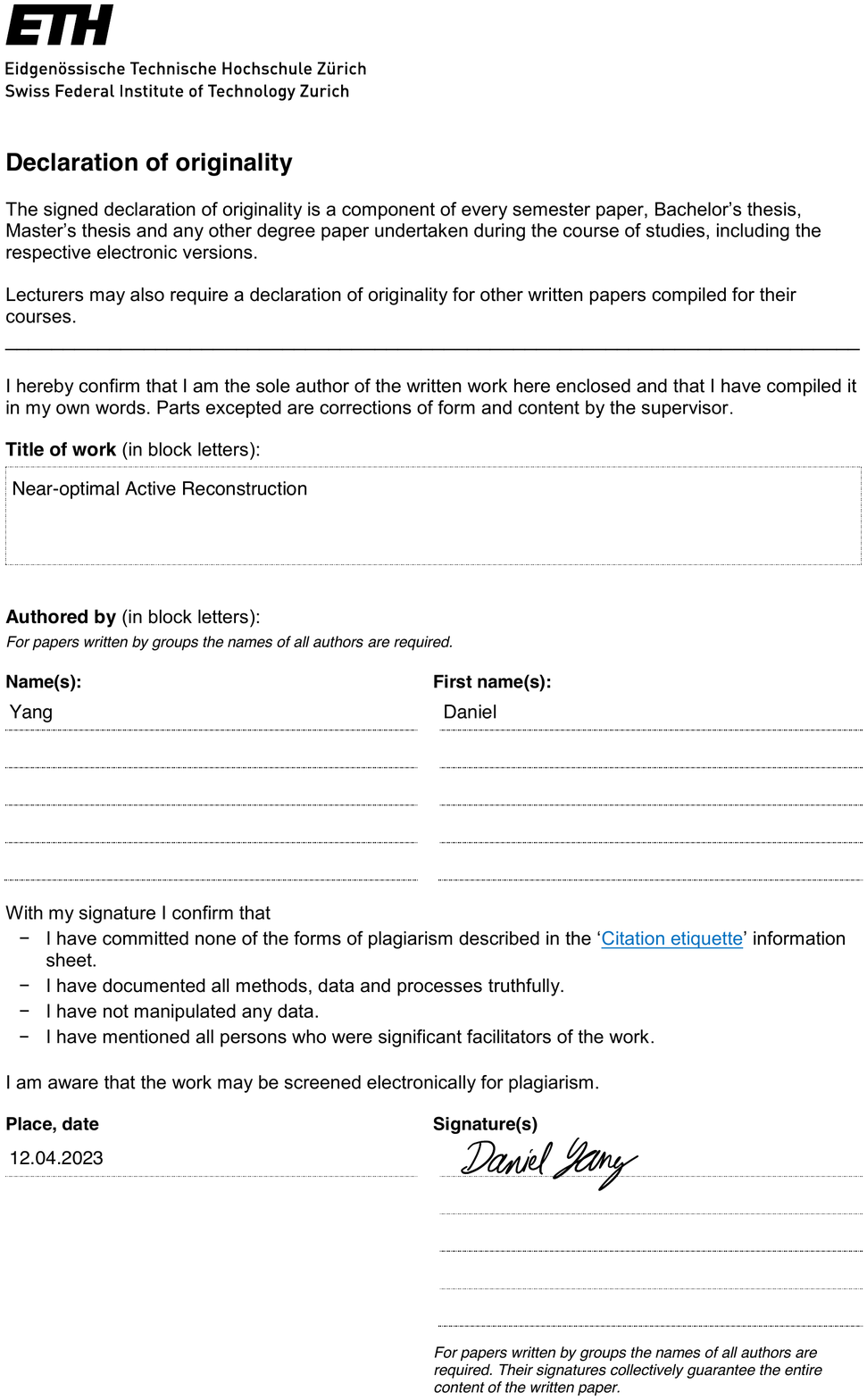}

\end{document}

%% file: setup/packages_thesis.tex
\usepackage[utf8]{inputenc} 
\usepackage[T1]{fontenc} 
\usepackage{lmodern} 
\usepackage[american]{babel} 
\usepackage[sc]{mathpazo} 
\usepackage[activate]{microtype} 

\input{setup/packages} 

\usepackage[
    backend=biber,
    style=authoryear-comp, date=year,
]{biblatex} 
\usepackage{hyperref} 
\usepackage[capitalise,nameinlink]{cleveref} 
\usepackage{crossreftools}
\makeatletter
\RenewDocumentCommand\nameref{s}{\IfBooleanTF{#1}{\@namerefstar}{\T@nameref}}
\makeatother

\usepackage{pdfpages} 

%% file: setup/packages.tex
\usepackage{suffix} 
\usepackage{etoolbox} 
\usepackage{xpatch} 

\usepackage{float}
\usepackage{csquotes} 
\usepackage{enumitem} 
\usepackage[dvipsnames]{xcolor} 

\usepackage{amsmath,amssymb,amsfonts,mathrsfs}
\usepackage{amsthm}
\usepackage{mathtools} 

\usepackage{subcaption} 

\usepackage{tcolorbox} 
\tcbuselibrary{skins,breakable}

\usepackage{graphicx} 
\usepackage{tikz}
\usetikzlibrary{positioning,arrows.meta}
\usetikzlibrary{patterns,shapes.arrows} 

\usepackage{pgfplots}
\usepgfplotslibrary{groupplots,dateplot}
\pgfplotsset{compat=newest}

\usepackage{booktabs} 
\usepackage{tabularx} 
\usepackage{ltablex} 
\usepackage{multicol} 

\usepackage{algorithm}
\usepackage[noend]{algpseudocode} 

%% file: setup/layout.tex
\nonzeroparskip
\defaultlists

\makeatletter

\if@twoside
  \copypagestyle{chapter}{Ruled}
\else
  \copypagestyle{chapter}{ruled}
\fi
\makeoddhead{chapter}{}{}{}
\makeevenhead{chapter}{}{}{}
\makeheadrule{chapter}{\textwidth}{0pt}
\copypagestyle{abstract}{empty}

\makechapterstyle{bianchimod}{%
  \chapterstyle{default}
  \renewcommand*{\chapnamefont}{\normalfont\Large\sffamily}
  
  \renewcommand*{\printchaptername}{%
    \chapnamefont\centering\@chapapp}

  }

\chapterstyle{bianchimod}

\setsecheadstyle{\Large\bfseries\sffamily}
\setsubsecheadstyle{\large\bfseries\sffamily}
\setsubsubsecheadstyle{\bfseries\sffamily}
\setparaheadstyle{\normalsize\bfseries\sffamily}
\setsubparaheadstyle{\normalsize\itshape\sffamily}
\setsubparaindent{0pt}

\captionnamefont{\sffamily\bfseries\footnotesize}
\captiontitlefont{\sffamily\footnotesize}
\setsecnumdepth{subsection}
\settocdepth{subsection}

\pretitle{\vspace{0pt plus 0.7fill}\begin{center}\HUGE\sffamily\bfseries}
\posttitle{\end{center}\par}
\preauthor{\par\begin{center}\let\and\\\Large\sffamily}
\postauthor{\end{center}}
\predate{\par\begin{center}\Large\sffamily}
\postdate{\end{center}}

\def\@advisors{}
\newcommand{\advisors}[1]{\def\@advisors{#1}}
\def\@department{}
\newcommand{\department}[1]{\def\@department{#1}}
\def\@thesistype{}
\newcommand{\thesistype}[1]{\def\@thesistype{#1}}

\renewcommand{\maketitlehookb}{\vspace{1in}%
  \par\begin{center}\Large\sffamily\@thesistype\end{center}}

\renewcommand{\maketitlehookd}{%
  \vfill\par
  \begin{flushright}
    \sffamily
    \@advisors\par
    \@department, ETH Z\"urich
  \end{flushright}
}

\checkandfixthelayout

\makeatother


%% file: setup/style.tex
\tcbset{
    every box/.style={
        enhanced, 
        boxsep=0mm, 
        boxrule=0.25mm, 
    },
    every box on layer 1/.style={
        every box,
        boxrule=0mm, frame hidden, 
    },
    defaultstyle/.style={
        coltitle=black,
        colback=#1!10!white,
        colbacktitle=#1!25!white,
        colframe=#1!50!white,
        halign title=center, halign upper=center, halign lower=center,
        valign upper=center, valign lower=center,
    },
    defaultstyle/.style/.default=gray,
    boxstyle/.style={
        defaultstyle={#1},
        left=0mm, right=0mm,
        top=1.75mm, bottom=1.75mm, middle=1.75mm,
        toptitle=1.25mm, bottomtitle=1.25mm,
    },
    circlestyle/.style={
        defaultstyle={#1},
        square, circular arc,
        tikznode, 
        left=1.75mm, right=1.75mm,
        top=1.75mm, bottom=1.75mm, middle=1.75mm,
    },
}

\setkeys{Gin}{ 
    keepaspectratio,
}

\theoremstyle{definition}

\newtheorem{definition}{Definition}[chapter]

\theoremstyle{plain}

\newtheorem{theorem}{Theorem}[chapter]
\newenvironment{theoremvar}[2][]{ 
    \let\oldthetheorem\thetheorem
    \renewcommand{\thetheorem}{\oldthetheorem{} (#2)}
    \renewcommand{\theHtheorem}{\oldthetheorem.#2} 
    \begin{theorem}[#1]\addtocounter{theorem}{-1}
}{
    \end{theorem}
}

\newtheorem{lemma}{Lemma}[chapter]
\newenvironment{lemmavar}[2][]{ 
    \let\oldthelemma\thelemma
    \renewcommand{\thelemma}{\oldthelemma{} (#2)}
    \renewcommand{\theHlemma}{\oldthelemma.#2} 
    \begin{lemma}[#1]\addtocounter{lemma}{-1}
}{
    \end{lemma}
}
\newtheorem{corollary}{Corollary}[chapter]

\theoremstyle{definition}

\tcbset{remarkstyle/.style={
    colback=gray!15,
    left=4mm, right=4mm,
    top=2mm, topsep at break=2mm, bottom=4mm, 
    sharp corners, breakable, use color stack,
}}
\newtheorem{remark}{Remark}[chapter]
\tcolorboxenvironment{remark}{remarkstyle}
\AtBeginEnvironment{remark}{\nonzeroparskip} 

\newtheorem*{remark*}{Remark}
\tcolorboxenvironment{remark*}{remarkstyle}
\AtBeginEnvironment{remark*}{\nonzeroparskip}

\newcounter{proof}
\newcounter{proofstep}
\numberwithin{proofstep}{proof}
\renewcommand{\theproofstep}{\arabic{proofstep}}

\newbool{proofstep} 
\newenvironment{proofstep}[1][]{
    \setbool{proofstep}{true}
    \ifstrempty{#1}
        {\paragraph{Proof Step~\theproofstep.}}
        {\paragraph{Proof Step~\theproofstep~\textnormal{(#1)}.}}
}{
    \setbool{proofstep}{false}
}

\AtBeginEnvironment{proof}{\refstepcounter{proof}}
\AtBeginEnvironment{proofstep}{\refstepcounter{proofstep}}

\numberwithin{equation}{chapter}

\newcommand{\eqnote}[1]{\textcolor{gray}{\text{(#1)}}} 
\WithSuffix\newcommand{\eqnote}*[1]{\textcolor{gray}{\text{#1}}} 

\colorlet{eqdefault}{black} 
\ifbool{FINAL}{
    \colorlet{eqchange}{black} 
    \colorlet{eqnochange}{gray!75!white} 
}{
    \colorlet{eqchange}{blue} 
    \colorlet{eqnochange}{lightgray} 
}

\newcommand{\eqcontrast}{\color{eqnochange}}
\newcommand{\eqnocontrast}{\color{eqdefault}} 
\newcommand{\eqchange}[1]{\mathcolor{eqchange}{#1}{}} 
\newcommand{\eqnochange}[1]{\mathcolor{eqnochange}{#1}{}} 

\newenvironment{derivation*}[1][2]{
    \allowdisplaybreaks
    \csname alignat*\endcsname{#1}
}{
    \csname endalignat*\endcsname
    \ignorespacesafterend
}

\newcounter{derivation} 
\newcounter{derivationstep}
\newcounter{derivationstepstart} 

\renewcommand{\thederivationstep}{\ifbool{proofstep}{\theproofstep.}{}\arabic{derivationstep}}

\newenvironment{derivation}[1][2]{
    \stepcounter{derivation}
    \setcounter{derivationstepstart}{\value{derivationstep}}
    \let\oldrefstepcounter\refstepcounter
    \renewcommand{\refstepcounter}[1]{\oldrefstepcounter{derivationstep}}
    \renewcommand{\theequation}{\thederivationstep}
    \allowdisplaybreaks
    \alignat{#1}
}{
    \endalignat
    \ifbool{proofstep}{}{\setcounter{derivationstep}{0}}%
}

\makeatletter
\newenvironment{derivationinline}[1][2]{
    \stepcounter{derivation}
    \setcounter{derivationstepstart}{\value{derivationstep}}
    \let\label@in@display\label
    \newcommand{\tageq}[1]{\refstepcounter{derivationstep}\stackrel{\eqnocontrast(\thederivationstep)}{##1}}
    \allowdisplaybreaks
    \csname alignat*\endcsname{#1}
}{
    \csname endalignat*\endcsname
    \ifbool{proofstep}{}{\setcounter{derivationstep}{0}}%
}
\makeatother
\AtEndEnvironment{proofstep}{\setcounter{derivationstep}{0}}

\newenvironment{justification}{
    \begin{enumerate}[label=(\ifbool{proofstep}{\theproofstep.}{}\arabic*), start=\value{derivationstepstart}+1,left=0.8cm]
}{
    \end{enumerate}
}

\definecolor{solblue}{HTML}{268BD2}
\hypersetup{
  bookmarksnumbered,
  colorlinks=true,
  linkcolor=\ifbool{FINAL}{black}{solblue},
  citecolor=solblue,
  filecolor=solblue,
  urlcolor=solblue,
}

\crefname{enumi}{}{} 
\crefname{proofstep}{Proof Step}{Proof Steps}
\crefname{derivationstep}{}{}
\creflabelformat{derivationstep}{#2(#1)#3} 
\creflabelformat{equation}{#2#1#3} 

\input{setup/style_biblatex}
\input{setup/style_biblatex_patch} 

%% file: setup/style_biblatex.tex
\DeclareNameAlias{sortname}{given-family} 
\DeclareDelimFormat{nameyeardelim}{\addcomma\space} 

\DefineBibliographyStrings{english}{
  page={page},
  pages={pages},
  backrefpage={see page},
  backrefpages={see pages},
}

\renewbibmacro{in:}{}
\AtEveryBibitem{
  \clearfield{eventtitle}
  \clearfield{pagetotal}
}

\DeclareFieldFormat[online,report]{title}{\mkbibquote{#1\isdot}}

\DeclareFieldFormat{urldate}{\addcomma\space\bibstring{urlseen}\space#1}

\DeclareFieldFormat*{volume}{\bibstring{volume}\addnbspace#1}
\DeclareFieldFormat*{number}{\bibstring{number}\addnbspace#1}
\renewbibmacro*{journal+issuetitle}{%
  \usebibmacro{journal}%
  \newunit
  \iffieldundef{series}
    {}
    {%
      \newunit
      \printfield{series}%
      \newunit
    }%
  \usebibmacro{volume+number+eid}%
  \setunit{\addspace}%
}
\renewbibmacro*{volume+number+eid}{%
  \printfield{volume}%
  \newunit
  \printfield{number}%
  \newunit
  \printfield{eid}%
}

%% file: setup/style_biblatex_patch.tex
\xpatchbibmacro{cite}
  {\usebibmacro{cite:label}%
   \setunit{\printdelim{nonameyeardelim}}%
   \usebibmacro{cite:labeldate+extradate}}
  {\printtext[bibhyperref]{%
     \DeclareFieldAlias{bibhyperref}{default}%
     \usebibmacro{cite:label}%
     \setunit{\printdelim{nonameyeardelim}}%
     \usebibmacro{cite:labeldate+extradate}}}
  {}
  {\PackageWarning{biblatex-patch}
     {Failed to patch cite bibmacro}}

\xpatchbibmacro{cite}
  {\printnames{labelname}%
   \setunit{\printdelim{nameyeardelim}}%
   \usebibmacro{cite:labeldate+extradate}}
  {\printtext[bibhyperref]{%
     \DeclareFieldAlias{bibhyperref}{default}%
     \printnames{labelname}%
     \setunit{\printdelim{nameyeardelim}}%
     \usebibmacro{cite:labeldate+extradate}}}
  {}
  {\PackageWarning{biblatex-patch}
     {Failed to patch cite bibmacro}}

\makeatletter
\protected\def\blx@imc@biblinkstart{%
  \@ifnextchar[
    {\blx@biblinkstart}
    {\blx@biblinkstart[\abx@field@entrykey]}}
\def\blx@biblinkstart[#1]{%
  \blx@sfsave\hyper@natlinkstart{\the\c@refsection @#1}\blx@sfrest}
\protected\def\blx@imc@biblinkend{%
  \blx@sfsave\hyper@natlinkend\blx@sfrest}
\blx@regimcs{\biblinkstart \biblinkend}
\makeatother

\newbool{cbx:link}

\def\iflinkparens{%
  \ifboolexpr{ test {\ifnumequal{\value{multicitetotal}}{0}} and
               test {\ifnumequal{\value{citetotal}}{1}} and
               test {\iffieldundef{prenote}} and
               test {\iffieldundef{postnote}} }}

\xpatchbibmacro{textcite}
  {\printnames{labelname}}
  {\iflinkparens
     {\DeclareFieldAlias{bibhyperref}{default}%
      \global\booltrue{cbx:link}\biblinkstart%
      \printnames{labelname}}
     {\printtext[bibhyperref]{\printnames{labelname}}}}
  {}
  {\PackageWarning{biblatex-patch}
     {Failed to patch textcite bibmacro}}

\xpatchbibmacro{textcite}
  {\usebibmacro{cite:label}}
  {\iflinkparens
     {\DeclareFieldAlias{bibhyperref}{default}%
      \global\booltrue{cbx:link}\biblinkstart%
      \usebibmacro{cite:label}}
     {\usebibmacro{cite:label}}}
  {}
  {\PackageWarning{biblatex-patch}
     {Failed to patch textcite bibmacro}}

\xpretobibmacro{textcite:postnote}
  {\ifbool{cbx:link}
     {\ifbool{cbx:parens}
        {\bibcloseparen\global\boolfalse{cbx:parens}}
        {}%
      \biblinkend\global\boolfalse{cbx:link}}
     {}}
  {}
  {\PackageWarning{biblatex-patch}
     {Failed to patch textcite:postnote bibmacro}}

%% file: setup/macros.tex
\newcommand{\ifset}[2]{\ifstrempty{#1}{}{#2}}

\newcommand{\ExtendWithDelims}[2]{
    \WithSuffix\def#1*##1{#1#2*{##1}} 
    \WithSuffix\def#1[##1]##2{#1#2[##1]{##2}} 
}
\WithSuffix\newcommand{\ExtendWithDelims}*[2]{ 
    \WithSuffix\newcommand{#1}*[2][]{#1_{##1}#2*{##2}} 
}

\newcommand{\DeclareMathOperatorWithDelims}[3]{
    \DeclareMathOperator{#1}{#3}
    \ExtendWithDelims{#1}{#2}
}
\WithSuffix\newcommand{\DeclareMathOperatorWithDelims}*[3]{ 
    \DeclareMathOperator{#1}{#3}
    \ExtendWithDelims*{#1}{#2}
}

\newcommand{\DeclareFunctionWithDelims}[3]{
    \newcommand{#1}{#3}
    \ExtendWithDelims{#1}{#2}
}
\WithSuffix\newcommand{\DeclareFunctionWithDelims}*[3]{ 
    \newcommand{#1}{#3}
    \ExtendWithDelims*{#1}{#2}
}

\newcommand{\TODO}[1]{\begingroup\color{red}TODO #1\endgroup}
\WithSuffix\newcommand{\TODO}*[1]{\begingroup\color{red}\textnormal{TODO #1}\endgroup} 
\newcommand{\TODOLIST}[2][]{\begingroup\color{red}\TODO{#1}\begin{itemize}[topsep=2pt]#2\end{itemize}\endgroup}
\newcommand{\MAYBE}[1]{\begingroup\color{orange}#1\endgroup} 

\ifbool{FINAL}{
    \renewcommand{\TODO}[1]{\ignorespaces}
    \renewcommand{\TODOLIST}[2]{\ignorespaces}
    \renewcommand{\MAYBE}[1]{\ignorespaces}
}{
    \newlistof{listoftodos}{todos}{\textcolor{red}{List of TODOs}}
    \newlistentry[chapter]{todo}{todos}{0}
    \let\oldTODO\TODO
    \renewcommand{\TODO}[1]{\refstepcounter{todo}\addcontentsline{todos}{todo}{\protect\numberline{\thetodo}TODO #1}\oldTODO{#1}}
    \robustify{\TODO}
    \let\oldMAYBE\MAYBE
    \renewcommand{\MAYBE}[1]{\refstepcounter{todo}\addcontentsline{todos}{todo}{\protect\numberline{\thetodo}MAYBE #1}\oldMAYBE{#1}}
    \robustify{\MAYBE}
}

\DeclareMathOperator*{\arggre}{arg\widetilde{max}}

\newcommand{\Camspace}{\mathcal{C}} 
\newcommand{\Surface}{\mathcal{S}} 
\newcommand{\Domain}{\mathcal{D}} 

\newcommand{\dmax}{d_{max}}
\newcommand{\dmin}{d_{min}}
\newcommand{\dcam}{d_{cam}}
\newcommand{\DOF}{d_{DOF}}
\newcommand{\FOV}{\alpha_{FOV}}

\let\oldtheta\theta
\renewcommand{\theta}[1][]{\oldtheta\ifset{#1}{^{(#1)}}}
\newcommand{\thetaopt}{\Theta^{\star}} 
\newcommand{\thetagre}{\theta^*} 
\newcommand{\sigmaeps}{\sigma_\varepsilon} 
\newcommand{\x}{\mathrm{x}} 
\newcommand{\y}{\mathrm{y}} 

\newcommand{\fm}{\tilde{f}} 
\newcommand{\A}[1][]{\mathcal{A}\ifset{#1}{^{(#1)}}} 
\newcommand{\Fu}[1][]{F_u\ifset{#1}{^{(#1)}}} 

\renewcommand{\hyphen}{\text{-}}
\newcommand{\addhyphen}[1]{\ifset{#1}{#1\hyphen}} 
\newcommand{\kpwarp}[1][]{k_{\addhyphen{#1}p_u}}
\newcommand{\kpsuminf}[1][]{k_{\addhyphen{#1}p_\infty}}
\newcommand{\kpsumfin}[2][]{k_{\addhyphen{#1}\tilde{p}_{#2}}}
\newcommand{\ktrunc}{\tilde{k}}
\newcommand{\kptrunc}[2][]{k_{\addhyphen{#1}p_{#2}}}

\newcommand{\PhiI}{\Phi^{(I)}} 
\newcommand{\PhiS}{\Phi^{(S)}} 
\newcommand{\fov}[2]{fov\paren*{#2;#1}}
\newcommand{\ray}[2]{ray\paren*{#2;#1}}

\newcommand{\matern}{Matérn}

\newcommand{\GPUCB}{\alg{GP-UCB}}
\newcommand{\SMUCB}{\alg{SM-UCB}}
\newcommand{\MACOPT}{\alg{MaCOpt}}

\newcommand{\algOS}{\alg{ObservedSurface}}
\newcommand{\algOCL}{\alg{ObservedConfidenceLower}}
\newcommand{\algOCU}{\alg{ObservedConfidenceUpper}}
\newcommand{\algIOA}{\alg{IntersectionOcclusionAware}}
\newcommand{\algI}{\alg{Intersection}}
\newcommand{\algC}{\alg{Confidence}}
\newcommand{\algCS}{\alg{ConfidenceSimple}}
\newcommand{\algCSP}{\alg{ConfidenceSimplePolar}}
\newcommand{\algCSW}{\alg{ConfidenceSimpleWeighted}}
\newcommand{\algU}{\alg{Uncertainty}}
\newcommand{\algUP}{\alg{UncertaintyPolar}}

\newcommand{\algGCS}{Greedy-\algCS{}}
\newcommand{\algGU}{Greedy-\algU{}}
\newcommand{\algGUP}{Greedy-\algUP{}}
\newcommand{\algTCSU}{TwoPhase-\algCS{}-\algU{}}
\newcommand{\TCSU}{CS\hyphen U}

\newenvironment{discussion}[1]{
    \par \textbf{(#1)}
}{} 

\newcommand{\ie}{i.e.,}
\newcommand{\eg}{e.g.,}
\newcommand{\wloog}{w.l.o.g.\@}
\newcommand{\iid}{i.i.d.\@}
\newcommand{\whp}{w.h.p.\@}

\renewcommand{\i}{\mathrm{i}} 
\newcommand{\1}{\mathbf{1}} 

\newcommand{\Natural}{\mathbb{N}}
\newcommand{\Integer}{\mathbb{Z}}

\newcommand{\Real}{\mathbb{R}}
\newcommand{\Xcal}{\mathcal{X}}

\newcommand{\Xhat}{\widehat{X}}
\newcommand{\Yhat}{\widehat{Y}}

\newcommand{\defeq}{\coloneqq}
\newcommand{\equivalent}{\Longleftrightarrow}

\newcommand{\ffrac}[2]{{#1/#2}}
\newcommand{\multlinedcmd}[2][5cm]{\begin{multlined}[#1]#2\end{multlined}}

\let\brack\undefine 
\DeclarePairedDelimiter\paren{\lparen}{\rparen}
\DeclarePairedDelimiter\brack{\lbrack}{\rbrack}
\DeclarePairedDelimiter\set{\{}{\}}
\DeclarePairedDelimiter\abs{\lvert}{\rvert}
\DeclarePairedDelimiter\ceil{\lceil}{\rceil}
\DeclarePairedDelimiter\floor{\lfloor}{\rfloor}

\DeclareMathOperator*{\argmin}{argmin}
\DeclareMathOperator*{\argmax}{argmax}

\ExtendWithDelims{\max}{\paren}
\ExtendWithDelims{\min}{\paren}
\ExtendWithDelims{\exp}{\paren}
\ExtendWithDelims{\log}{\paren}
\ExtendWithDelims{\ln}{\paren}
\ExtendWithDelims{\sin}{\paren}
\ExtendWithDelims{\cos}{\paren}
\ExtendWithDelims{\tan}{\paren}
\ExtendWithDelims{\cot}{\paren}
\ExtendWithDelims{\arcsin}{\paren}
\ExtendWithDelims{\arccos}{\paren}
\ExtendWithDelims{\arctan}{\paren}
\ExtendWithDelims{\arccot}{\paren}
\ExtendWithDelims{\sinh}{\paren}
\ExtendWithDelims{\cosh}{\paren}
\ExtendWithDelims{\tanh}{\paren}
\ExtendWithDelims{\coth}{\paren}
\DeclareMathOperatorWithDelims{\supp}{\paren}{supp}
\DeclareFunctionWithDelims{\bigO}{\paren}{\mathcal{O}}
\DeclareFunctionWithDelims{\littleO}{\paren}{\mathcal{o}}
\newcommand{\Fourier}[2]{\mathcal{F}\brack*{#1}\paren*{#2}}

\DeclarePairedDelimiter\norm{\lVert}{\rVert}
\DeclarePairedDelimiter\inner{\langle}{\rangle}

\ExtendWithDelims{\ker}{\paren}
\ExtendWithDelims{\dim}{\paren}
\DeclareMathOperatorWithDelims{\rank}{\paren}{rank}
\ExtendWithDelims{\det}{\paren}
\DeclareMathOperatorWithDelims{\tr}{\paren}{tr}
\DeclareMathOperatorWithDelims{\diag}{\paren}{diag}
\DeclareMathOperatorWithDelims{\vectorize}{\paren}{vec}

\newcommand{\diff}{\mathrm{d}}
\renewcommand{\d}[1]{\,\diff #1}
\newcommand{\dx}{\d{x}}
\newcommand{\deriv}[3][]{\frac{\diff^{#1} #2}{\diff #3}}

\renewcommand{\P}{\mathbb{P}} 
\ExtendWithDelims{\Pr}{\brack}
\DeclareMathOperatorWithDelims*{\E}{\brack}{\mathbb{E}}
\DeclareMathOperatorWithDelims*{\Var}{\brack}{Var}
\DeclareMathOperatorWithDelims{\Cov}{\paren}{Cov}
\DeclareFunctionWithDelims{\Normal}{\paren}{\mathcal{N}}
\DeclareFunctionWithDelims{\GP}{\paren}{\mathcal{GP}}

\DeclarePairedDelimiterX{\probdiv}[2]{(}{)}{#1\;\delimsize\|\;\mathopen{}#2} 
\newcommand{\DKL}{D_{KL}\probdiv}

\newcommand{\with}{\text{ with }}
\newcommand{\inst}[2]{#1 \rightarrow #2}
\newcommand{\subst}[2]{#1 \leftarrow #2}
\newcommand{\noqed}{\renewcommand{\qedsymbol}{}}

\newcommand{\alg}[1]{\textsc{#1}}

\newcommand{\cmark}{\checkmark}
\newcommand{\xmark}{\textit{\sffamily X}}
\newcommand{\omark}{\textit{\sffamily O}}

\newcommand{\fullref}[1]{\cref{#1} (\nameref{#1})}
\makeatletter
\DeclareRobustCommand{\crefnosort}[1]{\begingroup\@cref@sortfalse\cref{#1}\endgroup}
\makeatother


%% file: 00_abstract.tex
\markboth{}{} 

\begin{abstract}
    With the growing practical interest in vision-based tasks for autonomous systems, the need for efficient and complex methods becomes increasingly larger. In the rush to develop new methods with the aim to outperform the current state of the art, an analysis of the underlying theory is often neglected and simply replaced with empirical evaluations in simulated or real-world experiments. While such methods might yield favorable performance in practice, they are often less well understood, which prevents them from being applied in safety-critical systems.
    
    The goal of this work is to design an algorithm for the \emph{Next Best View (NBV) problem} in the context of active object reconstruction, for which we can provide qualitative performance guarantees with respect to true optimality. To the best of our knowledge, no previous work in this field addresses such an analysis for their proposed methods.
    Based on existing work on Gaussian process optimization, we rigorously derive sublinear bounds for the cumulative regret of our algorithm, which guarantees near-optimality. Complementing this, we evaluate the performance of our algorithm empirically within our simulation framework. We further provide additional insights through an extensive study of potential objective functions and analyze the differences to the results of related work.
\end{abstract}

%% file: 00_acknowledgements.tex
\markboth{}{} 

\renewcommand{\abstractname}{Acknowledgements}
\begin{abstract}
    ``Try to write the proof over the weekend, so we can discuss it next time,'' I was once told by my supervisors. As it turned out, it took me few weeks and more than 30 pages. But without their final push, I would still continue improving my simulation framework and completely miss the appreciation I later found for the theoretical part of my thesis.
    
    With this, I would like to especially thank my supervisors Manish and Jelena for giving me such a thorough research experience both practically and theoretically. Without their support and guidance, their patience in addressing my constant questioning, the freedom I received, but also the gentle pushes into the right direction, I would not have made it this far.
    I would also like to thank Prof.\ Andreas Krause for letting me write my thesis in his group.
    
    Further, I want to sincerely thank Slava Borovitskiy for the insightful discussions on periodic kernels and for patiently resolving my questions. Without his help, some pages would still be covered with question marks. I also appreciate the help from Constantin Koch who gave me a hint for a proof at a time when I was too tired to think of a solution.
    
    Special thanks go to my father Bin Yang, who took the time to rigorously proofread almost my complete thesis, which I would not have expected from anyone, and provided valuable feedback. I would also like to thank Dominik Proschmann for correcting some proofs and making constructive suggestions.
    
    Without the mental support (and sometimes food) from my friends and family, this thesis would have ended differently. Thanks to all of you!
\end{abstract}

%% file: 00_notations.tex
\chapter*[Notations]{Notations}\label{chp:notations}
\addcontentsline{toc}{chapter}{Notations}

\begingroup 
    \renewcommand{\arraystretch}{1.1}
    \newcommand{\sepspace}{\addlinespace[0.25cm]}
    
    \newcommand{\header}[1]{\multicolumn{2}{l}{\sffamily\bfseries #1}}
    \newcommand{\subheader}[1]{\multicolumn{2}{l}{\sffamily\itshape #1}}
    \begin{tabularx}{\textwidth}{p{0.2\textwidth}X}
        \header{Abbreviations}\\
        \wloog{} & without loss of generality\\
        \iid{} & independent and identically distributed\\
        \whp{} & with high probability\\
        NBV & next best view\\
        FOV & field of view\\
        DOF & depth of field\\
        LOS & line of sight\\
        GP & Gaussian process\\
        \sepspace
        \subheader{Objective Functions}\\
        OS & \algOS\\
        OCU & \algOCU\\
        OCL & \algOCL\\
        IOA & \algIOA\\
        I & \algI\\
        C & \algC\\
        CS & \algCS\\
        CSP & \algCSP\\
        CSW & \algCSW\\
        U & \algU\\
        UP & \algUP\\
        \sepspace
        \\ \\ \\ \\ \\
        \header{General Setting}\\
        \(\A\) & algorithm\\
        \(t\) & current round (current number of measurements)\\
        \(T\) & final round (total number of measurements)\\
        \sepspace
        \subheader{Camera Pose (decision)}\\
        \(\Camspace\) & space of camera poses\\
        \(\Theta \subseteq \Camspace\) & set of camera poses\\
        \(\theta \in \Camspace\) & camera pose\\
        \(\theta_t\) & \(t\)-th NBV estimate\\
        \(\theta_{1:t}\) & set of NBV estimates \(\theta_1,\dots,\theta_t\)\\
        \(\thetaopt_T\) & optimal solution for \(T\) rounds\\
        \(\thetagre_t\) & greedy decision for round \(t\)\\
        \sepspace
        \subheader{Objectives}\\
        \(F(\Theta)\) & utility of observing from \(\Theta\)\\
        \(F(\theta \mid \Theta)\) & marginal utility of observing from \(\theta\)\\
        \(\Fu(\theta \mid \Theta)\) & upper bound for marginal utility\\
        \(F_l(\theta \mid \Theta)\) & lower bound for marginal utility\\
        \sepspace
        \subheader{Regret}\\
        \(r(t)\) & simple regret in round \(t\)\\
        \(R(T)\) & cumulative regret up to round \(T\)\\
        \(r_{ind}(t)\) & simple individual regret in round \(t\)\\
        \(R_{ind}(T)\) & cumulative individual regret up to round \(T\)\\
        \sepspace
        \header{Simplified 2D Setting}\\
        \subheader{Object}\\
        \(\Domain\) & domain of parameterized surface function\\
        \(\Phi \subseteq \Domain\) & set of parameters for surface function\\
        \(\varphi \in \Domain\) & parameter for surface function\\
        \(f(\varphi)\) & surface function\\
        \(f(\Phi)\) & \([f(\varphi)]_{\varphi\in\Phi}\)\\
        \(\dmax\) & upper bound for surface function\\
        \(\dmin\) & lower bound for surface function\\
        \sepspace
        \subheader{Object Discretization}\\
        \(h\) & width of real world pixel\\
        \(\Surface \subseteq \Domain\) & set of all surface points\\
        \(X \subseteq \Surface\) & set of surface points\\
        \(x_i \in \Surface\) & \(i\)-th surface point\\
        \(N\) & total number of surface points\\
        \sepspace
        \subheader{Camera}\\
        \(\dcam\) & distance of camera to real world center\\
        \(\DOF\) & camera DOF\\
        \(\FOV\) & camera FOV\\
        \(\fov{\theta}{\varphi}\) & left-right FOV boundary of camera at \(\theta\)\\
        \(\ray{\theta,\alpha}{\varphi}\) & ray of camera at \(\theta\) casted at angle \(\alpha\)\\
        \(o(\theta)\) & observation function mapping \(\theta\) to set of observed surface points\\
        \(o(\Theta)\) & \(\bigcup_{\theta\in\Theta} o(\theta)\)\\
        \(\fm(\varphi)\) & measurement function mapping \(\varphi\) to measured distance between real world center and object surface\\
        \(\fm(\Phi)\) & \([\fm(\varphi)]_{\varphi\in\Phi}\)\\
        \(\varepsilon_\varphi\) & measurement noise when measuring surface at \(\varphi\)\\
        \(\varepsilon_\Phi\) & \([\varepsilon_\varphi]_{\varphi\in\Phi}\)\\
        \(\sigmaeps\) & standard deviation of measurement noise\\
        \(X_{1:t} \defeq o(\theta_{1:t})\) & set of surface points observed from \(\theta_{1:t}\)\\
        \(n_{1:t} \defeq \abs{X_{1:t}}\) & number of surface points observed from \(\theta_{1:t}\)\\
        \(Y_{1:t} \defeq \fm(X_{1:t})\) & measurements at observed surface points \(X_{1:t}\)\\
        \(f_{1:t} \defeq f(X_{1:t})\) & true surface function at observed surface points \(X_{1:t}\)\\
        \(\varepsilon_{1:t} \defeq \varepsilon_{X_{1:t}}\) & measurement noise at observed surface points \(X_{1:t}\)\\
        \sepspace
        \subheader{Gaussian Process}\\
        \(\GP(m, k)\) & Gaussian process\\
        \(m(\varphi)\) & mean function of Gaussian process\\
        \(k(\varphi, \varphi')\) & covariance/kernel function of Gaussian process\\
        \(\mu(\Phi)\) & mean vector of \(\Phi\)\\
        \(\mu_0(\Phi)\) & mean vector of prior\\
        \(\mu_t(\Phi)\) & mean vector of posterior based on \(Y_{1:t}\)\\
        \(\Sigma(\Phi,\Phi')\) & covariance matrix between \(\Phi\) and \(\Phi'\)\\
        \(\Sigma_0(\Phi,\Phi')\) & covariance matrix of prior\\
        \(\Sigma_t(\Phi,\Phi')\) & covariance matrix of posterior based on \(Y_{1:t}\)\\
        \(\Sigma(\Phi)\) & \(\Sigma(\Phi,\Phi)\)\\
        \(\sigma(\varphi)\) & standard deviation of marginal distribution at \(\varphi\)\\
        \(\sigma_0(\varphi)\) & standard deviation of prior\\
        \(\sigma_t(\varphi)\) & standard deviation of posterior based on \(Y_{1:t}\)\\
        \(u_t(\varphi)\) & upper confidence bound for round \(t\) based on \(Y_{1:t-1}\)\\
        \(l_t(\varphi)\) & lower confidence bound for round \(t\) based on \(Y_{1:t-1}\)\\
        \(\beta_t\) & confidence parameter for round \(t\)\\
        \sepspace
        \\ \\
        \subheader{Kernels for Gaussian Process}\\
        \(l\) & length scale\\
        \(\sigma_f\) & standard deviation of Gaussian process\\
        \(\nu\) & smoothness of Matérn kernel\\
        \(k_{RBF}(x,x')\) & RBF kernel\\
        \(k_M(x,x')\) & Matérn kernel\\
        \(k_{M_\nu}(x,x')\) & Matérn kernel with smoothness \(\nu\)\\
        \(\kpwarp[X](x,x')\) & kernel \(X\) periodized by warping with function \(u\)\\
        \(\kpsuminf[X](x,x')\) & kernel \(X\) periodized by infinite periodic summation\\
        \(\kpsumfin[X]{\kappa}(x,x')\) & kernel \(X\) \(\kappa\)-approximately periodized by finite periodic summation\\
        \(\kptrunc[X]{c_1 \to c_2}(x,x')\) & kernel \(X\) periodized by truncation from \(c_1\) to \(c_2\)\\
        \(S(\omega)\) & spectral density of kernel function (Fourier transform of \(k(r)\))\\
        \sepspace
        \subheader{Algorithms and Analysis}\\
        \(\PhiI_t(\theta)\) & summation interval defined by ``FOV-confidence intersection''\\
        \(\PhiS(\theta)\) & summation interval defined by ``simple FOV endpoint''\\
        \(\A(\Theta;\Fu)\) & greedy algorithm based on objective \(\Fu\) and given previous camera poses \(\Theta\)\\
        \(\A[X](\Theta)\) & \(\A(\Theta;\Fu[X])\)\\
        \(\A[*](\Theta)\) & \(\A(\Theta;F)\) (optimal greedy algorithm)\\
        \(\A(\Theta;\Fu[1],\Fu[2])\) & two-phase algorithm based on objectives \(\Fu[1]\) for phase 1 and \(\Fu[2]\) for phase 2 and given previous camera poses \(\Theta\)\\
        \(\A[X\hyphen Y](\Theta)\) & \(\A(\Theta;\Fu[X],\Fu[Y])\) \\
        \(I(Y_{1:T};f_{1:T})\) & information gain about \(f\) through measurements \(\theta_{1:T}\)\\
        \(\gamma_T\) & maximum information gain through \(T\) measurements\\
        \(\Domain_t\) & uniform discretization of \(\Domain\)\\
        \([\Phi]_t\) & uniform discretization of \(\Phi\) based on \(\Domain_t\)\\
        \([\varphi]_t\) & closest point in \(\Domain_t\) to \(\varphi\)\\
        \sepspace
        \header{General Notations}\\
        \(\sum_{x\in[a,b]}^{\Delta x} f(x)\) & \(\sum_{k=0}^{\floor{\frac{b-a}{\Delta x}}} f(a+k\cdot\Delta x)\) (sum from \(a\) to \(b\) with step size \(\Delta x\))\\
        \(\abs{[a,b]}\) & \(b-a\) (length of interval)\\
        \(\lambda_i(A)\) & \(i\)-th eigenvalue of \(A\)\\
        \(\in_{2\pi}\) & interval membership modulo \(2\pi\)
    \end{tabularx}
\endgroup

\begin{remark*}
    For the sake of simplicity and readability, we introduce some sloppiness in our notation, which we discuss in the following:
    \begin{itemize}
        \item Sets \(X = \set{x_1,\dots,x_t}\) and sequences/vectors \(x_{1:t} = (x_1,\dots,x_t)\) are interpreted equivalently.
        
        The set notion allows us to use set operations such as \(x \in X\) or \(X_1 \cup X_2\) on sets and sequences. The sequence notion imposes an order on the elements and allows us to use elementwise math \(x_{1:t} + y_{1:t} = [x_i+y_i]_{i=1}^t\) and elementwise function application \(f(x_{1:t}) = [f(x_i)]_{i=1}^{t}\) on sets and sequences. We use both notations interchangeably and the specific notion should be clear from the context. In most cases, we refer to it as a set.
        \item Indexing by single index \(x_t\) or singleton sequence \(x_{t:t}\) are interpreted equivalently.
        
        This allows us to use single elements and singleton sets interchangeably and the specific notion should be clear from the context. In addition, definitions provided only for sequence indexes naturally generalize to single indexes.
    \end{itemize}
\end{remark*}

%% file: 01_introduction.tex
\chapter{Introduction}\label{chp:introduction}

With the rise of computer vision, machine learning and robotics, an increasing number of technologies emerge at their intersection. \emph{3D reconstruction} is one of them, which aims to reconstruct a digital 3D model of an object or a scene from images or other measurement types. This has become an important way to provide a 3D understanding of our world to computers. It allows autonomous robots and AI systems to understand and explore their environment, enables offline inspection of infrastructure which is difficult to access, and brings real-world objects into digital reality.

In this work, we focus on the reconstruction of objects, for which the number of measurements is limited. This is of high importance, when the sampling process of new measurements is tedious or expensive. For example, positioning autonomous robots in difficult terrains or unmanned aerial vehicles at windy conditions for taking high-quality measurements can be a time-consuming process. At the same time, limited resources such as battery capacity and computational power require efficient strategies for performing reconstruction on-board.

The goal of \emph{active reconstruction}, an instance of active learning, is to make informed decisions for the next measurement locations based on previously acquired information. Sampling measurements from strategically chosen locations will lead to more efficient and faster reconstruction. In literature, this is commonly known as the \emph{Next Best View (NBV) problem}, which has been studied extensively \autocite{connolly1985determination, banta2000nextbestview, wenhardt2007active}. A wide range of methods, algorithms and heuristics for efficient active reconstruction has been developed over time \autocite{banta2000nextbestview, chen2005vision, delmerico2018comparison, schmid2020efficient}. So far, most of the existing work only validates their methods quantitatively in simulated or real-world experiments to show superiority over the state of the art \autocite{bissmarck2015efficient, schmid2020efficient, kompis2021informed}. To the best of our knowledge, none of them provide qualitative results with respect to the optimality of reconstruction. However, developing well-founded methods with strong theoretical guarantees is important for the deployment in safety-critical systems. In addition, methods with a theoretical foundation can be reasoned about more precisely and allow one to formally justify certain design choices in practice. With this work, we try to devise an algorithm, whose proposed NBV for reconstructing a given object is close to the optimal NBV. In other words, we aim for near-optimal active reconstruction.

\emph{Gaussian process optimization} is a powerful way to optimize unknown functions under uncertainty while providing provably near-optimal guarantees based on sublinear regret bounds \autocite{srinivas2012informationtheoretic}. This has been applied to many settings including sensor placement \autocite{krause2007nearoptimal}, interactive recommender systems \autocite{chen2017interactive}, and multi-agent coverage control \autocite{prajapat2022nearoptimal}.
Our goal is to apply similar approaches in the object reconstruction setting to obtain near-optimality guarantees for our algorithm.

\subsubsection{Structure of the Thesis}
The structure of this thesis is as follows.
In \cref{chp:background}, we provide background knowledge for subsequent chapters and in \cref{chp:related} we highlight related works in the area of active reconstruction and Gaussian process optimization. Continuing with \cref{chp:problem}, we formally define the setting and formulate the problem. Furthermore, we provide a list of all simplifications, which facilitate our later analysis, and compare our setting to other related work in more details. In \cref{chp:design}, the design ideas for different components of our algorithm are presented and compared with each other. After selecting a small set of candidate algorithms, \cref{chp:analysis} is devoted to their theoretical analysis with respect to near-optimality. \cref{chp:experiments} discusses the experimental results obtained from our simulation framework and contrasts them with our theoretical findings. Finally, \cref{chp:conclusion} concludes our work with a summarizing discussion and suggestions for potential future work.

We refer to \cref{fig:problem-overview} for an overview of our setting and problem formulation. Similarly, \cref{fig:design-overview} summarizes the design process of our algorithms and \cref{fig:analysis-overview} provides an overview of the structure of our theoretical analysis.

%% file: 02_background.tex
\chapter{Background}\label{chp:background}

In this chapter, we provide background knowledge for subsequent chapters. In \cref{sec:background-gp} we discuss Gaussian processes, which form the foundation of our algorithms. In \cref{sec:background-infotheory} we provide a brief introduction to information theory which is relevant for their analysis.


\section{Gaussian Process}\label{sec:background-gp}

Before diving into Gaussian processes, we first briefly discuss the general class of stochastic processes in \cref{ssec:background-gp-sp}. Then we discuss Gaussian processes in \cref{ssec:background-gp-gp} and in particular the influence of the kernel function on the Gaussian process in \cref{ssec:background-gp-kernels}. Finally, we describe how to perform Gaussian process regression of an unknown function using Bayesian inference in \cref{ssec:background-gp-gpr}. This becomes relevant for our setting, in which we want to learn the unknown object surface function.

For a ``systematic and unified treatment'' of Gaussian processes, we refer to the book \citetitle{rasmussen2005gaussian} written by \textcite{rasmussen2005gaussian}. A quick introduction with coding examples can be found on the website of \textcite{roelants2019gaussian} and for interactive visualizations we can recommend \textcite{gortler2019visual}. A brief overview over kernel functions for Gaussian processes is provided by \textcite{duvenaud2014kernel}.

\subsection{Stochastic Process}\label{ssec:background-gp-sp}

Recall from probability theory that a \textit{random variable} \(X\) informally represents a variable whose value is distributed according to some probability distribution.
A \textit{random vector} \(X = (X_1,\dots,X_n)\) is a finite collection of indexed random variables, which are distributed by some corresponding multivariate distribution.
A \emph{stochastic process} generalizes random variables once more to a collection of indexed random variables \(\set{X_t}_{t\in T}\) or \(\set{X(t)}_{t\in T}\) with some arbitrary index set \(T\), which can be infinitely large. Hence, a stochastic process can be interpreted as a \textit{random sequence} or \textit{random function} corresponding to a infinite-dimensional random vector and is distributed with some probability distribution over sequences or functions.

In practice, \(T\) often has some notion of time with \(X(t)\) describing the randomness of some quantity at time \(t\). Therefore, common choices for \(T\) are subsets of \(\Real\) such as the set of integers or an interval.
Based on \(T\), stochastic processes can be categorized into 
\emph{discrete-time stochastic processes} corresponding to random sequences or \emph{continuous-time stochastic processes} corresponding to random functions.

As one can sample concrete values \(x\in\Real\) of a random variable \(X\) and concrete vectors \((x_1,\dots,x_n)\in\Real^n\) of a random vector \((X_1,\dots,X_n)\), one can also sample sequences \((x_t)_{t\in T}\) or functions \(x\colon T\to\Real\) of a stochastic process which are called \emph{sample functions} or \emph{sample paths}. In theory, this requires us to sample from a possibly infinite-dimensional joint distribution of random variables \(x_t\) or \(x(t)\) over all \(t\in T\).
In practice, it suffices to realize \(x_t\) or \(x(t)\) only for a finite subset of indexes \(\set{t_1,\dots,t_n}\subset T\), which corresponds to sampling from the finite-dimensional joint distribution of the random variables with index \(t \in \set{t_1,\dots,t_n}\) while marginalizing the infinitely many remaining random variables with index \(t \in T\setminus\set{t_1,\cdots,t_n}\).

The goal of discussing stochastic processes or more specifically Gaussian processes in the next \cref{ssec:background-gp-gp} is to use them as a model for the unknown surface function \(f\) of our object. Hence, we only focus on continuous-time stochastic processes (\ie{} random functions)
\begin{equation*}
    f \defeq \set{f(x)}_{x\in \Xcal}
\end{equation*}
with random variables \(f(x)\) indexed by points on the continuous domain \(\Xcal\).

\subsection{Gaussian Process}\label{ssec:background-gp-gp}

A Gaussian process is a stochastic process which we denote as%
\footnote{Note that \textcite[Eq. 2.14]{rasmussen2005gaussian} write \(f(x) \sim \GP*{m(x),k(x,x')}\), but we change this notation to \(f\) instead of \(f(x)\) to put emphasis on the notion of a random function while avoiding the misconception that \(f(x)\) refers to the random variable at point \(x\).}
\begin{equation*}
    f \sim \GP*{m(x),k(x,x')}
\end{equation*}
and which represents a random function defined on \(\Xcal\).
It is fully characterized by a \emph{mean function} and a positive semi-definite \emph{covariance function}
\begin{equation}\label{eq:mean-covariance-function}
    \begin{aligned}
        m(x) &\defeq \E*{f(x)}\\
        k(x,x') &\defeq \E*{\paren*{f(x)-m(x)}\paren*{f(x') - m(x')}} 
        \\ &= \Cov*{f(x), f(x')}
        .
    \end{aligned}
\end{equation}
The mean function \(m(x)\) at point \(x\) corresponds to the mean of the random variable \(f(x)\) and therefore describes the average function values at \(x\). The covariance function \(k(x,x')\) at points \(x\) and \(x'\) corresponds to the covariance between the random variables \(f(x)\) and \(f(x')\) and therefore describes the pairwise similarity between the function values at \(x\) and \(x'\).

The important characteristic of a Gaussian process is that for every finite subset of points \(X = \set{x_1,\dots,x_n}\) the corresponding set of random variables \(f(X) = \set{f(x) \mid x \in X}\) is distributed with a multivariate Gaussian distribution
\begin{equation*}
    f(X) \sim \Normal*{\mu(X),\Sigma(X)}
\end{equation*}
with \textit{mean vector} and \textit{covariance matrix}
\begin{equation}\label{eq:mean-vector-covariance-matrix}
    \begin{aligned}
        \mu(X) &\defeq \brack*{m(x)}_{x\in X}\\
        \Sigma(X,X') &\defeq \brack*{k(x,x')}_{x\in X,x'\in X'}
    \end{aligned}
\end{equation}
and \(\Sigma(X)\) denoting \(\Sigma(X,X)\). Hence, we can interpret \(f\) as a random function or an infinite-dimensional random vector distributed with an infinite-dimensional multivariate Gaussian distribution. In addition, we can easily sample a finite set of function values \(f(X)\) from the multivariate Gaussian distribution and interpolate between these values to visualize the sample function. This is the advantage of Gaussian processes, since they stay computationally tractable despite the infinite dimensionality.

The choice of \(m(x)\) is less important for the behavior of functions sampled from the Gaussian process, since it only specifies the offset around which the function values are expected to lie. Often the mean function for the Gaussian process is assumed to be zero, since \(f \sim \GP*{m(x), k(x,x')}\) can always be represented as \(f(x) = f_0(x) + m(x)\) with \(f_0 \sim \GP*{0, k(x,x')}\). As one can see, the choice of the covariance function \(k(x,x')\), which is also called the \emph{kernel function}, governs most of the properties of the Gaussian process such as continuity, smoothness, stationarity or periodicity. The next \cref{ssec:background-gp-kernels} is reserved for a discussion on kernel functions.

\subsection{Kernel Functions}\label{ssec:background-gp-kernels}

It is commonly known that a kernel function \(k\colon\Xcal \times \Xcal \to \Real\) is used as a similarity measure between pairwise points \(x\) and \(x'\) \autocite[Chapter 4]{rasmussen2005gaussian}.%
\footnote{Recall that the kernel function represents an inner product \(k(x,x') = \inner{\phi(x),\phi(x')}\) with \(\phi\) typically interpreted as some feature map. In Euclidean space, this inner product is maximally positive or negative if \(\phi(x)\) and \(\phi(x')\) are aligned (\ie{} high similarity) and it is zero if they are orthogonal (\ie{} no similarity).}
For a Gaussian process \(f\), the kernel \(k(x,x')\) is defined as the covariance between the random variables \(f(x)\) and \(f(x')\) as defined in \cref{eq:mean-covariance-function}, which corresponds to measuring the similarity between \(f(x)\) and \(f(x')\) in how they deviate from the mean function.
Hence, depending on the choice of \(k(x,x')\), functions of different types and with different characteristics can be sampled from the Gaussian process.

\subsubsection{Stationarity of Kernel}
One common type of kernel functions are \emph{stationary kernels} characterized as
\begin{equation}\label{eq:stationary-kernel}
    k(x,x') = k_S(x-x')
\end{equation}
with some univariate function \(k_S\colon\Xcal\to\Real\). These kernels are translation invariant, since the returned similarity only depends on the relative positions of \(x\) and \(x'\) and not on the absolute position of each individual point. Without the knowledge of the absolute positions, stationary kernels are only capable of influencing the local, stationary behavior of sample functions and are not able to describe global trends \autocite{gortler2019visual}. Hence, a sample function from a stationary process behaves similarly at all locations as shown in \cref{fig:background-gp-kernels-stationarity-stationary}.

\begin{figure}
    \centering
    \begin{subfigure}{0.45\linewidth}
        \centering
        \includegraphics[width=\linewidth]{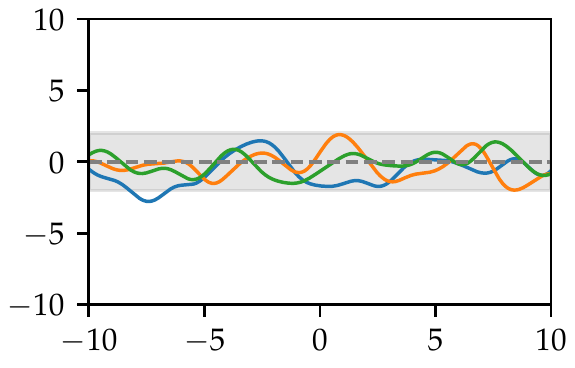}
        \caption{stationary Gaussian process}
        \label{fig:background-gp-kernels-stationarity-stationary}
    \end{subfigure}%
    \begin{subfigure}{0.45\linewidth}
        \centering
        \includegraphics[width=\linewidth]{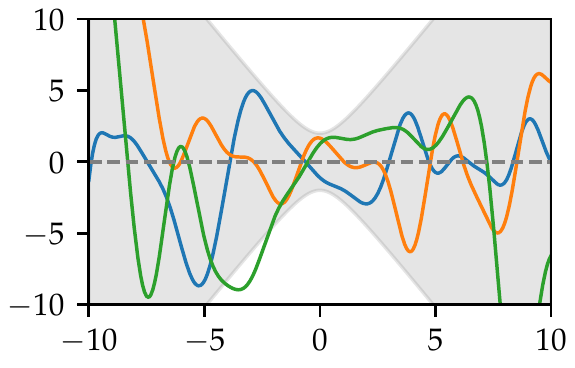}
        \caption{non-stationary Gaussian process}
        \label{fig:background-gp-kernels-stationarity-nonstationary}
    \end{subfigure}
    \caption[Stationary vs.\ Non-Stationary Gaussian Process]{
        Stationary vs.\ Non-Stationary Gaussian Process.
        (a) visualizes three sample functions from a stationary Gaussian process obtained by sampling from a multivariate Gaussian distribution at finitely many evaluation points and interpolating these sampled function values. The dotted gray line denotes the mean and the gray region denotes twice the standard deviation at each of the sample locations.
        Similarly, (b) shows three functions sampled from a Gaussian process, which is however defined with a non-stationary kernel.
        Observe, how the random behavior of sample functions from a stationary Gaussian process in (a) stays globally the same, while the non-stationary process in (b) also captures functions which globally change in their magnitude.
    }
    \label{fig:background-gp-kernels-stationarity}
\end{figure}

We can further categorize stationary kernels into \emph{isotropic kernels} defined as
\begin{equation}\label{eq:isotropic-kernel}
    k(x,x') = k_S(\norm{x-x'}_2)
\end{equation}
or as \(k(x,x') = k_S(\abs{x-x'})\) with \(\Xcal\subseteq\Real\). Isotropic kernels are translation and rotation invariant, since they only depend on the Euclidean distance between pairs of points. Hence, the similarity \(k(x,\cdot)\) to neighbor points around \(x\) is symmetric. The class of \emph{anisotropic kernels} can be written as
\begin{equation}\label{eq:anisotropic-kernel}
    k(x,x') = k_S(\norm{x-x'}_M)
\end{equation}
with some positive semi-definite matrix \(M\), which can take the direction of deviation into account \autocite[Chapter 4.2.1]{rasmussen2005gaussian}.

On the other side, \emph{non-stationary kernels} determine the similarity based on the absolute positions of each individual point, which allows the Gaussian process to capture global trends as seen in \cref{fig:background-gp-kernels-stationarity-nonstationary}. A common class of non-stationary kernels are dot product kernels, which can be written as \(k(x,x') = k_{dot}(x^Tx')\) \autocite[Chapter 4.1]{rasmussen2005gaussian}.

Since for our setting, we are only interested in stationary kernels as described later in \cref{ssec:design-gp-covariance}, we continue focusing only on stationary kernels. This also reduces the dependence of the kernel function
\begin{equation}\label{eq:stationary-kernel-univariate}
    k(x, x+r) = k_S(r)
\end{equation}
to a single variable \(r\) denoting the distance between pairs of points. This allows us to reason about the similarity only based on the distance \(r\) around some unspecified \(x\) instead of based on a pair of \(x\) and \(x'\). In addition, we can now easily plot the kernel functions over \(r\).%
\footnote{In fact, we plot the kernel functions over the distances \(r=x-x'\) instead of the absolute distances \(\abs{x-x'}\), since this allows the intuition of centering the plotted kernel function on \(x\) to imagine the similarity value to the left and right neighbor points \(x+r\).}

\subsubsection{Smoothness of Kernel}
The choice of kernel function also determines the smoothness properties of a stationary Gaussian process. The smoothness of a process is described in terms of \emph{mean-square (MS) continuity} and \emph{MS differentiability}, which is loosely related to the continuity and differentiability of sample functions as explained by \textcite[Chapter 4.1.1]{rasmussen2005gaussian}.
For stationary kernels, the behavior of \(k_S(r)\) around zero is decisive for the smoothness of the Gaussian process, since it determines how similar \(f(x)\) and \(f(x+r)\) are. Later in \cref{fig:background-gp-kernels-rbf-matern}, we provide visualizations of sample functions for kernels with different behavior at \(r=0\).

\subsubsection{Examples of Stationary Kernels}
Stationary kernels often have the form
\begin{equation}\label{eq:stationary-kernel-parameters}
    k(x,x+r) = \sigma_f^2 \cdot k_S\paren*{\frac{r}{l}}
\end{equation}
with \(k_S(r)\) normalized to \(k_S(0)=1\). The parameter \(\sigma_f\) corresponds to the \emph{process standard deviation} and describes the average distance of sample functions \(f(x)\) to the mean function \(m(x)\).
The parameter \(l\) represents the \emph{characteristic length scale} of the Gaussian process and is used to adjust how far the similarity to \(x\) can reach the neighbor points \(x'\). This length roughly corresponds to the distance to \(x\) before the function value changes significantly \autocite[Chapter 2.2]{rasmussen2005gaussian}.
Choosing a smaller length scale results into sample functions with more rapid behavior as seen in \cref{fig:background-gp-kernels-parameters-lengthshort}, since \(l\) scales up the actual distances between \(x\) and \(x'\) such that the similarity to \(x\) reduces faster around \(x\). On the other hand, a larger length scale results into sample functions with slower variation as visualized in \cref{fig:background-gp-kernels-parameters-lengthlong}, since the similarity to \(x\) decreases slower around \(x\).

\begin{figure}
    \centering
    \begin{subfigure}{0.45\linewidth}
        \centering
        \includegraphics[width=\linewidth]{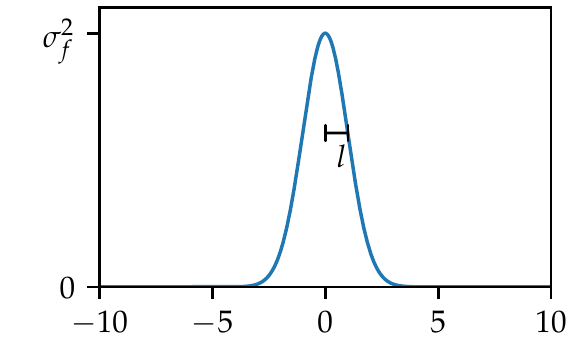}
        \includegraphics[width=\linewidth]{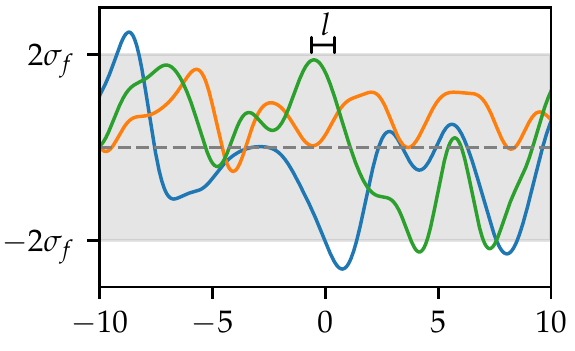}
        \caption{RBF kernel with \(l=1\)}
        \label{fig:background-gp-kernels-parameters-lengthshort}
    \end{subfigure}%
    \begin{subfigure}{0.45\linewidth}
        \centering
        \includegraphics[width=\linewidth]{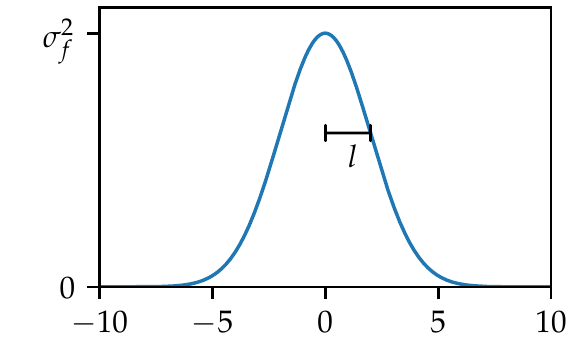}
        \includegraphics[width=\linewidth]{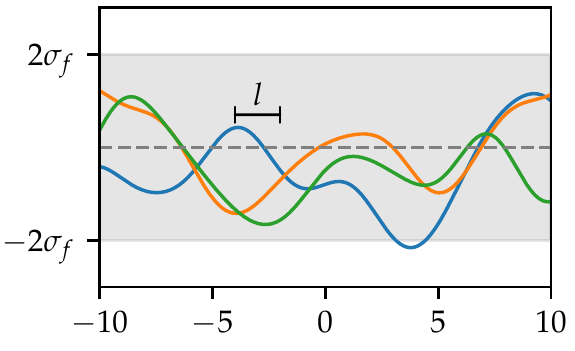}
        \caption{RBF kernel with \(l=2\)}
        \label{fig:background-gp-kernels-parameters-lengthlong}
    \end{subfigure}
    \caption[Parameters of Stationary Kernels]{
        Parameters of Stationary Kernels.
        (a) and (b) visualize the RBF kernel function (top) from \cref{eq:rbf-kernel} and sample functions (bottom) from the corresponding Gaussian process with different length scale parameters \(l\). Observe how the length scale represents the distance up to which the kernel function value is still large, precisely \(\sigma_f^2\cdot k_S(1)\). Afterwards, the kernel function decays rapidly and the points on the sample functions increasingly differ beyond a pairwise distance of \(l\).
        Additionally, one can visually verify that the process standard deviation \(\sigma_f\) corresponds to the average distance of sample functions to the mean function.
    }
    \label{fig:background-gp-kernels-parameters}
\end{figure}

The \emph{RBF kernel}, also called \emph{Squared Exponential kernel}, is defined as
\begin{equation}\label{eq:rbf-kernel}
    k_{RBF}(r) \defeq \sigma_f^2 \exp*{-\frac{r^2}{2l^2}}
    ,
\end{equation}
which results into an infinitely MS differentiable Gaussian process as visualized in \cref{fig:background-gp-kernels-rbf} \autocite[Eq. 4.9]{rasmussen2005gaussian}. This is often too smooth and unrealistic for modeling physical processes \autocite{stein1999interpolation}.

\begin{figure}
    \centering
    \begin{subfigure}{0.45\linewidth}
        \centering
        \includegraphics[width=\linewidth]{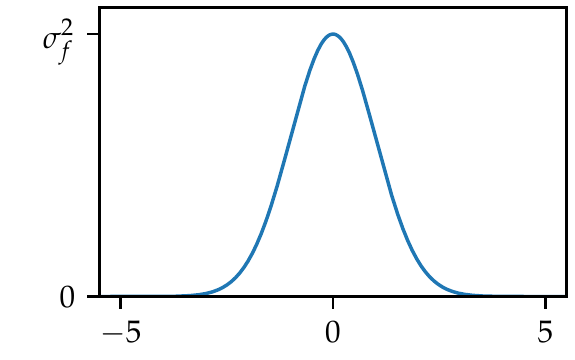}
        \includegraphics[width=\linewidth]{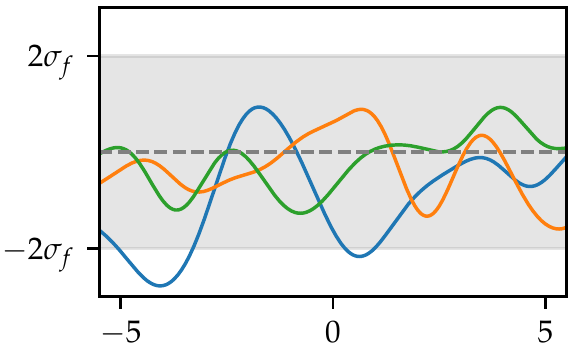}
        \caption{RBF kernel}
        \label{fig:background-gp-kernels-rbf}
    \end{subfigure}%
    \begin{subfigure}{0.45\linewidth}
        \centering
        \includegraphics[width=\linewidth]{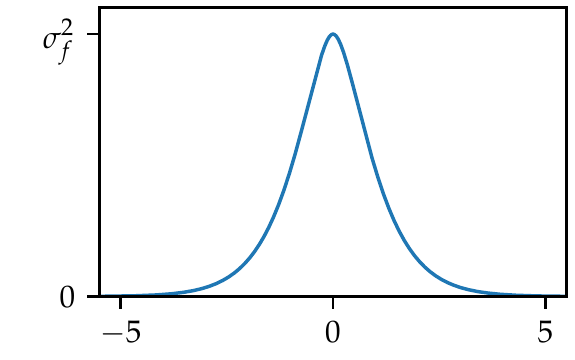}
        \includegraphics[width=\linewidth]{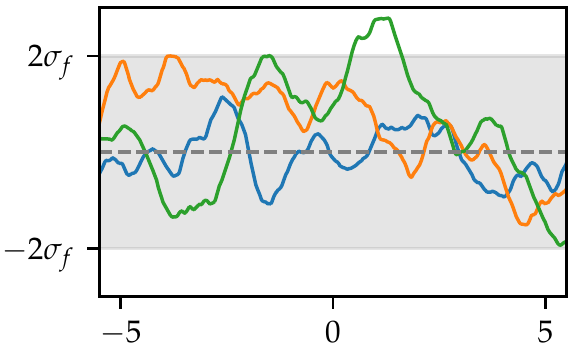}
        \caption{\matern{} kernel with \(\nu=3/2\)}
        \label{fig:background-gp-kernels-matern}
    \end{subfigure}
    \caption[RBF vs.\ \matern{} Kernel]{
        RBF vs.\ \matern{} Kernel.
        (a) shows the RBF kernel function (top) and sample functions (bottom) of the corresponding Gaussian process and similarly (b) for the \matern{} kernel with smoothness parameter \(\nu=3/2\). Observe how the less smooth behavior of the \matern{} kernel around \(r=0\) leads to less smooth sample functions.
    }
    \label{fig:background-gp-kernels-rbf-matern}
\end{figure}

Alternatively, the \emph{\matern{} kernel} is defined as
\begin{equation}\label{eq:matern-kernel}
    k_M(r) \defeq \sigma_f^2 \frac{2^{1-\nu}}{\Gamma(\nu)}\paren*{\sqrt{2\nu}\frac{r}{l}}^\nu K_\nu\paren*{\sqrt{2\nu}\frac{r}{l}}
\end{equation}
with gamma function \(\Gamma\) and modified Bessel function \(K_\nu\) \autocite[Eq. 4.14]{rasmussen2005gaussian}. The additional parameter \(\nu > 0\) allows us to adapt the smoothness of the Gaussian process. It is known that the Gaussian process is \(k\)-times MS differentiable if and only if \(\nu>k\) \autocite[Chapter 4.2.1]{rasmussen2005gaussian}. Hence, for the following choices of \(\nu\) we obtain
\begin{itemize}
    \item \(\nu=1/2\): MS continuous, but not MS differentiable (\ie{} very rough)
    \item \(\nu=3/2\): 1-times MS differentiable
    \item \(\nu=5/2\): 2-times MS differentiable
    \item \(\nu\geq7/2\): no large differences for different \(\nu\) (\ie{} all similar smooth)
\end{itemize}
This powerful kernel allows us to adapt the smoothness to the specific task or our prior knowledge with an example shown in \cref{fig:background-gp-kernels-matern}.

It is also possible to design new kernels based on existing ones as long as they stay positive semi-definite. We refer to \textcite[Chapter 4.2.4]{rasmussen2005gaussian} and \textcite[Chapter 13.1]{scholkopf2002learning} for various design techniques.

\subsection{Gaussian Process Regression}\label{ssec:background-gp-gpr}

The goal of Gaussian processes is not only to sample random functions from it, but also to learn an unknown function \(f\). Given a set of measured function values \(Y\) at the data points \(X\), we can use Bayesian inference
to refine the Gaussian process distribution, such that it takes the information from the dataset \(D=(X,Y)\) into account.

We define \(\Xhat\) as the points, where we want to evaluate the posterior distribution of the Gaussian process. For example, \(\Xhat\) can be a set of uniformly spaced points over \(\Xcal\) which allows us to interpolate and plot sample functions of the posterior Gaussian process distribution.

Then we choose a mean and covariance function which defines our prior Gaussian process distribution. As described in \cref{ssec:background-gp-gp}, the choice of the covariance function encodes most of our prior knowledge about the behavior or shape of the unknown function, while the mean function encodes most information about the magnitude or offset.
Before taking \(D\) into account, the random variables \(\Yhat=f(\Xhat)\) at the evaluation points are distributed with the \emph{prior distribution}
\begin{equation}\label{eq:prior-regression}
    \Yhat \sim \Normal*{\mu(\Xhat),\Sigma(\Xhat)}
\end{equation}
with mean and covariance as defined in \cref{eq:mean-vector-covariance-matrix}.

\subsubsection{Noise-free Gaussian Process Regression}
We first assume that the given set of measurements \(Y = f(X)\) does not contain noise and
\begin{equation*}
    Y \sim \Normal*{\mu(X),\Sigma(X)}
\end{equation*}
with mean and covariance as defined in \cref{eq:mean-vector-covariance-matrix}. The joint Gaussian distribution for the function values \(Y\) and \(\Yhat\) at the data and evaluation points is given as
\begin{equation*}
    \begin{pmatrix}Y \\ \Yhat\end{pmatrix}
    \sim \Normal*{\begin{pmatrix}
        \mu(X) \\ \mu(\Xhat)
    \end{pmatrix}, \begin{pmatrix}
        \Sigma(X) & \Sigma(X,\Xhat) \\
        \Sigma(\Xhat,X) & \Sigma(\Xhat)
    \end{pmatrix}}
    .
\end{equation*}
Since we already know concrete values for \(Y\), we apply Bayesian inference and arrive at the following \emph{posterior distribution}
\begin{equation}\label{eq:posterior-noisefree-regression}
    \Yhat \mid Y \sim \Normal*{\mu_D(\Xhat), \Sigma_D(\Xhat)}
\end{equation}
with closed-form expressions for
\begin{equation*}
    \begin{aligned}
        \mu_D(\Xhat) &= \eqcontrast \mu(\Xhat) \eqchange{+ \Sigma(\Xhat,X)\Sigma(X)^{-1} (Y - \mu(X))} \\
        \Sigma_D(\Xhat) &= \eqcontrast \Sigma(\Xhat) \eqchange{- \Sigma(\Xhat,X)\Sigma(X)^{-1} \Sigma(X,\Xhat)}
        .
    \end{aligned}
\end{equation*}
as shown by \textcite[Eq. 2.19]{rasmussen2005gaussian} and visualized in \cref{fig:background-gp-gpr-prior}. The posterior distribution of \(\Yhat\) provides us the posterior mean vector \(\mu_D(\Xhat)\), which we can use as a regression estimate for \(f\), and we additionally obtain confidence bounds for \(f\) based on the posterior covariance matrix \(\Sigma_D(\Xhat)\) as visualized in \cref{fig:background-gp-gpr-posterior-noisefree}. Observe how the posterior distribution restricts the Gaussian process only to sample functions which perfectly interpolate the dataset \(D\).

\begin{figure}
    \centering
    \begin{subfigure}{0.33\linewidth}
        \centering
        \includegraphics[width=\linewidth]{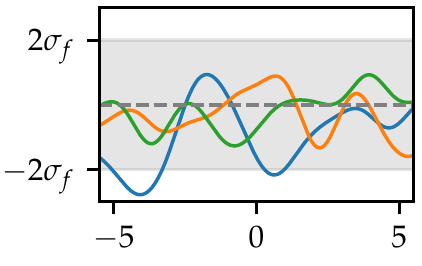}
        \caption{prior}
        \label{fig:background-gp-gpr-prior}
    \end{subfigure}%
    \begin{subfigure}{0.33\linewidth}
        \centering
        \includegraphics[width=\linewidth]{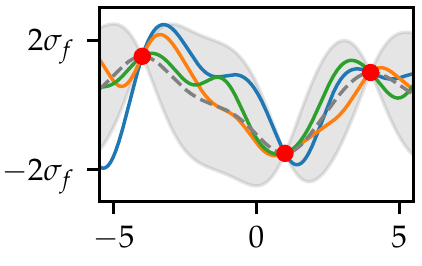}
        \caption{noise-free posterior}
        \label{fig:background-gp-gpr-posterior-noisefree}
    \end{subfigure}%
    \begin{subfigure}{0.33\linewidth}
        \centering
        \includegraphics[width=\linewidth]{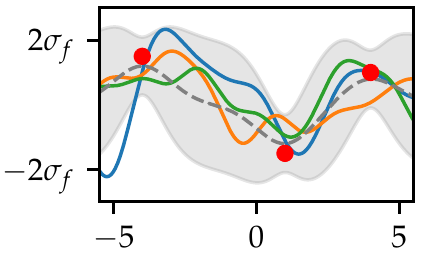}
        \caption{noisy posterior}
        \label{fig:background-gp-gpr-posterior-noisy}
    \end{subfigure}
    \caption[Gaussian Process Regression]{
        Gaussian Process Regression.
        (a) visualizes the prior distribution of the Gaussian process with the mean function initially set to zero \wloog{} and the same uncertainty or standard deviation from the mean for all function values. Given a dataset of three data points (red dots), the noise-free regression in (b) updates the Gaussian process distribution and returns a posterior distribution with the mean function perfectly interpolating the data points and the uncertainty reduced to zero at these points. Consequently, all sample functions from this posterior distribution similarly interpolate the given dataset. In contrast, the posterior distribution returned by the noisy regression model (c) only approximates the data points with its mean function while it keeps some remaining uncertainty about the exact values. This allows sample functions to fluctuate more around the given dataset.
    }
    \label{fig:background-gp-gpr}
\end{figure}

\subsubsection{Noisy Gaussian Process Regression}

To take measurement noise into account, we assume that the given set of measurements \(Y = f(X) + \varepsilon\) contains \iid{} Gaussian noise \(\varepsilon\sim\Normal*{0,\sigmaeps^2I}\) with some noise variance \(\sigmaeps^2\). The \emph{prior distribution} is given as
\begin{equation*}
     \eqcontrast Y \sim \Normal*{\mu(X) \eqchange{+ \sigmaeps^2I},\Sigma(X)}
\end{equation*}
with mean and covariance as defined in \cref{eq:mean-vector-covariance-matrix}. Then the joint Gaussian distribution corresponds to
\begin{equation*}
    \eqcontrast \begin{pmatrix}Y \\ \Yhat\end{pmatrix}
    \sim \Normal*{\begin{pmatrix}
        \mu(X) \\ \mu(\Xhat)
    \end{pmatrix}, \begin{pmatrix}
        \Sigma(X) \eqchange{+ \sigmaeps^2I} & \Sigma(X,\Xhat) \\
        \Sigma(\Xhat,X) & \Sigma(\Xhat)
    \end{pmatrix}}
\end{equation*}
and by applying Bayesian inference we arrive at the following \emph{posterior distribution}
\begin{equation}\label{eq:posterior-noisy-regression}
    \Yhat \mid Y \sim \Normal*{\mu_D(\Xhat), \Sigma_D(\Xhat)}
\end{equation}
with closed-form expressions for
\begin{equation*}
    \begin{aligned}
        \eqcontrast \mu_D(\Xhat) &\eqcontrast = \mu(\Xhat) + \Sigma(\Xhat,X)(\Sigma(X) \eqchange{+ \sigmaeps^2I})^{-1} (Y - \mu(X)) \\
        \eqcontrast \Sigma_D(\Xhat) &\eqcontrast = \Sigma(\Xhat) - \Sigma(\Xhat,X)(\Sigma(X) \eqchange{+ \sigmaeps^2I})^{-1} \Sigma(X,\Xhat)
        .
    \end{aligned}
\end{equation*}
as shown by \textcite[Eq. 2.22 to 2.24]{rasmussen2005gaussian}. \cref{fig:background-gp-gpr-posterior-noisy} visualizes how the posterior distribution preserves part of the initial uncertainty even at the measured data points \(X\) and allows sample functions to not perfectly pass through the data points in \(D\).


\section{Information Theory}\label{sec:background-infotheory}

While we need Gaussian processes for our algorithm, information theory is relevant only for the analysis. We first describe entropy, the key quantity for quantifying the amount of information in a random variable, in \cref{ssec:background-infotheory-entropy}. Then we discuss information gain as the quantity for the mutual information between random variables in \cref{ssec:background-infotheory-infogain}.

We refer to the book \citetitle{murphy2012machine} written by \textcite[Chapter 2.8]{murphy2012machine} for an additional treatment of information theory in the context of machine learning.

\subsection{Information Entropy}\label{ssec:background-infotheory-entropy}

Consider the random variable \(X\) distributed with some probability distribution \(\P_X\). Since \(X\) is random, there is a notion of ``uncertainty'' contained in the \(X\) before we measure it. After measuring \(X\), there is a notion of ``information'' contained in the measured value \(x\in\Xcal\).

The \emph{information content}, also called \emph{Shannon information}, was first introduced by \textcite{shannon1948mathematical} and unifies both notions.
For a specific value \(x\in\Xcal\), it quantifies the amount of uncertainty for measuring this value \(x\) or equivalently the amount of information obtained by measuring \(x\).

\begin{definition}[Information Content]\label{def:information-content}
    Let \(X\) be a random variable and \(p(x)\) the probability of a specific value \(x\in\Xcal\). The information content of \(x\) is defined as%
    \footnote{Originally, \textcite{shannon1948mathematical} proposed three axioms for the definition of information content and showed that \(-\log p(x)\) is the only function satisfying these axioms up to a multiplicative factor.}
    \begin{equation*}
        h(x) \defeq -\log p(x)
        .
    \end{equation*}
\end{definition}

The negative logarithm maps probabilities close to one to the information content zero and very small probabilities to an exponentially increasing information content. Intuitively, values with smaller probabilities are less likely to be measured and therefore contain more uncertainty and provide more information than values with high probabilities.
An alternative way is to describe the information content as the amount of ``surprise'' with less likely values leading to a larger surprise.



The \emph{entropy}, also called \emph{Shannon entropy}, for a given random variable \(X\) is defined as the expected information content over all possible values \(x\in\Xcal\). This means it quantifies the expected uncertainty in the random variable \(X\) or equivalently the expected information obtained when sampling a value from \(X\).

\begin{definition}[Information Entropy]\label{def:information-entropy}
    Let \(X\) and \(Y\) be random variables. The entropy of \(X\) or of its probability distribution \(\P_X\) is defined as
    \begin{equation*}
        H(X) \defeq \E*[X]{-\log p(X)} = \E*[X]{h(X)}
    \end{equation*}
    and similarly the joint and conditional entropy as
    \begin{equation*}
        \begin{aligned}
            H(X,Y) &\defeq \E*[X,Y]{-\log p(X,Y)}\\
            H(Y\mid X) &\defeq \E*[X,Y]{-\log p(Y\mid X)}
            .
        \end{aligned}
    \end{equation*}
\end{definition}

Observe that a random variable \(X\) which in fact is deterministic with \(p(x) = 1\) for some \(x\in\Xcal\) has entropy \(H(X) = 0\), while a random variable with a uniform distribution has maximum entropy, which corresponds to having maximum average uncertainty for its values.
This can be easily verified for a discrete random variable with \(\Xcal = \set{x_1,\dots,x_n}\) using Jensen's inequality
\begin{equation*}
    H(X) = \E*{-\log p(X)} = \E*{\log\frac{1}{p(X)}} \leq \log\E*{\frac{1}{p(X)}} = \log n
\end{equation*}
with \(\log(x)\) being concave. Since Jensen's inequality holds with equality only for \(p(x_1)=\dots=p(x_n)\), the entropy is maximized for uniformly distributed random variables. For a continuous random variable defined on \(\Xcal=[a,b]\), one can proceed similarly using Jensen's inequality for integrals on probability measures \autocite[Theorem 1.6.2]{durrett2019probability}.

In addition, observe that for random variables \(X\) and \(Y\) we obtain
\begin{equation}\label{eq:information-entropy-properties}
    \begin{aligned}
        H(X,Y) &= H(X\mid Y) + H(Y)\\
        &= H(Y\mid X) + H(X)\\
        H(X,Y) &= H(X) + H(Y) \quad\with X,Y \text{ independent}
    \end{aligned}
\end{equation}
by additivity of the logarithm and linearity of expectation. This shows that information or uncertainty is an additive quantity.

The entropy for a Gaussian distribution can be derived in closed-form.
\begin{lemma}[Information Entropy of Gaussian distribution]\label{lem:information-entropy-gaussian}
    \begin{equation*}
        \begin{alignedat}{2}
            H(X) &= \frac{1}{2}\log*{2\pi\sigma^2} + \frac{1}{2}
            &&\quad\with X\sim\Normal*{\mu,\sigma^2}\\
            H(X) &= \frac{1}{2}\log\det*{2\pi e\Sigma}
            && \quad\with X\in\Normal*{\mu,\Sigma}
        \end{alignedat}
    \end{equation*}
\end{lemma}
\begin{proof}
    \cref{ssec:proofs-lemma-information-entropy-gaussian}
    \noqed
\end{proof}

\subsection{Information Gain}\label{ssec:background-infotheory-infogain}
An important quantity for our analysis is the \emph{information gain} \(I(X;Y)\), which is also known as the \emph{mutual information} between \(X\) and \(Y\). It can be interpreted as the expected amount of information one can gain about one random variable \(X\) by measuring the other random variable \(Y\) or vice versa. Equivalently, it can also be seen the amount of information or uncertainty which is mutually contained in \(X\) and \(Y\). Both perspectives indicate that the information is a measure for the mutual dependence between both random variables.

\begin{definition}[Information Gain]\label{def:information-gain}
    Let \(X\) and \(Y\) be random variables with distributions \(\P_X\) and \(\P_Y\). The information gain or mutual information between \(X\) and \(Y\) is defined as
    \begin{equation*}
        \begin{aligned}
            I(X;Y) &\defeq \DKL{\P_{X,Y}}{\P_{X}\cdot\P_{Y}}
            \\ &= \E*[X,Y\sim\P_{X,Y}]{\log p(X,Y) - \log{p(X)p(Y)}}
        \end{aligned}
    \end{equation*}
    with Kullback-Leibler divergence \(\DKL{\cdot}{\cdot}\).
\end{definition}

It directly follows
\begin{equation}\label{eq:information-gain-property-joint}
    I(X;Y) = H(X) + H(Y) - H(X,Y)
\end{equation}
from \cref{def:information-entropy} by additivity of logarithm and linearity of expectation. We can further derive
\begin{equation}\label{eq:information-gain-property-conditional}
    \begin{aligned}
        I(X;Y) = H(X) - H(X\mid Y)\\
        I(X;Y) = H(Y) - H(Y\mid X)
    \end{aligned}
\end{equation}
using \cref{eq:information-entropy-properties,eq:information-gain-property-joint}.

Observe that for independent random variables \(X\) and \(Y\), we naturally have \(I(X;Y) = 0\) similarly from \cref{eq:information-entropy-properties,eq:information-gain-property-joint}. Maximum amount of mutual information between two random variables \(X\) and \(Y\) is achieved if they are deterministic functions of each other with \(H(X\mid Y) = H(Y\mid X) = 0\). In this case, we can further observe that \(I(X;Y) = H(X) = H(Y)\) from \cref{eq:information-gain-property-conditional}.

%% file: 03_related_work.tex
\chapter{Related Work}\label{chp:related}

Our work on near-optimal active reconstruction lies at the intersection of two main areas of previous research. It combines the more applied topics from active reconstruction, for which we discuss related work in \cref{sec:related-active-reconstruction}, with the theoretical analysis of near-optimality based on work related to Gaussian process optimization, which we present in \cref{sec:related-gp-optimization}.


\section{Active Reconstruction}\label{sec:related-active-reconstruction}

More generally, the field of \emph{active vision} is concerned with sensor planning strategies which actively select new sensor placements to fulfill vision-based tasks which depend on multiple views on some target object or scene.

Some of the earliest work was done by \textcite{aloimonos1988active}, who formally investigated the advantages of an active observer, who is able to actively explore and sample new visual data, over a passive observer, who is given a fixed set of visual data.
At the same time, \textcite{cowan1988automatic} presented an approach to automate the selection of camera viewpoints based on geometric constraints on the sensor locations.
A survey on various active vision topics in robotic applications was given by \textcite{chen2011active}, who categorized active vision tasks into \emph{model-based tasks}, for which a model of the target is provided, such that all sensor placements can be computed offline, or \emph{model-free tasks}, for which new sensor placements are determined online without prior information about the target.

The goal of \emph{active reconstruction}, a model-free task, is to find a sequence of views for reconstructing a complete model of the target object or scene.
In particular, in each round one aims to find a NBV which is typically defined as the view with the maximum amount of new information about the target. 

\textcite{connolly1985determination} was one of the first who presented two algorithms for determining the NBVs using octree data structures. Given an partial octree with not yet observed nodes labeled ``unseen'', the first algorithm greedily maximizes the number of unseen nodes inside the view plane over a densely sampled view sphere, while the second algorithm additionally takes information about node faces into account to save computation time.
Later, \textcite{banta2000nextbestview} presented three approaches based on human-intuitive heuristics for finding the NBV with the help of an octree-based world occupancy grid. The heuristics include positioning the sensor towards detected edges to reveal the occluded areas behind the edges, towards centroids of previously occluded surface to maximize the amount of newly observed surface, or towards clustered unobserved patches.
\textcite{chen2005vision} proposed a method based on the target's trend surface which predicts the local curvature of the unknown surface area. This trend surface then determines the exploration direction and the parameters of the NBV.

Besides geometrical considerations, methods for finding the NBV from a probabilistic perspective were proposed, where the NBV is commonly defined as the view which maximally reduces the uncertainty about the target.
\textcite{wenhardt2007active} used extended Kalman filters with sensor actions for probabilistic state estimations. This allowed them to adopt the alphabetical optimality criterions from optimal experimental design such as the A-, D- or E-optimality criterion \autocite{pukelsheim2006optimal}.
\textcite{delmerico2018comparison} surveyed different and proposed new volumetric information gain metrics based on a probabilistic world occupancy map. Their metrics include both non-probabilistic counting metrics as well as probabilistic entropy-based metrics for determining the NBV.

Following the classification of \textcite{bissmarck2015efficient}, NBV methods can be categorized into \emph{global methods}, which exploit some global data characteristics, \emph{surface-based methods}, which operate based on occlusion or frontier surfaces, and \emph{volumetric methods}, which evaluate candidate views based on the potential information gain inside the view frustum.

In summary, most of the presented related work discretize their world with an octree-based occupancy grid into occupied and free voxels, each optionally associated with some occupancy probability. Different viewpoints are then evaluated based on ray-casting and certain metrics and the best viewpoint is returned as the NBV.
For our later described 2D setting, we mostly adopt their approaches and discretize our world into a set of pixels, place the camera at a fixed view circle and perform ray-casting to evaluate our metrics for different camera positions. We further combine the incorporation of a trend surface for predicting the object shape a priori \autocite{chen2005vision} with a probabilistic approach by modeling the object with a Gaussian process which provides us confidence bounds on the object shape.

However, in what the described related work differs from our work is that their methods mostly depend on heuristics. None of them addresses the optimality of their methods with respect to the true NBV defined as the best view given full knowledge about the target and infinite computation power for an exhaustive search over the state space.
In particular, \textcite{chen2005vision} stated that it is impossible to give the true NBV without information about the unknown target. We quote \textcite{delmerico2018comparison} that \enquote{based on the current state of the art, it is not clear that there is an optimal way to quantify the volumetric information [...] with respect to choosing views based on maximizing information gain.}
This is exactly our focus to show performance guarantees with respect to the true NBV for our algorithm.



\section{Gaussian Process Optimization}\label{sec:related-gp-optimization}



The goal of \emph{Gaussian process optimization} is to sequentially optimize an unknown function over multiple rounds, which can be formulated as maximizing an unknown reward function in a multi-armed bandit setting.
\textcite{srinivas2012informationtheoretic} published, originally in 2010, the Gaussian process upper confidence bound (GP-UCB) algorithm, which models the unknown function with a Gaussian process and repeatedly samples the reward function at the maximizer of the current upper confidence bound. They were able to show sublinear regret, which implies that their algorithm is guaranteed to converge to a true optimal solution.

This seminal algorithm was further improved and built upon in many later works. 
\textcite{contal2014gaussian} introduced the GP-MI algorithm, which significantly improved the sublinear bounds for the cumulative regret. 
\textcite{krause2011contextual} presented CGP-UCB, which generalized GP-UCB for contextual bandit problems, and \textcite{chen2017interactive} presented SM-UCB, which further generalized CGP-UCB for interactive contextual bandit problems. For the latter problem, the algorithm deals with sequentially constructing a set which jointly maximizes the unknown reward function by interactively querying the marginal reward.
A more complex variation of GP-UCB was developed for safe multi-agent coverage control by \textcite{prajapat2022nearoptimal}, who modeled the unknown density function with Gaussian processes to achieve near-optimal coverage of the density.

The goal of our work is to apply a variant of GP-UCB to the active reconstruction setting, such that we can make use of the strong theoretical guarantees.
The setting of \textcite{chen2017interactive} is closest to our setting, as we similarly sample measurements from the object surface in an interactive manner to maximize the overall observation coverage. After introducing our setting in \cref{sec:problem-general,sec:problem-simplified}, we compare it with the ones of \textcite{srinivas2012informationtheoretic}, \textcite{prajapat2022nearoptimal}, and \textcite{chen2017interactive} in more detail in \cref{sec:problem-comparison}.

%% file: 04_problem.tex
\chapter{Problem Formulation}\label{chp:problem}

In this chapter, we introduce the mathematical formulation of the \emph{near-optimal active reconstruction problem}. We provide an overview of our setting in \cref{fig:problem-overview}.

In \cref{sec:problem-general} we start with the most general setting focusing only on the decision making process for the NBV. We formalize the selected camera pose as the algorithm's decision and describe the underlying objective, our precise understanding of a near-optimal decision and the notion of regret with respect to this near-optimal decision.

In \cref{sec:problem-simplified} we instantiate this general setting and formalize the notion of a world, an object and the observation and measurement model of a camera.
To keep the theoretical analysis of our algorithm feasible, we make certain simplifications on the setting, which we list at end of this section in detail.

In \cref{sec:problem-comparison} we compare our described setting to related work from a more technical perspective which provides additional understanding of the similarities and differences to other problem formulations.

\begin{figure}
    \newtcolorbox{boxcontainer}[3][]{
        boxstyle=#3,
        title={#2},
        fonttitle=\footnotesize,
        #1
    }
    \newtcbox{\circlecontainer}[1][]{
        circlestyle=gray,
        #1
    }
    \centering
    \begin{tikzpicture}[
        remember picture,
        every node/.style={inner sep=1mm, outer sep=0mm},
        >/.tip=Stealth,
        node distance=0cm,
        font=\scriptsize,
    ]
        \node (algorithm) {
            \begin{boxcontainer}[width=6.5cm, top=0.75mm, bottom=0.75mm]{Algorithm \(\A\)}{green}
            \centering
            \begin{tikzpicture}[remember picture]
                \node (gp) {
                    \begin{boxcontainer}[width=2.9cm]{}{gray}
                        Gaussian process \\ \(\GP*{m(\varphi),k(\varphi,\varphi')}\)
                    \end{boxcontainer}
                };
                \node (objective) [right=0.5cm of gp] {
                    \begin{boxcontainer}[width=2.1cm]{}{gray}
                        Objective \\ \(\Fu(\theta\mid\theta_{1:t-1})\)
                    \end{boxcontainer}
                };
            \end{tikzpicture}
            \end{boxcontainer}
        };
        \node (camera) [below=0.1cm of algorithm] {
            \begin{boxcontainer}[width=6.5cm, top=0.75mm, bottom=0.75mm]{Camera}{blue}
            \centering
            \begin{tikzpicture}
                \node (positioncam) {
                    \begin{boxcontainer}[width=2.4cm]{}{gray}
                        position camera
                    \end{boxcontainer}
                };
                \node (takeimage) [left=0.5cm of positioncam] {
                    \begin{boxcontainer}[width=2.6cm]{}{gray}
                        take depth image
                    \end{boxcontainer}
                };
            \end{tikzpicture}
            \end{boxcontainer}
        };
        
        \node (input) [left=0.75cm of gp] {
            \circlecontainer[height=1.25cm, bean arc]{
                observed \& \\ measured surface \\ at \(\theta_{1:t-1}\)
            }
        };
        \node (output) [right=0.75cm of objective] {
            \circlecontainer[width=1.15cm, tcbox width=forced center]{
                NBV \\ \(\theta_t\)
            }
        };
        
        \draw[<-] (input.west) -- node[above=0.5mm]{\(t\gets 0\)} +(left:0.75cm);
        \draw[->] (input.east) -- (gp.west);
        \draw[->] (gp.east) -- (objective.west);
        \draw[->] (objective.east) -- (output.west);
        \draw[->] (output.south) .. controls +(down:1cm) and +(right:1cm) .. (positioncam.east);
        \draw[->] (positioncam.west) -- (takeimage.east);
        \draw[->] (takeimage.west) .. controls +(left:1cm) and +(down:1cm) .. node[below,sloped] {\(t\gets t+1\)} (input.south);
    \end{tikzpicture}
    \caption[Overview of the Setting]{
        Overview of the Setting.
        The goal is to maximize the observed object surface through sequential decisions. This is an active learning problem, in which the algorithm \(\A\) interactively queries new locations \(\theta_t\) (data points), for which the camera returns the corresponding measurements (labels). The goal for \(\A\) is to find the most informative locations.
    }
    \label{fig:problem-overview}
\end{figure}
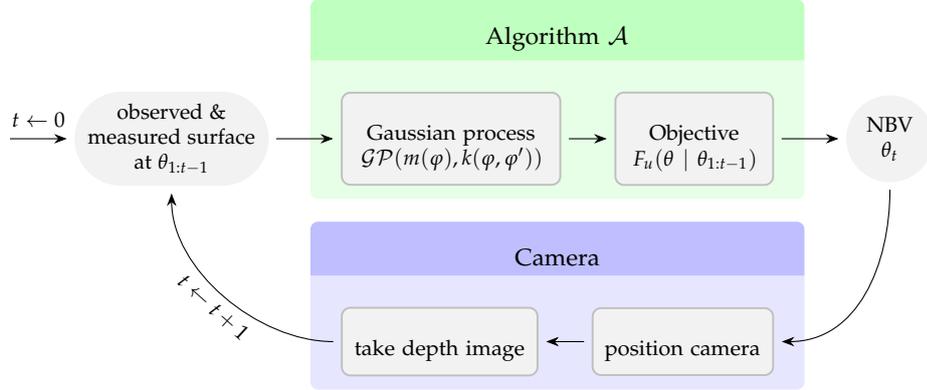


\section{General Setting}\label{sec:problem-general}

We start with introducing the decision in \cref{ssec:problem-general-decision} and objective in \cref{ssec:problem-general-objective}. Based on the objective, we discuss the notion of near-optimal decisions in \cref{ssec:problem-general-near-optimality}. We continue describing the regret with respect to near-optimality in \cref{ssec:problem-general-regret} and finish with highlighting the requirements for convergence to near-optimality as our final goal in \cref{ssec:problem-general-convergence}.

\subsection{NBV Estimate (Decision)}\label{ssec:problem-general-decision}
We define \(\Camspace\) to be the space of camera poses \(\theta\), over which the algorithm searches for the NBV. The \emph{NBV estimate} returned by our algorithm \(\A\) in round \(t\) is denoted with \(\theta_t\). For convenience, we use \(\theta_{1:t}\) for the set of NBV estimates returned by the algorithm in the first \(t\) rounds. Since the algorithm returns \(\theta_t\) based on the measurements of the object surface from previous camera poses \(\theta_{1:t-1}\), we denote the algorithm's decision with
\begin{equation*}
    \theta_t = \A(\theta_{1:t-1})
    .
\end{equation*}
These decisions are typically based on optimizing some objective over \(\Camspace\).

\begin{remark}\label{remark:different-camspaces}
    The space of camera poses \(\Camspace\) can describe any reasonable space uniquely specifying the camera's position and orientation, such as
    \begin{itemize}
        \item \(\Camspace=[0,2\pi]\) (polar angle) for a 2D camera with fixed camera orientation and radial distance. This is what we later use in \cref{ssec:problem-simplified-camera}.
        \item \(\Camspace=[0,2\pi] \times [0,\infty)\) (polar angle, radial distance) for a 2D camera with fixed camera orientation.
        \item \(\Camspace=[0,2\pi] \times [0,\infty)\times[0,2\pi]\) (polar angle, radial distance, orientation) for an unconstrained 2D camera with polar coordinates.
        \item \(\Camspace=\Real^2 \times[0,2\pi]\) (\(xy\)-coordinate, orientation) for an unconstrained 2D camera with Cartesian coordinates.
        \item \(\Camspace=\Real^3 \times[0,\pi]\times[0,2\pi]\) (\(xyz\)-coordinate, polar angle, azimuthal angle) for an unconstrained 3D camera.
    \end{itemize}
    For now, we keep \(\Camspace\) general, since our results in this section do not depend on the specific choice for \(\Camspace\).
\end{remark}

\subsection{Objective}\label{ssec:problem-general-objective}
We define the \emph{utility} as the set function \(F\colon 2^\Camspace \to \Real, \Theta \mapsto F(\Theta)\) which measures the reconstruction progress for a given set of camera poses \(\Theta \subseteq \Camspace\). The returned utility value represents how good the selected set of camera poses is for reconstructing the given object. The objective of the algorithm is to maximize this utility.
We define the \emph{marginal utility}
\begin{equation}\label{eq:marginal-utility}
    F(\theta \mid \Theta) \defeq F(\Theta \cup \set{\theta}) - F(\Theta)
\end{equation}
to be the utility of measuring the object from \(\theta\) after already measuring it from \(\Theta\). It corresponds to the increase in the utility value caused by the additional measurement at \(\theta\).
We naturally have monotonicity
\begin{equation}\label{eq:monotonic-utility}
    F(\Theta_1) \leq F(\Theta_2) \with \Theta_1 \subseteq \Theta_2
    ,
\end{equation}
since the reconstruction progress cannot be reversed by making additional measurements. In addition, we require submodularity 
\begin{equation}\label{eq:submodular-utility}
    F(\Theta_1 \cup \set{\theta}) - F(\Theta_1) \geq F(\Theta_2 \cup \set{\theta}) - F(\Theta_2)
    \text{ for all } \Theta_1 \subseteq \Theta_2,\theta \notin \Theta_2
    .
\end{equation}
This property can be equivalently formulated as \(F(\theta \mid \Theta_1) \geq F(\theta \mid \Theta_2)\) using the marginal utility. It intuitively describes the diminishing returns property that the same element \(\theta\) provides a smaller marginal utility when added to a larger set \(\Theta_2\). The more measurements of the object surface have already been made previously -- in this case set \(\Theta_2\), the smaller the marginal utility of the new measurement taken from \(\theta\).
Note that the utility and marginal utility functions are unknown in practice, since the algorithm has no knowledge about the true object surface. The idea is to design an \emph{upper bound for the marginal utility}
\begin{equation}\label{eq:marginal-utility-upper-bound}
    F(\theta \mid \theta_{1:t-1}) \leq \Fu(\theta \mid \theta_{1:t-1}) \text{ for all } \theta \in \Camspace
    ,
\end{equation}
which at the same time contains enough information about the true marginal utility, such that it serves as a reasonable objective function for \(\A\). Precise requirements and concrete designs for \(\Fu\) are stated in \cref{lem:lemma-2-2,sec:design-objective}.

\subsection{Near-Optimality}\label{ssec:problem-general-near-optimality}
The goal of this work is to devise an algorithm, which returns a NBV estimate \(\theta_t\) for each round \(t=1,\dots,T\) based on the information collected from \(\theta_{1:t-1}\). Ideally, the reconstruction progress of \(\theta_{1:T}\) as measured by the utility function should be as good as for an optimal solution. We define an \emph{optimal solution} to be a set of at most \(T\) camera poses maximizing the utility
\begin{equation}\label{eq:optimal-solution}
    \thetaopt_T \defeq \argmax_{\Theta \subseteq \Camspace,\abs{\Theta}\leq T} F(\Theta)
    .
\end{equation}
Finding an optimal solution is not possible in practice. Even if the space of camera poses \(\Camspace\) was finite and the true marginal utility known, this would require a combinatorial search over \(\Camspace\) due to the cardinality constraint \(\abs{\Theta}\leq T\), which is NP-hard as stated by \textcite[p. 266]{nemhauser1978analysis}. Note that the optimal solutions for different \(T\) do not necessarily have something in common. For example, the optimal camera locations for \(T=8\) can be completely different than for \(T=12\) and we have \(\thetaopt_T \nsubseteq \thetaopt_{T+1}\) in general.

For these reasons, we are satisfied with a \emph{near-optimal solution} \(\theta_{1:T}\) in practice, which is defined as
\begin{equation}
    F(\theta_{1:T}) \geq (1-\alpha) F(\thetaopt_T) \with \alpha \in (0,1)
    .
\end{equation}
This means a near-optimal solution is a \((1-\alpha)\)-approximation of the optimal solution guaranteeing its utility to be at least a constant fraction of the utility of an optimal solution.

Such a near-optimal solution can be theoretically found by a greedy algorithm with knowledge about the true object surface, which is normally not given in practice. The greedy algorithm sequentially constructs a set of camera poses by greedily selecting \(\theta\) in each round which maximizes the true marginal utility given the previous measurements from \(\theta_{1:t-1}\). We define the \emph{greedy decision} as
\begin{equation}\label{eq:greedy-decision}
    \thetagre_t \defeq \argmax_{\theta\in\Camspace} F(\theta \mid \theta_{1:t-1})
    .
\end{equation}
Note the subtle difference in notation between the optimal solution with superscript \(\star\) and greedy solution with superscript \(*\).
\textcite[Theorem 4.2]{nemhauser1978analysis} showed that such a greedy algorithm with a submodular objective function is guaranteed to achieve a \(\paren*{1-\frac{1}{e}} \approx 63\%\) approximation of \(\thetaopt_T\) in the worst-case and it is NP-hard to guarantee a better solution according to \textcite[Theorem 5.3]{feige1998threshold}. We use this greedy algorithm as our theoretical baseline for near-optimality.

One important problem in our setting is that a solution to the reconstruction problem does not only depend on the decision \(\theta_t\) in the current round, but on the sequentially made decisions \(\theta_{1:T}\) over all rounds. Since the algorithm initially starts with no information about the object, it is intuitively not possible to guarantee near-optimal decisions from the beginning on. Since the first few decisions can be arbitrarily bad, it cannot be guaranteed that the final solution \(\theta_{1:T}\) is near-optimal no matter how optimal the later decisions were. However, it is possible to eventually guarantee near-optimality for the individual decisions \(\theta_t\), which do not depend on previously made, less informed decisions. We formalize the notion of near-optimal decisions in the next \cref{ssec:problem-general-regret}.

\begin{remark}\label{remark:solution-vs-decision}
    Note the following difference in our wording. A \emph{solution} refers to all selected camera poses \(\theta_{1:t}\) up to the current round \(t\). It corresponds to a ``solution for the reconstruction problem'', which aims to find the best possible set of camera poses to reconstruct the object. A \emph{decision} refers to the single camera pose \(\theta_t\) selected by the algorithm in the current round \(t\). It corresponds to a ``solution for the NBV problem'', which aims to find the next best possible camera pose to reconstruct the object given the previous camera poses.
\end{remark}

\subsection{Regret}\label{ssec:problem-general-regret}
The \emph{regret} is a common quantity in decision theory and measures the difference between the utility of an optimal decision and the utility of a decision made under uncertainty.
As discussed in \cref{ssec:problem-general-near-optimality}, we can eventually guarantee near-optimal decisions, but not a near-optimal solution to the reconstruction problem due to arbitrarily bad decisions made in the beginning.
Hence, we define the cumulative regret as the difference in utility between a near-optimal solution guaranteed by the greedy algorithm and the solution of our algorithm.
\begin{equation}\label{eq:cumulative-regret}
    R(T) \defeq \paren*{1-\frac{1}{e}} F(\thetaopt_T) - F(\theta_{1:T})
    \quad\eqnote{cumulative}
\end{equation}
Let \(r(t)\) be the regret for the single decision in round \(t\), which we denote as the simple regret. Then the cumulative regret should naturally be the sum of all simple regrets \(R(T) = \sum_{t=1}^T r(t)\). To better understand the simple regret incurred by each decision, we derive from the above cumulative regret definition the following expression for the simple regret:
\begin{equation}\label{eq:simple-regret}
    \begin{aligned}
        r(t) &\defeq R(t) - R(t-1)\\
        &= \paren*{1-\frac{1}{e}} (F(\thetaopt_t) - F(\thetaopt_{t-1})) - F(\theta_t \mid \theta_{1:t-1})
    \end{aligned}
    \quad\eqnote{simple}
\end{equation}
Intuitively, the simple regret measures the difference between the increase in utility of an near-optimal solution when allowing it to make one more decision and the marginal utility of the decision made by our algorithm. Hence, we define a \emph{near-optimal decision} \(\theta_t\) in round \(t\) as
\begin{equation}
    F(\theta_t \mid \theta_{1:t-1}) \geq (1-\alpha) (F(\thetaopt_t) - F(\thetaopt_{t-1})) \with \alpha \in (0,1)
    .
\end{equation}
Since \(\thetaopt_{t-1} \nsubseteq \thetaopt_t\), we cannot refer to the increase in utility \(F(\thetaopt_t) - F(\thetaopt_{t-1})\) of an near-optimal solution as the marginal utility.

\begin{remark}\label{remark:negative-regret}
    Comparing the solution against a near-optimal solution in our specific setting instead of an optimal solution in general decision theory allows the regret to be negative.
    In particular, since the \(\paren*{1-\frac{1}{e}}\)-approximation guarantee only lower bounds the worst-case performance of a greedy algorithm, a reasonably good algorithm typically outperforms this guarantee in practice \autocite{leskovec2007costeffective}.
\end{remark}

Unfortunately, we cannot theoretically reason about the actual regret of our algorithm, because this requires us to know the utility of an optimal solution \(\thetaopt_t\). However, it is essential for us to quantify how the regret evolves over time and how close the algorithm's decision is to a near-optimal decision. As a remedy, we define the \emph{individual regret} as the difference in utility between the greedy decision from \cref{eq:greedy-decision} and the algorithm's decision.
\begin{equation}\label{eq:individual-regret}
    \begin{aligned}
        r_{ind}(t) &\defeq F(\thetagre_t \mid \theta_{1:t-1}) - F(\theta_t \mid \theta_{1:t-1})
        &&\quad\eqnote{simple}
        \\
        R_{ind}(T) &\defeq \sum_{t=1}^T r_{ind}(t)
        &&\quad\eqnote{cumulative}
    \end{aligned}
\end{equation}
Later we show in \cref{lem:lemma-2-1} that we can upper bound the actual regret \(r(t)\) with the help of \(r_{ind}(t)\). The advantage of the individual regret is that we are able to compute the greedy decision \(\thetagre_t\) in theory with the knowledge of the true marginal utility, while this is not even possible for the optimal solution \(\thetaopt_t\) due to NP-hardness. This allows us to reason about the individual regret, which in turn helps us to reason about the actual regret.

For better understanding the different utility and regret formulations, we provide a visualization of them in \cref{fig:problem-utility}.

\begin{figure}
    \centering
    \includegraphics[width=\linewidth]{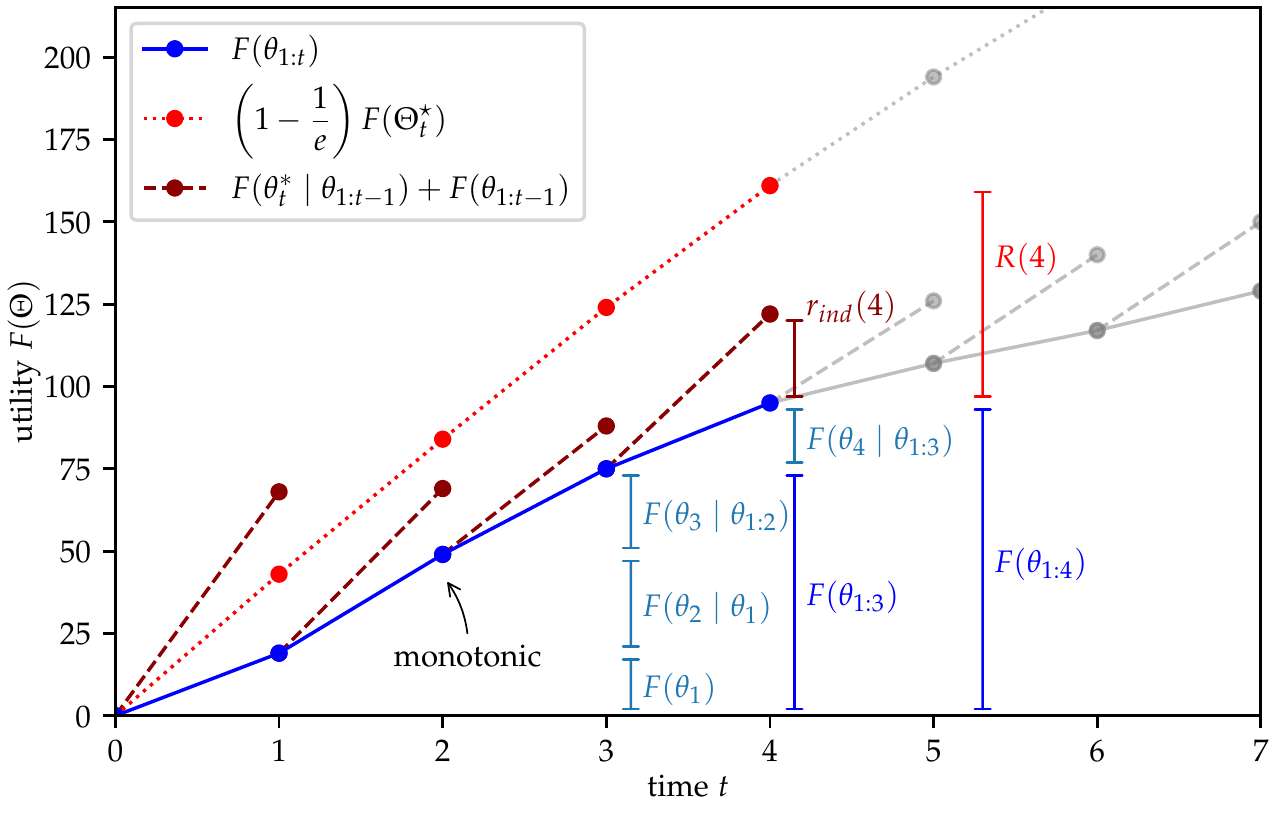}
    \caption[Utility and Regret Diagram]{
        Utility and Regret Diagram.
        Let us go through this time-utility diagram step by step. First focus on the algorithm's total utility (blue) which monotonically increases over the rounds. Observe that the utility of the decisions \(\theta_{1:3}\) is exactly the sum of marginal utilities (light blue) of \(\theta_1\) to \(\theta_3\), which each describe the increase in utility through the current decision relative to the previous decisions.
        \\
        Now focus on the utility of a near-optimal solution (red) which corresponds to \(1-\frac{1}{e}\approx63\%\) of the utility of an optimal solution \(\thetaopt_t\). As discussed above, it is important to remember that the optimal solutions for different \(t\) can be completely different and are not related to each other, which is why we only connect the dots loosely.
        We can observe that the utility of a near-optimal solution is not only larger, but also increases faster than the utility of the algorithm over the rounds. This gap between the algorithm's utility and the near-optimal utility is quantified as the cumulative regret \(R(T)\). The goal is to show that this cumulative regret increases sublinearly over the rounds, as we discuss later.
        \\
        To this end, we define the individual regret \(r_{ind}(t)\) (dark red) as the regret with respect to the greedy decision. In this regard, we specifically emphasize that this greedy decision is always defined relative to the algorithm's previous decisions \(\theta_{1:t-1}\) and not the previous greedy decisions as depicted.
        Since the first greedy decision coincides with the optimal solution for at most \(t=1\) measurement, the utility of a near-optimal solution (red) lies exactly at \(63\%\) of the greedy decision's utility (dark red) at \(t=1\).
    }
    \label{fig:problem-utility}
\end{figure}

\subsection{Convergence to Near-Optimality}\label{ssec:problem-general-convergence}
The ideal goal for \(\A\) is to achieve \emph{convergence to a near-optimal} solution, which is defined as
\begin{alignat}{2}\label{eq:near-optimality}
        &\forall \varepsilon > 0 \exists T_0 \geq 1 \forall T \geq T_0\colon
        r(T) < \varepsilon
        \intertext{or equivalently}
        &\begin{multlined}[0.8\linewidth]
            \forall \varepsilon > 0 \exists T_0 \geq 1 \forall T \geq T_0\colon
            \\ F(\theta_T \mid \theta_{1:T-1}) > \paren*{1-\frac{1}{e}} (F(\thetaopt_T) - F(\thetaopt_{T-1})) - \varepsilon
            .
        \end{multlined} \notag
\end{alignat}
In words, the marginal utility of the decision \(\theta_T\) returned by \(\A\) is guaranteed to be at least as high as the utility of a near-optimal decision up to precision \(\varepsilon\) for all \(T\) after some finite time \(T_0\). Note that \cref{eq:near-optimality} looks very similar to the definition \(\lim_{T\to\infty} r(T) = 0\), but it differs in upper bounding \(r(T)\) instead of \(\abs{r(T)}\) with \(\varepsilon\). Intuitively, this difference only requires the regret to be at most zero asymptotically, but it can be arbitrarily negative and the limit does not necessarily exist. This comes from the fact that the regret is defined with respect to a near-optimal solution in our setting and can be negative as stated in \cref{remark:negative-regret}. If the limit of the regret exists, \cref{eq:near-optimality} is equivalent to \(\lim_{T\to\infty} r(T) \leq 0\) as shown in \cref{ssec:proofs-auxiliary-convergence}.

Convergence to a near-optimal solution for \(\A\) is typically derived by showing that the cumulative regret \(R(T)\) increases at most sublinearly in \(T\). This is commonly referred to as \emph{sublinear regret} and can be written as
\begin{equation}\label{eq:sublinear-regret}
    R(T) \leq \bigO*{T^n}
    \quad\with n < 1
    .
\end{equation}
This implies that the average regret of \(\A\) is asymptotically zero, precisely \(\lim_{T\to\infty} R(T)/T = 0\), which is often referred to as \(\A\) being no-regret. Algorithms with this asymptotic property are guaranteed to converge to a near-optimal solution, since the simple regret in each round converges similarly to zero \autocite{vermorel2005multiarmed, chowdhury2017kernelized}. In our setting, we define \emph{no-regret} as
\begin{equation}\label{eq:no-regret}
    \forall \varepsilon > 0 \exists T_0 \geq 1 \forall T \geq T_0\colon
    \frac{R(T)}{T} < \varepsilon
    ,
\end{equation}
which describes that the average regret is asymptotically non-positive, since the regret is allowed to become negative. If the limit of the average regret exists, \cref{eq:no-regret} is equivalent to \(\lim_{T\to\infty} R(T)/T \leq 0\) as shown in \cref{ssec:proofs-auxiliary-convergence}.
Unfortunately, no-regret does not lead to convergence to near-optimality for our case, but only to a weaker statement as stated in \cref{thm:near-optimality}.

\begin{theorem}[Pseudo-Convergence to Near-Optimality]\label{thm:near-optimality}
    Let \(\A\) be an algorithm with sublinear cumulative regret as defined in \cref{eq:sublinear-regret}. Then \(\A\) makes a decision \(\theta_T\) within some finite time \(T \leq T_0 < \infty\), which is near-optimal up to precision \(\varepsilon\). Precisely,
    \begin{equation*}
        \forall \varepsilon > 0 \exists T_0 \geq 1 \exists T \leq T_0\colon
        r(T) < \varepsilon
        .
    \end{equation*}
\end{theorem}
\begin{proof}
    \cref{ssec:proofs-theorem-near-optimality}
    \noqed
\end{proof}

The difference to \cref{eq:near-optimality} is that \cref{thm:near-optimality} is only able to guarantee \(\exists T \leq T_0\) instead of \(\forall T \geq T_0\) as the third quantor. This means that \(\A\) can guarantee to make a near-optimal decision within some finite time \(T_0\), but cannot guarantee to always make near-optimal decisions after \(T_0\).

\begin{remark}\label{remark:true-vs-pseudo-near-optimality}
    Normally, true convergence to near-optimality does follow from no-regret as stated by \textcites[Section 2]{chowdhury2017kernelized}[Section 1]{vermorel2005multiarmed}.
    This is correct for settings, where the optimal decision \(x^\star\) and likewise the utility \(F(x^\star)\) is independent of the time \(t\). This relation is stated in the following corollary.
    
    \begin{corollary}[Convergence to Near-Optimality]\label{cor:near-optimality}
        Let \(\A\) be an algorithm with sublinear cumulative regret as defined in \cref{eq:sublinear-regret} and assume one of the following conditions:
        \begin{enumerate}
            \item \(r(t) \defeq (1-\alpha)F(x^\star) - F(x_t)\) with \(\alpha \in (0,1)\) is defined with respect to a time-independent optimal decision \(x^\star\) and an objective function \(F(x_t)\) monotonically increasing in \(t\).
            \item \(r(t)\) decreases monotonically in \(t\).
        \end{enumerate}
        Then \(\A\) converges towards a near-optimal decision. Precisely,
        \begin{equation*}
            \forall \varepsilon > 0
            \exists T_0 \geq 1
            \forall T \geq T_0\colon
            r(T) < \varepsilon
            .
        \end{equation*}
    \end{corollary}
    \begin{proof}
        \cref{ssec:proofs-corollary-near-optimality}
        \noqed
    \end{proof}
    
    As discussed above we are not able to show this convergence guarantee in our setting. Recall the definition of the simple regret in \cref{eq:simple-regret}, 
    in which \(F(x^\star)\) corresponds to the time-dependent ``marginal utility`` \(F(\thetaopt_t) - F(\thetaopt_{t-1})\) of an optimal decision and \(F(x_t)\) corresponds to the not necessarily monotonic marginal utility \(F(\theta_t \mid \theta_{1:t-1})\).%
    \footnote{We only require the utility \(F(\Theta)\) to be monotonic as stated in \cref{eq:monotonic-utility}, but not the marginal utility \(F(\theta\mid\Theta)\).}
    Hence, the stronger assumption 1 of \cref{cor:near-optimality} does not hold.
    The more general assumption 2 does not hold either, because it cannot be guaranteed that the marginal utility \(F(\theta_t \mid \theta_{1:t-1})\) increases faster than the ``marginal utility'' \(F(\thetaopt_t) - F(\thetaopt_{t-1})\) of an optimal decision.
\end{remark}

In order to apply \cref{thm:near-optimality}, we have to show sublinear regret for our algorithms. We end this section with presenting the first two general lemmas which help us in achieving this. A complete overview of the structure of theorems and lemmas is given later in \cref{fig:analysis-overview}.

The first lemma upper bounds the cumulative regret with the cumulative individual regret.

\begin{lemma}\label{lem:lemma-2-1}
    \(R(T) < R_{ind}(T)\) for all \(T \geq 1\).
\end{lemma}
\begin{proof}
    \cref{ssec:proofs-lemma-2-1}
    \noqed
\end{proof}

The second lemma upper bounds the cumulative individual regret with the sum of upper bounds on the marginal utility for each \(\theta_t\) under an additional assumption.

\begin{lemma}\label{lem:lemma-2-2}
    Assume
    \begin{equation*}
        F(\thetagre_t \mid \theta_{1:t-1}) \leq \Fu(\theta_t \mid \theta_{1:t-1})
        \quad\text{ for all } t \geq 1
        .
    \end{equation*}
    Then we can show
    \begin{equation*}
        R_{ind}(T) \leq \sum_{t=1}^T \Fu(\theta_t \mid \theta_{1:t-1})
        .
    \end{equation*}
\end{lemma}
\begin{proof}
    \cref{ssec:proofs-lemma-2-2}
    \noqed
\end{proof}

\begin{remark}\label{remark:lemma-2-2-assumption}
    The assumption for \cref{lem:lemma-2-2} typically holds for a reasonable objective function \(\Fu\) and algorithm \(\A\). If \(\A\) is a greedy algorithm and returns the maximizer of \(\Fu\) as the NBV estimate \(\theta_t\), we can show
    \begin{derivation*}
        F(\thetagre_t \mid \theta_{1:t-1})
        &\leq \eqcontrast \eqchange{\Fu}(\thetagre_t \mid \theta_{1:t-1})
        &&\quad\eqnote{since upper bound}
        \\
        &\leq \eqcontrast \Fu(\eqchange{\theta_t} \mid \theta_{1:t-1})
        &&\quad\eqnote{since upper bound maximizer}
        .
    \end{derivation*}
\end{remark}


\section{Simplified 2D Setting}\label{sec:problem-simplified}

After describing the most general setting from a decision theoretical perspective, we continue with describing the actual setting in which we analyze our algorithms. First we discuss our notion of a real world and a polar world in \cref{ssec:problem-simplified-world}. Then we describe how we model our object in \cref{ssec:problem-simplified-object} and how we discretize it in \cref{ssec:problem-simplified-discretization}. Based on the discrete representation, we define the observation and measurement model of the camera in \cref{ssec:problem-simplified-camera}. Finally, we discuss the Gaussian process model in \cref{ssec:problem-simplified-gp}, which is used by the algorithm to model its uncertainty about the unknown target object. Throughout these sections, we make certain simplifications and assumptions on our setting, such that it is feasible for us to analyze our algorithms later. We provide a summary of these simplifications in \cref{ssec:problem-simplified-simplifications} and discuss their implications on the applicability of our results.

\subsection{2D World}\label{ssec:problem-simplified-world}
The main simplification is that we deal with the reconstruction problem in 2D (\cref{item:simp-2D-world}). This means the goal is to reconstruct the boundary of a 2D object given ``1D measurements'' of our camera. In this 2D setting, we use \emph{real world} to refer to the \(x\)-\(y\)-coordinate system with the center of the real world defined as the center of this coordinate system. However, we mostly work with polar coordinates \((\varphi,r)\) and rarely Cartesian coordinates \((x,y)\) to refer to 2D points in the real world. With \emph{polar world} we refer to the polar representation of the real world defined by the \(\varphi\)-\(r\)-coordinate system. We provide visualizations of both world representations in \cref{fig:problem-worlds}.

\begin{figure}
    \centering
    \begin{subfigure}{0.45\linewidth}
        \centering
        \includegraphics[width=\linewidth]{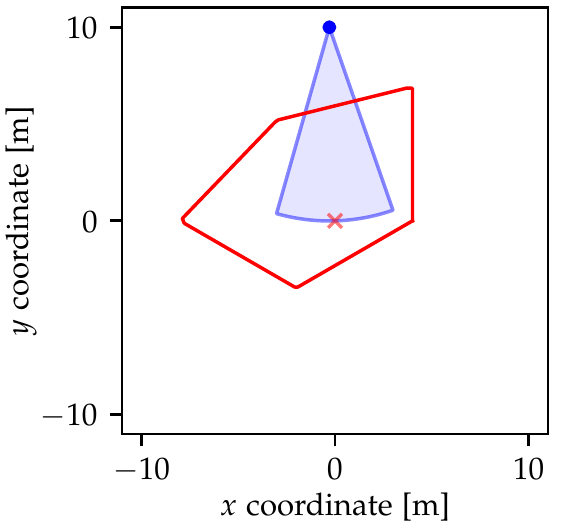}
        \caption{real world}
    \end{subfigure}%
    \quad%
    \begin{subfigure}{0.45\linewidth}
        \centering
        \includegraphics[width=\linewidth]{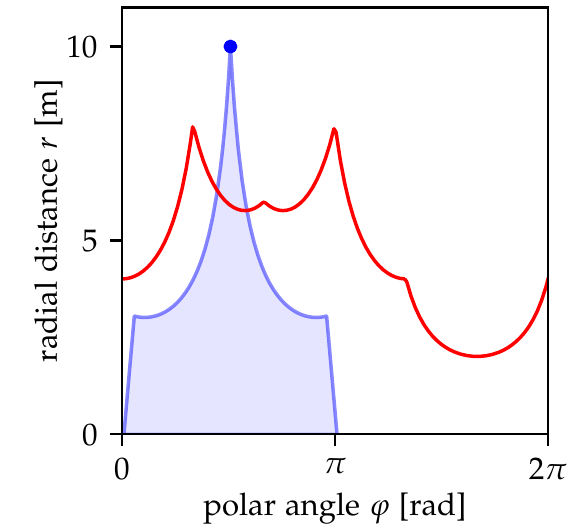}
        \caption{polar world}
    \end{subfigure}
    \caption[Real vs.\ Polar World]{
        Real vs.\ Polar World.
        The real world in (a) is defined as the geometry in which the object (red) and the camera (blue dot) reside. The shape of the camera's FOV is modeled with a sector (blue region).
        The polar world in (b) is obtained through a nonlinear transformation from Cartesian to polar coordinates. Observe how the center of the real world representation (red cross) corresponds to the complete \(\varphi\)-axis in the polar world.
    }
    \label{fig:problem-worlds}
\end{figure}

\subsection{Object}\label{ssec:problem-simplified-object}
The object is represented by a \emph{surface function} \(f\colon \Domain \to \Real, \varphi \mapsto f(\varphi)\) which parameterizes the object surface relative to the center of the real world. In this setting, we use a \(2\pi\)-periodic polar function defined on the domain \(\Domain \defeq [0,2\pi]\) as our surface function. This implicitly limits us to ``well-shaped'' objects with a unique surface point \(f(\varphi)\) for each polar angle \(\varphi \in \Domain\) (\cref{item:simp-object-complexity}).
In addition, the location of the object must be known to the algorithm in practice (\cref{item:simp-object-centered}), since the real world center must lie inside the object for the surface function.
Further we assume the object surface is bounded between \(\dmin \leq f(\varphi) \leq \dmax\) for all \(\varphi \in \Domain\) (\cref{item:simp-object-bounded}) as shown in \cref{fig:problem-object-camera-specification}. Finally, we assume the object surface function is sampled from a Gaussian process \(f \sim \GP*{m(\varphi), k(\varphi,\varphi')}\) (\cref{item:simp-object-sampled}), such that we can model the algorithm's uncertainty regarding the true object shape. Specific details of the Gaussian process used in our setting are discussed in \cref{ssec:problem-simplified-gp}.

\begin{remark}\label{remark:different-domains}
    Depending on the choice for the object surface parameterization, the corresponding parameter space \(\Domain\) differs, such as
    \begin{itemize}
        \item \(\Domain = [0,2\pi]\) (polar angle) for 2D object surfaces parameterized with a \(2\pi\)-periodic polar function. This is our choice of parameterization.
        \item \(\Domain = [0,1]\) (curve parameter) for 2D object surfaces parameterized with a non-self-intersecting, \(1\)-periodic path.
        \item \(\Domain = [0,\pi]\times[0,2\pi]\) (polar angle, azimuthal angle) for 3D object surfaces parameterized with a spherical function.
    \end{itemize}
    Note that in our setting, the space of parameters for the object surface \(\Domain\) coincide with the space of camera poses \(\Camspace=[0,2\pi]\). This is not the general case and we provided examples for other \(\Camspace\) in \cref{remark:different-camspaces}.
    
    In practice, one can think of more complex parameterizations or use more general representations such as point clouds for describing the object surface. The difficulty comes with modeling the object surface with a suitable Gaussian process model as described later in \cref{ssec:problem-simplified-gp}.
\end{remark}

\subsection{Object Discretization}\label{ssec:problem-simplified-discretization}
We consider the real world to be uniformly discretized into a grid of real world pixels of size \(h \times h\), which represents the smallest measurable unit in this setting. Based on the real world discretization, we similarly discretize the object into a set of \emph{surface points} \(\Surface \defeq \set{x_1,\dots,x_N} \subseteq \Domain\), which is a finite set of parameters each uniquely defining a point on the object surface. In our setting, each surface point defines a 2D point \((x_i, f(x_i))\) on the object surface in polar coordinates relative to the real world center. The set of surface points is constructed by finding exactly one surface point in each real world pixel, which contains the object surface as visualized in \cref{fig:problem-object-camera-discretization-observation}.

\subsection{Camera}\label{ssec:problem-simplified-camera}
Given an object of bounded size, we assume the \emph{camera} moves on a fixed view circle with radius \(\dcam\) around the real world center similar to the setup of \textcite{banta2000nextbestview}. To avoid collisions between the camera and the object, we require \(\dcam > \dmax\) (\cref{item:simp-camera-position}). The camera pose \(\theta\in\Camspace=[0,2\pi]\) is defined as the polar angle of the camera relative to the real world center, which uniquely defines the camera position as \((\theta,\dcam)\) in polar coordinates. By assuming the camera is always oriented towards the center of the real world, the camera orientation is given as \(\theta + \pi\) (\cref{item:simp-camera-orientation}). Hence, the full camera pose is completely specified by \(\theta\) only. The camera's field of view (FOV) is modeled with a simple 2D cone described by \(\FOV\) for the FOV angle and \(\DOF\) for the camera's depth of field (DOF) as in \cref{fig:problem-object-camera-specification}. Observe how the shape of the FOV in the polar world changes for different choices of \(\DOF\) which we visualize in \cref{fig:problem-different-fovs}.

\begin{figure}
    \centering
    \begin{subfigure}[t]{0.66\linewidth}
        \centering
        \includegraphics[width=0.5\linewidth]{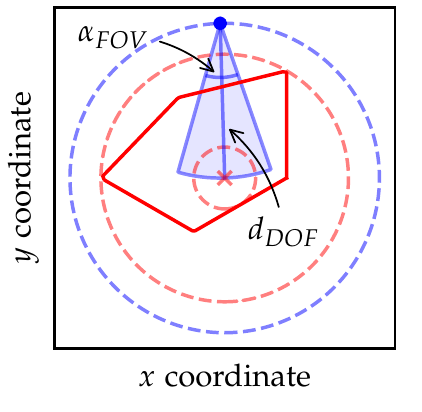}%
        \includegraphics[width=0.5\linewidth]{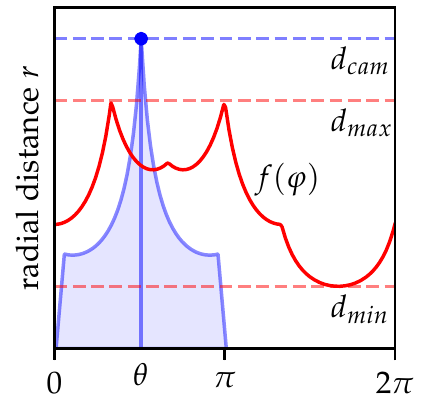}
        \caption{specification of object and camera}
        \label{fig:problem-object-camera-specification}
    \end{subfigure}%
    \begin{subfigure}[t]{0.33\linewidth}
        \centering
        \includegraphics[width=\linewidth]{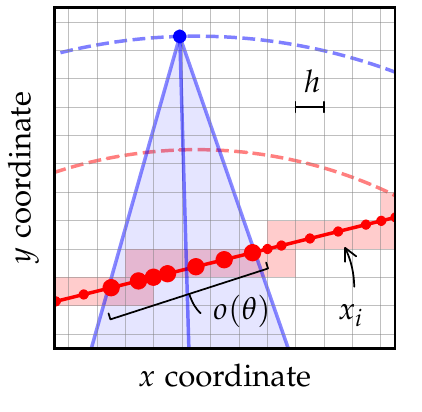}
        \caption{object discretization and camera observation model}
        \label{fig:problem-object-camera-discretization-observation}
    \end{subfigure}
    \caption[Different Aspects of Object and Camera]{
        Different Aspects of Object and Camera.
        (a) visualizes different constants for the object (red) and camera (blue) and shows how the object surface function \(f\) and camera position \(\theta\) are defined. (b) highlights the discretization of the real world into world pixels (grid) and of the object surface into surface points (red dots). In particular, the bold surface points correspond to the observed surface points contained in \(o(\theta)\).
    }
    \label{fig:problem-object-camera}
\end{figure}

\begin{figure}
    \centering
    \begin{subfigure}{0.33\linewidth}
        \centering
        \includegraphics[width=\linewidth]{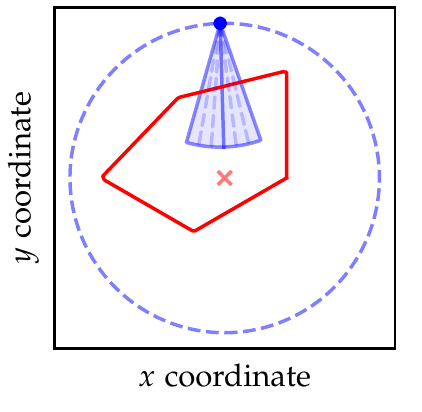}
        \includegraphics[width=\linewidth]{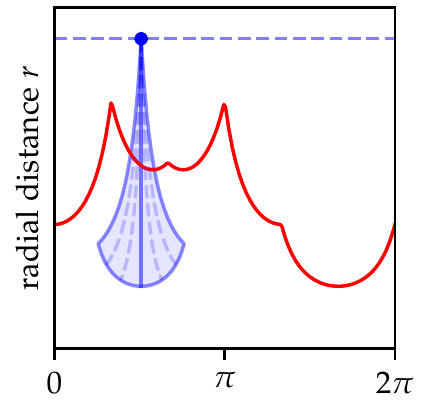}
        \caption{\(\DOF < \dcam\)}
    \end{subfigure}%
    \begin{subfigure}{0.33\linewidth}
        \centering
        \includegraphics[width=\linewidth]{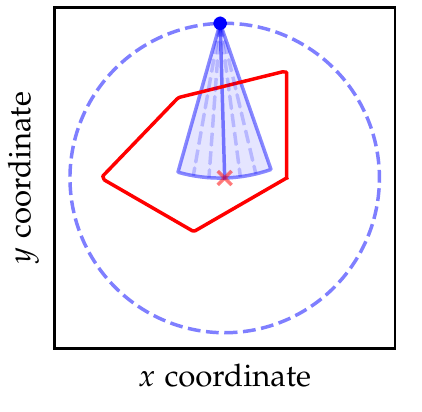}
        \includegraphics[width=\linewidth]{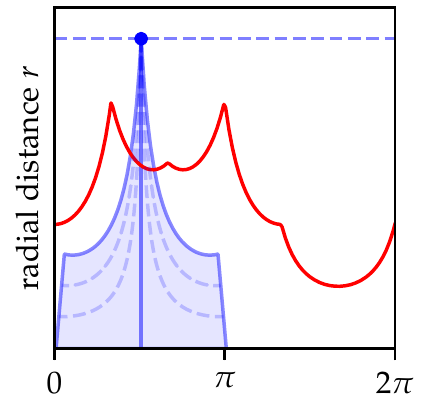}
        \caption{\(\DOF = \dcam\)}
    \end{subfigure}%
    \begin{subfigure}{0.33\linewidth}
        \centering
        \includegraphics[width=\linewidth]{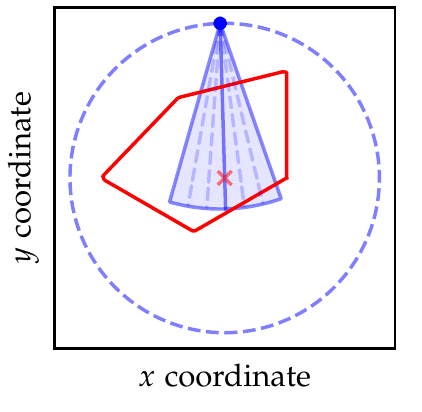}
        \includegraphics[width=\linewidth]{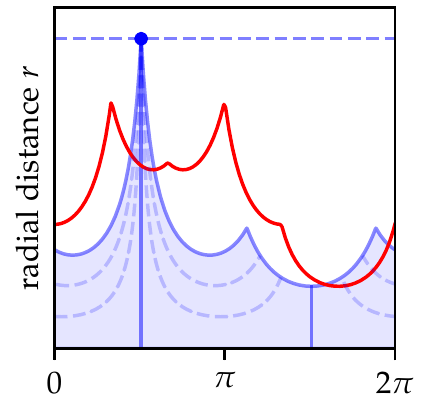}
        \caption{\(\DOF > \dcam\)}
    \end{subfigure}
    \caption[Shape of the FOV for different DOFs]{
        Shape of the FOV for different DOFs.
        The complexity of the FOV shape in the polar world mostly depends on whether it contains the world center or not. This can be illustrated by a case distinction on the relation between \(\DOF\) and \(\dcam\). In addition, observe how the straight rays (dotted lines) casted from the camera in the real world are transformed into bent rays in the polar world, which periodically wrap around every \(2\pi\).
    }
    \label{fig:problem-different-fovs}
\end{figure}

After formalizing the object and camera, we continue defining the \emph{observation function} \(o\colon \Camspace \to 2^\Surface, \theta \mapsto o(\theta)\), which describes the process of observing a subset of surface points \(o(\theta) \subseteq \Surface\) from a camera location \(\theta\) as visualized in \cref{fig:problem-object-camera-discretization-observation}. The observation model is defined as
\begin{equation}\label{eq:observation-function}
    \begin{aligned}
        o(\theta) &\defeq \set*{x_i \in \Surface \mid \text{\(x_i\) within FOV and not occluded\footnotemark}}\\
        o(\Theta) &\defeq \bigcup_{\theta\in\Theta} o(\theta)
        \quad\with \Theta\subseteq\Camspace
        .
    \end{aligned}
\end{equation}
\footnotetext{More precisely, the angle of observing \(x_i\) relative to the camera must lie in \(\brack*{-\frac{\FOV}{2},\frac{\FOV}{2}}\), the distance to the camera must not exceed \(\DOF\), and it must be possible to cast a ray from the camera to \(x_i\) without passing through the world pixels of other surface points.}%
We use the observation function as a black-box abstraction to avoid dealing with the technical details of the camera's FOV and mutual occlusion of surface points.

While the observation function only provides the set of observed surface points, more precisely the set of polar angles defining the surface points, the \emph{measurement function} \(\fm\colon \Domain \to \Real, \varphi \mapsto \fm(\varphi)\) describes the process of measuring the distances between the observed surface points and the real world center, which corresponds to measuring the true surface function \(f\).%
\footnote{Note that the output of the observation function provides an input to the measurement function, since \(\Surface \subseteq \Domain\).}
In practice, we would measure the distance between the observed surface points and the camera instead and compute \(\fm\) from the distance of the camera to the real world center \(\dcam\). To capture the inherent noise in these measurements, we define the measurement model as
\begin{equation}\label{eq:measurement-function}
    \begin{aligned}
        \fm(\varphi) &\defeq f(\varphi) + \varepsilon_\varphi
        &&\with \varepsilon_\varphi \sim \Normal*{0, \sigmaeps^2} \text{ \iid{}}
        \\
        \fm(\Phi) &\defeq \brack*{\fm(\varphi)}_{\varphi\in\Phi}
        &&\with \Phi\subseteq\Domain
    \end{aligned}
\end{equation}
Using this, we not only make the \iid{} assumption on the noise (\cref{item:simp-camera-noise}), but also assume that the noise has the same standard deviation \(\sigmaeps\) for all observed surface points. In practice, measuring surface points farther away typically incur more noise than measuring closer surface points due to the finite resolution of the camera. However, this measurement model implicitly sets the camera resolution to the width \(h\) of a real world pixel and observes all surface points with the same accuracy independently of the distance to the surface point (\cref{item:simp-camera-resolution}). We use this measurement function as an abstraction for the technical details of the camera resolution.

We define the set of surface points observed from \(\theta_{1:t}\) as
\begin{equation}\label{eq:observed-surface}
    X_{1:t} \defeq o(\theta_{1:t})
    \quad\with
    n_{1:t} \defeq \abs{X_{1:t}}
\end{equation}
and the measurements made for these surface points as
\begin{equation}\label{eq:measured-surface}
    \begin{aligned}
        Y_{1:t} &\multicolumn{3}{l}{\(\;\defeq\fm(X_{1:t}) = f_{1:t} + \varepsilon_{1:t}\)}
        \\
        f_{1:t} &\defeq \brack*{f(x)}_{x \in X_{1:t}}
        &&\with f \sim \GP*{m(\varphi), k(\varphi,\varphi')}
        \\
        \varepsilon_{1:t} &\defeq \brack*{\varepsilon_x}_{x \in X_{1:t}}
        &&\with \varepsilon_x \sim \Normal*{0, \sigmaeps^2}
        .
    \end{aligned}
\end{equation}

\begin{remark}
    In practice, distance measurements to the object surface can be obtained from depth images, in which each pixel is associated with a measured distance to the object surface. Such images can be taken with depth cameras, also known as Time-of-Flight (ToF) cameras, or extracted from monocular or stereo RGB images based on depth estimation methods.
    
    For a given discretized world representation, one can compute the set of observed 2D pixels (or 3D voxels) from such a depth image based on the current camera pose. We assume that this set is given as the input to our algorithm.
\end{remark}

The overall model of observing surface points (polar angles) and measuring the surface function (radial distances) can be summarized into
\begin{equation*}
    \theta
    \;\xrightarrow{\;\substack{\text{observing}\\o(\theta)}\;}\;
    X
    \;\xrightarrow{\;\substack{\text{measuring}\\\fm(X)}\;}\;
    Y
    .
\end{equation*}
Note that the observation process is deterministic and depends on the object surface, the camera's FOV and the occlusion of surface points, whereas the measurement process is stochastic and subject to measurement noise.

\subsection{Gaussian Process Model}\label{ssec:problem-simplified-gp}
By assuming that the surface function of the target object is sampled from a \emph{Gaussian process} distribution
\begin{equation*}
    f \sim \GP*{m(\varphi), k(\varphi,\varphi')}
\end{equation*}
with mean and covariance function
\begin{equation}\label{eq:mean-covariance-function-surface}
    \begin{aligned}
        m(\varphi) &\defeq \E*{f(\varphi)} \\
        k(\varphi,\varphi') &\defeq \Cov*{f(\varphi), f(\varphi')}
    \end{aligned}
\end{equation}
as defined in \cref{eq:mean-covariance-function}, we can model the algorithm's uncertainty about the true object shape. Recall from \cref{ssec:background-gp-gp} that the marginal distribution of finitely many points \(f(\Phi)\) on such a sampled function follows a multivariate Gaussian distribution. In a Bayesian setting, this corresponds to the \emph{prior distribution} of points on the surface
\begin{equation}\label{eq:prior-surface}
    f(\Phi) \sim \Normal*{\mu_0(\Phi), \Sigma_0(\Phi)}
\end{equation}
with mean vector and covariance matrix
\begin{equation*}
    \begin{alignedat}{3}
        \mu_0(\Phi) &\defeq \mu(\Phi)
        && \quad\with
        & \mu(\Phi) &= \brack*{m(\varphi)}_{\varphi\in\Phi}\\
        \Sigma_0(\Phi) &\defeq \Sigma(\Phi)
        && \quad\with
        & \Sigma(\Phi,\Phi') &= \brack*{k(\varphi,\varphi')}_{\varphi\in\Phi,\varphi'\in\Phi'}
        .
    \end{alignedat}
\end{equation*}
as defined in \cref{eq:mean-vector-covariance-matrix} and visualized in \cref{fig:problem-perspective-alg-t1}.%
\footnote{Recall that we use the notation \(\Sigma(\Phi) \defeq \Sigma(\Phi,\Phi)\).}
After measuring \(f\) at \(X_{1:t}\), we obtain the \emph{posterior distribution} of points on the surface
\begin{equation}\label{eq:posterior-surface}
    f(\Phi) \mid Y_{1:t} \sim \Normal*{\mu_t(\Phi), \Sigma_t(\Phi)}
\end{equation}
with mean vector and covariance matrix
\begin{equation*}
    \begin{aligned}
        \mu_t(\Phi) &= \mu(\Phi) + \Sigma(\Phi,X_{1:t})(\Sigma(X_{1:t}) + \sigmaeps^2 I)^{-1} (Y_{1:t} - \mu(X_{1:t}))
        \\
        \Sigma_t(\Phi) &= \Sigma(\Phi) - \Sigma(\Phi,X_{1:t})(\Sigma(X_{1:t}) + \sigmaeps^2 I)^{-1} \Sigma(X_{1:t},\Phi)
    \end{aligned}
\end{equation*}
obtained from \cref{eq:posterior-noisy-regression} and visualized in \cref{fig:problem-perspective-alg-t2,fig:problem-perspective-alg-t3}. Since they are given in closed-form, the algorithm is able to compute them. Intuitively, computing the posterior distribution using Bayesian inference is a form of learning the unknown object surface from a given set of observed surface points. The prior and posterior variance
\begin{equation}\label{eq:prior-posterior-variance-surface}
    \begin{aligned}
        \sigma_0(\Phi)^2 &\defeq \diag(\Sigma_0(\Phi)) = \brack*{k(\varphi,\varphi)}_{\varphi\in\Phi} \\
        \sigma_t(\Phi)^2 &\defeq \diag(\Sigma_t(\Phi))
    \end{aligned}
\end{equation}
can be found on the diagonal of the covariance matrix.

With the Gaussian process model the goal is to represent the uncertainty about the true surface function with \emph{confidence bounds} of the form
\begin{equation}\label{eq:confidence-region-ideal}
    l_t(\varphi) \leq f(\varphi) \leq u_t(\varphi)
    \quad\text{ for all \(\varphi \in \Domain\), \(t \geq 1\) \whp{}}
\end{equation}
which should ideally envelop the true surface function with high probability (\whp{}). This information can then be used by \(\A\) in its objective function to infer an estimate \(\theta_t\) for the NBV.
The upper and lower confidence bounds based on \(Y_{1:t-1}\) and used by \(\A\) in round \(t\) are symmetrically defined as
\begin{equation}\label{eq:confidence-bounds-ideal}
    \begin{aligned}
        u_t(\varphi) &\defeq \mu_{t-1}(\varphi) + \beta_t^{\ffrac{1}{2}}\sigma_{t-1}(\varphi)\\
        l_t(\varphi) &\defeq \mu_{t-1}(\varphi) - \beta_t^{\ffrac{1}{2}}\sigma_{t-1}(\varphi)\\
    \end{aligned}
    .
\end{equation}
with \emph{confidence parameter} \(\beta_t\), which is responsible for scaling the width of the confidence region. In \cref{lem:lemma-2-4} we show how to appropriately choose \(\beta_t\) in each round to provide a guarantee similar to \cref{eq:confidence-region-ideal} with arbitrarily high probability \(1-\delta\).

\begin{figure}
    \centering
    \begin{subfigure}{0.33\linewidth}
        \centering
        \includegraphics[width=\linewidth]{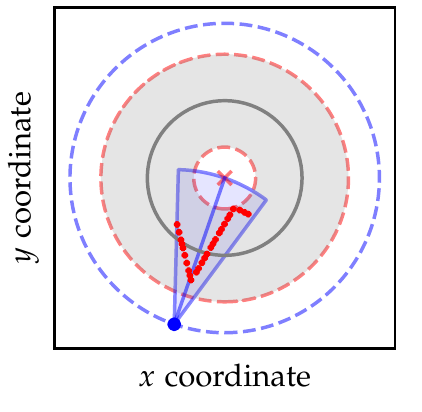}
        \includegraphics[width=\linewidth]{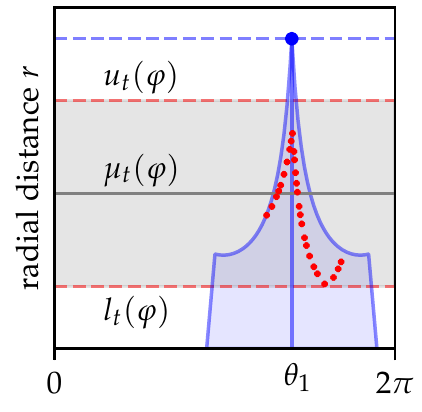}
        \caption{\(t = 1\)}
        \label{fig:problem-perspective-alg-t1}
    \end{subfigure}%
    \begin{subfigure}{0.33\linewidth}
        \centering
        \includegraphics[width=\linewidth]{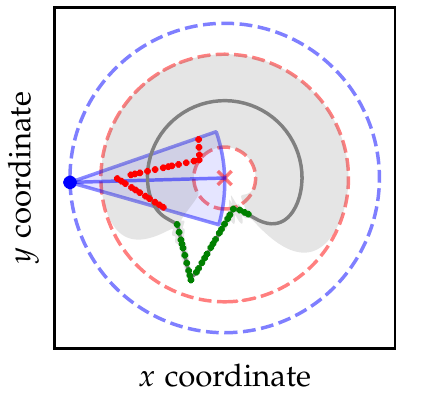}
        \includegraphics[width=\linewidth]{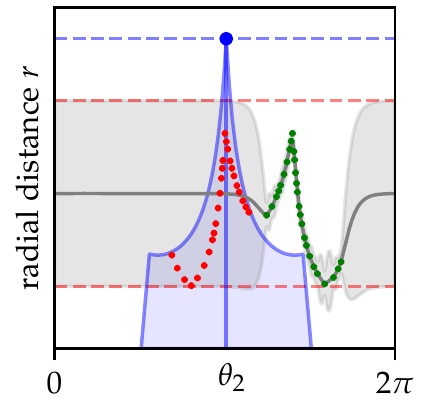}
        \caption{\(t = 2\)}
        \label{fig:problem-perspective-alg-t2}
    \end{subfigure}%
    \begin{subfigure}{0.33\linewidth}
        \centering
        \includegraphics[width=\linewidth]{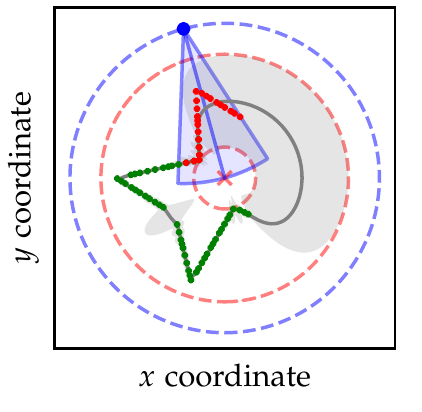}
        \includegraphics[width=\linewidth]{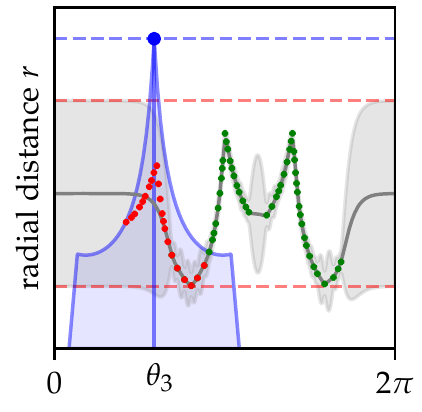}
        \caption{\(t = 3\)}
        \label{fig:problem-perspective-alg-t3}
    \end{subfigure}
    \caption[Perspective of the Algorithm]{
        Perspective of the Algorithm.
        In the beginning, algorithm \(\A\) has no information about the object except for the confidence region (gray) representing our prior knowledge and the surface points observed from the current location (red dots). After each measurement at \(\theta_1\) and \(\theta_2\), \(\A\) obtains an updated posterior distribution through Bayesian inference, where the uncertainty is reduced at the measured locations.
    }
    \label{fig:problem-perspective-alg}
\end{figure}

What remains is to choose a suitable mean function \(m(\varphi)\) and covariance function \(k(\varphi,\varphi')\) for the Gaussian process model used by \(\A\). This is discussed in \cref{sec:design-gp}.

\subsection{True Objective Function}\label{ssec:problem-simplified-objective}

Now we can define the notion of reconstruction progress for a set of camera poses as quantified by the utility function in \cref{ssec:problem-general-objective} more precisely.

The utility function \(F(\theta_{1:t})\) evaluates the overall reconstruction progress for \(\theta_{1:t}\), which we measure in terms of the total number of observed surface points as visualized in \cref{fig:problem-true-objective-full} and written as
\begin{equation}\label{eq:true-utility}
    F(\theta_{1:t}) \defeq \abs{o(\theta_{1:t})}
    .
\end{equation}
This is the true objective with respect to the reconstruction problem.

Since \(\A\) makes its decisions sequentially, we defined the marginal utility function \(F(\theta \mid \theta_{1:t-1})\) in \cref{eq:marginal-utility} to evaluate the reconstruction progress made by a single decision \(\theta\). This corresponds to the number of newly observed surface pixels as visualized in \cref{fig:problem-true-objective-marginal} and formally defined as
\begin{equation}\label{eq:true-marginal-utility}
    F(\theta \mid \theta_{1:t-1}) = \abs{o(\theta) \setminus o(\theta_{1:t-1})}
    .
\end{equation}
This corresponds to the true objective with respect to the NBV problem.
Because of the sequential decision making process, the marginal utility naturally becomes the \emph{true objective function} for \(\A\) in each round.

\begin{figure}
    \centering
    \begin{subfigure}{0.45\linewidth}
        \centering
        \includegraphics{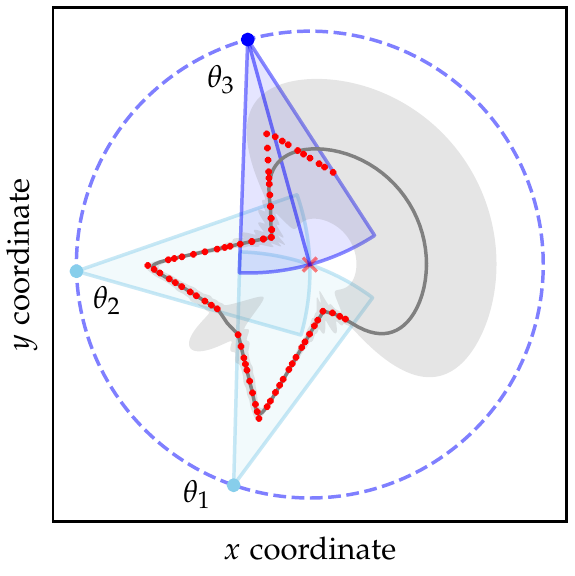}
        \caption{\(F(\theta_{1:t})\)}
        \label{fig:problem-true-objective-full}
    \end{subfigure}%
    \quad%
    \begin{subfigure}{0.45\linewidth}
        \centering
        \includegraphics{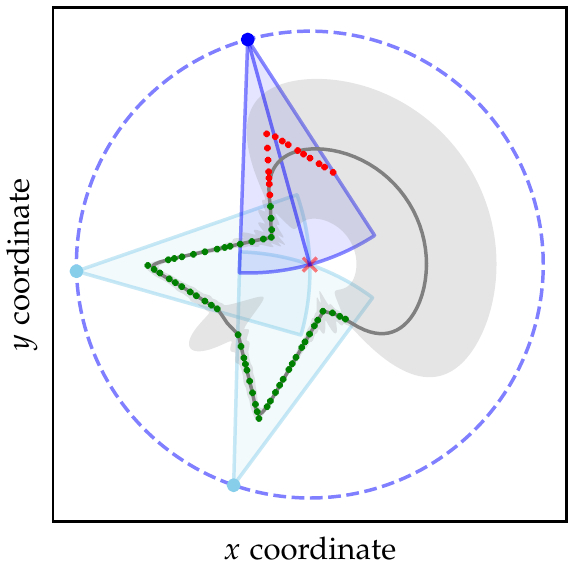}
        \caption{\(F(\theta \mid \theta_{1:t})\)}
        \label{fig:problem-true-objective-marginal}
    \end{subfigure}
    \caption[True Objective Functions]{
        True Objective Functions.
        (a) visualizes the utility function, which measures the total number of observed surface points (red) from all chosen camera locations \(\theta_{1:t}\). This is the true objective to be maximized for solving the reconstruction problem optimally. (b) visualizes the marginal utility, which only counts the newly observed surface points (red) from the current camera location \(\theta\). Previously observed surface points (green) do not contribute anymore to the marginal utility. This is the true objective to be maximized for solving the NBV decision problem greedily.
    }
    \label{fig:problem-simplified-objective}
\end{figure}

However, in practice the shape of the target object is unknown and \(o(\theta)\) can only be evaluated if the camera makes a measurement from \(\theta\). Without global knowledge about the true objective function \(F(\theta\mid\theta_{1:t-1})\), \(\A\) must somehow infer information about the object shape based on previous measurements of the object surface to make informed decisions. This is where the Gaussian process model described in \cref{ssec:problem-simplified-gp} comes into play, which captures the information from previous measurements and provides an upper and lower confidence bound on the surface function with high probability as defined in \cref{eq:confidence-bounds-ideal}. Later in \cref{sec:design-objective}, we use these confidence bounds to design objective functions \(\Fu(\theta\mid\theta_{1:t-1})\) which estimate the true objective function well and can be globally evaluated by \(\A\) to find the NBV estimate.%
\footnote{The subscript \(u\) refers to the fact that \(\Fu(\theta\mid\theta_{1:t-1})\) must be an upper bound of the marginal utility as defined in \cref{eq:marginal-utility-upper-bound} and later refined in \cref{item:req-necessary-bound}.}

\subsection{List of Simplifications}\label{ssec:problem-simplified-simplifications}

In this section we summarize all previously made simplifications and assumptions and briefly discuss their implications on our results. We further provide ideas on how to relax or completely lift them. We distinguish between \textit{strong} and \textit{weak} assumptions depending on how strong it restricts the applicability of our results.

\begin{enumerate}[label=(S\arabic*), ref=Simp.~\arabic*]
    \item \label{item:simp-2D-world} We assume to live in a 2D world.
    \begin{discussion}{strong}
        This simplification is the main limiting factor of our work and prevents the results from being directly applied in the 3D world.
        
        This simplification can be lifted by extending the used methods to 3D. This includes the object and its corresponding Gaussian process model and likewise the camera pose and FOV. In \cref{remark:different-domains,remark:different-camspaces}, we briefly presented ways to parameterize 3D object surfaces and to define the camera pose in 3D.
        
        An interesting approach, which saves us from searching for a suitable surface parameterization, is to define the Gaussian process model on the camera space \(\Camspace\) (\eg{} camera sphere in 3D) instead of the surface function domain \(\Domain\).%
        \footnote{Credits for this idea go to Viacheslav Borovitskiy.}
        The idea is to project the distance measured at the current location \(\theta\) to the camera sphere, where the Gaussian process model is updated. The advantage is we do not restrict us to certain classes of objects. However, it completely relies on \cref{item:simp-camera-position}.
    \end{discussion}
    
    \item \label{item:simp-object-complexity} We assume the object surface can be modeled with a polar function \(f\).
    \begin{discussion}{strong}
        This simplification limits the complexity of target objects to ``well-shaped'' objects in the sense that every \(\varphi\in\Domain\) must uniquely identify a point on the surface. For example, it is not possible to describe an object with the shape of a horseshoe using a polar function.
        
        This simplification can be lifted by using a periodic parametric function \(f\colon\Domain \to \Real^2\) with \(\Domain = [0,1]\) as described in \cref{remark:different-domains}. However, the difficulty is to find an appropriate Gaussian process model for such functions.
    \end{discussion}
    
    \item \label{item:simp-object-centered} We assume the object is roughly centered in the real world.
    \begin{discussion}{weak}
        This simplification requires \(\A\) to know the location of the target object, such that it can model it with a polar surface function. This is typically the case in practice, since the object reconstruction problem is concerned about reconstructing the object and not about localizing the object.
    \end{discussion}
    
    \item \label{item:simp-object-bounded} We assume the object surface is bounded between \(\dmin\) and \(\dmax\) relative to the center of the real world.
    \begin{discussion}{weak/strong}
        This simplification limits the size of the object and prevents our methods from being applied to too large or too small objects. In practice, having an upper bound for the object size is the limiting factor, since the usual goal is to efficiently reconstruct arbitrarily large objects.
        
        This simplification can be relaxed by changing \(\dmin\) or \(\dmax\) to allow objects of smaller or larger size. The difficulty is \cref{item:simp-camera-position}, which requires the camera to move on a circle with fixed radius outside of \(\dmax\). By setting \(\dmax\) too large, the camera might not be able to fully observe the object anymore due to its limited DOF and resolution. However, if we can weaken \cref{item:simp-camera-position} sufficiently enough, we bounded size simplification will not be a significant restriction anymore.
    \end{discussion}
    
    \item \label{item:simp-object-sampled} We assume the object surface is sampled from a Gaussian process.
    \begin{discussion}{weak}
        This simplification limits us to objects which can be sampled from the Gaussian process. It typically comes along with restrictions on the smoothness of the object shape depending on the used covariance function. 
        
        This simplification can be relaxed by using an appropriate covariance function, such that almost any realistic object can be sampled from the Gaussian process. The difficulty is to preserve the theoretical guarantees.
    \end{discussion}
    
    \item \label{item:simp-camera-position} We assume the camera moves on a circle with fixed radius \(\dcam > \dmax\) around the center of the real world.
    \begin{discussion}{strong}
        This simplification limits the motion range of the camera, which plays an important role in the reconstruction problem. Especially varying the distance between camera and object represents a crucial tradeoff between observed surface area and measurement accuracy in practice. Being farther away from the object allows the camera to observe more of the object surface at the cost of larger noise. In addition, a finite DOF limits the measurement range of the camera and it might not be able to observe the complete object from a fixed circular orbit as remarked in \cref{item:simp-object-bounded}.
        
        This simplification can be relaxed by specifying the camera pose with more parameters to increase its motion range. Different suggestions for the parameter space of the camera are given in \cref{remark:different-camspaces}. However, this increases the difficulty for our algorithm, since it enlarges the space of decisions in which the algorithm must search for the NBV estimate.
    \end{discussion}
    
    \item \label{item:simp-camera-orientation} We assume the camera is always directed towards the center of the real world.
    \begin{discussion}{strong}
        This simplification limits the camera to a single view direction for each camera position. In practice, this fixed direction might be non-optimal for a given camera position.
        
        This simplification can be relaxed by similarly increasing the parameter space of the camera as for \cref{item:simp-camera-position}.
    \end{discussion}
    
    \item \label{item:simp-camera-noise} We assume the camera measurements of the object surface are subject to \iid{} Gaussian noise.
    \begin{discussion}{weak}
        This simplification is a classical assumption in statistics. Typically, the \iid{} assumption has only negligible influence on the overall result and the distributional assumption is justified by the Central Limit Theorem, which states that the average noise of a large sample size is close to a normal distribution.
    \end{discussion}
    
    \item \label{item:simp-camera-resolution} We assume the camera measures all surface points with the same accuracy.
    \begin{discussion}{strong}
        This simplification implicitly assumes that the camera resolution matches the granularity \(h\) of the real world discretization, since it observes all surface points no matter how far they are. From the perspective of an image taken by such a camera, this means that the image resolution is dynamically adapted to how far the surface points are. The farther a part of the object surface is located from the camera, the finer the image resolution will become locally where this part of the surface is depicted, and vice versa. This model is not completely realistic.
        
        This simplification can be relaxed by dynamically adapting the noise standard deviation \(\sigmaeps\) based on the true distance between the camera and object. A simple linear relationship can be \(\sigmaeps(\varphi) \propto \dcam - f(\varphi)\) for \(\varepsilon_\varphi \sim \Normal*{0,\sigmaeps^2(\varphi)}\). This ensures that surface points further away are measured with higher uncertainty, which implicitly reflects the reality that the taken image provides less information about these surface points.
    \end{discussion}
\end{enumerate}


\section{Comparison to other Settings}\label{sec:problem-comparison}

Many previous related work applied Gaussian processes in a sequential decision making framework to maximize some unknown reward function, which is expensive to evaluate. Among three of them, we want to highlight similarities and differences with our setting to provide some additional insights in certain problem formulations and subsequent design choices.

We start comparing our setting with the most general one in the foundational work of \textcite{srinivas2012informationtheoretic} in \cref{ssec:problem-comparison-gpucb} and then continue with \textcite{prajapat2022nearoptimal} in \cref{ssec:problem-comparison-macopt} and \textcite{chen2017interactive} in \cref{ssec:problem-comparison-smucb} ordered by increasing similarity to our setting. For clarity and comparability, we present their work in a simplified and stripped-down version and try to slightly unify their notations with ours.

Before we start, we first briefly summarize our own setting to facilitate the comparison.

\subsubsection{Problem}
The problem of object reconstruction is to sequentially construct a set of camera poses \(\theta_{1:T}\) from which the unknown surface function \(f\) of an object can be maximally observed with \(T\) measurements.
We assume that the surface function is sampled from a Gaussian process
\begin{equation*}
    f\colon\Domain\to\Real
    \quad\with f\sim\GP*{m(\varphi),k(\varphi,\varphi')}
\end{equation*}
as described in \cref{ssec:problem-simplified-object} and the utility or reward function is defined as the number of surface points on \(f\) observed from \(\theta_{1:T}\)
\begin{equation*}
    F(\theta_{1:T}) = \abs*{o(\theta_{1:T};f)}
\end{equation*}
as described in \cref{ssec:problem-simplified-objective}. We use the different notation \(o(\cdot;f)\) to highlight the natural dependence of the observation function on \(f\).
An optimal solution for \(T\) measurements is then defined as
\begin{equation*}
    \thetaopt_T = \argmax_{\Theta\subseteq\Camspace,\abs{\Theta}\leq T} F(\Theta)
\end{equation*}
as described in \cref{ssec:problem-general-near-optimality}.

\subsubsection{Decisions \& Regret}
An algorithm \(\A\) tries to find such a solution by sequentially deciding, from which camera location
\begin{equation*}
    \theta_t = \A(\theta_{1:t-1})
\end{equation*}
the object surface should be observed next.
The set of all locations \(\theta_{1:T}\) after round \(T\) is then returned as a solution to the object reconstruction problem.
The simple and cumulative regret are defined as
\begin{equation*}
    \begin{aligned}
        r(t) &= \paren*{1-\frac{1}{e}} (F(\thetaopt_t) - F(\thetaopt_{t-1})) - F(\theta_t \mid \theta_{1:t-1})\\
        R(T) &= \sum_{t=1}^T r(t) = \paren*{1-\frac{1}{e}} F(\thetaopt_T) - F(\theta_{1:T})
    \end{aligned}
\end{equation*}
as described in \cref{ssec:problem-general-regret}.

\subsection{Gaussian Process Optimization \autocite{srinivas2012informationtheoretic}}\label{ssec:problem-comparison-gpucb}

The setting of Gaussian process optimization is the most general setting and forms the foundation of many related work based on upper-confidence bound (UCB) algorithms.

\subsubsection{Problem}
The problem of Gaussian process optimization is to find a location \(x\) which maximizes an unknown reward function \(F\).
This reward function is assumed to be sampled from a Gaussian process
\begin{equation*}
    F\colon\Domain\to\Real
    \quad\with F\sim\GP*{m(x),k(x,x')}
\end{equation*}
with optimal solution defined as
\begin{equation*}
    x^\star = \argmax_{x\in \Domain} F(x)
    .
\end{equation*}

\subsubsection{Decisions \& Regret}
The proposed algorithm \GPUCB{} tries to find such a solution by sequentially deciding, at which location \(x_t\) the reward function \(F(x)\) should be evaluated next. The strategy is to use the Gaussian process model to obtain an upper confidence bound \(F_u(x)\) on the unknown reward function \(F(x)\). The next location is then chosen as the maximizer of this upper bound which can be written as
\begin{equation*}
    \begin{aligned}
        x_t &= \argmax_{x\in \Domain} F_u(x)\\
        &= \argmax_{x\in \Domain} \mu_{t-1}(x) + \beta_t^{1/2}\sigma_{t-1}(x)
        .
    \end{aligned}
\end{equation*}
The final location \(x_T\) at round \(T\) is then returned as a solution to the Gaussian process optimization problem. The simple and cumulative regret are defined as
\begin{equation*}
    \begin{aligned}
        r(t) &= F(x^\star) - F(x_t)\\
        R(T) &= \sum_{t=1}^T r(t)
        .
    \end{aligned}
\end{equation*}

\subsubsection{Comparison}
Different from our setting is that the sample function from the Gaussian process directly corresponds to the reward function. This means that maximizing the reward corresponds to ``exploiting'' the unknown sample function by maximizing \(\mu_{t-1}(x)\), which finds the location with largest expected reward. As noted by \textcite{srinivas2012informationtheoretic} that only maximizing \(\mu_{t-1}(x)\) is too greedy and results into bad performance, it requires a tradeoff with ``exploring'' the unknown sample function. Exploration corresponds to maximizing \(\sigma_{t-1}(x)\) to find the location with largest expected uncertainty. This exploration-exploitation tradeoff can be seen in the decision rule of \GPUCB{}, which maximizes a linear combination of \(\mu_{t-1}(x)\) and \(\sigma_{t-1}(x)\).

However, in our setting the sample function enters the reward function through the black-box observation function \(o(\cdot;f)\) and a large surface function value does not directly correlate with a large reward. In particular, we are not concerned about finding the maximum of the surface function \(f\), but about the observation coverage of \(f\), which corresponds to pure exploration. We show later in \cref{sec:design-objective} that our designed objective functions, in fact, mostly depend on the uncertainty \(\sigma_{t-1}(x)\). Besides this conceptual difference, it is also more difficult for us to make use of the confidence bounds for \(f\), which is encapsulated by the black-box observation function.

Another difference is that a solution to the Gaussian process optimization problem consists of a single location \(x_T\), which coincides with the decisions made in a single round. This means, that an optimal solution to this problem corresponds to an optimal decision in a single round, which is why their simple regret is defined with respect to the time-independent \(x^\star\).
In contrast, the solution \(\theta_{1:T}\) to the reconstruction problem contains the decisions made in all rounds. Hence, we defined our cumulative regret of all decisions instead of the simple regret with respect to \(\thetaopt_T\). This is the reason, why we only analyze the NBV decision problem instead of the complete object reconstruction problem. In addition, due to the dependence of the optimal solution \(\thetaopt_T\) on the total number of rounds \(T\), we can only show pseudo-convergence to near-optimality as formalized in \cref{thm:near-optimality}, since it is unknown to us, how fast \(\thetaopt_T\) improves with increasing \(T\).

We also want to highlight the difference that their solution consists of a single decision \(x_T\), while our solution consists of multiple decisions \(\theta_{1:T}\). Constructing a solution with multiple decisions is typically done greedily to avoid the combinatorial search. Hence, our regret for the set of decisions is defined with respect to the greedy \(1-\frac{1}{e}\) approximation guarantee of an optimal solution, which is not needed for the regret of \textcite{srinivas2012informationtheoretic}. This is the reason, why we only show near-optimality.

\subsection{Multi-Agent Coverage Control \autocite{prajapat2022nearoptimal}}\label{ssec:problem-comparison-macopt}

The setting of multi-agent coverage control is tied to a more specific application which exhibits more similarities with our setting.
Note that \textcite{prajapat2022nearoptimal} additionally present safe multi-agent coverage control, but for comparability with our setting we only focus on coverage control without safety constraints.

\subsubsection{Problem}
The problem of multi-agent coverage control is to find a set of positions \(x^{1:N}\) for \(N\) agents which maximizes the coverage of some unknown density function \(f\). 
which is assumed to be sampled from a Gaussian process
\begin{equation*}
    f\colon \Domain\to\Real
    \quad\with f\sim\GP*{m(x),k(x,x')}
    .
\end{equation*}
The reward function is defined as the sum of densities \(f(x)\) covered by all \(N\) agents
\begin{equation*}
    F(x^{1:N}) = \sum_{x\in D(x^{1:N})} f(x)
\end{equation*}
with \(D(x^{1:N})\) defined as the union of sensing regions of the agents located at \(x^{1:N}\).
An optimal solution is defined as
\begin{equation*}
    X^\star = \argmax_{X\subseteq\Domain, \abs{X}\leq N} F(X)
    .
\end{equation*}

\subsubsection{Decisions \& Regret}
The proposed algorithm \MACOPT{} tries to find such a solution by sequentially deciding, at which locations \(x^{1:N}_t\) the agents should jointly measure the unknown density \(f\) next. The strategy is to use the Gaussian process model to obtain an upper confidence bound \(f_u(x)\) on the unknown density function \(f(x)\), which directly translates into an upper confidence bound \(F_u(X)\) on the reward. 
The next locations are then chosen as the maximizer of this upper bound which can be written as
\begin{equation*}
    \begin{aligned}
        x^{1:N}_t &= \arggre_{x^{1:N}\gets x^1,\dots,x^N} F_u(x^{1:N}) \\
        &= \arggre_{x^{1:N}\gets x^1,\dots,x^N} \sum_{x\in D(x^{1:N})} \paren*{\mu_{t-1}(x) + \beta_t^{1/2}\sigma_{t-1}(x)}
        .
    \end{aligned}
\end{equation*}
This is done greedily for the \(N\) agents in each round,%
\footnote{We use \(\arggre_{X\gets x^1,\dots,x^n} f(X)\) only as an abstract notation to refer to greedy maximization. The intuition is that \(X\) is constructed by sequentially choosing \(x^i\), such that \(f(X)\) is maximally increased by each \(x^i\). For the precise formulation, we refer to \textcite[Section 4.1]{prajapat2022nearoptimal}.}
since maximizing the reward over all possible sets of agent locations is a combinatorial problem.
The final set of locations \(x^{1:N}_T\) at round \(T\) is then returned as a solution to the multi-agent coverage control problem.
The simple and cumulative regret are defined as
\begin{equation*}
    \begin{aligned}
        r(t) &= \paren*{1-\frac{1}{e}}F(X^\star) - F(x^{1:N}_t)\\
        R(T) &= \sum_{t=1}^T r(t)
        .
    \end{aligned}
\end{equation*}

\subsubsection{Comparison}
In this setting, the sampled density function from the Gaussian process corresponds to the reward function for covering each individual point \(x\in\Domain\). The difference to the setting of \textcite{srinivas2012informationtheoretic} in \cref{ssec:problem-comparison-gpucb} is that the rewards of all points \(D(x^{1:N})\) inside the agents' sensing regions are collected instead of the reward only at the current points \(x^{1:N}\). Note the similarity to our setting, where the goal is to maximize the observation coverage of the object surface and for which the reward function can be rewritten as 
\begin{equation*}
    F(\theta_{1:T}) = \abs{o(\theta_{1:T};f)} = \sum_{\varphi\in o(\theta_{1:T};f)} 1
    .
\end{equation*}
Basically, each surface point on the object has reward one and the rewards of all visible points \(o(\theta_{1:T};f)\) inside the FOV are collected. The difference to our setting is how the sample function \(f\) from the Gaussian process enters the reward function. In the setting of \textcite{prajapat2022nearoptimal}, a large density directly relates to a large reward and they can apply the upper confidence bound in their decision rule for balancing exploration and exploitation of the density function as for \GPUCB{}. In \cref{ssec:problem-comparison-gpucb}, we discussed that our setting is only concerned about exploration.

Regarding the solution to the multi-agent coverage control, note that the locations \(x^{1:N}_T\) of the \(N\) agents coincide with the decisions made in a single round as in \cref{ssec:problem-comparison-gpucb}, which is different from our setting. Hence, their simple regret is defined with respect to a time-independent \(X^*\), whereas in our setting the cumulative regret is defined with respect to \(\thetaopt_T\).

Similar to our setting is that the solution consists of multiple greedily made decisions \(x^1_T,\dots,x^N_T\). Hence, the regret is similarly defined with respect to the greedy \(1-\frac{1}{e}\) approximation guarantee of an optimal solution. This is why our cumulative regret over all rounds looks identical to their simple regret, which is the regret accumulated over all agents, but in a single round.

\subsection{Interactive Bandit Optimization \autocite{chen2017interactive}}\label{ssec:problem-comparison-smucb}

The setting of interactive bandit optimization comes closest to our setting.
Note that \textcite{chen2017interactive} focus on interactive \emph{contextual} bandits with \(m\) distinct reward functions \(F_{\phi_1},\dots,F_{\phi_m}\) each with additional context information \(\phi_i\). For comparability with our setting, we ignore the context and assume that we only encounter \(m=1\) distinct function.

\subsubsection{Problem}
The problem of interactive bandit optimization is to sequentially construct a set of items \(x_{1:T}\) over \(T\) rounds which maximizes an unknown reward function \(F(x_{1:T})\) by ``interacting'' with the marginal reward function \(F(x_t\mid x_{1:t-1})\) in each round, which is assumed to be sampled from a Gaussian process
\begin{equation*}
    F(\cdot\mid\cdot)\colon\Domain\times 2^\Domain\to\Real
    \quad\with F(\cdot\mid\cdot)\sim\GP*{m\paren[\Big]{x,X},k\paren[\Big]{(x,X),(x',X')}}
    .
\end{equation*}
This marginal reward function is related to the overall reward function with
\begin{equation*}
    F(x_{1:T}) = \sum_{t=1}^T F(x_t\mid x_{1:t-1})
    .
\end{equation*}
An optimal solution for \(T\) rounds is then defined as
\begin{equation*}
    X^*_T = \argmax_{X\subseteq\Domain,\abs{X}\leq T} F(X)
    .
\end{equation*}

\subsubsection{Decision \& Regret}
The proposed algorithm \SMUCB{} tries to find such an optimal solution by sequentially deciding, which item \(x_t\) should be selected next to reveal its marginal reward. The strategy is to use the Gaussian process model to obtain an upper confidence bound \(F_u(x\mid x_{1:t-1})\) for the unknown marginal reward function \(F(x\mid x_{1:t-1})\) given the previous items \(x_{1:t-1}\). The next item is then chosen as the maximizer of this upper bound which can be written as
\begin{equation*}
    \begin{aligned}
        x_t &= \argmax_{x\in\Domain} F_u(x\mid x_{1:t-1})\\
        &= \argmax_{x\in\Domain} \mu_{t-1}(x)+\beta_t^{1/2}\sigma_{t-1}(x)
    \end{aligned}
\end{equation*}
The set of all items \(x_{1:T}\) after round \(T\) is then returned as a solution to the interactive bandit optimization problem. The cumulative regret is defined as
\begin{equation*}
    R(T) = \paren*{1-\frac{1}{e}} F(X^*_T) - F(x_{1:T})
    .
\end{equation*}

\subsubsection{Comparison}
Different from our setting is that the sample function from the Gaussian process corresponds to the marginal reward function which directly contributes to the overall reward of all items \(x_{1:T}\). This difference to our setting is therefore the same as in \cref{ssec:problem-comparison-gpucb,ssec:problem-comparison-macopt}. While they can use the upper confidence bound in their decision rule similar to \GPUCB{} for balancing exploration and exploitation, we are only concerned about exploration.

Identical to our setting is that the solution to the interactive bandit optimization problem consists of the decisions \(x_{1:T}\) made in all rounds. Hence, they defined the cumulative regret of all decisions with respect to the time-dependent optimal solution \(X^*_T\) as it is the case for us. In fact, alternating between choosing a camera location \(\theta_t\) and receiving the marginal reward \(F(\theta_t\mid\theta_{1:t-1})\) is exactly what \textcite{chen2017interactive} described as interacting with the marginal reward function.

Similarly, the solution in their setting consists of multiple decisions \(x_1,\dots,x_T\) greedily made over all rounds. Hence, their and our cumulative regrets are defined with respect to the greedy \(1-\frac{1}{e}\) approximation guarantee of an optimal solution as explained in \cref{ssec:problem-comparison-gpucb}.

\subsection{Summary}

For the comparison, we focused on two major differences in the settings.

First, the relation between the sample function of the Gaussian process and the reward function determines the design of objective functions for the algorithm.
If the sample function is monotonically related to the reward function, maximizing the reward function corresponds to maximizing the sample function by balancing exploration and exploitation. A suitable objective function can
then be obtained by replacing the sample function
with its upper confidence bound in the reward function as seen in \cref{ssec:problem-comparison-gpucb,ssec:problem-comparison-macopt,ssec:problem-comparison-smucb}. If the relationship between sample function \(f\) and reward function is more complex like in our setting through the observation function \(o(\cdot;f)\), one has to design new objective functions.

Second, the definition of a solution to the considered problem determines the formulation of the regret and the convergence guarantees.
\begin{itemize}
    \item If a solution coincides with a single decision made in a single round, then the simple regret can be formulated with respect to an optimal solution and no-regret leads to convergence to an optimal solution.
    \item If a solution consists of multiple decisions made in a single round, then the simple regret can be formulated with respect to a near-optimal solution and no-regret leads to convergence to a near-optimal solution.
    \item If a solution consists of multiple decisions made over multiple rounds, then the cumulative regret is formulated with respect to a near-optimal solution and no-regret leads to pseudo-convergence to a near-optimal decision.
\end{itemize}

%% file: 05_design.tex
\chapter{Algorithm Design}\label{chp:design}

In this chapter, we discuss different design choices for our algorithm and present the final candidates for the analysis in \cref{chp:analysis}. We provide an overview of the algorithm design in \cref{fig:design-overview}.

In \cref{sec:design-gp} we discuss the design of the Gaussian process model, which consists of the choice for the mean and covariance function. Due to the need of \(2\pi\)-periodic surface functions, we discuss different periodization techniques to obtain periodic covariance functions.

In \cref{sec:design-objective} we focus on the design of suitable objective functions and formulate specific requirements and heuristics for them based on our insights. We design multiple different types of objective functions and highlight their advantages and disadvantages.

In \cref{sec:design-algorithm} we introduce the greedy and the two-phase algorithm design as two ways for making decisions based on predefined objective functions.

Finally, in \cref{sec:design-summary} we filter all possible designs choices into a small set of candidate algorithms for the final analysis.

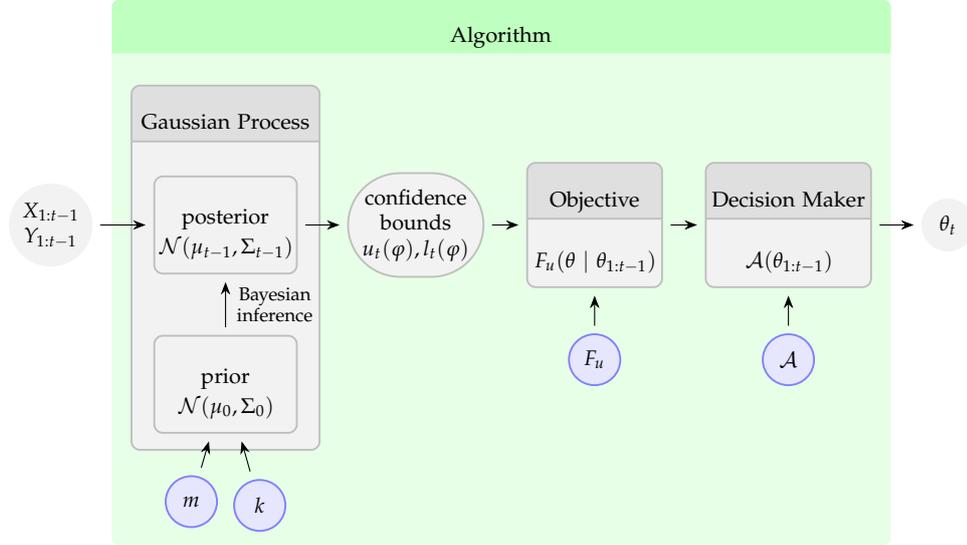
\begin{figure}
    \newtcolorbox{boxcontainer}[3][]{
        boxstyle={#3},
        adjusted title={#2},
        toptitle=0.75mm, bottomtitle=0.75mm,
        fonttitle=\scriptsize,
        #1
    }
    \newtcbox{\circlecontainer}[2][]{
        circlestyle=#2,
        #1
    }
    \centering
    \begin{tikzpicture}[
        remember picture,
        every node/.style={inner sep=1mm, outer sep=0mm},
        >/.tip=Stealth,
        node distance=0cm,
        font=\scriptsize,
    ]
        \node (algorithm) {
            \begin{boxcontainer}[width=10.25cm, top=0.75mm, bottom=0.75mm]{Algorithm}{green}
            \centering
            \begin{tikzpicture}[remember picture]
            \node (gp) {
                \begin{boxcontainer}[width=2.5cm, top=1mm, bottom=1mm]{Gaussian Process}{gray}
                \centering
                \begin{tikzpicture}
                \node (posterior) {
                    \begin{boxcontainer}[width=1.9cm, notitle]{}{gray}
                        \renewcommand{\baselinestretch}{1.2}\selectfont
                        posterior \\ \(\Normal*{\mu_{t-1},\Sigma_{t-1}}\)
                    \end{boxcontainer}
                };
                \node (prior) [below=0.6cm of posterior] {
                    \begin{boxcontainer}[width=1.9cm, notitle]{}{gray}
                        \renewcommand{\baselinestretch}{1.2}\selectfont
                        prior \\ \(\Normal*{\mu_0,\Sigma_0}\)
                    \end{boxcontainer}
                };
                \end{tikzpicture}
                \end{boxcontainer}
            };
            \node (confidence-bounds) [right=0.45cm of posterior] {
                \circlecontainer[height=1.4cm, bean arc]{gray}{
                    confidence \\ bounds \\ \(u_t(\varphi),l_t(\varphi)\)
                }
            };
            \node (objective) [right=0.35cm of confidence-bounds] {
                \begin{boxcontainer}[width=1.8cm]{Objective}{gray}
                    \(\Fu(\theta\mid\theta_{1:t-1})\)
                \end{boxcontainer}
            };
            \node (decision) [right=0.35cm of objective] {
                \begin{boxcontainer}[width=2.2cm]{Decision Maker}{gray}
                    \(\A(\theta_{1:t-1})\)
                \end{boxcontainer}
            };
            
            \node (mean) [below left=0.35cm and 0cm of prior.south] {
                \circlecontainer[width=0.7cm, tcbox width=forced center]{blue}{
                    \(m\)
                }
            };
            \node (kernel) [below right=0.4cm and 0cm of prior.south] {
                \circlecontainer[width=0.7cm, tcbox width=forced center]{blue}{
                    \(k\)
                }
            };
            \node (objective-design) [below=0.4cm of objective.south] {
                \circlecontainer[width=0.7cm, tcbox width=forced center]{blue}{
                    \(F_u\)
                }
            };
            \node (decision-design) [below=0.4cm of decision.south] {
                \circlecontainer[width=0.7cm, tcbox width=forced center]{blue}{
                    \(\A\)
                }
            };
            \end{tikzpicture}
            \end{boxcontainer}
        };
        \node (input) [left=0.6cm of posterior] {
            \circlecontainer{gray}{
                \(X_{1:t-1}\) \\ \(Y_{1:t-1}\)
            }
        };
        \node (output) [right=0.45cm of decision] {
            \circlecontainer[width=0.7cm, tcbox width=forced center]{gray}{
                \(\theta_t\)
            }
        };
        
        \draw[<-] (prior) -- (mean);
        \draw[<-] (prior) -- (kernel);
        \draw[<-] (decision) -- (decision-design);
        \draw[<-] (objective) -- (objective-design);
        \draw[->] (prior) -- node[right=0.5mm, align=center, font=\tiny] {Bayesian \\ inference} (posterior);
        \draw[->] (input) -- (posterior);
        \draw[->] (posterior) -- (confidence-bounds);
        \draw[->] (confidence-bounds) -- (objective);
        \draw[->] (objective) -- (decision);
        \draw[->] (decision) -- (output);
    \end{tikzpicture}
    \caption[Overview of the Algorithm Design]{
        Overview of the Algorithm Design.
        This figure expands the structure of the algorithm visualized in \cref{fig:problem-overview} with more details and highlights the main design choices (blue). Given observed and measured surface points \((X_{1:t-1},Y_{1:t-1})\), the algorithm updates its Gaussian process model with Bayesian inference to obtain confidence bounds \(u_t\) and \(l_t\) for the surface function from the posterior distribution. These are fed into an objective function \(F_u\) which is maximized by \(\A\) based on certain decision rules to obtain the NBV estimate \(\theta_t\).}
    \label{fig:design-overview}
\end{figure}


\section{Design of Gaussian Process}\label{sec:design-gp}

In this section we discuss our specific design choices for the Gaussian process model used by \(\A\) to model the uncertainty in the true surface function. In \cref{ssec:design-gp-mean} we define a straightforward mean function and in \cref{ssec:design-gp-covariance} we focus on the more complex covariance function. In particular, we present different periodization methods for the covariance function, which is necessary to obtain \(2\pi\)-periodic sample functions from the Gaussian process.

We refer back to \cref{sec:background-gp} for the necessary background knowledge on Gaussian processes and to \cref{ssec:problem-simplified-gp} in which we discussed how we use the Gaussian process model in our setting. This section is devoted to the design only. The design choices we make in this section play an important role, since they encode our prior assumptions on the unknown object shape.

\subsection{Mean Function}\label{ssec:design-gp-mean}
The mean function as defined in \cref{eq:mean-covariance-function-surface} encodes our prior knowledge on the shape and particularly the size of the object. Since we know that the object is bounded in size and the surface function must lie between \(\dmin\) and \(\dmax\), a natural choice for the mean function is
\begin{equation}\label{eq:mean-function}
    m(\varphi) = \frac{1}{2}(\dmin + \dmax)
    .
\end{equation}
This means we assume that the average object surface is close to the center between the bounds. This might not hold in practice if we set \(\dmin\) too small or \(\dmax\) too large and only deal with objects of the same size far off the average of \(\dmin\) and \(\dmax\). But without this particular prior knowledge, this mean function is a reasonable choice.

\subsection{Covariance Function}\label{ssec:design-gp-covariance}
The covariance function or \emph{kernel} of the Gaussian process as defined in \cref{eq:mean-covariance-function-surface} encodes our remaining prior knowledge on the object shape, since it measures the similarity between different points \(f(\varphi)\) and \(f(\varphi')\) on the surface at \(\varphi\) and \(\varphi'\) as discussed in \cref{ssec:background-gp-kernels}.
This is important, since it allows \(\A\) to infer information about other, potentially unknown surface points \(\varphi'\) from some known surface point \(\varphi\) based on their similarity.

\subsubsection{Design Choices}
One important class are \emph{stationary} kernels \(k(\varphi,\varphi') = k_s(\varphi-\varphi')\) as defined in \cref{eq:stationary-kernel}, which measure the similarity only based on the points relative to each other and not on the points themself. For modeling the object shape, it makes sense to use a stationary kernel, since rotating or translating the object changes the absolute position of a surface point, but not its similarity to other surface points.

In our case, we want the similarity of surface points to not only depend on \(\varphi-\varphi'\), but more specifically only on the absolute difference \(\abs{\varphi-\varphi'}\) between the polar angles of surface points. Whether one surface point is to the left or to the right of another surface point should not influence the similarity, since objects with a mirrored shape of another object can also exist. These kind of desired kernel functions with \(k(\varphi,\varphi') = k_s(\abs{\varphi-\varphi'})\) are called \emph{isotropic} as defined in \cref{eq:isotropic-kernel} and form a subclass of stationary kernels.%
\footnote{Note that stationary kernel functions defined on one-dimensional inputs like in our setting are always isotropic due to symmetry \(k(\varphi,\varphi')=k(\varphi',\varphi)\).}

The last and most important property for our kernel function is periodicity. In particular, we want a \emph{\(2\pi\)-periodic} kernel function with \(k(\varphi,\varphi') = k(\varphi,\varphi'+2\pi i)\) for all \(i\in\Integer\), since our Gaussian process should model \(2\pi\)-periodic polar functions.%
\footnote{Of course, we also want \(k\) to be continuous to model continuous surface functions. Hence, we cannot just copy \([0,2\pi]\) of an existing kernel to \([2\pi,4\pi]\).}
The intuition behind periodic kernels is that the similarity between surface points is periodically ``wrapped around'' every \(2\pi\) as seen in \cref{fig:design-gp-periodic}. 
For example, the surface point at \(\varphi = 0\) should not only be similar to \(\varphi' = 0.1\), but also to \(\varphi' = 2\pi - 0.1\). In fact, since these two pairs of surface points are equidistant from each other, we want an \(2\pi\)-periodic isotropic stationary kernel to return the same similarity value for them.

\begin{figure}
    \centering
    \begin{subfigure}{\linewidth}
        \centering
        \includegraphics[width=0.33\linewidth]{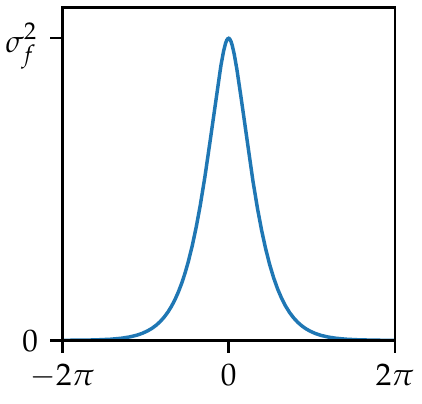}%
        \includegraphics[width=0.33\linewidth]{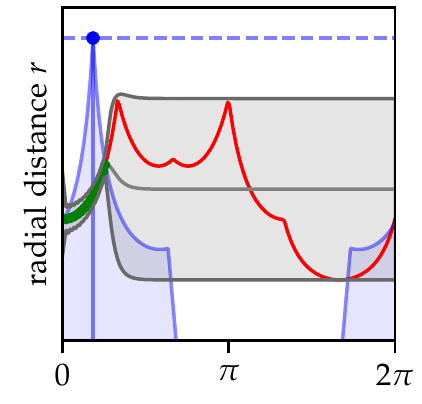}%
        \includegraphics[width=0.33\linewidth]{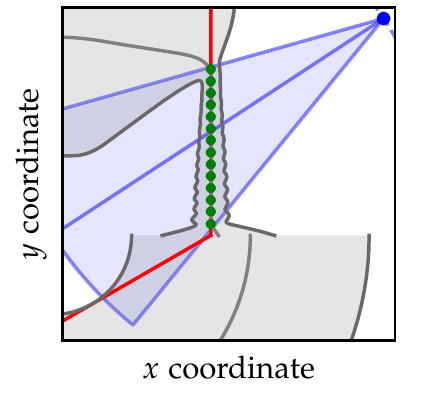}%
        \caption{non-periodic kernel}
        \label{fig:design-gp-nonperiodic}
    \end{subfigure}
    \par\smallskip
    \begin{subfigure}{\linewidth}
        \centering
        \includegraphics[width=0.33\linewidth]{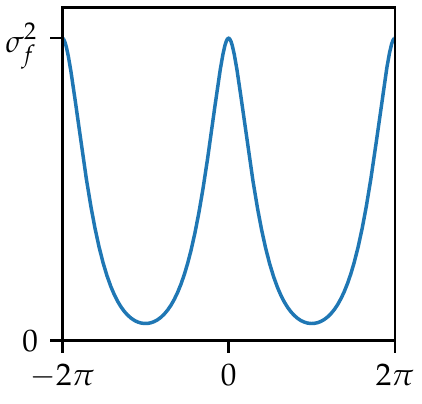}%
        \includegraphics[width=0.33\linewidth]{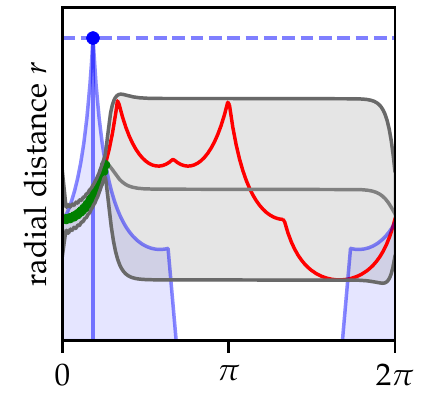}%
        \includegraphics[width=0.33\linewidth]{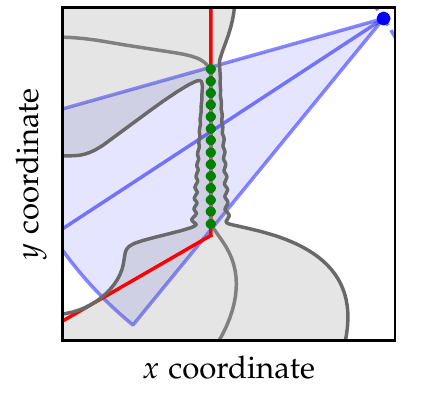}%
        \caption{periodic kernel}
        \label{fig:design-gp-periodic}
    \end{subfigure}
    \caption[Non-Periodic vs.\ Periodic Kernel]{
        Non-Periodic vs.\ Periodic Kernel.
        The left figures show the kernel functions plotted over the distance \(r=\varphi-\varphi'\) instead of the absolute distance \(\abs{\varphi-\varphi'}\) as we described in \cref{ssec:background-gp-kernels} for stationary kernels. The middle figures show the object surface function (red) and the posterior confidence region (gray) in the polar world after making measurements (green) of the object surface close to the wrap-around at \(0\to2\pi\). The right figures show the same confidence region in the real world focused on the wrap-around. Observe in (a) that the confidence bounds at \(2\pi\) are not influenced by the measurements at \(0\) due to the vanishing similarity between points close to \(2\pi\) and close to \(0\) as imposed by the non-periodic kernel. This is resolved in (b) by using a periodic kernel function which wraps around the similarity and causes the confidence bounds to adapt according to the measurements at \(0\).
    }
    \label{fig:design-gp-nonperiodic-periodic}
\end{figure}

In summary, we are looking for a stationary, more specifically isotropic, and \(2\pi\)-periodic kernel. We define
\begin{equation}\label{eq:distance-euclidean}
    r \defeq \norm{\varphi - \varphi'}_2 = \abs{\varphi - \varphi'}
\end{equation}
and simplify our kernel function to a single-argument function \(k(r)\) with \(r\in[0,2\pi]\) as in \cref{eq:stationary-kernel-univariate}.
The length scale parameter \(l\) and standard deviation of the Gaussian process \(\sigma_f\) are explained in \cref{eq:stationary-kernel-parameters}.

\subsubsection{Non-periodic Kernels}
We first discuss non-periodic kernels, because they are the most commonly used ones and are backed up with extensive theory \autocites[Chapter 13]{scholkopf2002learning}[Chapter 4]{rasmussen2005gaussian}. Later, we use them as the foundation for designing periodic kernels.

The \emph{RBF kernel} from \cref{eq:rbf-kernel} is defined as
\begin{equation*}
    k_{RBF}(r) \defeq \sigma_f^2 \exp*{-\frac{r^2}{2l^2}}
    .
\end{equation*}
The infinite smoothness of the RBF kernel leads to overly smooth sample functions, which restricts us to smooth target objects due to \cref{item:simp-object-sampled}. Other realistic, but non-smooth object shapes such as rectangles are not captured by this kernel function as shown in \cref{fig:design-gp-rbf}.

A more suitable candidate is the \emph{\matern{} kernel}
\begin{equation*}
    k_M(r) \defeq \sigma_f^2 \frac{2^{1-\nu}}{\Gamma(\nu)}\paren*{\sqrt{2\nu}\frac{r}{l}}^\nu K_\nu\paren*{\sqrt{2\nu}\frac{r}{l}}
\end{equation*}
from \cref{eq:matern-kernel}. The additional parameter \(\nu > 0\) allows us to adapt the smoothness of the functions sampled from the corresponding Gaussian process and is more discussed in \cref{ssec:background-gp-kernels}. By choosing \(\nu\) small enough, it is possible to model highly non-smooth functions with sharp surface edges as depicted in \cref{fig:design-gp-matern}.

\begin{figure}
    \centering
    \begin{subfigure}{\linewidth}
        \centering
        \includegraphics[width=0.33\linewidth]{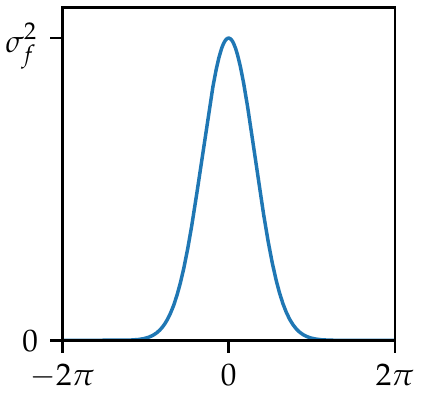}%
        \includegraphics[width=0.33\linewidth]{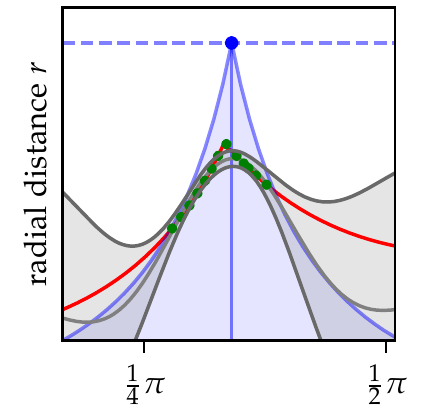}%
        \includegraphics[width=0.33\linewidth]{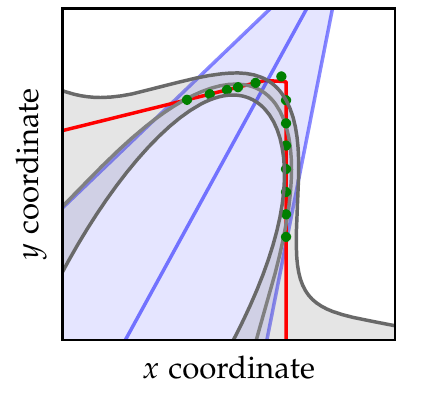}%
        \caption{RBF kernel}
        \label{fig:design-gp-rbf}
    \end{subfigure}
    \par\smallskip
    \begin{subfigure}{\linewidth}
        \centering
        \includegraphics[width=0.33\linewidth]{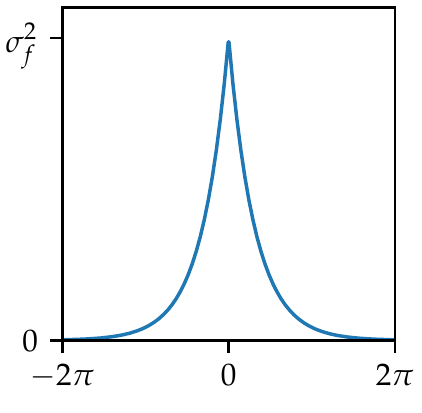}%
        \includegraphics[width=0.33\linewidth]{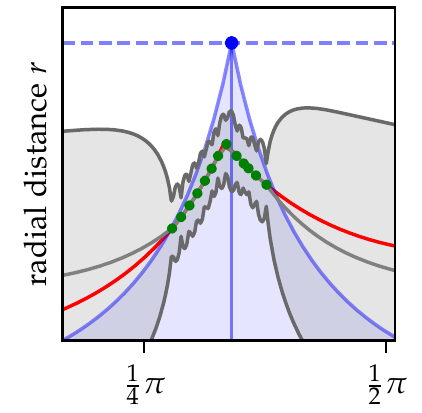}%
        \includegraphics[width=0.33\linewidth]{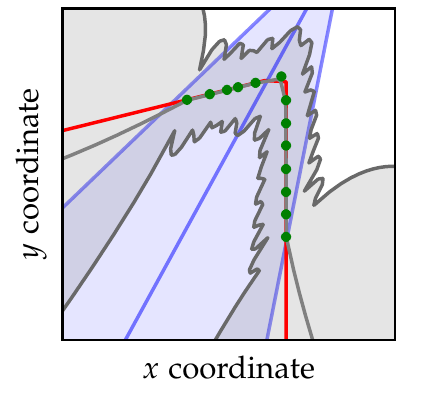}%
        \caption{\matern{} kernel with \(\nu=1/2\)}
        \label{fig:design-gp-matern}
    \end{subfigure}
    \caption[RBF vs.\ \matern{} Kernel]{
        RBF vs.\ \matern{} Kernel. 
        The left figures show the kernel functions plotted over the distance \(r=\varphi-\varphi'\). The middle figures show the object surface function (red) and the posterior confidence region (gray) in the polar world after measuring measurements (green) of the sharp edge of the object surface. The right figures show the same confidence region in the real world. Observe in (a) that the confidence bounds are very smooth under the assumption that the surface functions are sampled from a Gaussian process based on the smooth RBF kernel. Hence, they do not properly encapsulate the actual, less smooth surface function. This is resolved in (b) by using the non-smooth \matern{} kernel with \(\nu=1/2\).
    }
    \label{fig:design-gp-rbf-matern}
\end{figure}

Hence, the goal is to periodize the \matern{} kernel to capture realistic \(2\pi\)-periodic surface functions. In the following, we discuss different periodization techniques to transform existing non-periodic kernels into periodic ones.

\subsubsection{Periodization by Warping}

This method was proposed by \textcite[Chapter 5.4.3]{mackay1998introduction} and uses an arbitrary non-linear mapping \(u\) to define a new kernel function
\begin{equation*}
    k_u(\varphi,\varphi') \defeq k(u(\varphi),u(\varphi')) \quad\with u\colon\Real^n \to \Real^m,\varphi\mapsto u(\varphi)
    .
\end{equation*}
This simply corresponds to feature composition and preserves the positive definiteness of the kernel. To periodize a kernel defined for one-dimensional \(\varphi\in\Real\), \citeauthor{mackay1998introduction} used
\begin{equation*}
    u(\varphi) = \begin{pmatrix} \cos(\varphi) \\ \sin(\varphi) \end{pmatrix}
\end{equation*}
to map each point on the real line to a point on the unit circle. The warped kernel \(k_u\) naturally is periodic, since the input \(u(\varphi)\) and \(u(\varphi')\) to the kernel function becomes \(2\pi\)-periodic. An isotropic stationary kernel, which only depends on \(r = \norm{\varphi-\varphi'}_2\), depends now on
\begin{equation}\label{eq:distance-on-circle}
    r_p = \norm{u(\varphi)-u(\varphi')}_2
    = 2\abs*{\sin*{\frac{\varphi-\varphi'}{2}}}
    = 2\abs*{\sin*{\frac{r}{2}}}
    .
\end{equation}
Intuitively, instead of computing the similarity between two points based on their absolute distance on \(\Real\) as defined in \cref{eq:distance-euclidean}, it is computed based on their Euclidean distance on the unit circle as in \cref{eq:distance-on-circle}. Hence, the similarity naturally wraps-around every \(2\pi\) as visualized in \cref{fig:design-gp-warping-wraparound}.

\begin{figure}
    \centering
    \includegraphics[width=0.6\linewidth]{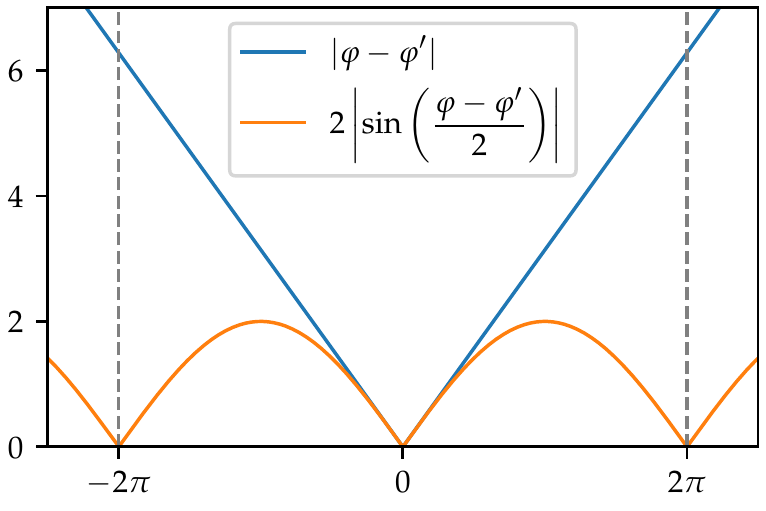}
    \caption[Absolute Distance on \(\Real\) vs.\ Euclidean Distance on Unit Circle]{
        Absolute Distance on \(\Real\) vs.\ Euclidean Distance on Unit Circle.
        Observe how the absolute distance (blue) is monotonically increasing to the left and right side, while the Euclidean distance on the unit circle (orange) naturally reaches its maximum at \(\pi\) and decreases again towards \(2\pi\).
        The dotted lines (gray) indicate our sampling region \(\varphi-\varphi'\in[-2\pi,2\pi]\).
    }
    \label{fig:design-gp-warping-wraparound}
\end{figure}

\begin{definition}[Periodization by Warping]\label{def:periodization-warping}
    Let \(k(r)\) be a stationary kernel defined on \(\Domain \subseteq \Real\).%
    \footnote{This periodization technique is only possible on \(\Real\) as far as we know.}
    Its \(2\pi\)-periodization by warping is defined as
    \begin{equation*}
        \kpwarp(r) \defeq k\paren*{2\abs*{\sin*{\frac{r}{2}}}}
        .
    \end{equation*}
\end{definition}

\citeauthor{mackay1998introduction} used this method to derive the most commonly found periodic kernel
\begin{equation*}
    \kpwarp[RBF](r) = k_{RBF}\paren*{2\abs*{\sin*{\frac{r}{2}}}}
    = \sigma_f^2 \exp*{-\frac{2\sin*{\frac{r}{2}}^2}{l^2}}
\end{equation*}
from the RBF kernel defined in \cref{eq:rbf-kernel}. Confusingly for our setting, this is called \emph{the periodic kernel}. We however use ``periodic kernel'' to refer to a general periodic kernel.

We continue applying this transformation to the \matern{} kernel defined in \cref{eq:matern-kernel} and obtain the \emph{periodic \matern{} kernel by warping} given as
\begin{equation}\label{eq:periodic-matern-kernel-warped}
    \kpwarp[M](r) = k_M\paren*{2\abs*{\sin*{\frac{r}{2}}}}
\end{equation}
which is visualized in \cref{fig:design-gp-periodic-matern}.

\subsubsection{Periodization by Periodic Summation}
A straightforward periodization technique was proposed by \textcite[Chapter 4.4.4]{scholkopf2002learning} based on the infinite sum of periodically shifted kernels.

\begin{definition}[Periodization by Periodic Summation]\label{def:periodization-infinite-sum}
    Let \(k(r)\) be a stationary kernel defined on \(\Domain \subseteq \Real\). Its \(2\pi\)-periodization by periodic summation is defined as
    \begin{equation*}
        \kpsuminf(r) \defeq \frac{\sigma_f^2}{C} \sum_{i\in\Integer} k(r+2\pi i)
    \end{equation*}
    with \(C\) ensuring \(\kpsuminf(0) = \sigma_f^2\).
\end{definition}

Compared to the previous method, periodicity is not introduced in the kernel function argument by feature composition, but by the additive combination of shifted kernels as visualized in \cref{fig:design-gp-periodic-summation}.
According to \textcite[Section 3]{borovitskiy2020matern}, positive definiteness is preserved by periodic summation, since it does not change the positivity of the Fourier transform of the kernel function, which is equivalent to positive definiteness by Bochner's theorem.%
\footnote{A discussion on Bochner's theorem and the Fourier transform of a kernel function is provided in \cref{remark:stationary-kernel-eigenfunctions} in another context.}
\TODO{REMOVABLE CHECK because of linearity of FT (summation) and additional exponential factor (from shifting)? And hence S(w) is positive, i.e. positive finite measure, i.e. positive definite by Bochner. But exponential factor is complex and can be negative...hmmm}

Unfortunately, a closed-form solution for the infinite summation is not always given and we have to approximate it by truncating the sum after finite terms.

\begin{figure}
    \centering
    \includegraphics[width=0.6\linewidth]{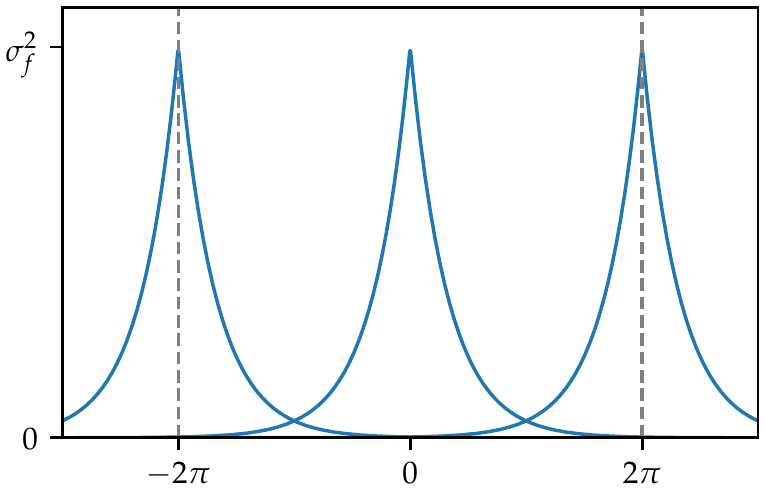}
    \caption[Periodic Summation of Kernels]{
        Periodic Summation of Kernels.
        This figure visualizes the idea of periodizing a stationary kernel through periodic summation. By placing copies of the stationary kernel function at multiples of \(2\pi\), one can simulate the periodical wrap-around of the similarity.
        The dotted lines (gray) indicate our sampling region \(\varphi-\varphi'\in[-2\pi,2\pi]\).
    }
    \label{fig:design-gp-periodic-summation}
\end{figure}

\begin{definition}[\(\kappa\)-approximative Periodization by Periodic Summation]\label{def:periodization-finite-sum}
    Let \(k(r)\) be a stationary kernel defined on \(\Domain \subseteq \Real\). Its \(\kappa\)-approximative \(2\pi\)-periodization by periodic summation is defined as
    \begin{equation*}
        \kpsumfin{\kappa}(r) \defeq \frac{\sigma_f^2}{C}\sum_{i=-\kappa}^\kappa k(r+2\pi i)
    \end{equation*}
    with \(C\) ensuring \(\kpsumfin{\kappa}(0) = \sigma_f^2\).
\end{definition}
This approximation is close to a truly periodized kernel, when the to be periodized kernel function decays fast enough. For example, consider the periodized \matern{} kernel
\begin{equation*}
    \begin{alignedat}{7}
        &\kpsuminf[M](0)\\
        &= \frac{\sigma_f^2}{C}\paren{\dots &&+ k_M(-4\pi) &&+ k_M(-2\pi) &&+ k_M(0) &&+ k_M(2\pi) &&+ k_M(4\pi) &&+ \dots}\\
        &= \frac{\sigma_f^2}{C}\paren{\dots &&+ 8.0 \cdot 10^{-9} &&+ 2.2 \cdot 10^{-4} &&+ \quad 1 &&+ 2.2 \cdot 10^{-4} &&+ 8.0 \cdot 10^{-9} &&+ \dots}
    \end{alignedat}
\end{equation*}
with \(\sigma_f=1, l=1\) and \(\nu=1.5\). Since each of the terms only significantly contribute to the sum within \(\pm2\pi\) around their own peak, it allows us to drop all additional terms with negligibly small values within the domain \([0,2\pi]\) of \(r\) given in \cref{eq:distance-euclidean}. However, we keep the finite sum symmetric for technical reasons as described in \cref{remark:symmetric-periodic-summation}. We suspect that the finite sum of periodically shifted kernels preserves positive semi-definiteness as it is the case for the infinite sum from \cref{def:periodization-infinite-sum}, but it remains an open question for us.
\TODO{REMOVABLE CHECK maybe similar reasoning as above?}

The \emph{periodic \matern{} kernel by 1-approximative periodic summation} is given as
\begin{equation}\label{eq:periodic-matern-kernel-finite-sum}
    \kpsumfin[M]{1}(r) = \frac{\sigma_f^2}{C}\paren*{k_M(r-2\pi) + k_M(r) + k_M(r+2\pi)}
\end{equation}
which is visualized in \cref{fig:design-gp-periodic-matern}.

\begin{remark}\label{remark:symmetric-periodic-summation}
    Kernel functions must satisfy symmetry and positive semi-definiteness to form a valid covariance function as defined in \cref{eq:mean-covariance-function}. Given a stationary kernel function \(k(r)\), symmetry is satisfied when periodizing the kernel by periodic summation even with a non-symmetric sum
    \begin{equation*}
        \kpsumfin{\kappa}(x,x') = \sum_{i=0}^\kappa k(\abs{x-x'}+2\pi i)
    \end{equation*}
    with \(\kpsumfin{\kappa}(x,x') = \kpsumfin{\kappa}(x',x)\).%
    \footnote{In fact, \(k_p(r) \defeq k(r) + k(r+2\pi)\) is a perfectly periodic kernel with \(r\in[0,2\pi]\).}
    However, it is often the case that libraries only provide an interface for \(k(x,x')\) instead of \(k(r)\).%
    \footnote{For example, the implementation of the \matern{} kernel in the \href{https://scikit-learn.org/stable/modules/generated/sklearn.gaussian_process.kernels.Matern.html}{scikit-learn library}.}
    To still use their implementation, one can define the periodic kernel as
    \begin{equation*}
        \kpsumfin{\kappa}(x,x') = \sum_{i=-\kappa}^\kappa k(x+2\pi i,x') = \sum_{i=-\kappa}^\kappa k(\abs{x-x'+2\pi i})
        .
    \end{equation*}
    One can verify by case distinction on the sign of \(x-x'\) that
    \begin{equation*}
        \begin{aligned}
            k(x+2\pi i,x') + k(x-2\pi i,x')
            &= k(\abs{x-x'+2\pi i}) + k(\abs{x-x'-2\pi i})
            \\ &= k(\abs{x-x'}+2\pi i) + k(\abs{x-x'}-2\pi i)
        \end{aligned}
    \end{equation*}
    is satisfied for all \(i\). Therefore, it follows that the sum of this expression over \(i\) is symmetric.
\end{remark}


To our luck, \textcite[Eq. 47]{borovitskiy2020matern} provide closed-form expressions for the infinite sum of \matern{} kernels
\begin{equation}\label{eq:periodic-matern-kernel-infinite-sum}
    \kpsuminf[M_\nu](r) = \frac{\sigma_f^2}{C} \sum_{i\in\Integer} k_{M_\nu}(r+2\pi i)
\end{equation}
as in \cref{def:periodization-infinite-sum} for half-integer smoothness parameters \(\nu=n+\frac{1}{2}, n\in\Natural\). We refer to them as \emph{periodic \matern{} kernels by periodic summation} and list only the most important ones:%
\footnote{In the referenced paper, the \matern{} kernel is periodized on \([0,1]\). To obtain a \matern{} kernel periodized on \([0,2\pi]\), we instantiate \(\inst{x}{\frac{\varphi}{2\pi}}\) and \(\inst{\kappa}{\frac{l}{2\pi}}\).}
\begingroup
    \allowdisplaybreaks
    \begin{align*}
        \kpsuminf[M_{1/2}](r) &= \frac{\sigma_f^2}{C_{1/2}} \cosh(u)\\
        u &=\textstyle \frac{r-\pi}{l}\\
        \kpsuminf[M_{3/2}](r) &= \frac{\sigma_f^2}{C_{3/2}} \paren*{a_0\cosh(u) + a_1 \sinh(u)}\\
        u &=\textstyle \sqrt{3}\frac{r-\pi}{l}\\
        a_0 &=\textstyle \frac{\pi l}{6}\paren*{\frac{l}{\pi} + \sqrt{3}\coth*{\frac{\sqrt{3}\pi}{l}}}\\
        a_1 &=\textstyle -\frac{l^2}{6}
        \refstepcounter{equation}\tag{\theequation}\label{eq:periodic-matern-kernel-infinite-sum-examples}\\
        \kpsuminf[M_{5/2}](r) &= \frac{\sigma_f^2}{C_{5/2}} \paren*{a_0\cosh(u) + a_1 u\sinh(u) + a_2 u^2\cosh(u)}\\
        u &=\textstyle \sqrt{5}\frac{r-\pi}{l}\\
        a_0 &=\textstyle -\frac{\pi^2 l^2}{200}\paren[\Big]{-5+\frac{3l^2}{\pi^2}+\frac{3\sqrt{5}l}{\pi}\coth*{\frac{\sqrt{5}\pi}{l}} + 10\coth*{\frac{\sqrt{5}\pi}{l}}^2}\\
        a_1 &=\textstyle \frac{\pi l^3}{100}\paren*{\frac{3l}{200} + \sqrt{5}\coth*{\frac{\sqrt{5}\pi}{l}}}\\
        a_2 &=\textstyle -\frac{l^4}{200}
    \end{align*}
\endgroup
with \(C_\nu\) ensuring \(\kpsuminf[M_\nu](0) = \sigma_f^2\) and \(\kpsuminf[M_{1/2}](r)\) visualized in \cref{fig:design-gp-periodic-matern}.

\subsubsection{Periodization by Truncation}
The last technique is inspired by the work of \textcite{bachmayr2018representations} and continues the idea behind approximative periodic summation from \cref{def:periodization-finite-sum}. The approximation error comes from the infinite support%
\footnote{The support of a function \(f\colon\Domain\to\Real\) is defined as \(\supp(f) \defeq \set{x\in\Domain\mid f(x)\neq0}\).}
of the periodized kernel functions, although the error is negligibly small. This error can be completely eliminated when using kernels with finite support, which allows us to write the infinite sum as a finite sum of shifted kernels. Such a kernel with finite support can be obtained by truncating existing kernels to a finite domain. The difficulty is to preserve the existing smoothness in the kernel. \textcite[Section 5.1]{bachmayr2018representations} provides a suitable truncation function
\newcommand{\omegafunc}[1]{\omega\paren*{\frac{#1}{\Delta c}}}
\begin{equation*}
    \begin{aligned}
        t_{c_1 \to c_2}(r) \defeq \frac{\omegafunc{c_2 - \abs{r}}}{\omegafunc{c_2 - \abs{r}} + \omegafunc{\abs{r} - c_1}}
        &\quad\with \omega(r) \defeq \begin{cases}
            \exp*{-r^{-1}}, &r > 0\\
            0, &r \leq 0
        \end{cases}
        \\ &\quad\text{and } \Delta c = c_2 - c_1
    \end{aligned}
\end{equation*}
which smoothly shrinks from 1 to 0 from \(c_1\) to \(c_2\) and \(-c_1\) to \(-c_2\) as depicted in \cref{fig:design-gp-truncation-function}. It guarantees smoothness \(t_{c_1 \to c_2} \in C^\infty(\Real)\) and ensures finite support \(\supp(t_{c_1 \to c_2}) = [-c_2,c_2]\).

\begin{figure}
    \centering
    \begin{subfigure}{0.45\linewidth}
        \centering
        \includegraphics{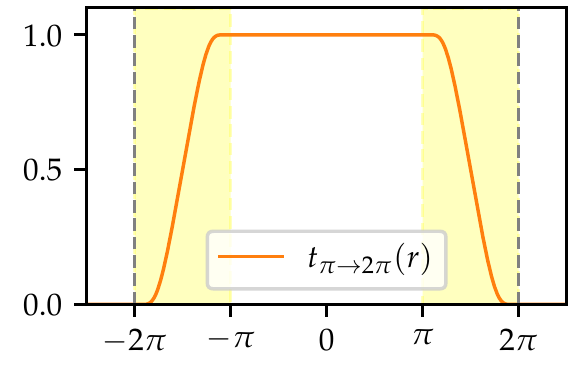}%
        \caption{truncation function}
        \label{fig:design-gp-truncation-function}
    \end{subfigure}%
    \quad%
    \begin{subfigure}{0.45\linewidth}
        \centering
        \includegraphics{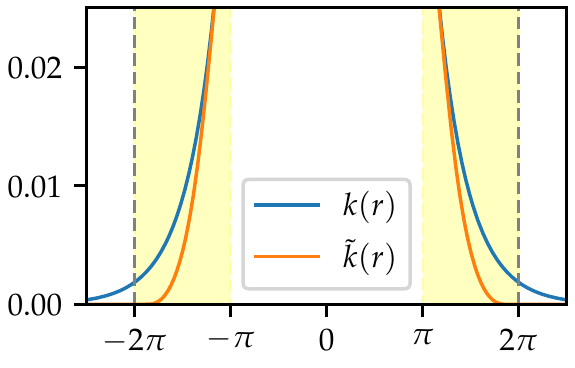}
        \caption{original vs.\ truncated kernel}
        \label{fig:design-gp-truncated-kernel}
    \end{subfigure}
    \caption[Truncation of Kernel Function]{
        Truncation of Kernel Function.
        (a) shows the truncation function (orange) smoothly transitioning between \(\pi\) and \(2\pi\) and similarly between \(-\pi\) and \(-2\pi\) (yellow) from \(1\) to \(0\). (b) compares the truncated kernel function (orange) with the original kernel function (blue) from a strongly zoomed-in perspective. Observe how the truncated kernel achieves finite support \([-2\pi,2\pi]\) and at the same time only marginally differs from the original kernel.
        The dotted lines (gray) indicate our sampling region \(\varphi-\varphi'\in[-2\pi,2\pi]\).
    }
\end{figure}

\begin{definition}[Periodization by Truncation]\label{def:periodization-truncation}
    Let \(k(r)\) be a stationary kernel defined on \(\Domain \subseteq \Real\). Its \(2\pi\)-periodization by truncation is defined as
    \begin{equation*}
        \kptrunc{c_1 \to c_2}(r) \defeq \frac{\sigma_f^2}{C}\sum_{i=-\kappa}^\kappa \ktrunc(r+2\pi i)
        \quad\with \ktrunc(r) \defeq t_{c_1 \to c_2}(r)k(r)
    \end{equation*}
    with \(C\) ensuring \(\kptrunc{c_1 \to c_2}(0) = \sigma_f^2\) and \(\kappa = \ceil{\frac{c_2}{2\pi}}\) ensuring perfect periodization.
\end{definition}

Besides being perfectly periodic, this kernel hardly differs from the one obtained by \(\kappa\)-approximative periodic summation given in \cref{eq:periodic-matern-kernel-finite-sum}, since the difference between the original and truncated kernel function is negligible as seen in \cref{fig:design-gp-truncated-kernel}.
However, it is unknown to us whether the sum of periodically shifted and truncated kernel functions remains positive semi-definite on complete \([0,2\pi]\).
\TODO{REMOVABLE CHECK}

The \emph{periodic \matern{} kernel by truncation} is defined as
\begin{equation}\label{eq:periodic-matern-kernel-truncated}
    \kptrunc[M]{\pi\to2\pi}(r) \defeq \frac{\sigma_f^2}{C} \paren*{\ktrunc_M(r-2\pi) + \ktrunc_M(r+2\pi i) + \ktrunc_M(r+2\pi)}
    .
\end{equation}

\subsubsection{Summary}
We discussed three different techniques and one approximation to periodize non-periodic kernel functions. \cref{def:periodization-warping} uses feature composition to map the kernel function inputs to the unit circle before feeding them to the kernel function. \cref{def:periodization-infinite-sum} is based on the periodic summation of kernels, while \cref{def:periodization-finite-sum} approximates this infinite sum with finitely many terms. Finally, \cref{def:periodization-truncation} eliminates the approximation error of the finite sum by truncating the kernel function to a finite support.

\begin{figure}
    \centering
    \includegraphics[width=0.7\linewidth]{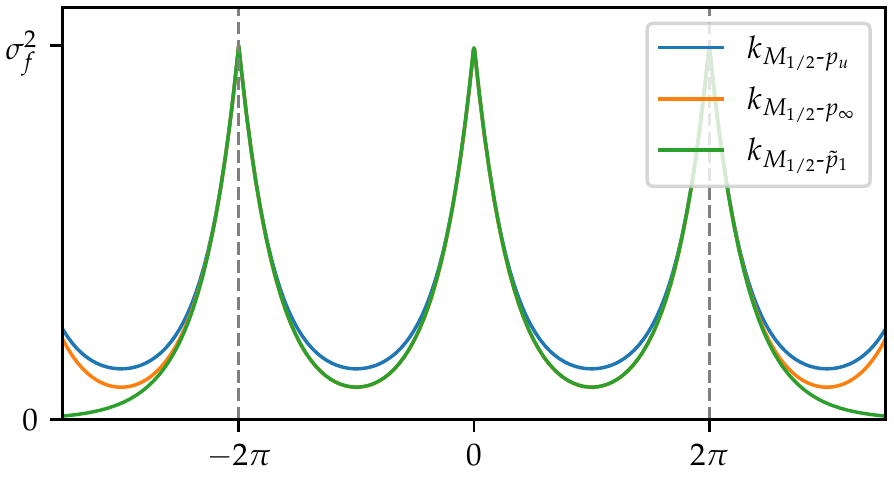}
    \caption[Comparison of Periodic \matern{} Kernels]{
        Comparison of Periodic \matern{} Kernels.
        This figure visualizes the periodic \matern{} kernels obtained by warping (blue), infinite periodic summation (orange) and finite periodic summation (green). The one obtained by truncation is not visualized here, since there is no visible difference to the one obtained by finite periodic summation.
        Observe how warping produces a periodic kernel which is a bit ``smoother'' between multiples of \(2\pi\). The reason is that periodicity is introduced in the input to the original kernel function through the smooth, sinusoidal Euclidean distance on the unit circle given in \cref{eq:distance-on-circle} and shown in \cref{fig:design-gp-warping-wraparound}, whereas periodic summation introduces periodicity by additive combination of the outputs.
        One can also see that the finite periodic summation only ensures (1-approximative) periodicity on \([-2\pi,2\pi]\). This is sufficient to us, since our sampling region is \(\varphi-\varphi'\in[-2\pi,2\pi]\) indicated by the dotted lines (gray).
    }
    \label{fig:design-gp-periodic-matern}
\end{figure}

They all preserve or most likely preserve positive definiteness \TODO{REMOVABLE CHECK with above} and provide similar behavior as shown in \cref{fig:design-gp-periodic-matern}, but for our final analysis, we use the periodic \matern{} kernel \(\kpsuminf[M_\nu]\) with \(\nu=n+\frac{1}{2},n\in\Natural\) obtained by infinite summation as defined in \cref{eq:periodic-matern-kernel-infinite-sum}. The reason is that \textcite{borovitskiy2020matern} provides further properties such as the spectral density for these kernels, which we require for our analysis.


\section{Design of Objective Functions}\label{sec:design-objective}

This section is devoted to various designs of the objective function and an initial screening of them to find suitable candidate objective functions for the later analysis. In the previous \cref{ssec:problem-simplified-objective} we described the unknown true objective function, which should be estimated by our designed objective function. \cref{ssec:design-objective-requirements} provides a list of formal requirements and important heuristics for efficient and theoretically founded objective functions, which we use for the initial screening. In \cref{ssec:design-objective-observation}, we start designing an objective function in the naive way based on the observed length of the confidence bounds. \cref{ssec:design-objective-length-vs-area} discusses the issues in this design and contrasts it with area-based objective functions. Based on our insights, we present new designs in \cref{ssec:design-objective-intersection,ssec:design-objective-confidence,ssec:design-objective-uncertainty} each with a different tradeoff between accuracy and simplicity and summarize them in \cref{ssec:design-objective-summary}.

\subsection{Requirements}\label{ssec:design-objective-requirements}

Through the design of multiple candidate objective functions, their theoretical analysis in \cref{chp:analysis} and their evaluation in a simulation framework described in \cref{chp:experiments}, we gained valuable insights into the requirements for an efficient and theoretically founded objective function. We categorize these requirements into \textit{required}, which are properties required by the theoretical analysis, \textit{important}, which are heuristics important for the efficiency of an objective function in practice, and \textit{useful}.

\begin{enumerate}[label=(R\arabic*), ref=Req.~\arabic*]
    \item \label{item:req-necessary-bound} Necessary upper bound
    \begin{discussion}{required}
        We define a \emph{necessary upper bound} as the condition
        \begin{equation*}
            F(\thetagre_t \mid \theta_{1:t-1}) \leq \Fu(\thetagre_t \mid \theta_{1:t-1})
            \quad\text{ for all } t \geq 1
        \end{equation*}
        which is required to guarantee the assumption of \cref{lem:lemma-2-2} together with a reasonable \(\A\) which finds the NBV estimate by maximizing \(\Fu\) as discussed in \cref{remark:lemma-2-2-assumption}. A \emph{sufficient upper bound}, which refers to
        \begin{equation*}
            F(\theta \mid \theta_{1:t-1}) \leq \Fu(\theta \mid \theta_{1:t-1})
            \quad\text{ for all } \theta\in\Camspace \text{ and } t \geq 1
            ,
        \end{equation*}
        is a sufficient criterion for being a necessary upper bound.
    \end{discussion}
    
    \item \label{item:req-closed-form} Closed-form expression
    \begin{discussion}{required}
        A closed-form expression of \(\Fu\), which is ideally simple in its form, is required to derive theoretical guarantees for the behavior of \(\A\).
    \end{discussion}
    
    \item \label{item:req-real-world} Real world information
    \begin{discussion}{important}
        Real world information is important for objective functions which are designed based on the geometry of the real world. More specifically, they must be aware of the significant deformations between the polar world, in which the surface function and its confidence bounds are defined, and the real world, in which the actual object resides, as visualized earlier in \cref{fig:problem-worlds}.
        For example, objective functions defined in terms of the area between the confidence bounds must compute this area in the real world. If such objective functions are defined in the polar world, they do not have real world information. Later, we show in \cref{ssec:design-objective-length-vs-area} how FOV information can be included using the area of circle sectors.
    \end{discussion}
    
    \item \label{item:req-marginal} Marginal information
    \begin{discussion}{important}
        Marginal information refers to the ability of objective functions to take previously observed surface points into account. This is important to estimate the number of only newly observed points and to properly estimate the marginal utility, which is our true objective function. An objective function which returns an upper bound for the number of all instead of newly observed surface points does not have marginal information.
    \end{discussion}
    
    \item \label{item:req-fov} FOV information
    \begin{discussion}{important}
        FOV information refers to the ability of objective functions to take the shape of the camera's FOV into account. This is important to only estimate the number of surface points inside the FOV. In particular, it captures the information that the farther the camera is from the object, the more surface points can be potentially observed. Without this information an objective function might excessively overestimate the actual number of observed surface points and provide false information to \(\A\).
    \end{discussion}
    
    \item \label{item:req-occlusion} Occlusion information
    \begin{discussion}{useful}
        Occlusion information refers to the ability of objective functions to take the occlusion between different surface points into account. This is useful for providing more accurate estimates to \(\A\). The difficulty is to find a closed-form expression required by \cref{item:req-closed-form} which captures this information. 
    \end{discussion}
\end{enumerate}

\subsection{Observation-based Objective Functions}\label{ssec:design-objective-observation}

We start designing objective functions based on the number of \textit{observed} points on the confidence bounds. This seems to be the natural way to estimate the number of newly observed points on the surface function, but it appears to be the naive way as described afterwards in \cref{ssec:design-objective-length-vs-area}.

\subsubsection{\algOS{}}\label{sssec:design-objective-os}
This objective function counts the number of observed points on the object surface and is defined as
\begin{equation*}
    \Fu[OS](\theta \mid \theta_{1:t-1})
    \defeq \abs{o(\theta)} = \text{number of points on \(f\) observed from \(\theta\)}
    .
\end{equation*}
This actually corresponds to the utility \(F(\set{\theta})\) as defined in \cref{eq:true-utility} and requires complete knowledge about the shape of the target object similar to the marginal utility, the true objective function. Hence, \algOS{} is only of theoretical interest.

\begin{itemize}
    \item[\cmark] (\cref{item:req-necessary-bound}) It provides a sufficient upper bound, since the number of observed points is always larger than the number of newly observed points.
    \item[\xmark] (\cref{item:req-closed-form}) A closed-form expression is not known to us due to its dependence on the black-box observation function from \cref{eq:observation-function} and in particular the unknown true object shape.
    \item[\cmark] (\cref{item:req-real-world}) It has real world information by definition.
    \item[\xmark] (\cref{item:req-marginal}) It does not have marginal information, since it counts all observed surface points and not only the newly observed ones.
    \item[\cmark] (\cref{item:req-fov}) It has FOV information by definition.
    \item[\cmark] (\cref{item:req-occlusion}) It has occlusion information by definition.
\end{itemize}

\subsubsection{\algOCU{} \& \algOCL{}}\label{sssec:design-objective-ocu-ocl}
These objective functions count the number of observed points on the upper and lower confidence bounds, respectively,%
\footnote{Similar to the object surface, the upper and lower confidence bounds are discretized into points based on the real world discretization.}
and are defined as
\begin{equation*}
    \begin{aligned}
        F_l^{(OCU)}(\theta \mid \theta_{1:t-1})
        &\defeq \text{number of points on \(u_t\) observed from \(\theta\)}
        \\
        \Fu[OCL](\theta \mid \theta_{1:t-1})
        &\defeq \text{number of points on \(l_t\) observed from \(\theta\)}
        .
    \end{aligned}
\end{equation*}
Because of the shape of the FOV, more points can potentially be observed if the camera is farther away as seen in \cref{fig:design-objective-observation-ocl-ocu}. Since the lower confidence bound for the object surface is farther away from the camera than the upper confidence bound, \algOCU{} would intuitively provide a lower bound and \algOCL{} an upper bound for the true objective function.

Unfortunately, we observed that these bounds are not always guaranteed. In particular, for a surface function which oscillates arbitrarily inside the FOV region as in \cref{fig:design-objective-observation-degenerate}, it is possible to observe less points on the lower confidence bound than on the surface function. 

\begin{itemize}
    \item[\xmark] (\cref{item:req-necessary-bound}) It does not provide a necessary upper bound as discussed above.
    \item[\xmark] (\cref{item:req-closed-form}) A closed-form expression is not known to us due to the dependence on the black-box observation function from \cref{eq:observation-function} defined with respect to the lower confidence bound.
    \item[\cmark] (\cref{item:req-real-world}) It has real world information by definition.
    \item[\xmark] (\cref{item:req-marginal}) It does not have marginal information, since the number of points on the lower confidence bound does not tell, where the camera has already measured the surface. For example, once the camera measured a part of the surface function and the lower confidence bound fits itself to the measured surface, this objective function still includes the points on the lower confidence bound located at these positions as seen in \cref{fig:design-objective-observation-counterexample}.
    \item[\cmark] (\cref{item:req-fov}) It has FOV information, since it computes its estimates based on the camera's observation model.
    \item[\cmark] (\cref{item:req-occlusion}) It has occlusion information, since it computes its estimates based on the camera's observation model.
\end{itemize}

\begin{figure}
    \centering
    \begin{subfigure}{0.33\linewidth}
        \centering
        \includegraphics[width=\linewidth]{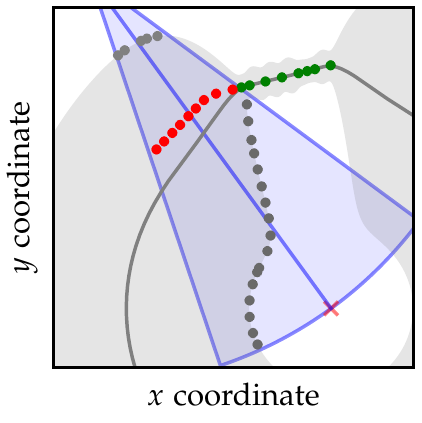}
        \caption{\(F_l^{(OCU)}\) and \(\Fu[OCL]\)}
        \label{fig:design-objective-observation-ocl-ocu}
    \end{subfigure}%
    \begin{subfigure}{0.33\linewidth}
        \centering
        \includegraphics[width=\linewidth]{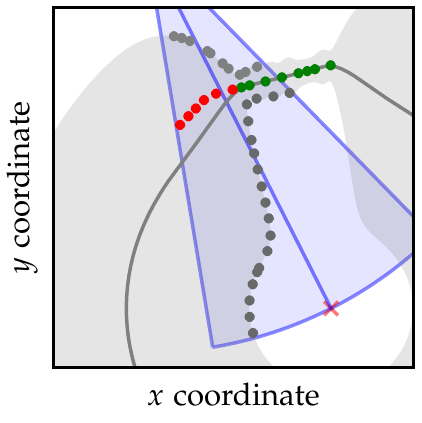}
        \caption{\vphantom{\(\Fu[OCU]\)}no marginal information}
        \label{fig:design-objective-observation-counterexample}
    \end{subfigure}%
    \begin{subfigure}{0.33\linewidth}
        \centering
        \includegraphics[width=\linewidth]{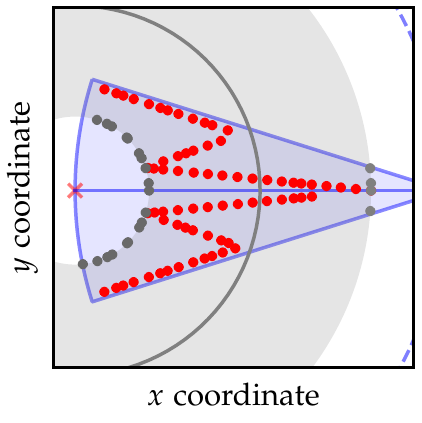}
        \caption{\vphantom{\(\Fu[OCU]\)}not an upper bound}
        \label{fig:design-objective-observation-degenerate}
    \end{subfigure}
    \caption[Observation-based Objective Functions]{
        Observation-based Objective Functions.
        (a) illustrates the intuition that more surface points can be seen if the camera is further away from them. Hence, the number of observed points on the lower confidence boundary (dark gray) would correspond to an upper bound on the number of observed surface points (red) and vice versa for the upper confidence boundary (light gray). However, this intuition is incorrect. One reason is depicted in (b) which shows that both functions do not take previously observed surface points (green) into account and therefore do not provide information about the marginal utility \(F(\theta\mid\theta_{1:t-1})\). In addition, (b) provides a counterexample for OCU being a lower bound. Finally, (c) shows the extreme case of a counterexample for OCL being an upper bound.
    }
    \label{fig:design-objective-observation}
\end{figure}

\subsection{Length-based vs.\ Area-based Objective Functions}\label{ssec:design-objective-length-vs-area}

Counting the number of observed points can be intuitively interpreted as estimating the length of the surface function inside the FOV.
In particular, this is the case if we shrink the pixel width of the real world discretization \(h\to0\). Hence, we call the observation-based objective functions in \cref{ssec:design-objective-observation} \emph{length-based objective functions}.

However, the counterexample of ill-shaped oscillating objects given for \hyperref[sssec:design-objective-ocu-ocl]{\algOCL{}} shows that the observed length of the lower confidence bound does not properly upper bound the observed length of potential surface functions.
In theory, if we take the number of oscillations of such an object inside the FOV to infinity, the surface function would ``fill out'' the complete area of the FOV. We realized that the \textit{observed area} between the confidence bounds is a better indicator than the \textit{observed length} of the lower confidence bound. This area corresponds to counting the number of real world pixels in the intersection of the FOV and confidence region. It is naturally guaranteed that there is no surface function inside the confidence region of which a camera can observe more surface points than this number.

This realization motivates the design of \emph{area-based objective functions}. The basic idea is to compute the area between the confidence bounds \(l_t(\varphi)\) and \(u_t(\varphi)\) given in \cref{eq:confidence-bounds-ideal} by integration and divide it by the area of a real world pixel. Since the Gaussian process can only be evaluated on a finite set of points, we approximate the integral with the sum
\begin{equation*}
    \begin{aligned}
        \Fu(\theta\mid\theta_{1:t-1})
        &= \frac{1}{h^2}\int_{\varphi_1}^{\varphi_2} g(l_t(\varphi),u_t(\varphi)) \d{\varphi}
        \quad\text{ with some function } g\\
        &\approx \frac{1}{h^2}\sum_{\varphi\in[\varphi_1,\varphi_2]}^{\Delta\varphi} g(l_t(\varphi),u_t(\varphi)) \Delta\varphi
        .
    \end{aligned}
\end{equation*}
We use this sum notation to describe a sum from \(\varphi_1\) to \(\varphi_2\) with step size \(\Delta\varphi\).%
\footnote{More formally, \(\sum_{x\in[a,b]}^{\Delta x} f(x) \defeq \sum_{k=0}^{\floor{\frac{b-a}{\Delta x}}} f(a+k\cdot\Delta x)\).}
This sum also requires us to define a summation interval \([\varphi_1,\varphi_2]\), over which the area should be computed approximately. For that reason, we define two kinds of summation intervals which we use later for the design of area-based objective functions. 

Beforehand, we introduce a notation for the \emph{left-right FOV boundary} which corresponds to the polar function parameterizing the rays cast by the camera at viewing angle \(\frac{\FOV}{2}\) and \(-\frac{\FOV}{2}\) as shown in \cref{fig:design-objective-fov-boundary}. The closed-form expression is given as
\begin{equation}\label{eq:fov-boundary}
    \begin{aligned}
        \fov{\theta}{\varphi} &\defeq \begin{cases}
            \ray{\theta,\frac{\FOV}{2}}{\varphi}, & \varphi \in_{2\pi} [\theta - \Delta\varphi, \theta]\\
            \ray{\theta,-\frac{\FOV}{2}}{\varphi}, & \varphi \in_{2\pi} [\theta, \theta + \Delta\varphi]\\
        \end{cases}\\
        \with \Delta\varphi &= \arctan*{\frac{\DOF\sin(\FOV/2)}{\dcam - \DOF\cos(\FOV/2)}}\\
    \end{aligned}
\end{equation}
and \(\ray{\theta,\alpha}{\varphi}\) derived in \cref{ssec:simulation-background-camray}. The notation \(\in_{2\pi}\) denotes interval membership modulo \(2\pi\). The constant \(\Delta\varphi\) corresponds to the polar angle offset of the left and right endpoint of the FOV boundary. We refer to \cref{ssec:simulation-background-fovendpoint} for its derivation. 

\begin{figure}
    \centering
    \begin{subfigure}{0.33\linewidth}
        \centering
        \includegraphics[width=\linewidth]{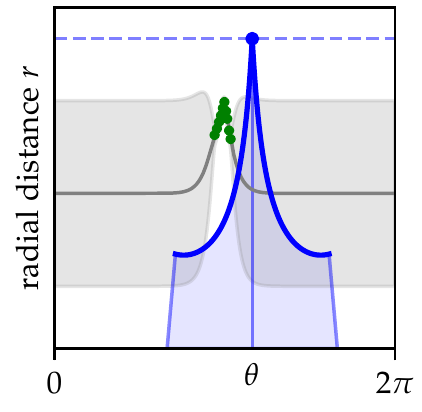}
        \caption{\vphantom{\(\PhiI_t\)}\(\fov{\theta}{\varphi}\)}
        \label{fig:design-objective-fov-boundary}
    \end{subfigure}%
    \begin{subfigure}{0.33\linewidth}
        \centering
        \includegraphics{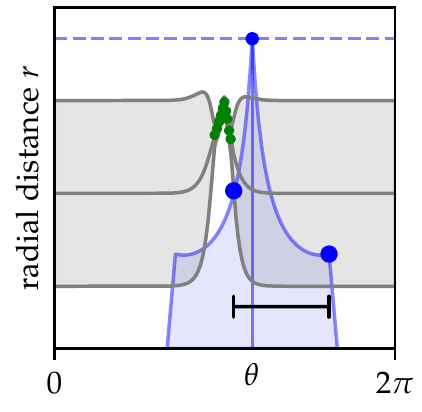}
        \caption{\(\PhiI_t(\theta)\)}
        \label{fig:design-objective-suminterval-intersect}
    \end{subfigure}%
    \begin{subfigure}{0.33\linewidth}
        \centering
        \includegraphics{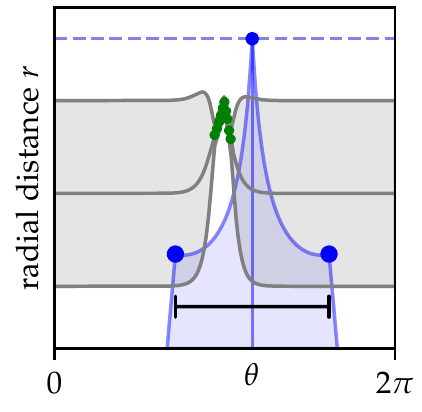}
        \caption{\(\PhiS(\theta)\)}
        \label{fig:design-objective-suminterval-simple}
    \end{subfigure}
    \caption[Left-Right FOV Boundary and Summation Intervals]{
        Left-Right FOV Boundary and Summation Intervals.
        The current setting in these figures is that the algorithm made a measurement (green) of the unknown object surface (not shown) and represents its remaining uncertainty about the object with the confidence region (gray) obtained from the Gaussian process. For the design of objective functions, we want to formalize multiple geometric quantities in this setting.
        The bold line in (a) describes the FOV boundary to the left and right side of the current camera position \(\theta\) as a polar function of the variable \(\varphi\). We use this function to include FOV information in our objective functions as recommended in \cref{item:req-fov}.
        In (b), we visualize the summation interval (black bar) based on the FOV-confidence intersection points (blue dots), which are each defined as the first intersection point between the FOV boundary and the lower confidence bound, or the endpoint of the FOV boundary in case of no intersection. This summation interval is used to include occlusion information in our objective functions as recommended in \cref{item:req-occlusion}.
        In (c), we visualize the summation interval (black bar) based on the simple FOV endpoints (blue dots), which are simply defined as the endpoints of the FOV boundary without considering potential occlusions. As a side remark, note that \(\fov{\theta}{\varphi}\) is precisely defined on  \(\PhiS(\theta)\).
    }
    \label{fig:design-objective-fov-suminterval}
\end{figure}

The first summation interval, which we denote as \emph{FOV-confidence intersection}, is bounded by the intersection points of the lower confidence bound and left-right FOV boundary. We define it as
\begin{equation}\label{eq:sum-interval-intersection}
    \begin{aligned}
        \PhiI_t(\theta) &\defeq [\varphi_1,\varphi_2]\\
        \varphi_1 &= \varphi \text{ such that } l_t(\varphi) = \fov{\theta}{\varphi}, \varphi \text{ closest to the left of } \theta\\
        \varphi_2 &= \varphi \text{ such that } l_t(\varphi) = \fov{\theta}{\varphi}, \varphi \text{ closest to the right of } \theta
    \end{aligned}
\end{equation}
and provide a visualization in \cref{fig:design-objective-suminterval-intersect}. Since multiple intersections are possible, the second constraint ensures that the closest one is taken. This implicitly handles occlusion, since intersection points farther away are occluded by the closest one. As the definition already suggests, no closed-form solution is known to us.

This summation interval can be simplified to \emph{Simple FOV endpoint}, which takes the endpoints of the FOV boundary function at \((\frac{\FOV}{2},\DOF)\) and \((-\frac{\FOV}{2},\DOF)\) given in polar coordinates relative to the camera as the left and right boundary. The definition is
\begin{equation}\label{eq:sum-interval-simple}
    \begin{aligned}
        \PhiS(\theta) &\defeq [\varphi_1,\varphi_2]\\
        \varphi_1 &= \theta - \arctan*{\frac{\DOF\sin(\FOV/2)}{\dcam - \DOF\cos(\FOV/2)}}\\
        \varphi_2 &= \theta + \arctan*{\frac{\DOF\sin(\FOV/2)}{\dcam - \DOF\cos(\FOV/2)}}
    \end{aligned}
\end{equation}
with the derivation provided in \cref{ssec:simulation-background-fovendpoint} and a visualization in \cref{fig:design-objective-suminterval-simple}. Besides the closed-form expression, observe that \(\PhiS(\theta)\) is independent of \(t\) and in addition the width \(\abs{\PhiS(\theta)}\) independent of \(\theta\). This is useful for the later analysis.

Referring back to \cref{item:req-real-world}, we want to focus on the difference between the real world and the polar world once more. In particular, a rectangle \([\varphi_1,\varphi_2]\times[0,r]\) with area \(\Delta\varphi r\) in the polar world is transformed into a \emph{circle sector} with area \(\frac{\Delta\varphi}{2\pi} \pi r^2 = \frac{1}{2}\Delta\varphi r^2\) in the real world. Similarly, a rectangle \([\varphi_1,\varphi_2]\times[l,u]\) in the polar world has area \(\frac{1}{2}\Delta\varphi(u^2 - l^2)\) in the real world. The latter expression is used frequently in the definitions of area-based objective functions.

\subsection{Intersection-based Objective Functions}\label{ssec:design-objective-intersection}

Motivated by \cref{ssec:design-objective-length-vs-area}, we design objective functions based on intersection area of the camera's FOV and the confidence region.

\subsubsection{\algIOA{}}\label{sssec:design-objective-ioa}
This objective function counts the number of real world pixels in the \textit{intersection} of the camera's FOV and the confidence region which are \textit{not occluded} by the already measured object surface. We define it as
\begin{equation*}
    \Fu[IOA](\theta\mid\theta_{1:t-1})
    \defeq \text{number of points in intersection and visible from \(\theta\)}
    .
\end{equation*}
This objective function is the most accurate area-based objective function, since it only covers the visible area in which the surface function can reside and nothing more.

\begin{itemize}
    \item[\cmark] (\cref{item:req-necessary-bound}) It provides a sufficient upper bound, since it completely covers the visible area in which the surface function can lie as shown in \cref{fig:design-objective-intersection-ioa}.
    \item[\xmark] (\cref{item:req-closed-form}) A closed-form expression is not known to us.
    \item[\cmark] (\cref{item:req-real-world}) It has real world information by definition.
    \item[\cmark] (\cref{item:req-marginal}) Interestingly, it implicitly contains marginal information, since wherever the surface is measured, the area of the confidence region shrinks towards zero and so does this objective function at these measured locations. This can be observed in \cref{fig:design-objective-intersection-ioa}.
    \item[\cmark] (\cref{item:req-fov}) It has FOV information by definition.
    \item[\cmark] (\cref{item:req-occlusion}) It has occlusion information by definition.
\end{itemize}

\begin{figure}
    \centering
    \begin{subfigure}{0.33\linewidth}
        \centering
        \includegraphics[width=\linewidth]{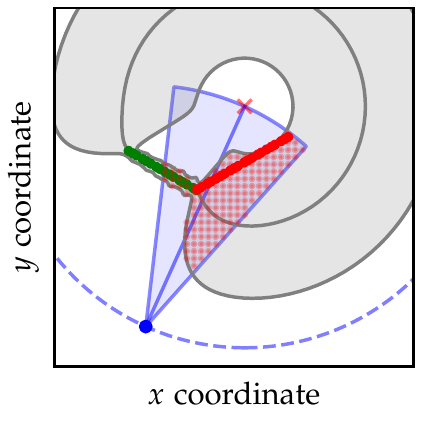}
        \includegraphics[width=\linewidth]{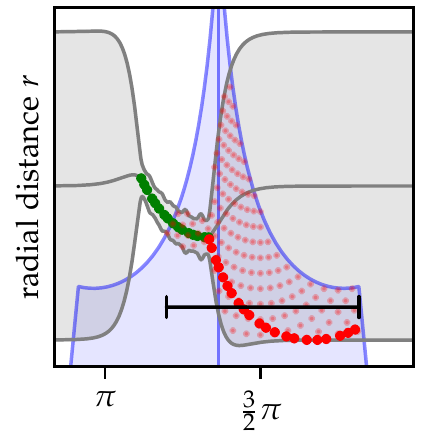}
        \caption{\(\Fu[IOA]\)}
        \label{fig:design-objective-intersection-ioa}
    \end{subfigure}%
    \begin{subfigure}{0.33\linewidth}
        \centering
        \includegraphics[width=\linewidth]{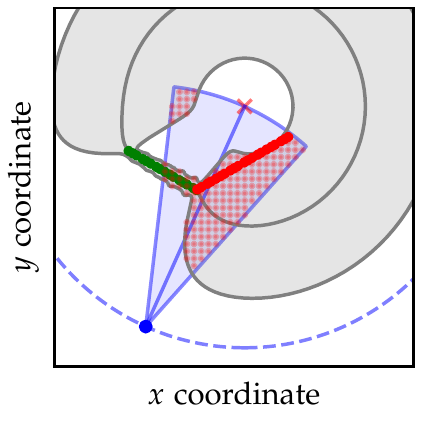}
        \includegraphics[width=\linewidth]{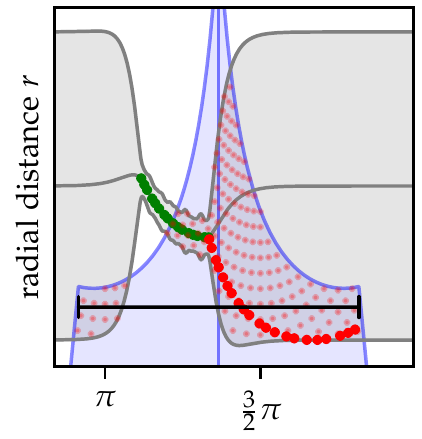}
        \caption{\(\Fu[I]\)}
        \label{fig:design-objective-intersection-i}
    \end{subfigure}
    \caption[Intersection-based Objective Functions]{
        Intersection-based Objective Functions.
        (a) visualizes the points (light red) inside the occlusion-aware intersection of the FOV (blue) and the confidence region (gray) in the real world (top) and polar world (bottom). This intersection is essentially defined as the intersection over \(\PhiI(\theta)\) (black bar). In contrast, the intersection in (b) is defined over the simpler interval \(\PhiS(\theta)\)  (black bar) without the awareness of occlusion as shown in \cref{fig:design-objective-suminterval-simple}. Observe that in both cases all newly visible surface points (red) inside the FOV are covered by both intersections. Hence, the proposed objective functions form valid upper bounds.
    }
    \label{fig:design-objective-intersection}
\end{figure}

\subsubsection{\algI{}}\label{sssec:design-objective-i}
This objective function estimates the number of real world pixels in the \textit{intersection} of the camera's FOV and the confidence region and is defined as
\begin{equation*}
    \Fu[I](\theta\mid\theta_{1:t-1})
    \defeq \frac{1}{h^2} \sum_{\varphi\in\PhiS(\theta)}^{\Delta\varphi}
    \frac{1}{2} \max*{\min*{u_t(\varphi),\fov{\theta}{\varphi}}^2 - l_t(\varphi)^2, 0} \Delta\varphi
    .
\end{equation*}
It trades off the occlusion awareness of \nameref{sssec:design-objective-ioa} for the sake of a closed-form expression. Since the visible area containing the surface function is both upper bounded by \(\fov{\theta}{\varphi}\) and \(u_t(\varphi)\), we take the minimum of both as the overall upper bound. Since \(l_t(\varphi)\) can also be above \(\fov{\theta}{\varphi}\), we take the maximum of the difference and zero to avoid negative summation terms.
This term does not perfectly compute the number of pixels in the intersection, since it ignores the finite DOF of the camera as seen from its definition. Including knowledge about the DOF in the objective function would cause more technical difficulties (see \cref{fig:problem-different-fovs}) than it would potentially benefit us. So we decided to keep \(\Fu[I]\) relatively simple. 

\begin{itemize}
    \item[\cmark] (\cref{item:req-necessary-bound}) It provides a sufficient upper bound, since it covers more than the visible area in which the surface function can lie as seen in \cref{fig:design-objective-intersection-i}.
    \item[\omark] (\cref{item:req-closed-form}) A closed-form expression is known to us, which however is too complex to analyze due to the usage of \(\max\), \(\min\) and \(\fov{\theta}{\varphi}\).
    \item[\cmark] (\cref{item:req-real-world}) It has real world information same as \nameref{sssec:design-objective-ioa}.
    \item[\cmark] (\cref{item:req-marginal}) It has marginal information same as \nameref{sssec:design-objective-ioa}.
    \item[\cmark] (\cref{item:req-fov}) It has FOV information same as \nameref{sssec:design-objective-ioa}.
    \item[\xmark] (\cref{item:req-occlusion}) It does not have occlusion information as stated above.
\end{itemize}

\subsection{Confidence-based Objective Functions}\label{ssec:design-objective-confidence}

Demotivated by the lack of a useful closed-form expression in \cref{ssec:design-objective-intersection}, which mostly comes from \(\fov{\theta}{\varphi}\), we design objective functions based on the area of the confidence region and only take the FOV into account for determining the summation boundaries.

\subsubsection{\algC{}}\label{sssec:design-objective-c}
This objective function estimates the number of real world pixels in the \textit{confidence} region on the interval \(\PhiI_t(\theta)\) from \cref{eq:sum-interval-intersection}. We define it as
\begin{equation*}
    \Fu[C](\theta\mid\theta_{1:t-1})
    \defeq \frac{1}{h^2} \sum_{\varphi\in\PhiI_t(\theta)}^{\Delta\varphi}
    \frac{1}{2}\paren*{u_t(\varphi)^2 - l_t(\varphi)^2} \Delta\varphi
    .
\end{equation*}

\begin{itemize}
    \item[\cmark] (\cref{item:req-necessary-bound}) It provides a sufficient upper bound, since it covers a lot more than the visible area in which the surface function can lie, which we visualize in \cref{fig:design-objective-confidence-c}.
    \item[\xmark] (\cref{item:req-closed-form}) A closed-form expression for \(\PhiI_t(\theta)\) is not known to us.
    \item[\cmark] (\cref{item:req-real-world}) It has real world information by definition.
    \item[\cmark] (\cref{item:req-marginal}) It has marginal information similar to \nameref{sssec:design-objective-ioa}, since the area of the confidence region shrinks wherever measurements were made.
    \item[\xmark] (\cref{item:req-fov}) It does not have FOV information, since it only considers the endpoints of the left-right FOV boundary, but does not take the shape of the FOV into account as visualized in \cref{fig:design-objective-confidence-c}. 
    \item[\cmark] (\cref{item:req-occlusion}) It has occlusion information through the use of \(\PhiI_t(\theta)\), which bounds the summation interval with the closest intersection point of the left-right FOV boundary and the lower confidence bound.
\end{itemize}

\subsubsection{\algCS{}}\label{sssec:design-objective-cs}
This objective function estimates the number of real world pixels in the \textit{confidence} region on the \textit{simple} interval \(\PhiS(\theta)\) from \cref{eq:sum-interval-simple}. We define it as
\begin{equation*}
    \Fu[CS](\theta\mid\theta_{1:t-1})
    \defeq \frac{1}{h^2} \sum_{\varphi\in\PhiS(\theta)}^{\Delta\varphi}
    \frac{1}{2}\paren*{u_t(\varphi)^2 - l_t(\varphi)^2} \Delta\varphi
    .
\end{equation*}
Similar to \nameref{sssec:design-objective-i}, it trades off the occlusion awareness contained in \(\PhiI_t(\theta)\) for the sake of a closed-form expression provided by the simpler summation interval \(\PhiS(\theta)\).

\begin{itemize}
    \item[\cmark] (\cref{item:req-necessary-bound}) It provides a sufficient upper bound, since it covers a lot more than the visible area in which the surface function can lie as shown in \cref{fig:design-objective-confidence-cs}.
    \item[\cmark] (\cref{item:req-closed-form}) A simple closed-form expression is known to us.
    \item[\cmark] (\cref{item:req-real-world}) It has real world information same as \nameref{sssec:design-objective-c}.
    \item[\cmark] (\cref{item:req-marginal}) It has marginal information same as \nameref{sssec:design-objective-c}.
    \item[\xmark] (\cref{item:req-fov}) It does not have FOV information same as \nameref{sssec:design-objective-c}.
    \item[\xmark] (\cref{item:req-occlusion}) It does not have occlusion information as stated above.
\end{itemize}

\begin{figure}
    \centering
    \begin{subfigure}{0.33\linewidth}
        \centering
        \includegraphics[width=\linewidth]{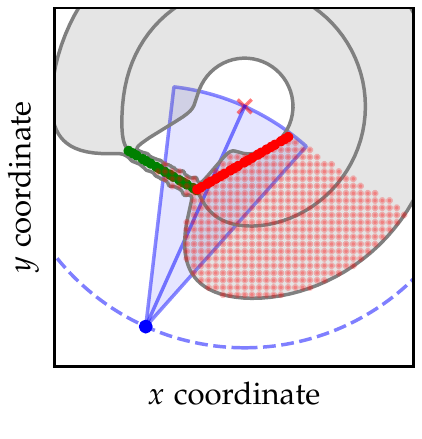}
        \includegraphics[width=\linewidth]{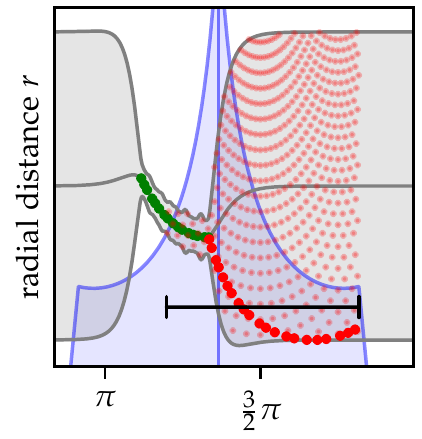}
        \caption{\(\Fu[C]\)}
        \label{fig:design-objective-confidence-c}
    \end{subfigure}%
    \begin{subfigure}{0.33\linewidth}
        \centering
        \includegraphics[width=\linewidth]{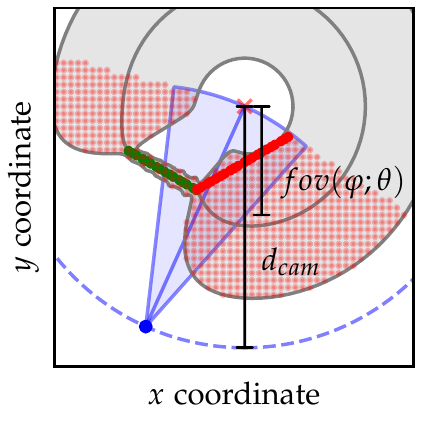}
        \includegraphics[width=\linewidth]{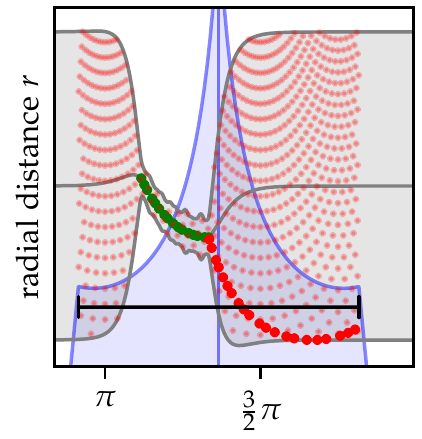}
        \caption{\(\Fu[CS]\) (and \(\Fu[CSW]\))}
        \label{fig:design-objective-confidence-cs}
    \end{subfigure}%
    \begin{subfigure}{0.33\linewidth}
        \centering
        \includegraphics[width=\linewidth]{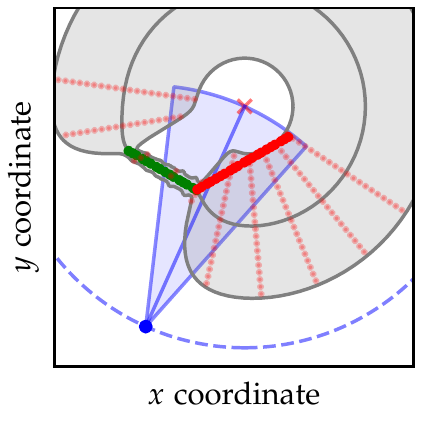}
        \includegraphics[width=\linewidth]{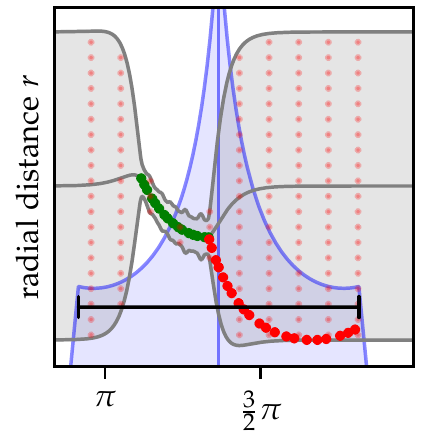}
        \caption{\(\Fu[CSP]\)}
        \label{fig:design-objective-confidence-csp}
    \end{subfigure}%
    \caption[Confidence-based Objective Functions]{
        Confidence-based Objective Functions.
        These figures visualize the different upper bounds for the number of observed surface points (red dots) provided by the confidence-based objective functions. Besides the already discussed main differences, observe in the polar world representation how the density of real world pixels (light red dots) inside some range of polar angles \(\Delta\varphi\) increases with the distance from the world center. In particular, the majority of these points are not inside the FOV, which results into heavy overestimations of the actual observations. As a side note, the larger horizontal spacings between the polar world points in (c) comes from different axis scales.
        \\
        In addition, we provide a visualization for the weight factor \(\frac{\fov{\theta}{\varphi}}{\dcam}\) of the \nameref{sssec:design-objective-csw} objective function in the real world (top) of (b). The shorter bar indicates \(\fov{\theta}{\varphi}\) while the longer one represents \(\dcam\). This should help in understanding how the ratio of both lengths estimates the number of world pixels inside the FOV.
    }
    \label{fig:design-objective-confidence}
\end{figure}

\subsubsection{\algCSP{}}\label{sssec:design-objective-csp}
This objective function estimates the number of \textit{polar} world pixels in the \textit{confidence} region on the \textit{simple} interval \(\PhiS(\theta)\) from \cref{eq:sum-interval-simple}. We define it as
\begin{equation*}
    \Fu[CSP](\theta\mid\theta_{1:t-1})
    \defeq \frac{1}{h^2} \sum_{\varphi\in\PhiS(\theta)}^{\Delta\varphi}
    \paren*{u_t(\varphi) - l_t(\varphi)} \Delta\varphi
    .
\end{equation*}
It corresponds to \nameref{sssec:design-objective-cs}, but defined in the polar world. Intuitively, this objective function counts the number of polar pixels in the \(\varphi\)-\(r\)-coordinate system, which certainly does not provide an upper bound for the number of real world pixels. Hence, this objective function is only of theoretical interest.%
\footnote{We are interested in this objective function, since it allows us to analyze \nameref{sssec:design-objective-cs} without the squared upper and lower confidence bounds as an initial step.}

\begin{itemize}
    \item[\xmark] (\cref{item:req-necessary-bound}) It does not provide a necessary upper bound as discussed above.
    \item[\cmark] (\cref{item:req-closed-form}) A simple closed-form expression is known to us.
    \item[\xmark] (\cref{item:req-real-world}) It does not have real world information as discussed above.
    \item[\cmark] (\cref{item:req-marginal}) It has marginal information same as \nameref{sssec:design-objective-cs}.
    \item[\xmark] (\cref{item:req-fov}) It does not have FOV information same as \nameref{sssec:design-objective-cs}.
    \item[\xmark] (\cref{item:req-occlusion}) It does not have occlusion information same as \nameref{sssec:design-objective-cs}.
\end{itemize}

\subsubsection{\algCSW{}}\label{sssec:design-objective-csw}
The lack of FOV information in \nameref{sssec:design-objective-c} and its two variants has proven to be serious in our simulation experiments. A more detailed explanation is given later in \cref{ssec:experiments-results-deficiency}. In short, without taking the shape of the camera's FOV into account, these objective functions excessively overestimate the actually observable region and prevents the algorithm from convergence. This motivates the following design of a variant of \nameref{sssec:design-objective-c} which takes the FOV into account by weighting.

This objective function estimates the number of real world pixels in the \textit{confidence} region on the \textit{simple} interval \(\PhiS(\theta)\) from \cref{eq:sum-interval-simple} \textit{weighted} based on the shape of FOV. We define it as
\begin{equation*}
    \Fu[CSW](\theta\mid\theta_{1:t-1})
    \defeq \frac{1}{h^2} \sum_{\varphi\in\PhiS(\theta)}^{\Delta\varphi}
    \frac{\fov{\theta}{\varphi}}{\dcam} \cdot \frac{1}{2}\paren*{u_t(\varphi)^2 - l_t(\varphi)^2} \Delta\varphi
    .
\end{equation*}
The additional weight factor tries to approximate the shape of the FOV by estimating the relative amount of each infinitesimal circle sector area inside the visible FOV region for different \(\varphi\). The circle sector area at \(\varphi=\theta\) is completely preserved due to \(\fov{\theta}{\theta}=\dcam\), while the further away \(\varphi\) is from \(\theta\), the more the contribution of the circle sector area is damped. A more visual explanation for the weight factor is provided in \cref{fig:design-objective-confidence-cs}. The idea is closely related to \nameref{sssec:design-objective-i}, which is a more accurate, but also more complex version of this objective function.

\begin{itemize}
    \item[\omark] (\cref{item:req-necessary-bound}) It does not provide a necessary upper bound, since it is possible to design very degenerate counterexamples, which however hardly appear in practice. For example, assume the current confidence region has the shape of the camera's FOV and the worst-case surface function completely fills out this confidence region by infinite oscillations. Then all circle sector areas completely contribute to the visible area, but the additional weighting factor incorrectly scales down the circle sector areas at all \(\varphi\neq\theta\).
    \item[\omark] (\cref{item:req-closed-form}) A closed-form expression is known to us, but might be difficult to analyze due to the use of \(\fov{\theta}{\varphi}\).
    \item[\cmark] (\cref{item:req-real-world}) It has real world information same as \nameref{sssec:design-objective-cs}.
    \item[\cmark] (\cref{item:req-marginal}) It has marginal information same as \nameref{sssec:design-objective-cs}.
    \item[\cmark] (\cref{item:req-fov}) It has FOV information through the weighting factor.
    \item[\xmark] (\cref{item:req-occlusion}) It does not have occlusion information same as \nameref{sssec:design-objective-cs}.
\end{itemize}

\subsection{Uncertainty-based Objective Functions}\label{ssec:design-objective-uncertainty}

The last class of objective functions we designed are based on the difference between upper and lower confidence bound only at \(\theta\), which corresponds to the uncertainty of the surface function at \(\theta\).

\subsubsection{\algU{}}\label{sssec:design-objective-u}
This objective function estimates the number of real world pixels in the difference of two circle sectors based on the current \textit{uncertainty} and with angle \(\abs{\PhiS}\). We define it as
\begin{equation*}
    \begin{aligned}
        \Fu[U](\theta\mid\theta_{1:t-1})
        &\defeq \frac{1}{h^2} \sum_{\varphi\in\PhiS(\theta)}^{\Delta\varphi}
        \frac{1}{2}\paren*{u_t(\theta_t)^2 - l_t(\theta_t)^2} \Delta\varphi\\
        &= \frac{1}{h^2} \cdot \frac{1}{2}\abs*{\PhiS}\paren*{u_t(\theta)^2 - l_t(\theta)^2}
    \end{aligned}
    .
\end{equation*}
It corresponds to a simplified version of \nameref{sssec:design-objective-cs} which does not take all uncertainties within \(\PhiS(\theta)\) into account, but only the current uncertainty at \(\theta\). This allows us to eliminate the sum, which is helpful for the beginning.

\begin{itemize}
    \item[\cmark] (\cref{item:req-necessary-bound}) This objective function is unique in the sense that it only provides a necessary upper bound, but not a sufficient upper bound as visualized in \cref{fig:design-objective-uncertainty-u,fig:design-objective-uncertainty-u-max}. The reason is that this objective function only depends on the uncertainty at the current location and provides a valid upper bound only if this uncertainty upper bounds all uncertainties at other locations.
    \item[\cmark] (\cref{item:req-closed-form}) A simple closed-form expression is known to us.
    \item[\cmark] (\cref{item:req-real-world}) It has real world information by definition.
    \item[\cmark] (\cref{item:req-marginal}) It has marginal information similar to \nameref{sssec:design-objective-cs}.
    \item[\xmark] (\cref{item:req-fov}) It does not have FOV information, since it only depends on the current uncertainty. In contrast to \nameref{sssec:design-objective-cs} and its variants, the lack of FOV information does not prevent convergence of \(\A\) in our simulation experiments, since this objective function still allows \(\A\) to find the positions with largest uncertainty, which is not the case for \nameref{sssec:design-objective-cs} as explained later in \cref{ssec:experiments-results-deficiency}.
    \item[\xmark] (\cref{item:req-occlusion}) It does not have occlusion information, since it only depends on the current uncertainty.
\end{itemize}

\begin{figure}
    \centering
    \begin{subfigure}{0.33\linewidth}
        \centering
        \includegraphics{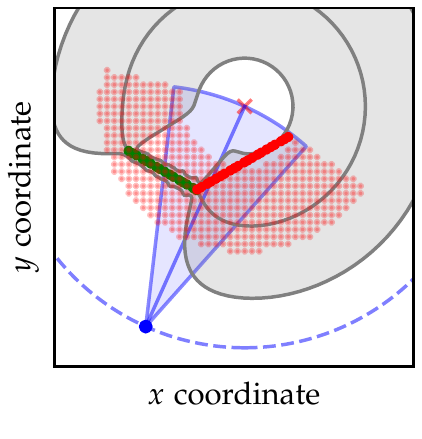}
        \includegraphics{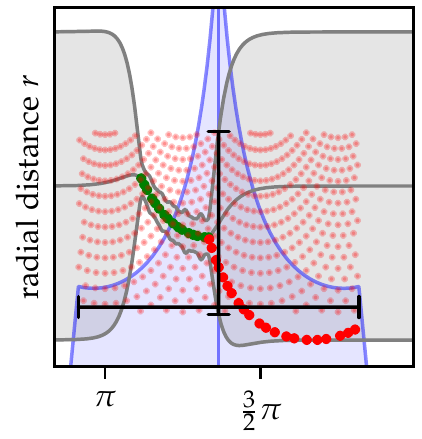}
        \caption{\(\Fu[U]\)}
        \label{fig:design-objective-uncertainty-u}
    \end{subfigure}%
    \begin{subfigure}{0.33\linewidth}
        \centering
        \includegraphics{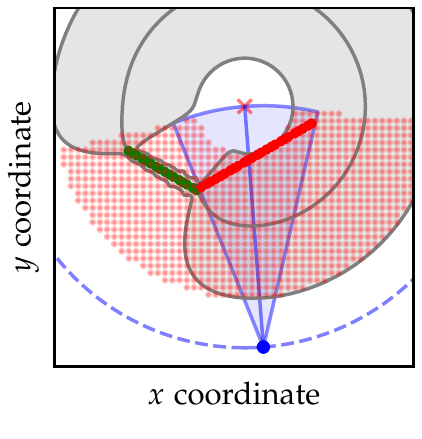}
        \includegraphics{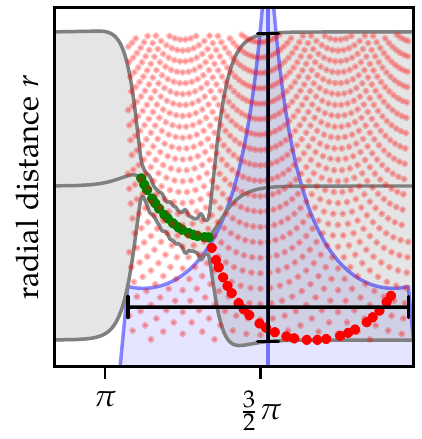}
        \caption{\(\Fu[U]\) at \(\argmax\)}
        \label{fig:design-objective-uncertainty-u-max}
    \end{subfigure}%
    \begin{subfigure}{0.33\linewidth}
        \centering
        \includegraphics{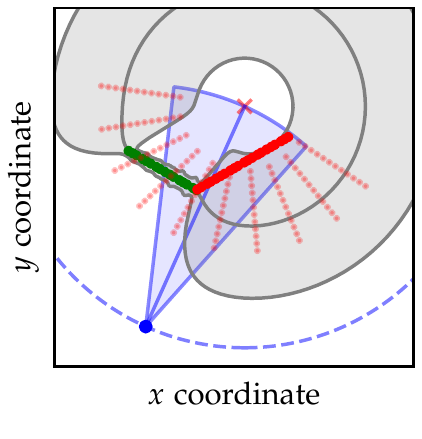}
        \includegraphics{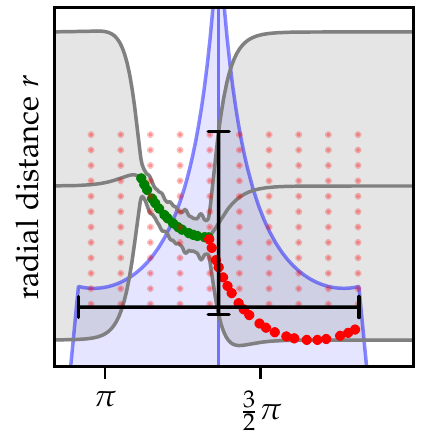}
        \caption{\(\Fu[UP]\)}
        \label{fig:design-objective-uncertainty-up}
    \end{subfigure}
    \caption[Uncertainty-based Objective Functions]{
        Uncertainty-based Objective Functions.
        The idea of these objective functions is to upper bound the number of visible surface points (red dots) with the number of real world points (light red dots) inside a rectangle in the polar world which is determined by the uncertainty at the current position (vertical bar) as the height and the length of the summation interval \(\PhiS(\theta)\) (horizontal bar) as the width. This upper bound is only valid if the current uncertainty is the maximum uncertainty over all locations, as one can easily see between (a) and (b).
    }
    \label{fig:design-objective-uncertainty}
\end{figure}

\subsubsection{\algUP{}}\label{sssec:design-objective-up}
This objective function estimates the number of \textit{polar} world pixels in the rectangle based on the current \textit{uncertainty} and with width \(\abs{\PhiS}\). We define it as
\begin{equation*}
    \begin{aligned}
        \Fu[UP](\theta\mid\theta_{1:t-1})
        &\defeq \frac{1}{h^2} \sum_{\varphi\in\PhiS(\theta)}^{\Delta\varphi}
        \paren*{u_t(\theta) - l_t(\theta)} \Delta\varphi\\
        &= \frac{1}{h^2} \cdot \abs*{\PhiS} \paren*{u_t(\theta) - l_t(\theta)}
    \end{aligned}
    .
\end{equation*}
It corresponds to \nameref{sssec:design-objective-u}, but defined in the polar world in the same way as \nameref{sssec:design-objective-csp} was defined. Hence, it is again only of theoretical interest, since it certainly does not provide an upper bound for the number of real world pixels.

\begin{itemize}
    \item[\xmark] (\cref{item:req-necessary-bound}) It does not provide a necessary upper bound as stated above.
    \item[\cmark] (\cref{item:req-closed-form}) A simple closed-form expression is known to us.
    \item[\xmark] (\cref{item:req-real-world}) It does not have real world information as stated above.
    \item[\cmark] (\cref{item:req-marginal}) It has marginal information same as \nameref{sssec:design-objective-u}.
    \item[\xmark] (\cref{item:req-fov}) It does not have FOV information same as \nameref{sssec:design-objective-u}.
    \item[\xmark] (\cref{item:req-occlusion}) It does not have occlusion information same as \nameref{sssec:design-objective-u}.
\end{itemize}

\subsection{Summary}\label{ssec:design-objective-summary}

In order to provide an overview on the design of all objective functions, we briefly summarize their main design idea below and the met requirements in \cref{tab:design-objective-requirements}.

Observation-based objective functions (\cref{ssec:design-objective-observation}) are length-based objective functions and hard to evaluate due to the black-box observation function.
\begin{itemize} 
    \item \nameref{sssec:design-objective-os} counts the number of observed points on the object surface.
    \item \hyperref[sssec:design-objective-ocu-ocl]{\algOCL{}} counts the number of observed points on the lower confidence bound.
\end{itemize}

Intersection-based objective functions (\cref{ssec:design-objective-intersection}) are the most accurate area-based objective functions, but lack simple closed-form expressions.
\begin{itemize}
    \item \nameref{sssec:design-objective-ioa} counts the number of visible real world pixels in the intersection of the FOV and the confidence region.
    \item \nameref{sssec:design-objective-i} counts the number of real world pixels in the intersection of the FOV and the confidence region.
\end{itemize}

Confidence-based objective functions (\cref{ssec:design-objective-confidence}) provide simple closed-form expressions, but excessively overestimate the visible area due to the lack of FOV information (\cref{item:req-fov}) except for the weighted variant.
\begin{itemize}
    \item \nameref{sssec:design-objective-c} counts the number of real world pixels in the confidence region on \(\PhiI_t(\theta)\).
    \item \nameref{sssec:design-objective-cs} counts the number of real world pixels in the confidence region on \(\PhiS(\theta)\).
    \item \nameref{sssec:design-objective-csp} counts the number of polar world pixels in the confidence region on \(\PhiS(\theta)\).
    \item \nameref{sssec:design-objective-csw} counts the number of real world pixels in the confidence region on \(\PhiS(\theta)\) weighted based on the FOV shape.
\end{itemize}

Uncertainty-based objective functions (\cref{ssec:design-objective-uncertainty}) provide even simpler closed-form expressions and at the same time do not overestimate the visible area.
\begin{itemize}
    \item \nameref{sssec:design-objective-u} counts the number of real world pixels in the circle sector area difference with radius based on the current uncertainty and angle \(\abs{\PhiS}\).
    \item \nameref{sssec:design-objective-up} counts the number of polar world pixels in the rectangle with height based on the current uncertainty and width \(\abs{\PhiS}\).
\end{itemize}

After we presented different ways to form a decision based on a given objective function in \cref{sec:design-algorithm}, we finalize our list of candidate objective functions in \cref{sec:design-summary} based on the insights gained here.

\begin{table}
    \centering
    \renewcommand{\arraystretch}{1.3}
    \newcommand{\sepspace}{\addlinespace[0.4cm]}
    
    \newcommand{\rot}{\rotatebox{90}}
    \let\xmarkold\xmark
    \renewcommand{\xmark}{\textcolor{lightgray}{\xmarkold}}
    \begin{tabular}{@{}lcccccc@{}}
        \toprule
        Objective function
        & \rot{\shortstack[l]{Upper bound \\ (\cref{item:req-necessary-bound})}}
        & \rot{\shortstack[l]{Closed-form \\ (\cref{item:req-closed-form})}}
        & \rot{\shortstack[l]{Real world \\ (\cref{item:req-real-world})}}
        & \rot{\shortstack[l]{Marginal \\ (\cref{item:req-marginal})}}
        & \rot{\shortstack[l]{FOV \\ (\cref{item:req-fov})}}
        & \rot{\shortstack[l]{Occlusion \\ (\cref{item:req-occlusion})}}
        \\
        \midrule
        \nameref{sssec:design-objective-os}
        & \cmark & \xmark & \cmark & \xmark & \cmark & \cmark \\
        \hyperref[sssec:design-objective-ocu-ocl]{\algOCL{}}
        & \xmark & \xmark & \cmark & \xmark & \cmark & \cmark \\
        \sepspace
        \nameref{sssec:design-objective-ioa}
        & \cmark & \xmark & \cmark & \cmark & \cmark & \cmark \\
        \nameref{sssec:design-objective-i}
        & \cmark & \omark & \cmark & \cmark & \cmark & \xmark \\
        \sepspace
        \nameref{sssec:design-objective-c}
        & \cmark & \xmark & \cmark & \cmark & \xmark & \cmark \\
        \nameref{sssec:design-objective-cs}
        & \cmark & \cmark & \cmark & \cmark & \xmark & \xmark \\
        \nameref{sssec:design-objective-csp}
        & \xmark & \cmark & \xmark & \cmark & \xmark & \xmark \\
        \nameref{sssec:design-objective-csw}
        & \omark & \omark & \cmark & \cmark & \cmark & \xmark \\
        \sepspace
        \nameref{sssec:design-objective-u}
        & \cmark & \cmark & \cmark & \cmark & \xmark & \xmark \\
        \nameref{sssec:design-objective-up}
        & \xmark & \cmark & \xmark & \cmark & \xmark & \xmark \\
        \bottomrule
    \end{tabular}
    \caption{Overview of Requirements and Objective Functions.}
    \label{tab:design-objective-requirements}
\end{table}


\section{Design of Algorithms}\label{sec:design-algorithm}

In this section, we focus on how to find the NBV estimate based on the above designed objective functions. A natural way is a greedy algorithm described in \cref{ssec:design-algorithm-greedy}. We address an issue of this design in combination with a certain type of objective functions and present the design of two-phase algorithms in \cref{ssec:design-algorithm-twophase}, which tries to combine the strengths of different objective functions.

\subsection{Greedy Algorithm Design}\label{ssec:design-algorithm-greedy}

As discussed in \cref{ssec:problem-general-near-optimality}, the greedy decision in \cref{eq:greedy-decision}, which maximizes the true objective function \(F(\theta\mid\Theta)\), forms our theoretical baseline for near-optimality. Hence, a natural and simple design of an algorithm is a greedy algorithm which finds the maximizer of one of our designed objective functions \(\Fu(\theta\mid\Theta)\) and returns it as the NBV estimate \(\theta_t\).

\begin{definition}[Greedy Algorithm]\label{def:greedy-algorithm}
    Given an objective function \(\Fu\), the greedy algorithm \(\A(\cdot;\Fu)\) finds the global maximizer
    \begin{equation*}
        \A(\Theta;\Fu) \defeq \argmax_{\theta\in\Camspace} \Fu(\theta\mid\Theta)
    \end{equation*}
    for a given set of previous camera poses \(\Theta\).
    We use \(\A[X](\Theta)\) to denote the greedy algorithm with respect to objective \(\Fu[X]\).
\end{definition}

In \cref{remark:lemma-2-2-assumption}, we showed that such a greedy algorithm together with a necessary upper bound \(\Fu\) (see \cref{item:req-necessary-bound}) satisfies the assumption of \cref{lem:lemma-2-2}, which is required for our theoretical analysis.

\subsection{Two-phase Algorithm Design}\label{ssec:design-algorithm-twophase}

The problem of the greedy algorithm design in combination with objective functions, which excessively overestimate the visible confidence region, is that it leads to inaccurate NBV estimates and prevents the algorithm from convergence. Confidence-based objective functions described in \cref{ssec:design-objective-confidence} are an example for such objective functions and their deficient behavior is explained later in \cref{ssec:experiments-results-deficiency}. However, the idea behind their design is not completely wrong. Since the camera does not only observe a single surface point, but multiple ones in each round, the NBV estimate should not only maximize the uncertainty at the current location \(\theta_t\), but also at all surrounding locations in some interval \(\Phi(\theta_t)\). The confidence-based objective functions precisely do this by summing over the total uncertainty in such an interval. This provides a NBV estimate close to a large, but not necessarily observable area of uncertainty.

On the other hand, uncertainty-based objective functions described in \cref{ssec:design-objective-uncertainty} maximize the uncertainty only at the current location, which is guaranteed to be visible. In general, this does not correlate with a large close-by area of uncertainty as visualized in \cref{fig:design-algorithm-confidence-vs-uncertainty}.

\begin{figure}
    \centering
    \includegraphics[width=0.7\linewidth]{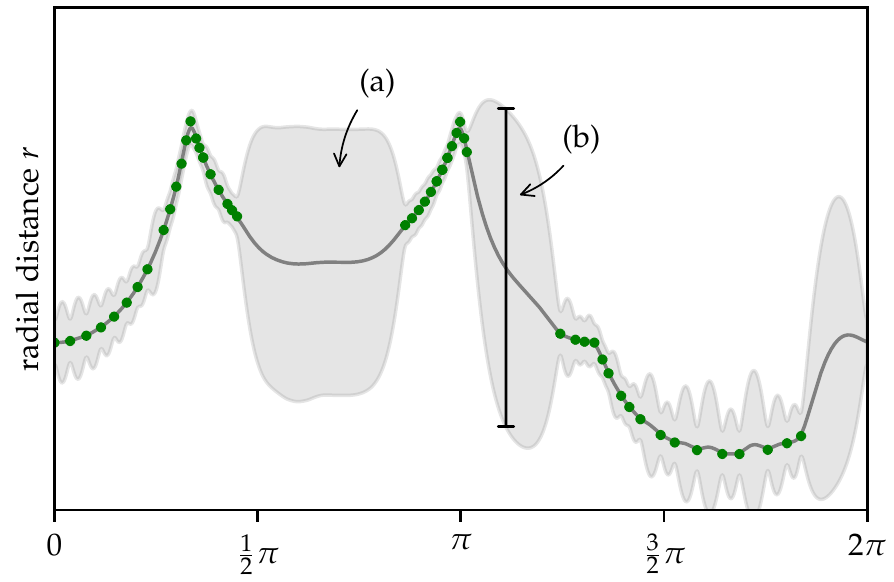}
    \caption[Confidence-based vs.\ Uncertainty-based Objective Functions]{
        Confidence-based vs.\ Uncertainty-based Objective Functions.
        This figure visualizes the confidence region (gray) after making measurements (green) of the object surface function. Observe that the location with the largest surrounding area of uncertainty (a), which is maximized by confidence-based objective functions, does not correspond to the location with largest current uncertainty (b), which is maximized by uncertainty-based objective functions.
    }
    \label{fig:design-algorithm-confidence-vs-uncertainty}
\end{figure}

To combine the strengths of both types, one can first use a confidence-based objective function to find a large area of uncertainty, which does not have to be completely visible, and then use a uncertainty-based objective function to find a location with largest uncertainty inside this area, which is visible. This motivates the design of two-phase algorithms, a variant of the single-phase greedy algorithms. The general idea is to find some interval \(\Phi^{(1)}\) in phase 1, which maximizes some objective \(\Fu[1]\), and then to run the greedy algorithm in phase 2 only on \(\Phi^{(1)}\) to maximize some other objective \(\Fu[2]\).

\begin{definition}[Two-phase Algorithm]\label{def:twophase-algorithm}
    Given an area-based objective function \(\Fu[1]\) defined over a summation interval \(\Phi^{(1)}(\theta)\) and an objective function \(\Fu[2]\), the two-phase algorithm \(\A(\Theta;\Fu[1],\Fu[2])\) returns
    \begin{equation*}
        \A(\Theta;\Fu[1],\Fu[2])
        \defeq \argmax_{\theta\in\Phi^{(1)}\paren[\Big]{\argmax\limits_{\theta\in\Camspace} \Fu[1](\theta\mid\Theta)}} \Fu[2](\theta\mid\Theta)
    \end{equation*}
    for a given set of previous camera poses \(\Theta\).
    We use \(\A[X\hyphen Y](\Theta)\) to denote the two-phase algorithm with respect to objective \(\Fu[X]\) in phase 1 and objective \(\Fu[Y]\) in phase 2.
\end{definition}

From the definition, one can easily see the difference to \cref{def:greedy-algorithm}, which consists only of the domain over which the phase 2 objective function is maximized to find the NBV estimate.

The algorithm \(\A[\TCSU]\) captures our previous idea and uses \nameref{sssec:design-objective-cs} in phase 1 to find an area with large uncertainty and \nameref{sssec:design-objective-u} in phase 2 to find the location with maximum uncertainty.


\section{Summary}\label{sec:design-summary}

We finish this chapter with a summary on the final design choices for our Gaussian process model, objective functions and algorithms, and we present the candidates for the analysis in \cref{chp:analysis}.

For the Gaussian process model, we use the mean function from \cref{eq:mean-function} and for the covariance function the periodic \matern{} kernels by periodic summation \(\kpsuminf[M_\nu](r)\) with \(\nu=n+\frac{1}{2},n\in\Natural\) from \cref{eq:periodic-matern-kernel-infinite-sum}. The choice of these kernel functions allows us to use the existing closed-form expressions and theoretical properties provided by \textcite{borovitskiy2020matern}.

In combination with the greedy algorithm design from \cref{def:greedy-algorithm}, we favor the objective functions \nameref{sssec:design-objective-up}, \nameref{sssec:design-objective-u} and \nameref{sssec:design-objective-cs} for their simplicity. We first analyze \(\A[UP]\) despite its objective function not being a necessary upper bound (see \cref{item:req-necessary-bound}). The reason is that we can extend a large part of the theoretical results to the analysis of \(\A[U]\) and \(\A[CS]\).

Together with the two-phase algorithm design from \cref{def:twophase-algorithm}, we analyze \(\A[\TCSU]\) based on \nameref{sssec:design-objective-cs} in phase 1 and \nameref{sssec:design-objective-u} in phase 2. This allows us to compare it with the previous analysis of the individual objective functions combined with the greedy algorithm design.

%% file: 06_analysis.tex
\chapter{Theoretical Analysis}\label{chp:analysis}

In this chapter, the goal is to show sublinear regret for the final candidates described in \cref{sec:design-summary}. The structure of this analysis is illustrated in \cref{fig:analysis-overview}.

In \cref{sec:analysis-tools} we present results, which are applicable to all candidates and help us to show sublinear regret. They include the choice of the confidence parameter \(\beta_t\) to ensure the validity of the confidence bounds and relating the measured uncertainties with an information-theoretic quantity.

In \cref{sec:analysis-greedy-up} we proceed with the proof for \(\A[UP]\). Despite this candidate not fulfilling the necessary requirement \cref{item:req-necessary-bound}, which prevents us from showing sublinear regret for \(\A[UP]\), it provides us with a framework for the following candidates, which are closely related to \(\A[UP]\).

In the \cref{sec:analysis-greedy-u,sec:analysis-greedy-cs,sec:analysis-twophase-csu} we show sublinear regret for \(\A[U],\A[CS]\) and \(\A[\TCSU]\) and summarize our findings in \cref{sec:analysis-summary}.

\begin{figure}
    \newtcolorbox{boxcontainer}[3][]{
        boxstyle=#3,
        title={#2},
        fonttitle=\scriptsize,
        #1
    }
    \newtcbox{\circlecontainer}[1][]{
        circlestyle=gray,
        width=0.7cm, tcbox width=forced center,
        #1
    }
    \centering
    \begin{tikzpicture}[
        every node/.style={inner sep=1.5mm, outer sep=0mm},
        circlenode/.style={inner sep=0mm, circle},
        >/.tip=Stealth,
        node distance=0cm,
        font=\tiny,
    ]
        \node (theorem-2) {
            \begin{boxcontainer}[width=2.6cm]{Theorem 6.1 *}{green}
                sublinear regret \\ \(R(T) \leq \bigO*{T^n}, n<1\)
            \end{boxcontainer}
        };
        \node (text-combinetheorem) [above left=-0.45cm and 0.1cm of theorem-2.west, font=\scriptsize] {
            \&
        };
        \node (theorem-1) [left=0.75cm of theorem-2] {
            \begin{boxcontainer}[width=3cm]{\cref{thm:near-optimality}}{blue}
                sublinear regret \(\implies\) near-optimality
            \end{boxcontainer}
        };
        \node (text-nearoptimality) [above right=-0.45cm and 0.1cm of theorem-2.east, font=\scriptsize] {
            \(\implies\) near-optimality
        };

        \node (lemma-2-3) [below right=0.75cm and 0cm of theorem-2.south west] {
            \begin{boxcontainer}[width=2.6cm]{Lemma 6.4 *}{green}
                 \(\sum \Fu \leq \bigO*{\sqrt{T\beta_T\gamma_T}}\)
            \end{boxcontainer}
        };
        
        \node (text-equals2) [below left=0.73cm and -0.27cm of lemma-2-3.north west, font=\scriptsize] {
            \(=\)
        };
        \node (lemma-2-2) [below left=of lemma-2-3.north west] {
            \begin{boxcontainer}[width=2.2cm]{\cref{lem:lemma-2-2}}{blue}
                 \(R_{ind}(T) \leq \sum \Fu\)
                 \tcblower
                 assuming \\ \(F(\thetagre_t \vert \cdot) \leq \Fu(\theta_t \vert \cdot)\)
            \end{boxcontainer}
        };
        \node (text-equals1) [below left=0.71cm and -0.27cm of lemma-2-2.north west, font=\scriptsize] {
            \(=\)
        };
        \node (lemma-2-1) [below left=of lemma-2-2.north west] {
            \begin{boxcontainer}[width=1.9cm]{\cref{lem:lemma-2-1}}{blue}
                \(R(T) < R_{ind}(T)\)
            \end{boxcontainer}
        };
        
        \node (lemma-2-4b) [below right=of lemma-2-3.north east] {
            \begin{boxcontainer}[width=2cm]{\cref{lem:lemma-2-4}}{ForestGreen}
                 \(\beta_T \leq \bigO*{\log T}\)
            \end{boxcontainer}
        };
        \node (lemma-2-6) [below right=of lemma-2-4b.north east] {
            \begin{boxcontainer}[width=2.7cm]{\cref{lem:lemma-2-6}}{ForestGreen}
                \(\gamma_T \leq \bigO*{T^\alpha\log(T)^{1-\alpha}}\)
            \end{boxcontainer}
        };
        \node[circlenode] (design-kernel) [below=1cm of lemma-2-6] {
            \circlecontainer{\(k\)}
        };

        \node[circlenode] (design-objective) [below=1cm of lemma-2-3] {
            \circlecontainer{\(\Fu\)}
        };
        \node[circlenode] (design-algorithm) [left=0.1cm of design-objective] {
            \circlecontainer{\(\A\)}
        };
        \node (lemma-2-5) [right=0.1cm of design-objective] {
            \begin{boxcontainer}[width=2.8cm]{\cref{lem:lemma-2-5}}{ForestGreen}
                 \(\sum \sigma(X_t)^2 \leq I(Y_{1:T};f_{1:T})\)
            \end{boxcontainer}
        };
        \node (lemma-2-4a) [below=0.5cm of design-objective] {
            \begin{boxcontainer}[width=2.6cm]{\cref{lem:lemma-2-4}}{ForestGreen}
                 choice for \(\beta_t\)
            \end{boxcontainer}
        };

        \node (text-general) [below right=2.7cm and 0cm of lemma-2-1.south west] {
            \begin{boxcontainer}[width=3cm]{}{blue}
                results for \\ general setting
            \end{boxcontainer}
        };
        \node (text-simplified) [below left=2.7cm and 0cm of lemma-2-6.south east] {
            \begin{boxcontainer}[width=3cm]{}{green}
                tools \& results for \\ simplified 2D setting
            \end{boxcontainer}
        };
        \node (text-star) [above left=-0.1cm and 0cm of text-simplified.north east, align=left] {
            * results specific to \\ \hphantom{* }choice of \(\A\) and \(\Fu\)
        };

        \draw[<-] (theorem-2) -- (lemma-2-3);
        \draw[<-] (theorem-2) -- (lemma-2-4b);
        \draw[<-] (theorem-2) -- (lemma-2-6.north west);
        
        \draw[dotted] (lemma-2-2) -- (design-algorithm);
        \draw[dotted] (lemma-2-2) -- (design-objective);

        \draw[<-] (lemma-2-3) -- (design-algorithm);
        \draw[<-] (lemma-2-3) -- (design-objective);
        \draw[<-] (lemma-2-3) -- (lemma-2-5);

        \draw[<-] (design-objective) -- (lemma-2-4a);
        
        \draw[<-] (lemma-2-6) -- (design-kernel);
        \draw[dotted] (lemma-2-4b) -- (design-kernel);
    \end{tikzpicture}
    \caption[Overview of the Theoretical Analysis]{
        Overview of the Theoretical Analysis.
        Initially, \cref{thm:near-optimality} shows for the general setting that pseudo-convergence to near-optimality follows from sublinear regret.
        Hence, we want to show sublinear regret for each of our algorithm candidates with the final results given in \cref{thm:no-sublinear-greedy-up,,thm:sublinear-greedy-u,,thm:sublinear-greedy-cs,,thm:sublinear-twophase-csu}. On our way towards sublinear regret, we use the \cref{lem:lemma-2-1,lem:lemma-2-2} from the general setting and the corresponding \cref{lem:lemma-2-3-greedy-up,lem:lemma-2-3-greedy-u,lem:lemma-2-3-greedy-cs,lem:lemma-2-3-twophase-csu} to upper bound the cumulative regret with \(\bigO*{\sqrt{T\beta_T\gamma_T}}\). Together with \cref{lem:lemma-2-4,lem:lemma-2-6}, we can show sublinear regret for most candidates depending on the design choices for \(\A\), \(\Fu\) and \(k\), which then implies near-optimality.
    }
    \label{fig:analysis-overview}
\end{figure}


\section{Tools for the Analysis}\label{sec:analysis-tools}

Before starting the analysis, we present some additional results which help us in showing sublinear regret and are applicable to all algorithm candidates. In \cref{ssec:analysis-tools-confidence} we present the specific choice for the confidence parameter \(\beta_t\) which ensures that the unknown surface function lies between the upper and lower confidence bound with high probability. In \cref{ssec:analysis-tools-information} we define an additional information theoretic quantity to describe the maximum information gain through each measurement, which allows us to derive the final sublinear bound on the regret.

\subsection{Choice of Confidence Parameter}\label{ssec:analysis-tools-confidence}
The ideal goal is to show that the surface function \(f\) deterministically lies between the upper and lower confidence bounds \(u_t\) and \(l_t\) by choosing a suitable confidence parameter \(\beta_t\) as described in \cref{eq:confidence-bounds-ideal}. However, since \(f\) is sampled from \(\GP*{m,k}\), there can always be outlier functions which escape the confidence bounds locally at some \(\varphi\) with very small probability. Hence, we can only show that \(f\) lies between \(u_t\) and \(l_t\) with probability of at least \(1-\delta\) as written in \cref{eq:confidence-region-ideal}, where the \emph{failure probability} \(\delta \in (0,1)\) can be chosen arbitrarily low.

The following lemma is obtained from \textcite[Theorem 2]{srinivas2012informationtheoretic} and adapted to our setting with domain \(\Domain = [0,2\pi]\).

\begin{lemma}[Confidence Parameter]\label{lem:lemma-2-4}
    Let \(k(r)\) be a stationary kernel defined on \(\Domain \subseteq \Real\). In addition, assume that the derivatives of \(f\sim\GP*{m,k}\) are bounded with probability
    \begin{equation*}
        \Pr*{\sup_{\varphi\in\Domain} \abs*{\deriv{f}{\varphi}(\varphi)} \leq L} \geq 1 - ae^{-L^2/b^2}
        \quad\text{ for some } a,b > 0
        .
    \end{equation*}
    By choosing the confidence parameter
    \begin{equation*}
        \beta_t = 2\log*{\frac{2\pi^3b}{3}\sqrt{\log*{\frac{2a}{\delta}}}\frac{t^4}{\delta}} \leq \bigO*{\log*{\frac{t}{\delta}}}
    \end{equation*}
    for an arbitrary small \(\delta\in(0,1)\), we can show the following high probability bound
    \begin{equation*}
        \Pr*{
            \forall t \geq 1
            \forall \varphi \in \Domain\colon
            \abs*{f(\varphi)-\mu_{t-1}([\varphi]_t)} \leq \beta_t^{1/2}\sigma_{t-1}([\varphi]_t) + \frac{1}{t^2}
        } \geq 1-\delta
        .
    \end{equation*}
    with \([\varphi]_t\) defined as the closest point to \(\varphi\) inside a uniform discretization \(\Domain_t\) of size \(\abs{\Domain_t} = 2\pi b\sqrt{\log*{\frac{2a}{b}}}\cdot t^2\).
\end{lemma}
\begin{proof}
    \cref{ssec:proofs-lemma-2-4}
    \noqed
\end{proof}

This lemma can be rewritten into
\begin{equation}\label{eq:confidence-region-actual}
    l_t([\varphi]_t) \leq f(\varphi) \leq u_t([\varphi]_t)
    \quad\text{ for all \(\varphi \in \Domain, t \geq 1\) \whp{}}
\end{equation}
with the following refined upper and lower confidence bounds
\begin{equation}\label{eq:confidence-bounds-actual}
    \begin{aligned}
        u_t(\varphi) &= \mu_{t-1}(\varphi) + \beta_t^{1/2}\sigma_{t-1}(\varphi) + \frac{1}{t^2}\\
        l_t(\varphi) &= \mu_{t-1}(\varphi) - \beta_t^{1/2}\sigma_{t-1}(\varphi) - \frac{1}{t^2}
        .
    \end{aligned}
\end{equation}
The difference to our initial version in \cref{eq:confidence-region-ideal,eq:confidence-bounds-ideal} is that we bound the unknown surface function \(f\) only with discretized upper and lower confidence bounds based on the discretized domain \(\Domain_t\). The additional uncertainty of \(\frac{1}{t^2}\) ensures that \(f(\varphi)\) also lies between the discretized confidence bounds at the remaining  points \(\varphi\notin\Domain_t\).

Intuitively, showing that \(f(\varphi)\) is bounded for all \(\varphi\in\Domain\) corresponds to infinitely many probabilistic statements, one for each \(\varphi\). Since each of these statements only hold with probability less than 1, it is impossible to lower bound the probability of all these statements together with a nonzero probability. This becomes possible with the union bound if we show the statement only for a finite set of points \(\Domain_t\).

The additional uncertainty of \(\frac{1}{t^2}\) due to the discretization depends on the assumption that the derivatives of functions sampled from \(\GP*{m,k}\) are bounded with high probability (\ie{} probabilistically Lipschitz-continuous). This allows us to prevent sample functions from escaping the confidence bounds between the discretization point with high probability. By increasing the discretization granularity with \(\abs{\Domain_t}\leq\bigO*{t^2}\) in each round, this additional uncertainty of \(\frac{1}{t^2}\) due to the discretization error shrinks towards zero. The formal reasoning can be found in the proof in \cref{ssec:proofs-lemma-2-4}.

As a final remark regarding the bounded derivatives assumption, we want to point out that this is satisfied for stationary Gaussian processes with four times differentiable kernel functions as stated by \textcite{srinivas2012informationtheoretic}. This guarantees that the sample functions \(f\) are almost surely continuously differentiable and the corresponding derivatives \(\deriv{f}{\varphi}\) are distributed with a Gaussian process again, which provides the exponentially decreasing probability bound for derivatives larger than \(L\) \autocite[Theorem 5]{ghosal2006posterior}. It is known that RBF and \matern{} kernels with \(\nu > 2\) have derivatives up to fourth order \autocite{stein1999interpolation}.

\subsection{Relation between Uncertainty and Information Gain}\label{ssec:analysis-tools-information}

In \cref{ssec:background-infotheory-infogain}, we defined the \emph{information gain} \(I(X;Y)\) as the amount of information gained about the random variable \(Y\) by observing \(X\) or vice versa due to symmetry. The ``information amount'' can also be understood as the amount of uncertainty contained in \(Y\) which is removed by observing \(X\) or vice versa. The precise notion is described in \cref{ssec:background-infotheory-entropy}.

In our setting, we are interested in \(I(Y_{1:t};f_{1:t})\), which quantifies the information gained about the unknown surface function \(f_{1:t}\) through the noisy measurements \(Y_{1:t}\). The following lemma obtained from \textcite[Lemma 5-7]{prajapat2022nearoptimal} relates the sum of all uncertainties \(\sigma_{t-1}(X_t)\) at the measured surface points \(X_{1:t}\) with this quantity.

\begin{lemma}[Uncertainty and Information Gain]\label{lem:lemma-2-5}
    Let \(X_t\) describe the observed surface points and \(Y_t\) the corresponding measurements at these points as defined in \cref{eq:observed-surface,eq:measured-surface}. Assume the uncertainties \(\sigma_{t-1}(\varphi)\) are obtained from \(\GP*{m,k}\) with a kernel function satisfying
    \begin{equation*}
        \abs{k(\varphi,\varphi')}\leq1
        \quad\text{ for all } \varphi,\varphi'\in\Domain
        .
    \end{equation*}
    \begin{enumerate}[label=\alph*),ref=\cref*{lem:lemma-2-5}\alph*]
        \item \label{lem:lemma-2-5a} If \(X_t = o(\theta_t)\) in each round, then the sum of all uncertainties at \(X_{1:T}\) can be upper bounded with
        \begin{equation*}
            \frac{1}{2} \sum_{t=1}^T \sum_{i=1}^{n_t} \sigma_{t-1}(X_{t,i})^2
            \leq \frac{N_T}{\log*{\sigmaeps^{-2}+1}} \cdot I(Y_{1:T};f_{1:T})
        \end{equation*}
        with \(\abs{X_{1:T}} = \sum_{t=1}^T n_t\) and \(N_T \defeq \max_{t=1,\dots,T} n_t\).
        \item \label{lem:lemma-2-5b} If \(X_t = \set{\theta_t}\) in each round, then the sum of all uncertainties at \(X_{1:T}\) can be upper bounded with
        \begin{equation*}
            \frac{1}{2} \sum_{t=1}^T \sigma_{t-1}(\theta_t)^2
            \leq \frac{1}{\log*{\sigmaeps^{-2}+1}} \cdot I(\fm(\theta_{1:T});f(\theta_{1:T}))
        \end{equation*}
        with \(\abs{X_{1:T}} = T\).
    \end{enumerate}
\end{lemma}
\begin{proof}
    \cref{ssec:proofs-lemma-2-5}
    \noqed
\end{proof}

Note that the general statement \cref{lem:lemma-2-5a} does not depend on the specific choice for \(X_t\) and can be instantiated with any set of measurement locations. This is how we obtained \cref{lem:lemma-2-5b} by assuming a camera model which only measures the single surface point \(\theta_t\) at the line of sight (LOS) instead of all visible surface points \(o(\theta_t)\) inside the FOV of the camera. This statement becomes useful when analyzing the uncertainty-based objective functions described in \cref{ssec:design-objective-uncertainty}.

We also want to remark that we use the inequality provided by \cref{lem:lemma-2-5} only as a mathematical tool to show sublinear regret. The set of measurement locations \(X_t\) might not exactly match the true set of measurement locations in practice and can be just an estimate of it. Hence, \(I(Y_{1:t};f_{1:t})\) does not necessarily reflect the real information gain of our algorithm either, but is similarly just an estimate.

The importance of this lemma comes from connecting both ends of the overall proof towards sublinear regret. On the one side, we start with \(R(T)\), which can be related to the measured uncertainties \(\sigma_{t-1}(X_t)\) through \cref{lem:lemma-2-1,lem:lemma-2-2} and the choice of the objective function \(\Fu\). On the other side, we can derive a sublinear regret bound based on the maximum information gain as we show next in \cref{lem:lemma-2-6}. The relation between the measured uncertainties and the information gain is established by this \cref{lem:lemma-2-5}.

The maximum amount of information one can gain about an unknown function by measuring it at \(T\) locations is commonly defined as the \emph{information capacity} \(\gamma_T\).

\begin{definition}[Information Capacity]\label{def:information-capacity}
    Let \(f\) be a function sampled from \(\GP*{m,k}\) and \(\fm\) the corresponding noisy measurement function as defined in \cref{eq:measurement-function}. The information capacity is defined as
    \begin{equation*}
        \gamma_T
        \defeq \sup_{\Phi\subseteq\Domain,\abs{\Phi}=T} I(\fm(\Phi);f(\Phi))
    \end{equation*}
    with \(T\geq1\).
\end{definition}

As part of \cref{lem:lemma-2-5}, we showed that
\begin{equation*}
    I(\fm(\Phi);f(\Phi)) = \frac{1}{2} \sum_{t=1}^T \log\det*{I_T + \sigmaeps^{-2}K_\varphi}
\end{equation*}
with \(K_\Phi \defeq \brack{k(\varphi,\varphi')}_{\varphi,\varphi'\in\Phi}\). By defining the information capacity as the maximum over the information gain, the dependence on the set of measurement locations \(\Phi\) is removed and \(\gamma_T\) becomes a kernel-specific quantity. This sounds intuitive, since the kernel quantifies the similarity between different surface points and therefore influences the amount of information one can obtain by measuring these surface points.

Previous work already derived various upper bounds on \(\gamma_T\) \autocite{srinivas2012informationtheoretic, vakili2021information}, but they instantiate their results only for the non-periodic RBF and \matern{} kernel. Since we require periodic kernels, we derive an upper bound for the periodic \matern{} kernel \(\kpsuminf[M_\nu]\) based on the work of \textcite{borovitskiy2020matern, vakili2021information}.

\begin{lemma}[Bound on Information Gain]\label{lem:lemma-2-6}
    Let \(\gamma_T\) be the information capacity from \cref{def:information-capacity} defined with respect to the periodic \matern{} kernel by periodic summation \(\kpsuminf[M_\nu]\) with \(\nu=n+\frac{1}{2},n\in\Natural\) from \cref{eq:periodic-matern-kernel-infinite-sum}. Assume
    \begin{equation*}
        \abs{\kpsuminf[M_\nu](r)} \leq 1
        \quad\text{ for all } r\in\Real
        .
    \end{equation*}
    Then the information capacity can be upper bounded with
    \begin{equation*}
        \gamma_T \leq \bigO*{T^\frac{1}{2\nu+1}\log(T)^{\frac{2\nu}{2\nu+1}}}
        .
    \end{equation*}
\end{lemma}
\begin{proof}
    \cref{ssec:proofs-lemma-2-6}
    \noqed
\end{proof}

In particular, we obtain
\begin{equation*}
    \begin{aligned}
        \gamma_T &\leq \bigO*{T^\ffrac{1}{2}\log(T)^{\ffrac{1}{2}}}
        &&\quad\with \nu = \ffrac{1}{2}\\
        \gamma_T &\leq \bigO*{T^\ffrac{1}{4}\log(T)^{\ffrac{3}{4}}}
        &&\quad\with \nu = \ffrac{3}{2}\\
        \gamma_T &\leq \bigO*{T^\ffrac{1}{6}\log(T)^{\ffrac{5}{6}}}
        &&\quad\with \nu = \ffrac{5}{2}
        .
    \end{aligned}
\end{equation*}
As one can see, the smoother the Gaussian process is with larger \(\nu\), the slower the maximum information gain \(\gamma_T\) increases with the number of measurements \(T\).
The intuition is that a smoother Gaussian process leads to smoother sample functions, for which most information about the function can be already obtained with the first few measurements, while more measurements do not provide significantly more information.
When reducing the smoothness parameter \(\nu\), the sample functions become rougher and more measurements still allow us to gain new information about them.

Since \cref{lem:lemma-2-4} requires \matern{} kernel with \(\nu>2\) and \cref{lem:lemma-2-6} provides bounds on the information gain only for \(\kpsuminf[M_\nu]\) with \(\nu=n+\frac{1}{2},n\in\Natural\), we continue with the kernel \(\kpsuminf[M_\nu]\) with \(\nu=\frac{5}{2}\) as our selected kernel function for the Gaussian process in the following sections.%
\footnote{However, the choices \(\nu=\frac{1}{2}\) and \(\nu=\frac{3}{2}\) still remain of practical relevance when dealing with objects with less smooth surface functions.}
Any larger \(\nu\) would unnecessarily restrict the smoothness of considered surface functions.


\section{\algGUP{}}\label{sec:analysis-greedy-up}

This section presents the results specific to the greedy algorithm \(\A[UP]\) based on the \nameref{sssec:design-objective-up} objective function.

Based on the result from \cref{lem:lemma-2-4} the upper and lower confidence bounds used by the objective function must be adapted to \cref{eq:confidence-bounds-actual} to ensure that the surface function lies inside the confidence region with probability \(1-\delta\), which is assumed by the objective function.%
\footnote{Although \nameref{sssec:design-objective-up} does not provide a necessary upper bound as required by \cref{item:req-necessary-bound} and this adaption does not matter to the overall result, we still proceed in preparation for the \nameref{sssec:design-objective-u} objective function.}

\begin{lemmavar}{UP}\label{lem:lemma-2-3-greedy-up}
    Consider surface functions sampled from a Gaussian process \(\GP*{m,k}\). Assume that the assumptions for \cref{lem:lemma-2-4,lem:lemma-2-5} are satisfied. Let \(u_t(\varphi)\) and \(l_t(\varphi)\) be defined according to \cref{eq:confidence-bounds-actual} with confidence parameter \(\beta_t\) chosen as in \cref{lem:lemma-2-4}.
    
    Consider the greedy algorithm \(\A[UP]\) from \cref{def:greedy-algorithm} with the objective function \nameref{sssec:design-objective-up}.
    Let \(\theta_{1:T}\) be arbitrary.
    Then we can show
    \begin{equation*}
        \sum_{t=1}^T \Fu[UP]\paren*{\theta_t \mid \theta_{1:t-1}}
        \leq C_1 \sqrt{T\beta_T\gamma_T} + C_2
    \end{equation*}
    with \(C_1 = \frac{2\abs*{\PhiS}}{h^2}\sqrt{\frac{3}{\log[]{\sigmaeps^{-2}+1}}}\) and \(C_2 = \frac{\abs{\PhiS}}{h^2}\frac{\pi}{\sqrt{3}}\).
\end{lemmavar}
\begin{proof}
    \cref{ssec:proofs-lemma-2-3-greedy-up}
    \noqed
\end{proof}

Combining all previous results as visualized in the overview in \cref{fig:analysis-overview}, we arrive at the final theorem for \(\A[UP]\).

\begin{theoremvar}[Sublinear Regret?]{UP}\label{thm:no-sublinear-greedy-up}
    Consider surface functions sampled from a Gaussian process with kernel function \(\kpsuminf[M_\nu]\) and \(\nu=\frac{5}{2},\sigma_f=1\) as defined in \cref{eq:periodic-matern-kernel-infinite-sum-examples}. Choose some failure probability \(\delta\in(0,1)\) and the confidence bounds according to \cref{eq:confidence-bounds-actual} with confidence parameter \(\beta_t\) as in \cref{lem:lemma-2-4}.
    
    Consider the greedy algorithm \(\A[UP]\) from \cref{def:greedy-algorithm} based on the objective function \nameref{sssec:design-objective-up} and the sequence of NBV estimates \(\theta[UP]_t \defeq \A[UP](\theta_{1:t-1})\) with \(t=1,\dots,T\).
    This does \textbf{not} suffice to show
    \begin{equation*}
        R(T) \leq \bigO*{T^{\frac{2\nu+2}{4\nu+2}}\log(T)^{\frac{4\nu+1}{4\nu+2}}}
        .
    \end{equation*}
\end{theoremvar}
\begin{proof}
    \cref{ssec:proofs-theorem-no-sublinear-greedy-up}
    \noqed
\end{proof}

If we were able to show sublinear regret for \(\A[UP]\), we would obtain
\begin{equation*}
    \begin{aligned}
        R(T) &\leq \bigO*{T^\ffrac{3}{4}\log(T)^{\ffrac{3}{4}}}
        &&\quad\with \nu = \ffrac{1}{2}\\
        R(T) &\leq \bigO*{T^\ffrac{5}{8}\log(T)^{\ffrac{7}{8}}}
        &&\quad\with \nu = \ffrac{3}{2}\\
        R(T) &\leq \bigO*{T^\ffrac{7}{12}\log(T)^{\ffrac{11}{12}}}
        &&\quad\with \nu = \ffrac{5}{2}
        .
    \end{aligned}
\end{equation*}


\section{\algGU{}}\label{sec:analysis-greedy-u}


This section presents the results specific to the greedy algorithm \(\A[U]\) based on the \nameref{sssec:design-objective-u} objective function.

Similar to \cref{sec:analysis-greedy-up}, we adapt the upper and lower confidence bounds used by the objective function to \cref{eq:confidence-bounds-actual} based on the result from \cref{lem:lemma-2-4}.

Note that the only difference to \nameref{sssec:design-objective-up} is that \nameref{sssec:design-objective-u} depends on \(\frac{1}{2} (u_t(\theta)^2 - l_t(\theta)^2)\) instead of \(u_t(\theta) - l_t(\theta)\) to take the circle sector area in the real world into account (see \cref{item:req-real-world}). This leads to
\begin{equation*}
    \begin{aligned}
        \Fu[U](\theta_t\mid\theta_{1:t-1})
        &= \mu_{t-1}(\theta_t) \cdot \Fu[UP](\theta_t\mid\theta_{1:t-1})
        \\ &\leq \dmax \cdot \Fu[UP](\theta_t\mid\theta_{1:t-1})
    \end{aligned}
\end{equation*}
for arbitrary \(\theta_{1:t}\) as shown in the proof for the following \cref{lem:lemma-2-3-greedy-u}. The additional factor \(\mu_{t-1}(\theta)\) introduced in the objective function can be constantly upper bounded with \(\dmax\) based on \cref{item:simp-object-bounded}. This allows us to almost directly derive the following sublinear regret bound:

\begin{lemmavar}{U}\label{lem:lemma-2-3-greedy-u}
    Consider surface functions sampled from a Gaussian process \(\GP*{m,k}\). Assume that the assumptions for \cref{lem:lemma-2-4,lem:lemma-2-5} are satisfied. Let \(u_t(\varphi)\) and \(l_t(\varphi)\) be defined according to \cref{eq:confidence-bounds-actual} with confidence parameter \(\beta_t\) chosen as in \cref{lem:lemma-2-4}.
    
    Consider the greedy algorithm \(\A[U]\) from \cref{def:greedy-algorithm} with the objective function \nameref{sssec:design-objective-u}.
    Let \(\theta_{1:T}\) be arbitrary.
    Then we can show
    \begin{equation*}
        \sum_{t=1}^T \Fu[U]\paren*{\theta_t \mid \theta_{1:t-1}}
        \leq C_1 \sqrt{T\beta_T\gamma_T} + C_2
    \end{equation*}
    with \(C_1 = \frac{2\dmax\abs*{\PhiS}}{h^2}\sqrt{\frac{3}{\log[]{\sigmaeps^{-2}+1}}}\) and \(C_2 = \frac{\dmax\abs{\PhiS}}{h^2}\frac{\pi}{\sqrt{3}}\).
\end{lemmavar}
\begin{proof}
    \cref{ssec:proofs-lemma-2-3-greedy-u}
    \noqed
\end{proof}

Combining all previous results as visualized in the overview in \cref{fig:analysis-overview}, we arrive at the final theorem for \(\A[U]\).

\begin{theoremvar}[Sublinear Regret]{U}\label{thm:sublinear-greedy-u}
    Consider surface functions sampled from a Gaussian process with kernel function \(\kpsuminf[M_\nu]\) and \(\nu=\frac{5}{2},\sigma_f=1\) as defined in \cref{eq:periodic-matern-kernel-infinite-sum-examples}. Choose some failure probability \(\delta\in(0,1)\) and the confidence bounds according to \cref{eq:confidence-bounds-actual} with confidence parameter \(\beta_t\) as in \cref{lem:lemma-2-4}.
    
    Consider the greedy algorithm \(\A[U]\) from \cref{def:greedy-algorithm} based on the objective function \nameref{sssec:design-objective-u} and the sequence of NBV estimates \(\theta[U]_t \defeq \A[U](\theta_{1:t-1})\) with \(t=1,\dots,T\).
    Then we can show
    \begin{equation*}
        R(T) \leq \bigO*{T^{\frac{2\nu+2}{4\nu+2}}\log(T)^{\frac{4\nu+1}{4\nu+2}}}
    \end{equation*}
    with probability at least \(1-\delta\).
\end{theoremvar}
\begin{proof}
    \cref{ssec:proofs-theorem-sublinear-greedy-u}
    \noqed
\end{proof}

Interestingly, the additional mean factor \(\mu_{t-1}(\theta)\) is contra-productive for our algorithm, since it rewards camera locations close to the object surface in practice.%
\footnote{The larger the mean at \(\theta\), the further away the surface function from the world center and the closer the surface function to the camera at \(\dcam\).}
Being closer to the surface mostly results into less observed surface points due to the shape of the FOV. The reason for this additional mean factor comes from counting the number of world pixels inside a range of polar angles \(\Delta\varphi\) instead of the number of polar pixels. Whereas the number of polar pixels stays constant with the distance to the world center, the number of world pixels becomes larger the further one moves away from the world center as previously visualized in \cref{fig:design-objective-uncertainty}. Later, we observe in our experiments that \nameref{sssec:design-objective-up} without the additional mean factor indeed performs better than \nameref{sssec:design-objective-u} as discussed in \cref{ssec:experiments-results-comparison}.

Without the information about the FOV shape (see \cref{item:req-fov}), the objective function assumes that all surface points within \(\Delta\varphi\) can be observed by the camera. Since \nameref{sssec:design-objective-u} only depends on the uncertainty at the current camera location \(\theta_t\), which corresponds to an infinitesimal small \(\Delta\varphi\), the assumption is satisfied and the lack of FOV information is not severe. Since \nameref{sssec:design-objective-cs} depends on the uncertainty inside a much larger range \(\Delta\varphi = \abs{\PhiS}\), we show in \cref{ssec:experiments-results-deficiency} how this results into bad performance with the lack of FOV information.


\section{\algGCS{}}\label{sec:analysis-greedy-cs}

This section presents the results specific to the greedy algorithm \(\A[CS]\) based on the \nameref{sssec:design-objective-cs} objective function.

Based on the result from \cref{lem:lemma-2-4} the objective function \nameref{sssec:design-objective-cs} must be adapted to ensure that the surface function lies between the used upper and lower confidence boundary with probability \(1-\delta\), which is assumed by the objective function. We redefine it as
\begin{equation}\label{eq:objective-function-cs-new}
    \Fu[CS](\theta\mid\theta_{1:t-1})
    \defeq \eqcontrast \frac{1}{h^2} \sum_{\varphi\in\eqchange{\brack[\big]{\PhiS(\theta)}_t}} \frac{1}{2}\paren*{u_t(\varphi)^2 - l_t(\varphi)^2} \eqchange{\frac{\abs{\PhiS}}{\abs[\big]{\brack[\big]{\PhiS}_t}}}
    \eqnocontrast
\end{equation}
with the summation interval restricted to
\begin{equation}\label{eq:sum-interval-simple-discretized}
    \brack[\big]{\PhiS(\theta)}_t \defeq \PhiS(\theta) \cap \Domain_t
    \quad\with \abs*{\brack[\big]{\PhiS}_t} = \frac{\abs{\PhiS}}{2\pi/\abs{\Domain_t}}
\end{equation}
based on the uniform discretization \(\Domain_t\) with granularity \(\frac{2\pi}{\abs{\Domain_t}}\) as defined in \cref{lem:lemma-2-4}.
The surface function is guaranteed to lie between the confidence boundaries at points in \(\brack[\big]{\PhiS(\theta)}\) with probability \(1-\delta\).

Note that the only difference to \nameref{sssec:design-objective-u} is that \nameref{sssec:design-objective-cs} as defined in \cref{eq:objective-function-cs-new} depends on \(\sum_{\varphi\in\brack{\PhiS(\theta)}_t} \paren*{u_t(\varphi)^2-l_t(\varphi)^2}\frac{\abs{\PhiS}}{\abs{\brack{\PhiS}_t}}\) instead of \(\abs{\PhiS}\paren*{u_t(\theta)^2 - l_t(\theta)^2}\) to also take the surrounding uncertainties at \(\brack{\PhiS(\theta)}_t\) into account (see \cref{eq:sum-interval-simple,eq:sum-interval-simple-discretized}). This leads to
\begin{equation*}
    \begin{aligned}
        \Fu[CS](\theta_t\mid\theta_{1:t-1})
        &\leq \max_{\theta\in\brack{\PhiS(\theta_t)}_t} \Fu[U](\theta\mid\theta_{1:t-1})
        \\ &\leq \max_{\substack{\theta\in\Domain \\ \hphantom{\theta\in\brack{\PhiS(\theta_t)}_t}}} \Fu[U](\theta\mid\theta_{1:t-1})
        \\ &= \Fu[U]\paren*{\theta[U]_t\mid\theta_{1:t-1}}
    \end{aligned}
\end{equation*}
for arbitrary \(\theta_{1:t}\) as shown in the proof for \cref{lem:lemma-2-3-greedy-cs}. The first inequality is obtained by upper bounding the sum with the maximum summation term and the second inequality by expanding the maximization from the summation interval to the whole domain.
We can immediately derive the same sublinear regret bound as in \cref{lem:lemma-2-3-greedy-u}.

\begin{lemmavar}{CS}\label{lem:lemma-2-3-greedy-cs}
    Consider surface functions sampled from a Gaussian process \(\GP*{m,k}\). Assume that the assumptions for \cref{lem:lemma-2-4,lem:lemma-2-5} are satisfied. Let \(u_t(\varphi)\) and \(l_t(\varphi)\) be defined according to \cref{eq:confidence-bounds-actual} with confidence parameter \(\beta_t\) chosen as in \cref{lem:lemma-2-4}.
    
    Consider the greedy algorithm \(\A[CS]\) from \cref{def:greedy-algorithm} with the objective function \nameref{sssec:design-objective-cs}.
    Let \(\theta_{1:T}\) be arbitrary.
    Then we can show
    \begin{equation*}
        \sum_{t=1}^T \Fu[CS]\paren*{\theta_t \mid \theta_{1:t-1}}
        \leq C_1 \sqrt{T\beta_T\gamma_T} + C_2
    \end{equation*}
    with \(C_1 = \frac{2\dmax\abs*{\PhiS}}{h^2}\sqrt{\frac{3}{\log[]{\sigmaeps^{-2}+1}}}\) and \(C_2 = \frac{\dmax\abs{\PhiS}}{h^2}\frac{\pi}{\sqrt{3}}\).
\end{lemmavar}
\begin{proof}
    \cref{ssec:proofs-lemma-2-3-greedy-cs}
    \noqed
\end{proof}

Combining all previous results as visualized in the overview in \cref{fig:analysis-overview}, we arrive at the final theorem for \(\A[CS]\).

\begin{theoremvar}[Sublinear Regret]{CS}\label{thm:sublinear-greedy-cs}
    Consider surface functions sampled from a Gaussian process with kernel function \(\kpsuminf[M_\nu]\) and \(\nu=\frac{5}{2},\sigma_f=1\) as defined in \cref{eq:periodic-matern-kernel-infinite-sum-examples}. Choose some failure probability \(\delta\in(0,1)\) and the confidence bounds according to \cref{eq:confidence-bounds-actual} with confidence parameter \(\beta_t\) as in \cref{lem:lemma-2-4}.
    
    Consider the greedy algorithm \(\A[CS]\) from \cref{def:greedy-algorithm} based on the objective function \nameref{sssec:design-objective-cs} and the sequence of NBV estimates \(\theta[CS]_t \defeq \A[CS](\theta_{1:t-1})\) with \(t=1,\dots,T\).
    Then we can show
    \begin{equation*}
        R(T) \leq \bigO*{T^{\frac{2\nu+2}{4\nu+2}}\log(T)^{\frac{4\nu+1}{4\nu+2}}}
    \end{equation*}
    with probability at least \(1-\delta\).
\end{theoremvar}
\begin{proof}
    \cref{ssec:proofs-theorem-sublinear-greedy-cs}
    \noqed
\end{proof}


\section{\algTCSU{}}\label{sec:analysis-twophase-csu}

This section presents the results specific to the two-phase algorithm \(\A[\TCSU]\) based on the \nameref{sssec:design-objective-cs} objective function in phase 1 and the \nameref{sssec:design-objective-u} objective function in phase 2.

Similar to \cref{sec:analysis-greedy-up} and \cref{sec:analysis-greedy-cs} we adapt the upper and lower confidence bounds used by the objective functions to \cref{eq:confidence-bounds-actual} and use the refined version of \nameref{sssec:design-objective-cs} from \cref{eq:objective-function-cs-new} based on the result from \cref{lem:lemma-2-4}.

Recall from \cref{def:twophase-algorithm} that the phase 1 objective function determines an interval over which phase 2 objective function is maximized. Hence, \(\A[\TCSU]\) is a variant of \(\A[U]\) with
\begin{equation*}
    \theta[\TCSU]_t
    = \argmax_{\theta\in\PhiS\paren{\theta[CS]_t}} \Fu[U](\theta\mid\theta_{1:t-1})
    ,
\end{equation*}
since \(\Fu[U]\) is the final objective function to be maximized, whereas \(\Fu[CS]\) only restricts the maximization to a certain interval around its own maximizer \(\theta[CS]\). Note that this is the only difference to \nameref{sssec:design-objective-u}, which instead maximizes \(\Fu[U]\) over complete \(\Domain\). This leads to
\begin{equation*}
    \begin{aligned}
        \Fu[U](\theta[\TCSU]_t\mid\theta_{1:t-1}) \leq \Fu[U](\theta[U]_t\mid\theta_{1:t-1})
    \end{aligned}
\end{equation*}
for arbitrary \(\theta_{1:t-1}\) as shown in the proof for \cref{lem:lemma-2-3-twophase-csu}. We can immediately derive the same sublinear regret bound as in \cref{lem:lemma-2-3-greedy-u}.

\begin{lemmavar}{CS-U}\label{lem:lemma-2-3-twophase-csu}
    Consider surface functions sampled from a Gaussian process \(\GP*{m,k}\). Assume that the assumptions for \cref{lem:lemma-2-4,lem:lemma-2-5} are satisfied. Let \(u_t(\varphi)\) and \(l_t(\varphi)\) be defined according to \cref{eq:confidence-bounds-actual} with confidence parameter \(\beta_t\) chosen as in \cref{lem:lemma-2-4}.
    
    Consider the two-phase algorithm \(\A[\TCSU]\) from \cref{def:twophase-algorithm} with phase 1 objective function \nameref{sssec:design-objective-cs} and phase 2 objective function \nameref{sssec:design-objective-u}.
    Let \(\theta_{1:T}\) be arbitrary.
    Then we can show
    \begin{equation*}
        \sum_{t=1}^T \Fu[U]\paren*{\theta[\TCSU]_t \mid \theta_{1:t-1}}
        \leq C_1 \sqrt{T\beta_T\gamma_T} + C_2
    \end{equation*}
    with \(C_1 = \frac{2\dmax\abs*{\PhiS}}{h^2}\sqrt{\frac{3}{\log[]{\sigmaeps^{-2}+1}}}\) and \(C_2 = \frac{\dmax\abs{\PhiS}}{h^2}\frac{\pi}{\sqrt{3}}\).

\end{lemmavar}
\begin{proof}
    \cref{ssec:proofs-lemma-2-3-twophase-csu}
    \noqed
\end{proof}

Combining all previous results as visualized in the overview in \cref{fig:analysis-overview}, we arrive at the final theorem for \(\A[\TCSU]\).

\begin{theoremvar}[Sublinear Regret]{CS-U}\label{thm:sublinear-twophase-csu}
    Consider surface functions sampled from a Gaussian process with kernel function \(\kpsuminf[M_\nu]\) and \(\nu=\frac{5}{2},\sigma_f=1\) as defined in \cref{eq:periodic-matern-kernel-infinite-sum-examples}. Choose some failure probability \(\delta\in(0,1)\) and the confidence bounds according to \cref{eq:confidence-bounds-actual} with confidence parameter \(\beta_t\) as in \cref{lem:lemma-2-4}.
    
    Consider the two-phase algorithm \(\A[\TCSU]\) from \cref{def:twophase-algorithm} with phase 1 objective function \nameref{sssec:design-objective-cs} and phase 2 objective function \nameref{sssec:design-objective-u} and the sequence of NBV estimates \(\theta[\TCSU]_t \defeq \A[\TCSU](\theta_{1:t-1})\) with \(t=1,\dots,T\).
    Then we can show
    \begin{equation*}
        R(T) \leq \bigO*{T^{\frac{2\nu+2}{4\nu+2}}\log(T)^{\frac{4\nu+1}{4\nu+2}}}
    \end{equation*}
    with probability at least \(1-\delta\).
\end{theoremvar}
\begin{proof}
    \cref{ssec:proofs-theorem-sublinear-twophase-csu}
    \noqed
\end{proof}


\section{Summary}\label{sec:analysis-summary}

To summarize the analysis, we found that all objective functions exhibit the same asymptotic behavior with the same sublinear regret bound. Measuring uncertainties at neighbor points as done by \(\A[CS]\) and analyzed in \cref{sec:analysis-greedy-cs} or restricting the maximization domain by a phase 1 objective as done by \(\A[\TCSU]\) and analyzed in \cref{sec:analysis-twophase-csu} does not differ from measuring the uncertainty only at the current position as done by \(\A[U]\) and analyzed in \cref{sec:analysis-greedy-u}. One reason is that in our setting all improvements such as measuring the uncertainty at additional additional points only contribute constantly more information, since the number of observed surface points can always be bounded by the finite size of the FOV.
In \cref{chp:experiments}, we however observe drastic differences in the performance of the three candidates \(\A[U]\), \(\A[CS]\) and \(\A[\TCSU]\), although we are able to show sublinear regret for all of them.

Finally, we summarize the relations of the three algorithm candidates.
\begin{equation*}
    \begin{alignedat}{3}
        \MoveEqLeft \hphantom{\leq} \Fu[CS](\theta_t\mid\theta_{1:t-1}) \span\span
        \\ &= \frac{1}{h^2} \cdot{} &&\sum_{\varphi\in\brack{\PhiS(\theta_t)}_t} \frac{1}{2} \frac{\abs{\PhiS}}{\abs{\brack{\PhiS}_t}} \paren*{u_t(\varphi)^2-l_t(\varphi)^2}
        &&\quad \color{gray} \stackrel{*}{=} \Fu[CS](\theta[CS]_t\mid\theta_{1:t-1})
        \\\\
        &\leq \frac{1}{h^2} \cdot{} &&\max_{\varphi\in\brack{\PhiS(\theta_t)}_t} \frac{1}{2} \abs*{\PhiS} \paren*{u_t(\varphi)^2-l_t(\varphi)^2}
        \\ &= &&\max_{\theta\in\brack{\PhiS(\theta_t)}_t} \Fu[U](\theta\mid\theta_{1:t-1})
        &&\quad \color{gray} \stackrel{*}{=} \Fu[U](\theta[\TCSU]_t\mid\theta_{1:t-1})
        \\\\
        &\leq &&\max_{\substack{\theta\in\Domain \\ \hphantom{\varphi\in\brack{\PhiS(\theta_t)}_t}}} \Fu[U](\theta\mid\theta_{1:t-1})
        &&\quad \color{gray} = \Fu[U](\theta[U]_t\mid\theta_{1:t-1})
        \\
        & && &&\quad \color{gray} \text{* only with \(\theta_t = \theta[CS]_t\)}
    \end{alignedat}
\end{equation*}



%% file: 07_experiments.tex
\chapter{Experimental Results}\label{chp:experiments}

In this chapter, we present the experimental results obtained from our simulation framework for the simplified 2D setting.

In \cref{sec:experiments-framework} we define our experiment framework including the setting, in which the experiments are conducted, the evaluation objects, on which the algorithms are tested, and the evaluation metrics, with which the performance is evaluated. We also specify conditions for the algorithm to terminate, after which the performance is then evaluated.

In \cref{sec:experiments-results} we present our experimental results in a condensed version and try to shed light into the obtained evaluation data.

In \cref{sec:experiments-summary} we provide a brief summary of our results.


\section{Experiment Framework}\label{sec:experiments-framework}

We first address the framework in which we conducted the experiments. This specify the different parameters of our simplified 2D setting and discuss the choices we make differently from what theory suggest in \cref{ssec:experiments-framework-settings}. Then we describe the objects against which we evaluate our algorithms in \cref{ssec:experiments-framework-objects} and finally define the evaluation metrics in \cref{ssec:experiments-framework-metrics}.

\subsection{Experiment Setting}\label{ssec:experiments-framework-settings}

Regarding the simplified 2D setting described in \cref{sec:problem-simplified}, we use the following parameters:
\begin{equation*}
    \begin{alignedat}{2}
        h &= 0.1\mathrm{m} &&\enspace\eqnote{world pixel width}
        \\ \dmax &= 8\mathrm{m} &&\enspace\eqnote{max. object size}
        \\ \dmin &= 2\mathrm{m} &&\enspace\eqnote{min. object size}
    \end{alignedat}
    \quad
    \begin{alignedat}{2}
        \dcam &= 10\mathrm{m} &&\enspace\eqnote{camera distance}
        \\ \DOF &= 10\mathrm{m} &&\enspace\eqnote{camera DOF}
        \\ \FOV &= 35^{\circ} &&\enspace\eqnote{camera FOV}
        \\ \sigmaeps &= 0.2 &&\enspace\eqnote{camera noise}
        .
    \end{alignedat}
\end{equation*}
The tradeoff in choosing the granularity \(h\) of the world discretization is to keep an as accurate model of the world as possible and at the same time ensuring computational efficiency. Making the discretization finer results into more surface points observed by each measurement and a cubic increase in computation time when updating the Gaussian process model. By setting \(\dcam = \DOF\), we ensure that the complete object surface is guaranteed to be observable by the camera while avoiding complex FOV shapes in the polar world as shown in \cref{fig:problem-different-fovs}. The observation noise \(\sigmaeps\) can be interpreted as the precision of the camera measuring the observed surface points with an average accuracy of \(\pm 0.2\mathrm{m}\).

In addition, we choose the following parameters for the kernel and confidence bounds universally for all algorithms:
\begin{equation*}
    \begin{alignedat}{2}
        \sigma_f &= 1.5 \mathcolor{gray}{\nleq 1} {} &&\enspace\eqnote{kernel deviation}
        \\ l &= 0.2 &&\enspace\eqnote{kernel length scale}
        \\ \nu &= 3/2 \mathcolor{gray}{\ngtr 2} {} &&\enspace\eqnote{kernel smoothness}
    \end{alignedat}
    \quad
    \begin{alignedat}{2}
        &\delta \text{ not chosen} &&\enspace\eqnote{failure prob.}
        \\ &\beta_t^{1/2} = 2 &&\enspace\eqnote{confidence param.}
    \end{alignedat}
\end{equation*}
Most notably, we choose a static confidence parameter \(\beta_t\) which defines the confidence bounds from \cref{eq:confidence-bounds-ideal} as twice the standard deviation around the mean. In contrast, \cref{lem:lemma-2-4} suggests to dynamically scale \(\beta_t\) with \(t\) to theoretically guarantee with a failure probability of at most \(\delta\) that the confidence boundaries enclose the surface function. This scaling is undesired in practice, since most of our objective functions depend on the geometric relation between the confidence region and the camera's FOV. Hence, increasing the confidence region in each round falsifies their estimates. We observed that choosing a less smooth \matern{} kernel through a smaller \(\nu\), modeling the surface functions with a larger standard deviation \(\sigma_f\) and assuming a smaller length scale \(l\) ensure that the confidence boundaries stay valid in most cases, despite violating the assumptions of \cref{lem:lemma-2-4,lem:lemma-2-5}.

Since the Gaussian process is defined on polar functions, the length scale parameter \(l\) roughly describes the polar angle instead of Euclidean distance to some \(\varphi\) before the function value changes significantly as described in \cref{ssec:background-gp-kernels}.
Hence, the choice for \(l\) implicitly depends on the considered size of objects, since the surface function for larger objects changes at a smaller scale in polar angles than for small objects. Specifically, the conversion rate from a Cartesian length scale \(l_c\) to a polar length scale \(l_p\) is
\begin{equation*}
    l_p = \frac{l_c}{2\pi r} \cdot 2\pi = \frac{l_c}{r}
\end{equation*}
for points at distance \(r\) to the world center.%
\footnote{To be correct, we should note that the length scale \(l_c\) is defined in terms of the geodesic distance on a circle with radius \(r\). Assuming \(l_c \ll r\) this roughly corresponds to the Euclidean distance.}

The reasons for presenting these specific choices is that due to our rather small sample size of experiments and evaluations the found results might have large variance with respect to different settings. For example, choosing \(h=0.1\) or \(h=0.2\) for the width of a real world pixel can already change the ranking of performances of our algorithm candidates. Similarly, different choices for the object bounds \(\dmax\) and \(\dmin\), the camera distance \(\dcam\) or the FOV shape given by \(\FOV\) and \(\DOF\) influence the level of observability of the object through the camera, with which the one or other objective function can cope better.

\subsection{Evaluation Objects}\label{ssec:experiments-framework-objects}

Next, we introduce the set of objects, on which we evaluate our algorithm candidates. The goal is to capture a large variety of real world characteristics and at the same time keep the modeling process through the polar function feasible. These characteristics do not only include different smoothness, but also varying distance of the object surface to the camera and different surface complexities, which leads to different potentials for self-occlusion.

\begin{figure}
    \centering
    \begin{subfigure}{0.25\linewidth}
        \centering
        \includegraphics[width=\linewidth]{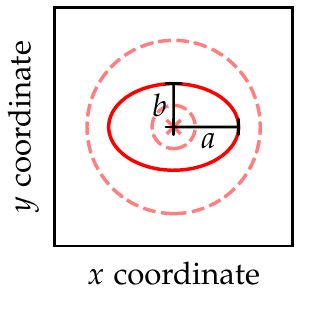}
        \includegraphics[width=\linewidth]{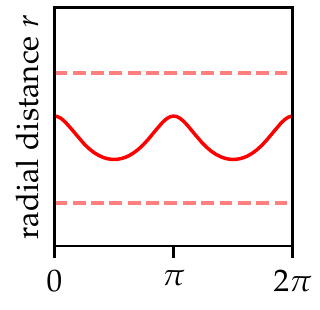}
        \caption{ellipse objects}
        \label{fig:experiments-framework-objects-ellipse}
    \end{subfigure}%
    \begin{subfigure}{0.25\linewidth}
        \centering
        \includegraphics[width=\linewidth]{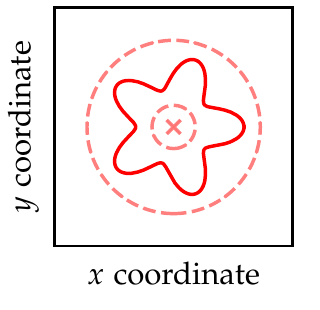}
        \includegraphics[width=\linewidth]{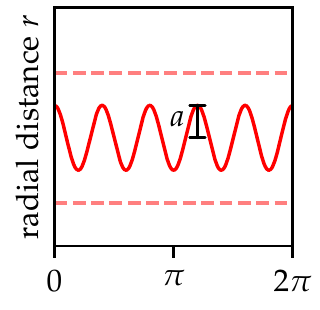}
        \caption{flower objects}
        \label{fig:experiments-framework-objects-flower}
    \end{subfigure}%
    \begin{subfigure}{0.25\linewidth}
        \centering
        \includegraphics[width=\linewidth]{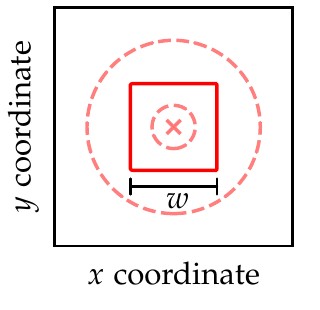}
        \includegraphics[width=\linewidth]{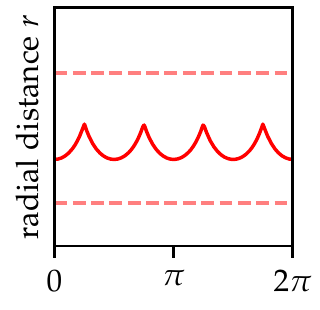}
        \caption{square objects}
        \label{fig:experiments-framework-objects-square}
    \end{subfigure}%
    \begin{subfigure}{0.25\linewidth}
        \centering
        \includegraphics[width=\linewidth]{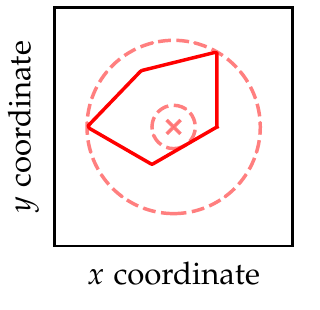}
        \includegraphics[width=\linewidth]{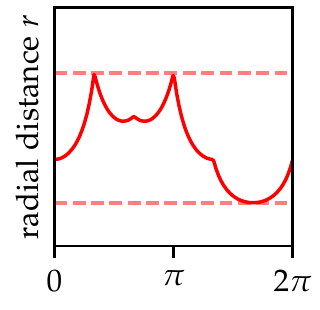}
        \caption{polygon objects}
        \label{fig:experiments-framework-objects-polygon}
    \end{subfigure}
    \caption[Different Classes of Objects]{
        Different Classes of Objects.
        (a) to (d) visualize the four different main classes of objects in the real (top) and polar world (bottom) which we use to evaluate our algorithms. Detailed description of the object parameters are given in \cref{ssec:experiments-framework-objects}. The dotted lines represent the object boundaries \(\dmax\) and \(\dmin\) as described in \cref{ssec:problem-simplified-object}.
    }
    \label{fig:experiments-framework-objects}
\end{figure}

The first type of objects are smooth objects represented by a smooth surface function. The simplest class consists of \emph{circle objects} corresponding to uniform polar functions. They can be generalized to the class of \emph{ellipse objects} with semi-major axis length \(a\) and semi-minor axis length \(b\) as visualized in \cref{fig:experiments-framework-objects-ellipse}. Another surface function for more complex object shapes is a cosine function oscillating around \(\frac{1}{2}(\dmax+\dmin)\) with frequency \(f\) and amplitude \(a\). We refer to them as \emph{flower objects} and one example is provided in \cref{fig:experiments-framework-objects-flower}.

The second type of objects are formed by straight edges and sharp corners corresponding to highly non-smooth surface functions, which can be modeled with piecewise polar functions. A simple class consists of \emph{square objects} with width \(w\) as shown in \cref{fig:experiments-framework-objects-square}. \emph{Polygon objects}, which are defined in terms of a sequence of vertices, are a more general and an example is given in \cref{fig:experiments-framework-objects-polygon}.

The complete set of objects for evaluating our algorithms is given in \cref{sec:simulation-objects}.

\subsection{Evaluation Metrics}\label{ssec:experiments-framework-metrics}

\newcommand{\rec}{\textsl{rec}}
\newcommand{\mea}{\ensuremath{T}}
\newcommand{\meathr}{\ensuremath{T_{\geq0.95}}}
\newcommand{\rmea}{\ensuremath{\widetilde{T}}}
\newcommand{\rmeathr}{\ensuremath{\widetilde{T}_{\geq0.95}}}
\newcommand{\reg}{\ensuremath{\overline{r}_{ind}}}
\newcommand{\rkrec}{\ensuremath{\#_{REC}}}
\newcommand{\rknbv}{\ensuremath{\#_{NBV}}}

\newcommand{\higherbetter}{\({}^\uparrow\)}
\newcommand{\lowerbetter}{\({}^\downarrow\)}

One unmentioned piece of information about our algorithms is the termination condition defining the time after which we evaluate and compare their performances. We define it as
\begin{equation*}
    T \defeq \text{smallest \(t\) with }\begin{cases}
        o(\theta_{1:t}) = \Surface &\eqnote{full reconstruction} \\
        o(\theta_{1:t}) = o(\theta_{1:t+1}) &\eqnote{early termination}
        .
    \end{cases}
\end{equation*}
Naturally, once the complete object surface \(\Surface\) is fully observed, the algorithm terminates. However, in case no new surface points are observed in some round \(t\) corresponding to \(F(\theta_{t+1}\mid\theta_{1:t}) = \abs{o(\theta_{t+1})\setminus o(\theta_{1:t})} = 0\), \emph{early termination} happens. The reason is that without new information about the object surface the algorithm keeps recommending the same NBV estimate \(\theta_{t+1}\), which is deterministically computed based on \(o(\theta_{1:t})\) up to measurement noise, and therefore the algorithm does not progress.%
\footnote{This is only the case, because our true objective function from \cref{eq:true-marginal-utility} considers the binary case, whether a surface point is observed or not. In reality, already observed surface points might still be associated with some uncertainty coming from large measurement noise and measuring these again can provide new information.}

To compare the performances of the algorithms not only over the same object, but also over different ones, we ideally want to find object-independent metrics, which quantify the general performance of an algorithm. In the following, we define metrics of two types. The first type measures the performance with respect to the reconstruction problem, while the second one measures the performance with respect to the NBV decision problem. We then deduce ranking schemes from these metrics for metric-independent comparison between the algorithms.

\subsubsection{Metrics for Reconstruction}
Since the total number of observed surface points varies depending on the object size, it is not suitable for measuring the reconstruction performance. Hence, we define the relative \emph{reconstruction amount} at termination time \(T\) as
\begin{equation*}
    \rec \defeq \frac{\text{total number of observed surface points}}{\text{total number of surface points}}
    = \frac{\abs{o(\theta_{1:T})}}{\abs{\Surface}} 
    .
\end{equation*}
The corresponding \emph{number of measurements} required to achieve this amount of reconstruction is equivalent to \(T\) and we write it as
\begin{equation*}
    \mea = \text{number of measurements until } \begin{cases}
        \text{rec \(=\) 100\%}\\
        \text{early termination}
    \end{cases}
    .
\end{equation*}
The caveat of this definition is that early terminating algorithms typically make less number of measurements than algorithms achieving full reconstruction. This means that one can only compare the number of measurements in combination with the achieved reconstruction amount. Therefore, we define the joint metric 
\begin{equation*}
    \meathr \defeq \begin{cases}
        \text{number of measurements until \rec{} \(\geq\) 95\%}\\
        \text{N/A for early termination with \rec{} \(<\) 95\%}
    \end{cases}
\end{equation*}
combining the number of measurements with a certain reconstruction threshold. Since the number of required measurements depends on the size of the object, the metrics \mea{} and \meathr{} are not suitable for comparisons over different objects. We solve this by defining them \emph{relative} to the number of measurements \(\meathr^*\) of the optimal greedy algorithm \(\A[*]\) as
\begin{equation*}
    \rmea \defeq \frac{\mea}{\mea^*} 
    \quad\text{and}\quad
    \rmeathr \defeq \frac{\meathr}{\meathr^*}
    ,
\end{equation*}
where \(\A[*]\) has knowledge about the true object shape.
Note that all above defined metrics measure the algorithm's performance with respect to the reconstruction problem, \ie{} maximally observing the object surface with minimal number of measurements.

\subsubsection{Metrics for NBV}
A metric which measures the performance with respect to the NBV decision problem, \ie{} finding the best decision in the current round, is
\begin{equation*}
    \reg \defeq \text{average individual regret}
    = \frac{1}{T}\sum_{t=1}^T r_{ind}(t)
    .
\end{equation*}
It corresponds to the average number of surface points additionally observed by a greedy decision \(\thetagre_t\) compared to the actual decision \(\theta_t\) given the same previous camera positions \(\theta_{1:t-1}\). It basically shows which algorithms make decisions closest to the optimal greedy decision.
Note that this evaluation metric is, in fact, mostly independent from the object size despite being a counting metric on surface points. Due to the finite size of the FOV, the deviation from the greedy decision is constantly upper bounded and does not scale with the object size, but rather depends on the algorithm itself. We only observed that the deviation from the greedy decision is typically larger for more complex objects which allow the optimal greedy algorithm to maximally exploit its knowledge about the true object shape.

However, a misconception is that small average individual regret is directly correlated to good reconstruction performance, which does not hold in general. Since the individual regret only measures the deviation from the current greedy decision individual to each round, it is possible that starting from some non-optimal previous decisions the greedy decision itself cannot perform much better than the actual decision. Although this can lead to low average individual regret, it only implies that the current decision is close to the (greedy-)optimum given the previous decisions, but not that the overall set of decisions is close to optimum.

\subsubsection{Ranking}
To compare the algorithms' performances over different objects, we typically compute the average of their metrics for these objects. For metrics such as \reg{}, more complex objects often lead to generally larger values for this metric than simpler objects. Hence, the average of such metrics over multiple objects is implicitly weighted based on the dependence of the metric on the object, which might or might not be desired.

A completely object-independent way for comparing algorithms is to derive a ranking scheme from the defined metrics which orders the algorithms on a uniform scale according to their values.
The advantage is that the average rank takes the performance of the algorithms for each object uniformly into account. However, it eliminates the information about the relative performances between the algorithms by placing them on a uniform scale.

For comparing the performances with respect to the reconstruction problem, we define the ranking of algorithms as
\begin{equation*}
    \rkrec \defeq \text{dense ranking\footnotemark{} of algorithms ordered by \meathr{}}
\end{equation*}
\footnotetext{A dense ranking assigns equal items the same rank and subsequent items the immediately following rank.}%
with ties resolved based on \mea{} and \rec{} in the given order. As discussed above it is preferable to use \meathr{} as the main comparison criterion, since \mea{} and \rec{} should be compared jointly.%
\footnote{Even then the order is unclear whether one should rank a 99\% reconstruction above a 95\% reconstruction if the 99\% reconstruction requires two more measurements, for example.}

For comparing the performances with respect to the NBV decision problem, we define the ranking of algorithms as
\begin{equation*}
    \rknbv \defeq \text{dense ranking of algorithms ordered by \reg}
\end{equation*}
with ties resolved based on \meathr{}, \mea{} and \rec{} in the given order. As discussed above, a high ranking with respect to the NBV decision problem does not directly correlate with a high ranking for the reconstruction problem.


\section{Experiment Results}\label{sec:experiments-results}

We first explain the issues of confidence-based objective functions in \cref{ssec:experiments-results-deficiency} and why we exclude them from further experimental analysis. Consequently, it remains to compare the intersection-based and uncertainty-based objective functions, which is done in \cref{ssec:experiments-results-comparison}.

Among the presented results, we give our best to only highlight generally applicable results which do not stem from some edge case of our specific framework.

\subsection{Deficiency of Confidence-based Objective Functions}\label{ssec:experiments-results-deficiency}

We observed during our experiments that the confidence-based objective functions are not able to reconstruct the object except for \nameref{sssec:design-objective-csw}. Due to the lack of FOV information as discussed in \cref{ssec:design-objective-confidence}, these objective functions excessively overestimate the actually observable region through the camera by falsely assuming that the camera's FOV completely covers a range of polar angles such as \(\PhiI_t(\theta)\) or \(\PhiS(\theta)\) as visualized in \cref{fig:experiments-results-deficiency-visited-location}. This leads to the degenerate problem that these objective functions have local maxima at already visited locations, which can be best explained by \cref{fig:experiments-results-deficiency-objective-function}. Hence, the algorithm starts recommending already visited locations at some point in time resulting into early termination, which in our experiments typically happens between 30\% to 60\% of reconstruction depending on the object complexity. For that reason, we exclude \nameref{sssec:design-objective-c}, \nameref{sssec:design-objective-cs} and \nameref{sssec:design-objective-csp} from our following discussion.

\begin{figure}
    \centering
    \begin{subfigure}{0.33\linewidth}
        \centering
        \includegraphics[width=\linewidth]{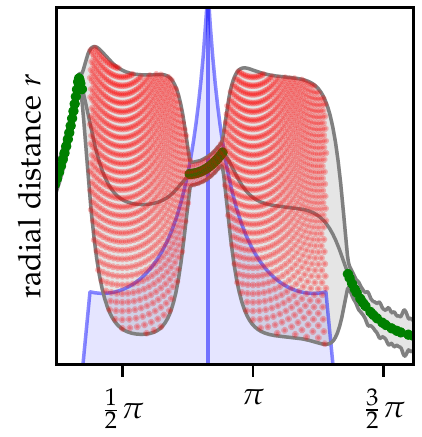}
        \caption{\vphantom{\(\Fu[CS]\)}camera at visited location}
        \label{fig:experiments-results-deficiency-visited-location}
    \end{subfigure}%
    \begin{subfigure}{0.33\linewidth}
        \centering
        \includegraphics[width=\linewidth]{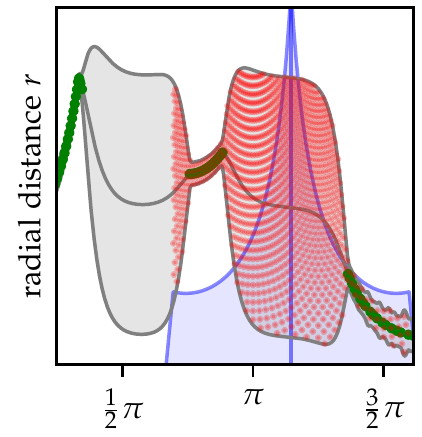}
        \caption{\vphantom{\(\Fu[CS]\)}camera at new location}
        \label{fig:experiments-results-deficiency-new-location}
    \end{subfigure}%
    \begin{subfigure}{0.33\linewidth}
        \centering
        \includegraphics[width=\linewidth]{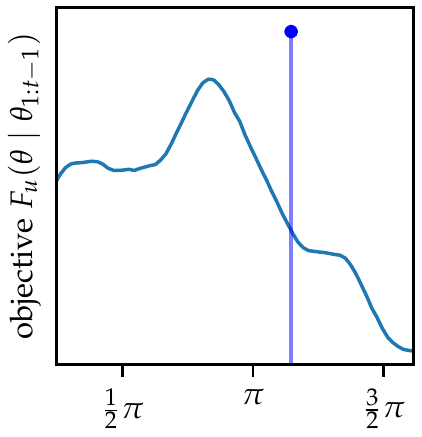}
        \caption{\(\Fu[CS](\theta\mid\theta_{1:t-1})\)}
        \label{fig:experiments-results-deficiency-objective-function}
    \end{subfigure}
    \caption[Degenerate Problem of Confidence-based Objective Functions]{
        Degenerate Problem of Confidence-based Objective Functions.
        The \nameref{sssec:design-objective-cs} objective function is visualized at the recently visited location in (a), where the object surface was observed and measured (green), and at a new location in (b), where the object surface has not been observed yet as represented by the confidence region (gray). From the comparison of the covered areas through the objective function at both locations (red dots), it should become apparent that the objective function recommends the recently visited location again, since this location yields a larger objective function value. Figure (c) visualizes the objective function over different camera positions \(\theta\) and clearly shows a maximum at the previously visited location.
        Similar issues arise with \nameref{sssec:design-objective-c} and \nameref{sssec:design-objective-csp} which are not further discussed.
    }
    \label{fig:experiments-results-deficiency}
\end{figure}

\subsection{Intersection- vs.\ Uncertainty-based Objective Functions}\label{ssec:experiments-results-comparison}

From a top-down approach, we first present the most general results summarizing the experiments over all objects. Afterwards, we analyze the results for each object class separately to provide some additional thoughts.

\subsubsection{Results for all Objects}
The following discussion refers to \cref{tab:experiments-results-all} which compares the average performance of the algorithms on all used objects depicted in \cref{fig:simulation-experiment-objects}.

\begin{table}
    \centering
    \renewcommand{\arraystretch}{1.3}
    \begin{subtable}{0.6\linewidth}
        \centering
        \begin{tabular}{@{}lccccc@{}}
            \toprule
            Algo. & \rkrec{} & \rmeathr{} \lowerbetter{} & \color{gray} \rmea{} \lowerbetter{} & \color{gray} \rec{} \higherbetter{} \\
            \midrule
            \color{black} \(\A[*]\)     & 1.2 & 100\% & \color{gray} 100\% & \color{gray} 100\% \\
            \color{Green} \(\A[CSW]\)   & 3.3 & 127\% & \color{gray} 112\% & \color{gray}  99\% \\
            \color{Green} \(\A[IOA]\)   & 3.3 & 127\% & \color{gray} 105\% & \color{gray}  98\% \\
            \color{Green} \(\A[I]\)     & 3.6 & 129\% & \color{gray} 107\% & \color{gray}  97\% \\
            \color{blue}  \(\A[U]\)     & 3.9 & 128\% & \color{gray} 116\% & \color{gray} 100\% \\
            \color{blue}  \(\A[UP]\)    & 4.0 & 128\% & \color{gray} 113\% & \color{gray} 100\% \\
            \color{blue}  \(\A[\TCSU]\) & 4.2 & 130\% & \color{gray} 116\% & \color{gray} 100\% \\
            \bottomrule
        \end{tabular}
        \caption{results ranked by \rkrec{}}
        \label{tab:experiments-results-all-rec}
    \end{subtable}%
    \begin{subtable}{0.4\linewidth}
        \centering
        \begin{tabular}{@{}lcc@{}}
            \toprule
            Algo. & \rknbv{} & \reg{} \lowerbetter{} \\
            \midrule
            \color{black} \(\A[*]\)     & 1.0 &  0.00 \\
            \color{blue}  \(\A[UP]\)    & 3.3 &  9.01 \\
            \color{Green} \(\A[IOA]\)   & 3.9 & 11.10 \\
            \color{blue}  \(\A[U]\)     & 4.0 & 11.91 \\
            \color{Green} \(\A[I]\)     & 4.5 & 12.28 \\
            \color{Green} \(\A[CSW]\)   & 4.6 & 11.63 \\
            \color{blue}  \(\A[\TCSU]\) & 4.7 & 12.14 \\
            \bottomrule
        \end{tabular}
        \caption{results ranked by \rknbv{}}
        \label{tab:experiments-results-all-nbv}
    \end{subtable}
    \caption[Averaged Results over all Objects]{
        Averaged Results over all Objects.
        These tables summarize the results from the experiments on all objects depicted in \cref{fig:simulation-experiment-objects} by averaging the corresponding ranks and metric values. (a) is ordered based on the average rank with respect to the reconstruction problem, while (b) is ordered based on the average rank with respect to the NBV decision problem. We provide explanations for the metrics in \cref{ssec:experiments-framework-metrics} and indicate whether higher \higherbetter{} or lower \lowerbetter{} values are better. For visual support, we color algorithms using intersection-based objective functions (green) different from algorithms using uncertainty-based objective functions (blue). Note that \(\A[CSW]\) can be seen as an approximation of an intersection-based objective function as described in \cref{sssec:design-objective-csw}, while \(\A[\TCSU]\) is a variant of an uncertainty-based objective function as explained in \cref{sec:analysis-twophase-csu}. The metrics \rmea{} and \rec{} (gray) are only provided for information and not for comparison as explained previously.
    }
    \label{tab:experiments-results-all}
\end{table}

In \cref{tab:experiments-results-all-rec}, we can clearly see that intersection-based objective functions score higher rankings on average than the uncertainty-based objective functions based on the colors. This reflects how we gradually designed new objective functions with less accurate upper bounds to obtain simpler closed-form expressions as discussed in \cref{ssec:design-objective-summary}. As we can see, using more accurate objective functions helps in achieving good rankings with respect to the reconstruction problem.

At the same time, we can observe that intersection-based objective functions tend to terminate early without achieving full reconstruction seen from the \rec{} metric. A potential reason is that they maximize the observed uncertainty area, which towards the final rounds can be maximal at already visited locations, where the uncertainty of many surface points, which still remains from the previous measurement noise, can be observed all at once.
In contrast, uncertainty-based objective functions almost always achieve full reconstruction, since they only maximize the uncertainty at the current location which normally is larger at new locations than at visited locations.

Note that the average ranking of the optimal greedy algorithm \(\A[*]\) underlines the fact that the greedy algorithm is optimal in the sense of finding a greedy solution based on the true object, but not finding an overall optimal solution. Hence, it is possible to sometimes achieve better performance without knowledge about the true object, but usually only with luck.

In \cref{tab:experiments-results-all-nbv}, we can observe that the rankings \rknbv{} are different from \rkrec{}. In particular, the greedy algorithm based on \nameref{sssec:design-objective-up} achieves significantly lower average individual regret than the others. One can carefully conclude that uncertainty-based objective functions tend to agree more with the optimal greedy decision than the more complex intersection-based objective functions and therefore solve the NBV decision problem better. 
Reasons for why they lag behind in the performance with respect to the reconstruction problem are vague, but recall that the greedy decision is always made with respect to the previously made decisions. Hence, agreeing with the greedy decisions after \(t\) rounds might be already too late if the first \(t\) decisions were selected too badly.

Note that the order of \rkrec{} does not necessarily coincide with the order given by \reg{}, since the average ranking takes all objects uniformly into account, while the average of \reg{} is weighted depending on the kinds of objects.

\subsubsection{Results per Object Class}
The following discussion refers to \cref{tab:experiments-results-perobject-rec,tab:experiments-results-perobject-nbv} which compare the average performances of the algorithms on each object class.

\begin{table}
    \centering
    \renewcommand{\arraystretch}{1.3}
    \renewcommand{\*}{\bfseries}
    \begin{subtable}{\linewidth}
        \centering
        \begin{small}
            \begin{tabular}{@{}lcclcclcclcc@{}}
                \toprule
                & \multicolumn{2}{c}{ellipse} && \multicolumn{2}{c}{flower} && \multicolumn{2}{c}{square} && \multicolumn{2}{c}{polygon} \\
                \cmidrule{2-3} \cmidrule{5-6} \cmidrule{8-9} \cmidrule{11-12}
                Algo. & \rkrec{} & \rmeathr{} && \rkrec{} & \rmeathr{} && \rkrec{} & \rmeathr{} && \rkrec{} & \rmeathr{} \\
                \midrule
                \color{black} \(\A[*]\)     &   1.2 &   100\% &&   1.2 &   100\% &&   1.0 &   100\% &&   1.2 &   100\% \\
                \color{Green} \(\A[CSW]\)   & \*3.2 &   132\% && \*3.4 & \*117\% &&   3.2 &   141\% && \*3.2 & \*134\% \\
                \color{Green} \(\A[IOA]\)   & \*2.8 & \*122\% && \*3.8 &   123\% &&   2.8 &   112\% &&   4.0 &   183\% \\
                \color{Green} \(\A[I]\)     &   5.2 &   140\% && \*3.8 &   127\% && \*2.2 &   114\% && \*3.0 & \*142\% \\
                \color{blue}  \(\A[U]\)     &   4.5 &   133\% &&   4.3 &   125\% && \*2.5 & \*109\% &&   3.5 &   149\% \\
                \color{blue}  \(\A[UP]\)    &   3.5 & \*121\% &&   4.7 &   127\% && \*2.5 & \*106\% &&   4.0 &   158\% \\
                \color{blue}  \(\A[\TCSU]\) &   5.5 &   137\% &&   4.1 & \*117\% &&   3.0 &   140\% &&   4.2 &   148\% \\
                \bottomrule
            \end{tabular}
        \end{small}
        \caption{results ranked by \rkrec{} for the reconstruction problem}
        \label{tab:experiments-results-perobject-rec}
    \end{subtable}
    \par\medskip
    \begin{subtable}{\linewidth}
        \centering
        \begin{small}
            \begin{tabular}{@{}lcclcclcclcc@{}}
                \toprule
                & \multicolumn{2}{c}{ellipse} && \multicolumn{2}{c}{flower} && \multicolumn{2}{c}{square} && \multicolumn{2}{c}{polygon} \\
                \cmidrule{2-3} \cmidrule{5-6} \cmidrule{8-9} \cmidrule{11-12}
                Algo. & \rknbv{} & \reg{} && \rknbv{} & \reg{} && \rknbv{} & \reg{} && \rknbv{} & \reg{} \\
                \midrule
                \color{black} \(\A[*]\)     &   1.0 &   0.00 &&   1.0 &   0.00 &&   1.0 &   0.00 &&   1.0 &    0.00 \\
                \color{blue}  \(\A[UP]\)    & \*2.0 & \*5.40 && \*3.8 & \*8.88 && \*2.5 & \*2.00 && \*4.0 & \*19.94 \\
                \color{Green} \(\A[IOA]\)   &   3.8 &   8.87 && \*3.8 &   9.43 &&   4.5 &   7.42 && \*4.0 &   21.21 \\
                \color{blue}  \(\A[U]\)     & \*3.0 & \*7.72 &&   4.9 &  12.17 && \*3.0 & \*2.73 && \*4.0 &   24.61 \\
                \color{Green} \(\A[I]\)     &   5.8 &  10.09 &&   4.7 &  10.48 &&   3.5 &   7.24 && \*4.0 &   24.00 \\
                \color{Green} \(\A[CSW]\)   &   6.2 &  11.44 &&   4.2 & \*8.88 &&   4.5 &   9.57 &&   4.2 & \*20.75 \\
                \color{blue}  \(\A[\TCSU]\) &   5.0 &   9.17 &&   4.3 &  10.03 &&   5.5 &   9.88 &&   4.5 &   22.63 \\
                \bottomrule
            \end{tabular}
        \end{small}
        \caption{results ranked by \rknbv{} for the NBV decision problem}
        \label{tab:experiments-results-perobject-nbv}
    \end{subtable}
    \caption[Averaged Results over each Object Class]{
        Averaged Results over each Object Class.
        These tables average the corresponding ranks and metric values over each object class and are ordered as in \cref{tab:experiments-results-all}. Explanations for the metrics are given in \cref{ssec:experiments-framework-metrics}. Among the algorithm candidates, we highlight for each metric the top two values (bold).
    }
    \label{tab:experiments-results-perobject}
\end{table}

In \cref{tab:experiments-results-perobject-rec,tab:experiments-results-perobject-nbv}, we observe that \rmea{} and \reg{} are significantly larger for polygons on average than for other object classes. The reason is that straight edges and sharp corners of polygons often provide favorable camera locations with large observation coverage of the surface. Hence, knowing the true object surface provides a much larger advantage for complex objects than for simple objects. This explains why the deviation from the optimal greedy algorithm in terms of performance is generally larger for high object complexities.

In \cref{tab:experiments-results-perobject-nbv}, it is notably that the greedy algorithm \(\A[UP]\) based on \nameref{sssec:design-objective-up} is ranked first with respect to the NBV decision problem for all object classes, while the variant \(\A[U]\) does not exhibit this performance. This underlines our statement in \cref{sec:analysis-greedy-u} that the additional mean factor \(\mu_{t-1}(\theta)\) in the objective function \(\Fu[U]\) is contra-productive, since it incentivizes instead of penalizes small distances between camera and object surface.

Interestingly, when ranking the algorithms based on the average \(\reg\) over each object class, \(\A[U]\) is placed second for the ellipse and square object classes, while it is on the last place for the flower and polygon object classes. This suggests that \(\A[U]\) only performs well on simple objects without self-occlusions, while it disagrees more with the optimal greedy decisions for more complex objects.

Finally, we want to remark that we are disappointed by the performance of the two-phase algorithm \(\A[\TCSU]\), as it is ranked last for both the reconstruction and NBV decision problems. We assume that the \nameref{sssec:design-objective-cs} objective function still negatively impacts the NBV estimates despite finding the \(\Fu[U]\) maximizer in the second phase.


\section{Summary}\label{sec:experiments-summary}

Despite many vague statements, we try to summarize the ones, for which we feel the most confident.
Clearly, all confidence-based algorithms which do not take the FOV shape into account demonstrate poor reconstruction performance as seen in \cref{fig:experiments-results-deficiency}.
For the reconstruction problem, we found out in \cref{tab:experiments-results-all-rec} that the average performance ranking follows the order of accuracy of the objective functions with intersection-based objective functions being the most accurate ones.
For the NBV decision problem, it seems very much that uncertainty-based objective functions tend to perform better and achieve lower average individual regret. In particular, the algorithm based on \nameref{sssec:design-objective-up}, for which we even did not show sublinear regret in \cref{sec:analysis-greedy-up}, performs remarkably good with respect to the NBV decision problem as seen in \cref{tab:experiments-results-perobject-nbv}. In contrast, the two-phase algorithm \(\A[\TCSU]\), for which we showed sublinear regret, does solve the issues of confidence-based algorithms, but does not do so sufficiently to compete with the other candidates.

%% file: 08_conclusion.tex
\chapter{Conclusion}\label{chp:conclusion}

To conclude our work on near-optimal active reconstruction, we want to highlight the most important observations from the different chapters.

In \cref{chp:problem,chp:design}, we faced the challenges of transferring the approaches for Gaussian process optimization to the active object reconstruction setting.

From the practical perspective, the main difficulty was the transformation between the real world, in which the target object and the camera reside, and the polar world, in which we defined our Gaussian process model and consequently our objective functions.
With the design of \(2\pi\)-periodic kernels for modeling the polar surface functions and various types of objective functions for estimating the NBV, we were able to find some candidate algorithms, which are simple enough for the analysis, but still performing sufficiently well in practice.
Based on our insights, we collected a list of requirements and heuristics for the design of objective functions in \cref{ssec:design-objective-requirements}.

From a theoretical perspective, we investigated the differences to settings of related work in \cref{sec:problem-comparison}, which appeared minor, but ultimately caused major problems for our results. Since a solution to the reconstruction problem ranges over the decisions of all rounds instead of a single round, it is not possible for us to show near-optimality for this problem. We conjecture that this is generally the case, since the performance in the first few rounds can be arbitrarily bad and the optimal solution constantly improves with increasing number of rounds. Therefore, we relaxed our initial goals to finding a near-optimal decision in each round or equivalently a near-optimal solution to the NBV problem. For the same reason as above and the resulting time-dependence of a near-optimal decision, we are only able to show pseudo-convergence to near-optimality as stated in \cref{thm:near-optimality}.

In \cref{chp:analysis}, we precisely showed under reasonable assumptions that our algorithm candidates have asymptotically zero or negative average regret with probability \(1-\delta\), which implies that they are guaranteed to make at least 
one decision with marginal utility of at least \(1-\frac{1}{e}\approx 63\%\) of an optimal decision up to some precision \(\varepsilon>0\) within some finite time \(T\) and with probability \(1-\delta\).
Interestingly, since the additional number of observed surface points per measurement is constantly upper bounded, any heuristic improvement of maximizing the number of surface points inside the FOV only leads to a constant improvement of the regret bounds.

In \cref{chp:experiments}, we however observed that these heuristics typically matter more than the asymptotic behavior of our algorithms. From the averaged results over all objects in \cref{tab:experiments-results-all-rec}, we conclude that algorithms equipped with more accurate objective functions such as \(\A[IOA]\), \(\A[I]\) or \(\A[CSW]\) generally exhibit better performance with respect to the reconstruction problem. But it is also notable that the uncertainty-based objective function \(\A[UP]\) significantly outperforms the other algorithms with respect to the NBV decision problem in terms of the individual average regret as seen in \cref{tab:experiments-results-all-nbv}.

This let us conclude that theory and practice are not always perfectly aligned and showing sublinear regret in theory does not directly correlate with superior performance in practice. However, theory did provide us with valuable insights for the design and understanding of our objective functions.

\subsubsection{Contributions}
The contributions of this thesis might not be as initially envisioned, but we showed how to apply Gaussian process optimization rigorously in a complex, novel setting, what difficulties can arise and how they can be solved.
We hope that we provided valuable insights with the comparison of our setting with previous work, our thoughts on the design of objective functions and periodic kernels, and finally our theoretical and experimental results. 

\subsubsection{Future Work}
Based on the list of simplifications in \cref{ssec:problem-simplified-simplifications}, much work can be done in gradually relaxing or lifting these simplifications, for which we provided initial food for thought. Most importantly, extending our methods to 3D and relaxing the restrictions on the camera pose would be major contributions to our work. This would allow one to conduct real-world experiments outside the simulation framework with robotic systems, for example.

An interesting idea is to interpret the reconstruction of a 3D room as an inverted 3D object reconstruction problem with the camera located ``inside'' the object and oriented towards outside.%
\footnote{Credits for this idea go to Manish Prajapat.}
Although not much thought has been given to this idea yet, it would be great to see how near-optimality results can be extended to even more complex tasks such as 3D scene reconstruction.

%% file: A_proofs.tex
\chapter{Proofs}\label{chp:proofs}

Here we provide all rigorous proofs for the previous chapters. To facilitate the parsing of the derivations, we highlight all changes made from the previous to the current derivation step.


\section{Proofs for \cref{chp:background}}

\subsection{\fullref{lem:information-entropy-gaussian}}\label{ssec:proofs-lemma-information-entropy-gaussian}

The goal of this proof is to derive the entropy of the Gaussian distribution.

\begin{proof}
    Let \(X\in\Normal*{\mu,\Sigma}\) with \(\mu\in\Real^n,\Sigma\in\Real^{n,n}\).
    \begin{derivation}
        H(X) &= \eqcontrast \eqchange{\E*{-\log p(X)}}
        \\ &= \eqcontrast \E*{-\log*{\eqchange{\frac{1}{\sqrt{(2\pi)^n\det*{\Sigma}}}e^{-\frac{1}{2}(X-\mu)^T\Sigma^{-1}(X-\mu)}}}}
        \\ &= \eqcontrast \E*{-\log*{\frac{1}{\sqrt{(2\pi)^n\det*{\Sigma}}}} \eqchange{+ \frac{1}{2}(X-\mu)^T\Sigma^{-1}(X-\mu)}}
        \\ &= \eqcontrast \eqchange{\frac{1}{2}\log\det*{2\pi\Sigma}} + \eqchange{\frac{1}{2}}\E*{(X-\mu)^T\Sigma^{-1}(X-\mu)}
        \\ &= \eqcontrast \frac{1}{2}\log\det*{2\pi\Sigma} + \frac{1}{2}
        \E*{\eqchange{\tr*{\eqnochange{(X-\mu)^T\Sigma^{-1}(X-\mu)}}}}
        \\ &= \eqcontrast \frac{1}{2}\log\det*{2\pi\Sigma} + \frac{1}{2}
        \E*{\tr*{\eqchange{\Sigma^{-1}}(X-\mu)(X-\mu)^T}}
        \\ &= \eqcontrast \frac{1}{2}\log\det*{2\pi\Sigma} + \frac{1}{2}
        \tr*{\eqchange{\E*{\eqnochange{\Sigma^{-1}(X-\mu)(X-\mu)^T}}}}
        \\ &= \eqcontrast \frac{1}{2}\log\det*{2\pi\Sigma} + \frac{1}{2}
        \tr*{\Sigma^{-1}\eqchange{\E*{\eqnochange{(X-\mu)(X-\mu)^T}}}}
        \\ &= \eqcontrast \frac{1}{2}\log\det*{2\pi\Sigma} + \frac{1}{2}
        \tr*{\Sigma^{-1}\eqchange{\Sigma}}
        \\ &= \eqcontrast \frac{1}{2}\log\det*{2\pi\Sigma} + \frac{1}{2}
        \tr*{\eqchange{I}} \notag
        \\ &= \eqcontrast \frac{1}{2}\log\det*{2\pi\Sigma} + \frac{1}{2}\eqchange{n} \tag{*}\label[derivationstep]{deriv:lemma-information-entropy-gaussian-univariate}
        \\ &= \eqcontrast \frac{1}{2}\log*{\eqchange{e^n}\det*{2\pi\Sigma}}
        \\ &= \eqcontrast \frac{1}{2}\log\det*{2\pi\eqchange{e}\Sigma}
    \end{derivation}
    \begin{justification}
        \item by \cref{def:information-entropy} (definition of information entropy)
        \item by definition of multivariate Gaussian distribution
        \item by property of logarithm
        \item by linearity of expectation and properties of logarithm and determinant
        \item since \(x = \tr*{x}\) with scalar \(x\)
        \item by cyclicity of trace
        \item by linearity of trace
        \item by linearity of expectation
        \item by definition of covariance matrix
        \item by property of logarithm
        \item by property of determinant
    \end{justification}
    The entropy for the univariate Gaussian distribution follows from \cref{deriv:lemma-information-entropy-gaussian-univariate}.
\end{proof}


\section{Proofs for \cref{chp:problem}}

\subsection{Quantor- vs.\ Limit-based Convergence}\label{ssec:proofs-auxiliary-convergence}

For the sake of completeness, we show the following result
\begin{equation}
    \begin{alignedat}{2}
        && \forall \varepsilon > 0 \exists N \geq 1 \forall n \geq N\colon
        & a_n < \varepsilon
        \quad\text{ and }\quad
        \lim_{n\to\infty} a_n \defeq a \text{ exists}
        \\
        \equivalent\quad
        && \forall \varepsilon > 0 \exists N \geq 1 \forall n \geq N\colon
        & \abs{a_n} < \varepsilon
        \quad\eqnote{or equivalently \(\lim_{n\to\infty} a_n \leq 0\)}
    \end{alignedat}
\end{equation}
which is intuitive and might appear trivial. This formally justifies that the definitions of convergence to near-optimality in \cref{eq:near-optimality} and of no-regret in \cref{eq:no-regret} are as strong as one would normally define them in terms of \(\lim_{T\to\infty} \ffrac{R(T)}{T} = 0\) and \(\lim_{t\to\infty} r(t) = 0\) for \textit{non-negative} regret functions -- up to the existence of the limit.

\begin{proof}
    We show ``\(\Longrightarrow\)'' in \cref{ps:auxiliary-convergence-ps1} and ``\(\Longleftarrow\)'' in \cref{ps:auxiliary-convergence-ps2}.
    
    \begin{proofstep}[\(\Longrightarrow\)]\label{ps:auxiliary-convergence-ps1}
        \begin{derivation}
            && & \lim_{n\to\infty} a_n \defeq a \text{ exists}
            \quad\text{ and }\quad
            \forall \varepsilon > 0 \exists N \geq 1 \forall n \geq N\colon a_n < \varepsilon \notag
            \\
            \implies\quad
            && & \eqcontrast \eqchange{\forall \varepsilon > 0 \exists N \geq 1 \forall n \geq N\colon
            \abs{a_n - a} < \varepsilon}
            \\
            \implies\quad
            && & \eqcontrast \forall \varepsilon > 0 \exists N \geq 1 \forall n \geq N\colon
            \eqchange{a - \varepsilon < a_n < a + \varepsilon} \notag
            \\
            \implies\quad
            && & \eqcontrast \forall \varepsilon > 0 \exists N \geq 1 \forall n \geq N\colon
            a - \varepsilon < a_n < \eqchange{\varepsilon}
            \\
            \implies\quad
            && & \eqcontrast \forall \varepsilon > 0\colon
            a < \eqchange{2\varepsilon \notag}
            \\
            \implies\quad
            && & \eqcontrast a \eqchange{\leq 0}
        \end{derivation}
        \begin{justification}
            \item by definition of limit
            \item since \(\forall \varepsilon > 0 \exists N \geq 1 \forall n \geq N\colon a_n < \varepsilon\)
            \item since \(\varepsilon\) can be arbitrarily small
        \end{justification}
    \end{proofstep}
    
    \begin{proofstep}[\(\Longleftarrow\)]\label{ps:auxiliary-convergence-ps2}
        \begin{derivation}
            && &\lim_{n\to\infty} a_n \defeq a \text{ exists}
            \quad\text{ and }\quad
            a \leq 0 \notag
            \\
            \quad\implies
            && & \eqcontrast \eqchange{\forall \varepsilon > 0 \exists N \geq 1 \forall n \geq N\colon
            \abs{a_n - a} < \varepsilon}
            \\
            \quad\implies
            && & \eqcontrast \forall \varepsilon > 0 \exists N \geq 1 \forall n \geq N\colon
            \abs{a_n - a} \eqchange{+ a} < \varepsilon \eqchange{+ a} \notag
            \\
            \quad\implies
            && & \eqcontrast \forall \varepsilon > 0 \exists N \geq 1 \forall n \geq N\colon
            \eqchange{a_n} < \eqchange{\varepsilon}
        \end{derivation}
        \begin{justification}
            \item by definition of limit
            \item since
            \begin{alignat*}{2}
                a_n - a \leq \abs{a_n - a}
                & \quad\equivalent\quad
                a_n \leq \abs{a_n-a}+a
                && \quad\eqnote{left side}
                \\
                a \leq 0
                & \quad\equivalent\quad
                \varepsilon + a \leq \varepsilon
                && \quad\eqnote{right side}
                \tag*{\qedhere}
            \end{alignat*}
        \end{justification}
    \end{proofstep}
\end{proof}

\subsection{\fullref{thm:near-optimality}}\label{ssec:proofs-theorem-near-optimality}

A similar proof was given by \textcite[Lemma 8]{prajapat2022nearoptimal}. The main difference in our setting is the dependence of the optimal solution \(\thetaopt_T\) on the number of measurements \(T\), which allows us to only show pseudo-convergence to near-optimality as stated in \cref{remark:true-vs-pseudo-near-optimality}. A more detailed discussion is provided in \cref{sec:problem-comparison}.

\begin{proof}
    In \cref{ps:theorem-near-optimality-ps1} we assume sublinear regret as defined in \cref{eq:sublinear-regret} and derive no-regret as defined in \cref{eq:no-regret}. In \cref{ps:theorem-near-optimality-ps2} we continue with no-regret and the desired pseudo-convergence to near-optimality.
    
    \begin{proofstep}\label{ps:theorem-near-optimality-ps1}
        Besides the technical difference of taking negative regret into account, the proof for showing no-regret from sublinear regret is straightforward.
        \begin{derivation}
            && &R(T) \leq \bigO*{T^n} \text{ with } n < 1 \notag
            \\
            \implies\quad
            && &\eqcontrast \eqchange{\forall c > 0\exists T_0 \geq 1 \forall T \geq T_0\colon} R(T) \leq \eqchange{c \cdot T^n}
            \\
            \implies\quad
            && &\eqcontrast \forall c > 0\exists T_0 \geq 1 \forall T \geq T_0\colon \eqchange{\frac{\eqnochange{R(T)}}{T}} \leq c \cdot T^{n \eqchange{- 1}} \eqchange{< c}
            \\
            \implies\quad
            && &\eqcontrast \forall \eqchange{\varepsilon} > 0\exists T_0 \geq 1 \forall T \geq T_0\colon \frac{R(T)}{T} < \eqchange{\varepsilon}
        \end{derivation}
        \begin{justification}
            \item by definition of \(\bigO\)-notation
            \item since \(T^{n-1} = \frac{1}{T^{1-n}} < 1\) with \(T\geq T_0 \geq 1\) and \(1 - n > 0\)
            \item by instantiating \(c\) with value smaller \(\varepsilon\)
        \end{justification}
    \end{proofstep}
    
    \begin{proofstep}\label{ps:theorem-near-optimality-ps2}
        The main idea for this proof is that the average regret upper bounds the minimum regret up to time \(T\). Hence, if the average regret is non-positive asymptotically, there must be one round within finite time where the simple regret is non-positive up to some precision \(\varepsilon\).
        \begin{derivation}
            && &\forall \varepsilon > 0\exists T_0 \geq 1 \forall T \geq T_0\colon \frac{R(T)}{T} < \varepsilon
            \\
            \implies\quad
            && &\eqcontrast \forall \varepsilon > 0\exists T_0 \geq 1 \forall T \geq T_0\colon \frac{1}{T} \eqchange{\sum_{t=1}^T r(t)} < \varepsilon \label{deriv:theorem-near-optimality-ps2-fork}
            \\
            \implies\quad
            && &\eqcontrast \forall \varepsilon > 0\exists T_0 \geq 1 \forall T \geq T_0\colon \eqchange{\min_{t=1,\dots,T} r(t)} < \varepsilon
            \\
            \implies\quad
            && &\eqcontrast \forall \varepsilon > 0\exists T_0 \geq 1\colon \min_{t=1,\dots,\eqchange{T_0}} r(t) < \varepsilon
            \\
            \implies\quad
            && &\eqcontrast \forall \varepsilon > 0\exists T_0 \geq 1 \eqchange{\exists T \leq T_0}\colon r(\eqchange{T}) < \varepsilon
        \end{derivation}
        \begin{justification}
            \item by \cref{ps:theorem-near-optimality-ps1}
            \item by \cref{eq:simple-regret} (definition of simple regret)
            \item since minimum \(\leq\) average
            \item by instantiating \(\inst{T_0}{T}\)
            \item by definition of minimum
        \end{justification}

        Hence, we have shown pseudo-convergence to near-optimality from no-regret. However, this is not possible for true convergence to near-optimality due to the counterexample
        \begin{equation*}
            x_n \defeq \begin{cases}
                \sqrt{n}, & \text{\(n = k^3\) with \(k\in\Natural\)}\\
                0,        & \text{otherwise}
                ,
            \end{cases}
        \end{equation*}
        whose average converges to zero, but the series itself does not due to sparks which become asymptotically sparser \autocite{quasi2018answer}. Depending on the setting, true convergence can be shown under additional conditions as described in \cref{cor:near-optimality}.
    \end{proofstep}
\end{proof}

\subsection{\fullref{cor:near-optimality}}\label{ssec:proofs-corollary-near-optimality}

\begin{proof}
    We first show the statement for the more general assumption 2 and then proceed with proving the statement under the stronger assumption 1.
    
    Assume that condition 2 is satisfied and \(r(t)\) decreases monotonically in \(t\). Given the result from \cref{thm:near-optimality}, we can immediately conclude
    \begin{derivation}
        && &\forall \varepsilon > 0 \exists T_0 \geq 1 \exists T \leq T_0\colon r(T) < \varepsilon
        \\ \implies\quad && & \eqcontrast \forall \varepsilon > 0 \exists T_0 \geq 1 \exists T \leq T_0 \eqchange{\forall t \geq T}\colon \eqchange{r(t) \leq} r(T) < \varepsilon
        \\ \implies\quad && & \eqcontrast \forall \varepsilon > 0 \exists T_0 \geq 1 \eqchange{\forall t \geq T_0}\colon r(t) < \varepsilon
    \end{derivation}
    \begin{justification}
        \item by \cref{thm:near-optimality}
        \item by \textit{assumption}
        \item by weakening the statement from \(t\geq T\) to \(t\geq T_0\)
    \end{justification}
    
    Assume that condition 1 is satisfied. Observe that \(r(t)\) decreases monotonically, since \(F(x^\star)\) is constant in \(t\) and \(F(x_t)\) increases monotonically in \(t\). Hence, condition 2 is satisfied and the statement follows.

    Alternatively for condition 1, we can follow the argumentation of \textcite[Section II]{srinivas2012informationtheoretic} that the maximum utility \(\max_{t\leq T} F(x_t)\) must be closer to the time-independent optimal utility \(F(x^\star)\) than the average utility \(\frac{1}{T}\sum_{t=1}^T F(x_t)\). Since the maximum utility corresponds to the final utility by monotonicity, the final regret must be smaller than the average regret.
    \begin{derivation}
        \frac{1}{T} \sum_{t=1}^T r(t)
        &= \eqcontrast \frac{1}{T} \sum_{t=1}^T \eqchange{\paren*{(1-\alpha)F(x^\star) - F(x_t)}}
        \\ &= \eqcontrast \eqchange{(1-\alpha)F(x^\star)} - \frac{1}{T} \sum_{t=1}^T F(x_t)
        \\ &\geq \eqcontrast (1-\alpha)F(x^\star) - \eqchange{\max_{t=1} F(x_t)}
        \\ &= \eqcontrast (1-\alpha)F(x^\star) - \eqchange{F(x_T)}
        \\ &= \eqcontrast \eqchange{r(T)}
    \end{derivation}
    \begin{justification}
        \item by condition 2 (\textit{assumption on definition or regret})
        \item by condition 2 (\textit{assumption on time-independence of optimal decision})
        \item since maximum \(\geq\) average
        \item by condition 2 (\textit{assumption on monotonicity of utility})
        \item by condition 2 (\textit{assumption on definition or regret})
    \end{justification}
    Hence, continuing at \cref{deriv:theorem-near-optimality-ps2-fork} in the proof for \cref{thm:near-optimality}, we can directly derive true convergence for the simple regret based on the convergence of the average regret.
    \begin{derivation*}
        && &\forall \varepsilon > 0\exists T_0 \geq 1 \forall T \geq T_0\colon \frac{1}{T} \sum_{t=1}^T r(t) < \varepsilon
        \\
        \implies\quad
        && &\eqcontrast \forall \varepsilon > 0\exists T_0 \geq 1 \forall T \geq T_0\colon \eqchange{r(T)} < \varepsilon
        \tag*{\qedhere}
    \end{derivation*}
\end{proof}

\subsection{\cref{lem:lemma-2-1}}\label{ssec:proofs-lemma-2-1}

The goal of this proof is to relate the optimal solution \(\thetaopt_T\) and the set of greedy decisions \(\thetagre_{1:T}\), such that we can upper bound \(R(T)\) from \cref{eq:cumulative-regret} defined in terms of \(\thetaopt_T\) with \(R_{ind}(T)\) from \cref{eq:individual-regret} defined with respect to \(\thetagre_t\).

The idea of this proof is to monitor the \emph{gap to optimality}
\begin{equation}\label{eq:gap-to-optimality}
    \delta_t \defeq F(\thetaopt_T) - F(\theta_{1:t})
    \quad\with t \geq 1
\end{equation}
\TODO{REMOVABLE CHECK \(t \leq T\) required?}
for a fixed \(T \geq 1\). Since \(F(\thetaopt_T)\) is constant and \(F(\theta_{1:t})\) monotonic increasing in \(t\) by \cref{eq:monotonic-utility}, this gap naturally decreases. By analyzing how \(\delta_t\) reduces over time, we can observe that the amount of decrease in each round can be related to \(r_{ind}(t)\). Since \(\delta_T\) implicitly reflects \(R(T)\), we can establish the upper bound relation between \(R(T)\) and \(R_{ind}(T)\).

This proof is taken from \textcite[Lemma 1 and 2]{prajapat2022nearoptimal} and adapted to our setting. Similar proofs were given by \textcites[Lemma 1 and 2]{yue2011linear}[Theorem 1]{chen2017interactive}.

\begin{proof}
    In \cref{ps:lemma-2-1-ps1}, we derive a recursive formula relating \(\delta_t\) with \(\delta_{t-1}\) and provides the insight that the gap from round \(t-1\) to \(t\) is reduced by at most \(r_{ind}(t)\). In \cref{ps:lemma-2-1-ps2}, we repeatedly apply this recursive formula to relate \(\delta_T\) with the initial \(\delta_0\). Finally, it is straightforward to show the desired result in \cref{ps:lemma-2-1-ps3}.
    
    \begin{proofstep}\label{ps:lemma-2-1-ps1}
        The goal is to show
        \begin{equation*}
            \delta_t \leq \paren*{1-\frac{1}{T}} \delta_{t-1} + r_{ind}(t)
        \end{equation*}
        for all \(t\geq1\) and \(T\geq1\). \TODO{REMOVABLE CHECK condition \(t\leq T\) required?}

        To make the derivation easier to parse at one point, we first define the marginal utility for a set of decisions as
        \begin{equation}\label{eq:marginal-utility-set}
            F(\theta_{t:t'} \mid \theta_{1:t-1}) \defeq \sum_{\tau=t}^{t'} F(\theta_\tau \mid F_{1:\tau-1})
            .
        \end{equation}
        The important relation to the utility is
        \begin{equation}\label{eq:marginal-utility-set-relation}
            \begin{alignedat}{2}
                F(\theta_{t:t'} \mid \theta_{1:t-1})
                &= \sum_{\tau=t}^{t'} (F(\theta_{1:\tau}) - F(\theta_{1:\tau-1}))
                &&\quad\eqnote{by \crefnosort{eq:marginal-utility-set,eq:marginal-utility}}\\
                &= F(\theta_{1:t'}) - F(\theta_{1:t-1})
                &&\quad\eqnote{since telescoping sum}
                ,
            \end{alignedat}
        \end{equation}
        which matches the intuition of marginal utility.
        \begin{derivation}
            && \delta_{t-1}
            &= \eqcontrast \eqchange{F(\thetaopt_T) - F(\theta_{1:t-1})}
            \\ && &\leq \eqcontrast F(\thetaopt_T \eqchange{\cup \theta_{1:t-1}}) - F(\theta_{1:t-1}) \label{deriv:lemma-2-1-ps1-monotonicity}
            \\ && &= \eqcontrast \eqchange{F(\thetaopt_T \mid \theta_{1:t-1})}
            \\ && &= \eqcontrast \eqchange{\sum_{\tau=1}^T} F(\eqchange{(\thetaopt_T)_\tau} \mid \theta_{1:t-1} \cup \eqchange{(\thetaopt_T)_{1:\tau-1}})
            \\ && &\leq \eqcontrast \sum_{\tau=1}^T F((\thetaopt_T)_\tau \mid \theta_{1:t-1}) \label{deriv:lemma-2-1-ps1-submodularity}
            \\ && &\leq \eqcontrast \sum_{\tau=1}^T F(\eqchange{\thetagre_\tau} \mid \theta_{1:t-1}) \label{deriv:lemma-2-1-ps1-greedy}
            \\ && &= \eqcontrast \eqchange{T \cdot} F(\thetagre_\tau \mid \theta_{1:t-1}) \notag
            \\
            \implies\quad    
            && \eqcontrast \eqchange{\frac{1}{T}} \delta_{t-1}
            &\leq \eqcontrast F(\thetagre_t \mid \theta_{1:t-1}) \tag{*}\label[derivationstep]{deriv:lemma-2-1-ps1-intuition}
            \\ && &= \eqcontrast \eqchange{r_{ind}(t) + F(\theta_t \mid \theta_{1:t-1})}
            \\ && &= \eqcontrast r_{ind}(t) + \eqchange{F(\theta_{1:t}) - F(\theta_{1:t-1})}
            \\ && &= \eqcontrast r_{ind}(t) + (\eqchange{F(\thetaopt_T)} - F(\theta_{1:t-1})) - (\eqchange{F(\thetaopt_T)} + F(\theta_{1:t})) \notag
            \\ && &= \eqcontrast r_{ind}(t) + \eqchange{\delta_{t-1}} - \eqchange{\delta_t}
            \\
            \implies\quad
            &&\eqcontrast \eqchange{\delta_t} &\leq \eqcontrast \paren*{1\eqchange{-\frac{1}{T}}} \delta_{t-1} + r_{ind}(t) \notag
        \end{derivation}
        \begin{justification}
            \item by \cref{eq:gap-to-optimality} (definition of gap to optimality)
            \item by \cref{eq:monotonic-utility} (monotonicity of utility)
            \item by \cref{eq:marginal-utility-set-relation} (relation between utility and marginal utility)
            \item by \cref{eq:marginal-utility-set} (definition of marginal utility for sets)
            \item by \cref{eq:submodular-utility} (submodularity of utility)
            \item by \cref{eq:greedy-decision} (definition of greedy decision)
            \item by \cref{eq:individual-regret} (definition of simple individual regret)
            \item by \cref{eq:marginal-utility} (definition of marginal utility)
            \item by \cref{eq:gap-to-optimality} (definition of gap to optimality)
        \end{justification}
        
        The insight is if we evaluate the marginal utility of each of the \(T\) optimal decisions \((\thetaopt_T)_1,\dots,(\thetaopt_T)_T\) individually with respect to \(\theta_{1:t-1}\) (see \cref{deriv:lemma-2-1-ps1-submodularity}), it cannot exceed the marginal utility of the greedy decision \(\thetagre_t\) with respect to \(\theta_{1:t-1}\) (see \cref{deriv:lemma-2-1-ps1-greedy}) by definition of the greedy decision in \cref{eq:greedy-decision}.
        Together with monotonocity and submodularity defined in \cref{deriv:lemma-2-1-ps1-submodularity,deriv:lemma-2-1-ps1-monotonicity}, this provides us the guarantee that the marginal utility of the greedy decision closes at least one \(T\)-th of the previous gap to optimality \(\delta_{t-1}\) in each round as stated in \cref{deriv:lemma-2-1-ps1-intuition}.
        With this relation between optimal and greedy decisions, we can upper bound the regret with the individual regret.
        
    \end{proofstep}
    
    \begin{proofstep}\label{ps:lemma-2-1-ps2}
        The goal is to show
        \begin{equation*}
            \delta_T < \frac{1}{e} \delta_0 + R_{ind}(T)
        \end{equation*}
        for all \(T\geq1\) by recursively applying \cref{ps:lemma-2-1-ps1}.
        \begin{derivation}
            && \delta_t
            &\leq \paren*{1-\frac{1}{T}} \delta_{t-1} + r_{ind}(t)
            \\ && &\leq \eqcontrast \paren*{1-\frac{1}{T}} \eqchange{\paren*{\paren*{1-\frac{1}{T}} \delta_{t-2} + r_{ind}(t-1)}} + r_{ind}(t)
            \\ && &= \eqcontrast \paren*{1-\frac{1}{T}}^{\eqchange{2}} \eqchange{\delta_{t-2}} + \paren*{1-\frac{1}{T}} \eqchange{r_{ind}(t-1)} + r_{ind}(t) \notag
            \\ && &\leq \eqcontrast \cdots \notag
            \\ && &\leq \eqcontrast \paren*{1-\frac{1}{T}}^{\eqchange{t}} \eqchange{\delta_0} + \eqchange{\sum_{\tau=1}^t \paren*{1-\frac{1}{T}}^{t-\tau} r_{ind}(\tau)}
            \\ && &\leq \eqcontrast \paren*{1-\frac{1}{T}}^t \delta_0 + \sum_{\tau=1}^t r_{ind}(\tau)
            \eqnocontrast \quad\text{ for all \(t,T\geq 1\)}
            \\
            \implies\quad
            &&\eqcontrast \delta_{\eqchange{T}}
            &\leq \eqcontrast \paren*{1-\frac{1}{T}}^{\eqchange{T}} \delta_0 + \sum_{\tau=1}^{\eqchange{T}} r_{ind}(\tau)
            \\ && &< \eqcontrast \eqchange{\frac{1}{e}} \delta_0 + \sum_{\tau=1}^T r_{ind}(\tau)
            \\ && &= \eqcontrast \frac{1}{e} \delta_0 + \eqchange{R_{ind}(T)}
        \end{derivation}
        \begin{justification}
            \item by \cref{ps:lemma-2-1-ps1}
            \item by \cref{ps:lemma-2-1-ps1}
            \item by \cref{ps:lemma-2-1-ps1} applied recursively
            \item since \(1 - \frac{1}{T} \leq 1\) for all \(T \geq 1\) and \(t-\tau \geq 0\)
            \item by instantiating \(\inst{t}{T}\)
            \item since \(\paren*{1-\frac{1}{T}}^T < \frac{1}{e}\) for all \(T \geq 1\)
            \item by \cref{eq:individual-regret} (definition of cumulative individual regret)
        \end{justification}
    \end{proofstep}
    
    \begin{proofstep}\label{ps:lemma-2-1-ps3}
        Using \cref{ps:lemma-2-1-ps2} it is straightforward to show the final result \(R(T) < R_{ind}(T)\) for all \(T\geq1\).
        \begin{derivation}
            && \delta_T
            &< \frac{1}{e} \delta_0 + R_{ind}(T)
            \\
            \implies\quad
            &&\eqcontrast \eqchange{F(\thetaopt_T) - F(\theta_{1:T})}
            &< \eqcontrast \frac{1}{e} \eqchange{F(\thetaopt_T)} + R_{ind}(T)
            \\
            \implies\quad
            &&\eqcontrast \paren*{1\eqchange{-\frac{1}{e}}} F(\thetaopt_T) - F(\theta_{1:T})
            &< \eqcontrast R_{ind}(T) \notag
            \\
            \implies\quad
            &&\eqcontrast \eqchange{R(T)}
            &< \eqcontrast R_{ind}(T)
        \end{derivation}
        \begin{justification}
            \item by \cref{ps:lemma-2-1-ps2}
            \item by \cref{eq:gap-to-optimality} (definition of gap to optimality)
            \item by \cref{eq:cumulative-regret} (definition of cumulative regret)
        \end{justification}
        
        Now it becomes obvious that the \(\paren*{1-\frac{1}{e}}\) approximation guarantee for the greedy algorithm comes from \cref{deriv:lemma-2-1-ps1-intuition} in \cref{ps:lemma-2-1-ps1}, which states that the greedy decisions close at least one \(T\)-th of the previous gap to optimality in every round. Hence, after all \(T\) rounds this gap is reduced to at most \(\paren*{1-\frac{1}{T}}^T < \frac{1}{e}\).
    \end{proofstep}
\end{proof}

\subsection{\cref{lem:lemma-2-2}}\label{ssec:proofs-lemma-2-2}

The basic idea of this proof is to upper and lower bound the marginal utility. This allows us to obtain an upper bound on the individual regret defined as the difference in marginal utility with respect to the greedy decision. The lower bound \(F_l\) is naively set to 0, while the upper bound is ensured by the assumption on the objective function \(\Fu\).

\begin{proof}
    We assume
    \begin{equation}\label{eq:proofs-lemma-2-2-assumption}
        F(\thetagre_t \mid \theta_{1:t-1}) \leq \Fu(\theta_t \mid \theta_{1:t-1})
        \quad\text{ for all } t \geq 1
        .
    \end{equation}
    
    This allows us to show:
    \begin{derivation}
        R_{ind}(T)
        &= \eqcontrast \eqchange{\sum_{t=1}^T r_{ind}(t)}
        \\ &= \eqcontrast \sum_{t=1}^T \eqchange{(F(\thetagre_t \mid \theta_{1:t-1}) - F(\theta_t \mid \theta_{1:t-1}))}
        \\ &\leq \eqcontrast \sum_{t=1}^T F(\thetagre_t \mid \theta_{1:t-1})
        \\ &\leq \eqcontrast \sum_{t=1}^T \eqchange{\Fu}(\eqchange{\theta_t} \mid \theta_{1:t-1})
    \end{derivation}
    \begin{justification}
        \item by \cref{eq:individual-regret} (definition of cumulative individual regret)
        \item by \cref{eq:individual-regret} (definition of simple individual regret)
        \item since \(F(\theta \mid \theta_{1:t-1}) \geq F_l(\theta \mid \theta_{1:t-1}) \defeq 0\) for all \(\theta \in \Camspace\)
        \item by \cref{eq:proofs-lemma-2-2-assumption} (\textit{assumption})
        \qedhere
    \end{justification}
\end{proof}


\section{Proofs for \cref{chp:analysis}}

\subsection{\fullref{lem:lemma-2-4}}\label{ssec:proofs-lemma-2-4}

The idea of this proof is to first show that the upper and lower bound holds for a single \(\varphi\in\Domain\) \whp{} and then to apply \textit{union bound} over all \(\varphi\in\Domain\). The problem is that \(\Domain = [0,2\pi]\) consists of infinitely many \(\varphi\) and union bound over \(\Domain\) does not lead to the desired result. The trick is to discretize \(\Domain\) into a finite set of points \(\Domain_t\), for which the confidence bounds can be ensured with the union bound. With the additional assumption of probabilistically bounded derivatives
\begin{equation}\label{eq:proofs-lemma-2-4-assumption}
    \Pr*{\sup_{\varphi\in\Domain} \abs*{\deriv{f}{\varphi}(\varphi)} \leq L} \geq 1 - ae^{-L^2/b^2}
    \quad\text{ for some } a,b > 0
    ,
\end{equation}
or equivalently Lipschitz-continuous functions, we can exclude very wild \(f\sim\GP*{m,k}\) which would be able to escape the confidence bounds between the discretization points. This allows us to ensure that for the chosen \(\beta_t\) the confidence bounds also hold for the remaining points \(\varphi\notin\Domain_t\), although with a small discretization error of \(\frac{1}{t^2}\) in the final bound.

This proof is taken from \textcite[Theorem 2]{srinivas2012informationtheoretic} and adapted to our setting.

\begin{proof}
    We choose the confidence parameter as
    \begin{align}
        \beta_t &= 2 \log*{\frac{\abs{\Domain_t}\pi_t}{\delta/2}} = 2\log*{\frac{2\pi^3b}{3}\sqrt{\log*{\frac{2a}{\delta}}}\frac{t^4}{\delta}} \label{eq:proofs-lemma-2-4-beta}\\
        \shortintertext{with}
        \pi_t &= \frac{\pi^2}{6}t^2 \label{eq:proofs-lemma-2-4-pi}\\
        \abs{\Domain_t} &= 2\pi b\sqrt{\log*{\frac{2a}{b}}}\cdot t^2 \label{eq:proofs-lemma-2-4-discretization}
    \end{align}
    and parameters \(a,b > 0\) specified in the assumption from \cref{eq:proofs-lemma-2-4-assumption} and \(\delta \in (0,1)\) selected at your own discretion. The reasons for these specific choices can be found below in the proof.
    
    This proof is divided in three steps. In \cref{ps:lemma-2-4-ps1}, we show that the confidence bounds hold for all \(\varphi\in\Domain_t\) with probability at least \(1-\frac{\delta}{2}\). In \cref{ps:lemma-2-4-ps2}, we show that the discretization error can be upper bounded with \(\frac{1}{t^2}\) with probability at least \(1-\frac{\delta}{2}\). \cref{ps:lemma-2-4-ps3} combines both of them and shows the desired result.

    \paragraph{Discretization}
    For a given discretization \(\Domain_t \subseteq \Domain\), we define
    \begin{equation*}
        [\varphi]_t \defeq \argmin_{\tilde{\varphi} \in \Domain_t} \abs{\tilde{\varphi} - \varphi}
    \end{equation*}
    to be the closest point in \(\Domain_t\) to \(\varphi\). We choose the discretized domain \(\Domain_t \subseteq \Domain\), such that
    \begin{equation}\label{eq:proofs-lemma-2-4-closest-neighbor}
         \abs{\varphi-[\varphi]_t} \leq \frac{2\pi}{\abs{\Domain_t}}
    \end{equation}
    is satisfied. For example, the uniform discretization of \(\Domain\) satisfies this condition. The specific choice for the discretization granularity is given by \cref{eq:proofs-lemma-2-4-discretization}.
    
    \begin{proofstep}\label{ps:lemma-2-4-ps1}
        The goal is to show
        \begin{equation*}
            \Pr*{\forall t \geq 1\forall\varphi\in\Domain_t\colon \abs{f(\varphi)-\mu_{t-1}(\varphi)} \leq \beta_t^{1/2}\sigma_{t-1}(\varphi)}
            \geq 1 - \frac{\delta}{2}
            .
        \end{equation*}
        We first show a general result based on the exponential decay of the Gaussian probability density function, which is also referred to as concentration guarantees. For some \(c > 0\) and an arbitrary \(z\sim\Normal(0,1)\), the probability of sampling outliers with \(\abs{z} > c\) can be upper bounded with an exponentially decreasing bound the larger we choose \(c\). This becomes useful, since points on the sampled surface function follow a Gaussian distribution.
        \begin{derivation}
            && \Pr*{z>c}
            & = \eqcontrast \eqchange{\int_c^\infty \frac{1}{\sqrt{2\pi}} e^{-\frac{1}{2}z^2} \d{z}}
            \eqnocontrast \quad\with z\sim\Normal*{0,1}
            \\ && & = \eqcontrast \int_c^\infty \frac{1}{\sqrt{2\pi}} e^{-\frac{1}{2} \eqchange{\paren*{(z-c)^2+2zc-c^2}}} \d{z}
            \\ && & = \eqcontrast \int_c^\infty \frac{1}{\sqrt{2\pi}} e^{\eqchange{-\frac{1}{2}}(z-c)^2 \eqchange{-}zc \eqchange{+ \frac{1}{2}}c^2} \d{z} \notag
            \\ && & = \eqcontrast \int_c^\infty \frac{1}{\sqrt{2\pi}} e^{-\frac{1}{2}(z-c)^2 - c(z\eqchange{-c}) \eqchange{- \frac{1}{2}c^2}} \d{z} \notag
            \\ && & = \eqcontrast \eqchange{e^{-\frac{1}{2}c^2} \frac{1}{\sqrt{2\pi}}} \int_c^\infty \eqchange{e^{-\frac{1}{2}(z-c)^2} e^{-c(z-c)}} \d{z} \notag
            \\ && & \leq \eqcontrast e^{-\frac{1}{2}c^2} \frac{1}{\sqrt{2\pi}} \int_c^\infty e^{-\frac{1}{2}(z-c)^2} \d{z}
            \\ && & = \eqcontrast e^{-\frac{1}{2}c^2} \frac{1}{\sqrt{2\pi}} \int_{\eqchange{0}}^\infty e^{-\frac{1}{2}\eqchange{z}^2} \d{z}
            \\ && & = \eqcontrast e^{-\frac{1}{2}c^2} \eqchange{\Pr*{z>0}}
            \\ && & = \eqcontrast \eqchange{\frac{1}{2}} e^{-\frac{1}{2}c^2}
            \\
            \implies\quad
            && \eqcontrast \Pr*{\eqchange{\abs{\eqnochange{z}}} > c} &\leq \eqcontrast e^{-\frac{1}{2}c^2} \label{deriv:lemma-2-4-ps1-7}
        \end{derivation}
        \begin{justification}
            \item by definition of \(\Normal*{0,1}\)
            \item by completing the square
            \item since \(e^{-c(z-c)} \leq 1\) with \(z \geq c\) (see integration bounds) and \(c > 0\)
            \item by change of integration bounds
            \item by definition of \(\Normal*{0,1}\)
            \item by symmetry of \(\Normal*{0,1}\)
            \item by symmetry of \(\Normal*{0,1}\)
        \end{justification}
        
        The result in \cref{deriv:lemma-2-4-ps1-7} holds for a single point \(z\). We use the union bound to apply this on all finitely many points in \(\Domain_t\) over infinitely many rounds \(t\geq1\). This is where the choice of \(\pi_t\) in \cref{eq:proofs-lemma-2-4-beta} comes into play, which upper bounds the outlier probability for a single round \(t\) with \(\bigO*{\frac{1}{t^2}}\). Hence, the outlier probability over all rounds \(t\geq1\) is constant.
        \begin{derivation}
            && \eqcontrast \Pr*{\eqchange{\frac{\eqnochange{\abs{\eqchange{f(\varphi)-\mu_{t-1}(\varphi)}}}}{\sigma_{t-1}(\varphi)}} > \eqchange{\beta_t^{1/2}}}
            &\leq \eqcontrast e^{-\frac{1}{2}\eqchange{\beta_t}}
            \\
            \implies\quad
            && \eqcontrast \Pr*{\eqchange{\exists\varphi\in\Domain_t\colon} \frac{\abs{f(\varphi)-\mu_{t-1}(\varphi)}}{\sigma_{t-1}(\varphi)} > \beta_t^{1/2}}
            &\leq \eqcontrast \eqchange{\abs{\Domain_t}} e^{-\frac{1}{2}\beta_t}
            \\ && &= \eqcontrast \eqchange{\frac{\delta/2}{\pi_t}}
            \\
            \implies\quad
            && \eqcontrast \Pr*{\multlinedcmd[6cm]{\eqchange{\exists t \geq 1}\exists\varphi\in\Domain_t\colon \\ \frac{\abs{f(\varphi)-\mu_{t-1}(\varphi)}}{\sigma_{t-1}(\varphi)} > \beta_t^{1/2}}}
            &\leq \eqcontrast \eqchange{\sum_{t\geq1}} \frac{\delta/2}{\pi_t}
            \\ && &= \eqcontrast \frac{\delta}{2}
            \\
            \implies\quad
            && \eqcontrast \Pr*{\multlinedcmd[6cm]{\eqchange{\forall} t \geq 1\eqchange{\forall}\varphi\in\Domain_t\colon \\ \frac{\abs{f(\varphi)-\mu_{t-1}(\varphi)}}{\sigma_{t-1}(\varphi)} \eqchange{\leq} \beta_t^{1/2}}}
            &\geq \eqcontrast \eqchange{1 -} \frac{\delta}{2} \notag
            \\
            \implies\quad
            && \eqcontrast \Pr*{\multlinedcmd[6cm]{\forall t \geq 1\forall\varphi\in\Domain_t\colon \\ \abs{f(\varphi)-\mu_{t-1}(\varphi)} \leq \beta_t^{1/2}\eqchange{\sigma_{t-1}(\varphi)}}}
            &\geq \eqcontrast 1 - \frac{\delta}{2} \notag
        \end{derivation}
        \begin{justification}
            \item by instantiating \(\inst{z}{\frac{f(\varphi)-\mu_{t-1}(\varphi)}{\sigma_{t-1}(\varphi)}} \sim \Normal*{0,1}\) and \(\inst{c}{\beta_t^{1/2}} > 0\)
            \item by union bound over \(\varphi\in\Domain_t\)
            \item by \cref{eq:proofs-lemma-2-4-beta} with \(\beta_t = 2 \log*{\frac{\abs{\Domain_t}\pi_t}{\delta/2}} \equivalent \abs{\Domain_t} e^{-\frac{1}{2}\beta_t} = \frac{\delta/2}{\pi_t}\)
            \item by union bound over \(t \geq 1\)
            \item by \cref{eq:proofs-lemma-2-4-pi} with \(\sum_{t\geq1} \frac{1}{\pi_t} = \frac{6}{\pi^2} \sum_{t\geq1} \frac{1}{t^2} = 1\)
        \end{justification}
    \end{proofstep}
    
    \begin{proofstep}\label{ps:lemma-2-4-ps2}
        The goal is to show
        \begin{equation*}
            \Pr*{\forall t\geq1\varphi\in\Domain\colon \abs*{f(\varphi) - f([\varphi]_t)} \leq \frac{1}{t^2}}
            \geq 1 - \frac{\delta}{2}
            .
        \end{equation*}
        The bounded derivatives assumption in \cref{eq:proofs-lemma-2-4-assumption} helps us in providing the lower bound on the probability, while the choice of \(\abs{\Domain_t} \in \bigO*{t^2}\) in \cref{eq:proofs-lemma-2-4-discretization} helps us to upper bound the discretization error by \(\bigO*{\frac{1}{t^2}}\) by making the discretization finer over time, such that the total discretization error for the confidence bounds over all rounds \(t\geq1\) is constant.
        \begin{derivation}
            && \Pr*{\sup_{\varphi\in\Domain} \abs*{\deriv{f}{\varphi}(\varphi)} \leq L}
            &\geq 1 - ae^{-L^2/b^2}
            \\
            \implies\quad
            && \eqcontrast \Pr*{\eqchange{\forall\varphi\in\Domain\colon} \abs*{\deriv{f}{\varphi}(\varphi)} \leq L}
            &\geq \eqcontrast 1 - ae^{-L^2/b^2}
            \\
            \implies\quad
            && \eqcontrast \Pr*{\forall\varphi,\eqchange{\varphi'}\in\Domain\colon \eqchange{\abs*{f(\varphi) - f(\varphi')} \leq L\abs*{\varphi-\varphi'}}}
            &\geq \eqcontrast 1 - ae^{-L^2/b^2}
            \\
            \implies\quad
            && \eqcontrast \Pr*{\multlinedcmd[7cm]{\forall\varphi,\varphi'\in\Domain\colon \\ \abs*{f(\varphi) - f(\varphi')} \leq \eqchange{b\sqrt{\log*{\tfrac{2a}{\delta}}}}\abs*{\varphi-\varphi'}}}
            &\geq \eqcontrast 1 - \eqchange{\frac{\delta}{2}}
            \\
            \implies\quad
            && \eqcontrast \Pr*{\multlinedcmd[7cm]{\eqchange{\forall t\geq1}\varphi\in\Domain\colon \\ \abs*{f(\varphi) - f(\eqchange{[\varphi]_t})} \leq b\sqrt{\log*{\tfrac{2a}{\delta}}}\abs*{\varphi-\eqchange{[\varphi]_t}}}}
            &\geq \eqcontrast 1 - \frac{\delta}{2}
            \\
            \implies\quad
            && \eqcontrast \Pr*{\multlinedcmd[7cm]{\forall t\geq1\varphi\in\Domain\colon \\ \abs*{f(\varphi) - f([\varphi]_t)} \leq \eqchange{\frac{2\pi}{\abs{\Domain_t}}}b\sqrt{\log*{\tfrac{2a}{\delta}}}}}
            &\geq \eqcontrast 1 - \frac{\delta}{2}
            \\
            \implies\quad
            && \eqcontrast \Pr*{\forall t\geq1\varphi\in\Domain\colon \abs*{f(\varphi) - f([\varphi]_t)} \leq \eqchange{\frac{1}{t^2}}}
            &\geq \eqcontrast 1 - \frac{\delta}{2}
        \end{derivation}
        \begin{justification}
            \item by \cref{eq:proofs-lemma-2-4-assumption} (\textit{bounded derivatives assumption})
            \item by definition of supremum
            \item by mean value theorem
            \item since \(ae^{-L^2/b^2} = \frac{\delta}{2} \equivalent L = b\sqrt{\log*{\frac{2a}{\delta}}}\) with \(a,\delta>0\)
            \item by instantiating \(\inst{\varphi'}{[\varphi]_t} \in \Domain\)
            \item by \cref{eq:proofs-lemma-2-4-closest-neighbor} (closest neighbor in \(\Domain_t\))
            \item by \cref{eq:proofs-lemma-2-4-discretization} with \(\abs{\Domain_t} = 2\pi b\sqrt{\log*{\frac{2a}{b}}}\cdot t^2 \equivalent \frac{2\pi}{\abs{\Domain_t}} b\sqrt{\log*{\frac{2a}{b}}} = \frac{1}{t^2}\)
        \end{justification}
    \end{proofstep}
    
    \begin{proofstep}\label{ps:lemma-2-4-ps3}
        The goal is to show
        \begin{equation*}
            \Pr*{\forall t\geq1\forall\varphi\in\Domain\colon \abs{f(\varphi) - \mu_{t-1}([\varphi]_t)} \leq \frac{1}{t^2} + \beta_t^{1/2}\sigma_{t-1}([\varphi]_t)} \geq 1-\delta
            .
        \end{equation*}
        Finally, we combine \cref{ps:lemma-2-4-ps1} and \cref{ps:lemma-2-4-ps2} and provide an probabilistic confidence bound for all \(\varphi\in\Domain\) based on the discretization error from \(\varphi\) to \([\varphi]_t\in\Domain_t\) and the confidence bounds for \([\varphi]_t\in\Domain_t\). Note that the probability of \(1-\delta\) is again obtained based on the union bound over the corresponding complement events.
        \begin{derivation}
            && \abs{f(\varphi) - \mu_{t-1}([\varphi]_t)}
            &= \eqcontrast \abs{f(\varphi) \eqchange{- f([\varphi]_t) + f([\varphi]_t)} - \mu_{t-1}([\varphi]_t)} \notag
            \\ && &\leq \eqcontrast \eqchange{\abs{\eqnochange{f(\varphi) - f([\varphi]_t)}}} + \eqchange{\abs{\eqnochange{f([\varphi]_t) - \mu_{t-1}([\varphi]_t)}}}
            \\ \implies\quad
            & \eqcontrast \mathrlap{\eqchange{\Pr*{\multlinedcmd{\forall t\geq1\forall\varphi\in\Domain\colon \\ \eqnochange{\abs{f(\varphi) - \mu_{t-1}([\varphi]_t)} \leq \eqchange{\frac{1}{t^2}} + \eqchange{\beta_t^{1/2}\sigma_{t-1}([\varphi]_t)}}}} \geq 1-\delta}}
        \end{derivation}
        \begin{justification}
            \item by triangle inequality
            \item by \cref{ps:lemma-2-4-ps1}, \cref{ps:lemma-2-4-ps2} and union bound
            \qedhere
        \end{justification}
    \end{proofstep}
\end{proof}

\subsection{\fullref{lem:lemma-2-5}}\label{ssec:proofs-lemma-2-5}

This proof is taken from \textcites[Lemma 5-7]{prajapat2022nearoptimal}[Lemma 5.3 and 5.4]{srinivas2012informationtheoretic} and adapted to our setting. We first show the general statement \cref{lem:lemma-2-5a}.
\begin{proof}
    We assume
    \begin{equation}\label{eq:proofs-lemma-2-5-assumption}
        \abs{k(\varphi,\varphi')} \leq 1
        \quad\text{ for all } \varphi,\varphi'\in\Domain
    \end{equation}
    for the given kernel function \(k\). In addition, we use the following notations
    \begin{equation*}
        \begin{aligned}
            \lambda_{i,t} &\defeq \lambda_i(\Sigma_{t-1}(X_t))\\
            N_T &\defeq \max_{t=1,\dots,T} n_t
            .
        \end{aligned}
    \end{equation*}

    This proof is structured into two steps:
    \begin{derivation}
        \frac{1}{2} \sum_{t=1}^T \sum_{i=1}^{n_t} \sigma_{t-1}(X_{t,i})^2
        &\leq \eqcontrast \eqchange{\frac{N_T}{\log*{\sigmaeps^{-2} + 1}}} \frac{1}{2} \sum_{t=1}^T \sum_{i=1}^{n_t} \eqchange{\log*{\sigmaeps^{-2} \lambda_{i,t} + 1}}
        \\ &= \eqcontrast \frac{N_T}{\log*{\sigmaeps^{-2} + 1}} \eqchange{I(Y_{1:T};f_{1:T})}
    \end{derivation}
    \begin{justification}
        \item In \cref{ps:lemma-2-5-ps1} we establish a relation between the measured uncertainties and the logarithm of the eigenvalues of the covariance matrix with the help of the auxiliary \cref{lem:lemma-aux-x_leq_c_logx_plus_1}.
        \item In \cref{ps:lemma-2-5-ps2} we use this relation to relate the measured uncertainties to the information gain. We mainly make use of the known expression for the entropy of a Gaussian distribution as stated in \cref{lem:information-entropy-gaussian}.
    \end{justification}
    
    \begin{proofstep}\label{ps:lemma-2-5-ps1}
        The goal is to show
        \begin{equation*}
            \frac{1}{2} \sum_{t=1}^T \sum_{i=1}^{n_t} \sigma_{t-1}(X_{t,i})^2
            \leq \frac{N_T}{\log*{\sigmaeps^{-2} + 1}} \frac{1}{2} \sum_{t=1}^T \sum_{i=1}^{n_t} \log*{\sigmaeps^{-2} \lambda_{i,t} + 1}
            .
        \end{equation*}
        
        In order to achieve this, we want to instantiate the auxiliary \cref{lem:lemma-aux-x_leq_c_logx_plus_1} with \(\inst{x}{\sigmaeps^{-2}\lambda_{i,t}}\) and \(\inst{c}{\sigmaeps^{-2}n_t}\) which provides us
        \begin{equation*}
            \sigmaeps^{-2} \lambda_{i,t}
            \leq \frac{\sigmaeps^{-2}n_t}{\log*{\sigmaeps^{-2} n_t + 1}} \cdot \log*{\sigmaeps^{-2} \lambda_{i,t} + 1}
            \quad\text{ for all } t \geq 1
            .
        \end{equation*}
        To this end, we have to ensure \(x\in[0,c]\). Since \(x \geq 0\) is satisfied, we have to show \(x \leq c\) or more specifically \(\lambda_{i,t} \leq n_t\) for all \(i = 1,\dots,n_t\) and all \(t\geq1\).
        
        The main ideas are to use the trace to relate the eigenvalues \(\lambda_{i,t}\) to the measured uncertainties \(\sigma_{t-1}(X_{t,i})^2\) and to bound these measured posterior uncertainties with the prior uncertainties. Based on the assumption of a bounded kernel function in \cref{eq:proofs-lemma-2-5-assumption}, these prior uncertainties are each bounded by 1 and their sum by \(n_t\).

        We first provide the formal proof that the uncertainties \(\sigma_t(\varphi)^2\) are monotonically decreasing in \(t\) and the posterior uncertainties are smaller than the prior uncertainties. This makes sense, since the uncertainties after gaining new information through measurements must naturally be smaller than the uncertainties without any information.
        \begin{derivation}
            \sigma_t(\varphi)^2
            &= \eqcontrast \eqchange{\Sigma(\varphi) - \Sigma(\varphi,X_{1:t})\paren*{\Sigma(X_{1:t}) + \sigmaeps^2 I}^{-1} \Sigma(X_{1:t},\varphi)} \label{deriv:lemma-2-5-ps1-aux1-start}
            \\ &= \eqcontrast \eqchange{\sigma_0(\varphi)} - \Sigma(\varphi,X_{1:t})\paren*{\Sigma(X_{1:t}) + \sigmaeps^2 I}^{-1} \Sigma(X_{1:t},\varphi)
            \\ &= \eqcontrast \sigma_0(\varphi) - \eqchange{v^T}(\eqchange{A}+\eqchange{B})^{-1}\eqchange{v}
            \\ &\leq \eqcontrast \sigma_0(\varphi) \label{deriv:lemma-2-5-ps1-aux1-end}
        \end{derivation}
        \begin{justification}
            \item by \cref{eq:posterior-surface} (definition of posterior variance)
            \item by \cref{eq:prior-surface} (definition of prior variance)
            \item by substituting \(\subst{v}{\Sigma(X_{1:t},\varphi)}, \subst{A}{\Sigma(X_{1:t})}, \subst{B}{\sigmaeps^2 I}\)
            \item by \cref{lem:lemma-aux-A_plus_B_inv_pd} with \(A\) positive semi-definite by definition of kernel function and \(B\) positive definite with \(\sigmaeps^2 > 0\)
        \end{justification}
        
        Now we can continue showing that \(\lambda_{i,t} \leq n_t\) holds for all \(i = 1,\dots,n_t\) and all \(t\geq1\), such that we can use \cref{lem:lemma-aux-x_leq_c_logx_plus_1}.
        \begin{derivationinline}
            \lambda_{i,t}
            &\leq \eqcontrast \eqchange{\sum_{i=1}^t} \lambda_{i,t} \eqnocontrast
            \tageq{=} \eqcontrast \eqchange{\tr*{\Sigma_{t-1}(X_t)}} \label{deriv:lemma-2-5-ps1-aux2-start} \eqnocontrast
            \tageq{=} \eqcontrast \eqchange{\sum_{i=1}^{n_t} \sigma_{t-1}(X_{t,i})^2} \eqnocontrast
            \\ &\tageq{\leq} \eqcontrast \sum_{i=1}^{n_t} \eqchange{\sigma_0}(X_{t,i})^2 \eqnocontrast
            \tageq{=} \eqcontrast \sum_{i=1}^{n_t} \eqchange{k(\eqnochange{X_{t,i}, X_{t,i}})} \eqnocontrast
            \\ &\tageq{\leq} \eqcontrast \sum_{i=1}^{n_t} \eqchange{1} \eqnocontrast \label{deriv:lemma-2-5-ps1-aux2-end}
            = \eqchange{n_t}
        \end{derivationinline}
        \begin{justification}
            \item by property of trace
            \item by definition of trace and \cref{eq:prior-posterior-variance-surface} (definition of prior and posterior variance)
            \item by \crefrange{deriv:lemma-2-5-ps1-aux1-start}{deriv:lemma-2-5-ps1-aux1-end} (posterior variance \(\leq\) prior variance)
            \item by \cref{eq:prior-posterior-variance-surface} (definition of prior variance)
            \item by \cref{eq:proofs-lemma-2-5-assumption} (\textit{assumption})
        \end{justification}

        This brings us into the position that we can show the desired result.
        \begin{derivation}
            \frac{1}{2} \sum_{t=1}^T \sum_{i=1}^{n_t} \sigma_{t-1}(X_{t,i})^2
            &= \eqcontrast \frac{1}{2} \sum_{t=1}^T \eqchange{\tr*{\Sigma_{t-1}(X_t)}}
            \\ &= \eqcontrast \frac{1}{2} \sum_{t=1}^T \eqchange{\sum_{i=1}^{n_t} \lambda_{i,t}}
            \\ &= \eqcontrast \eqchange{\sigmaeps^2} \frac{1}{2} \sum_{t=1}^T \sum_{i=1}^{n_t} \eqchange{\sigmaeps^{-2}} \lambda_{i,t} \notag
            \\ &\leq \eqcontrast \sigmaeps^2 \frac{1}{2} \sum_{t=1}^T \sum_{i=1}^{n_t} \eqchange{\frac{\sigmaeps^{-2}n_t}{\log*{\sigmaeps^{-2} n_t + 1}} \cdot \log*{\sigmaeps^{-2} \lambda_{i,t} + 1}}
            \\ &\leq \eqcontrast \eqchange{\frac{N_T}{\log*{\sigmaeps^{-2} + 1}}} \frac{1}{2} \sum_{t=1}^T \sum_{i=1}^{n_t} \log*{\sigmaeps^{-2} \lambda_{i,t} + 1}
        \end{derivation}
        \begin{justification}
            \item by definition of trace and \cref{eq:prior-posterior-variance-surface} (definition of prior and posterior variance)
            \item by property of trace
            \item by \cref{lem:lemma-aux-x_leq_c_logx_plus_1} instantiated with \(\inst{x}{\sigmaeps^{-2}\lambda_{i,t}}\) and \(\inst{c}{\sigmaeps^{-2}n_t}\) and \(x\in[0,c]\) ensured by \crefrange{deriv:lemma-2-5-ps1-aux2-start}{deriv:lemma-2-5-ps1-aux2-end}
            \item since \(1\leq n_t \leq N_T\) for all \(t=1,\dots,T\) 
            \\ \(\implies \frac{n_t}{\log*{\sigmaeps^{-2} n_t + 1}} \leq \frac{N_T}{\log*{\sigmaeps^{-2} + 1}}\)
        \end{justification}
    \end{proofstep}

    \begin{proofstep}\label{ps:lemma-2-5-ps2}
        The goal is to show
        \begin{equation*}
            I(Y_{1:T};f_{1:T})
            = \frac{1}{2} \sum_{t=1}^T \sum_{i=1}^{n_t} \log*{\sigmaeps^{-2}\lambda_{i,t} + 1}
            .
        \end{equation*}
        
        First recall the property of information gain
        \begin{equation*}
            I(Y_{1:T};f_{1:T}) = H(Y_{1:T}) - H(Y_{1:T} \mid f_{1:T})
        \end{equation*}
        given in \cref{eq:information-gain-property-conditional}. To show the desired equality, we first focus on the individual entropy terms \(H(Y_{1:T})\) and \(H(Y_{1:T} \mid f_{1:T})\) and derive corresponding expressions. These derivations mainly depend on the fact that the measurements \(Y_t\) are distributed with a Gaussian distribution, because \cref{lem:information-entropy-gaussian} provides us an expression for the entropy of a Gaussian distribution. Once we have both expressions, it is almost straightforward to show the final goal.

        We start with \(H(Y_{1:T})\).
        \begin{derivation}
            \MoveEqLeft[1.5] H(Y_{1:T})
            = \eqcontrast \eqchange{H(Y_T \mid Y_{1:T-1}) + H(Y_{1:T-1})} \label{deriv:lemma-2-5-ps3-aux1-start}
            \\ &= \eqcontrast \eqchange{\frac{1}{2} \log{\det*{2\pi e\paren*{\Sigma_{T-1}(X_T) + \sigmaeps^2 I_{n_T}}}}} + H(Y_{1:T-1})
            \\ &= \eqcontrast \frac{1}{2} \log*{\eqchange{\paren*{2\pi e\sigmaeps^2}^{n_T}}\det*{\eqchange{\sigmaeps^{-2}}\Sigma_{T-1}(X_T) + I_{n_T}}} + H(Y_{1:T-1})
            \\ &= \eqcontrast \frac{1}{2} \eqchange{n_T} \log*{2\pi e\sigmaeps^2} \eqchange{+} \frac{1}{2} \log{\det*{\sigmaeps^{-2}\Sigma_{T-1}(X_T) + I_{n_T}}} + H(Y_{1:T-1})  \label{deriv:lemma-2-5-ps3-aux1-rec}
            \\ &= \cdots \notag
            \\ &= \eqcontrast \frac{1}{2} \eqchange{\sum_{t=2}^T} \eqchange{n_t} \log*{2\pi e\sigmaeps^2} + \frac{1}{2} \eqchange{\sum_{t=2}^T} \log{\det*{\sigmaeps^{-2}\eqchange{\Sigma_{t-1}}(\eqchange{X_t}) + I_{\eqchange{n_t}}}} + H(\eqchange{Y_1})
            \\ &= \eqcontrast \frac{1}{2} \sum_{\eqchange{t=1}}^T n_t \log*{2\pi e\sigmaeps^2} + \frac{1}{2} \sum_{\eqchange{t=1}}^T \log{\det*{\sigmaeps^{-2}\Sigma_{t-1}(X_t) + I_{n_t}}} \label{deriv:lemma-2-5-ps3-aux1-end}
        \end{derivation}
        \begin{justification}
            \item by \cref{eq:information-entropy-properties} (property of joint entropy)
            \item by \cref{lem:information-entropy-gaussian} (entropy of Gaussian distribution) and \cref{eq:posterior-surface,eq:measured-surface} (definition of posterior and noise distribution) with
            \begin{equation*}
                Y_t \mid Y_{1:t-1} \sim \Normal*{\mu_{t-1}(X_t),\Sigma_{t-1}(X_t) + \sigmaeps^2 I_{n_t}}
                \quad\text{ for all } t \geq 2
            \end{equation*}
            obtained from \(Y_t = f_t + \varepsilon_t\) and
            \begin{equation*}
                \begin{alignedat}{2}
                    f_t \mid Y_{1:t-1} &\sim \Normal*{\mu_{t-1}(X_t),\Sigma_{t-1}(X_t)}
                    && \quad\eqnote{posterior distribution}
                    \\
                    \varepsilon_t &\sim \Normal*{0, \sigmaeps^2 I_{n_t}}
                    && \quad\eqnote{noise distribution}
                \end{alignedat}
            \end{equation*}
            \item by property of determinant
            \item by property of logarithm
            \item by \crefrange{deriv:lemma-2-5-ps3-aux1-start}{deriv:lemma-2-5-ps3-aux1-rec} applied iteratively to \(H(Y_{1:t})\)
            \item by \cref{lem:information-entropy-gaussian} (entropy of Gaussian distribution) and \cref{eq:prior-surface,eq:measured-surface} (definition of prior and noise distribution) with
            \begin{equation*}
                Y_1 \sim \Normal*{\mu_0(X_1),\Sigma_0(X_1) + \sigmaeps^2 I_{n_1}}
            \end{equation*}
            obtained from \(Y_1 = f_1 + \varepsilon_1\) and
            \begin{equation*}
                \begin{alignedat}{2}
                    f_1 &\sim \Normal*{\mu_0(X_1),\Sigma_0(X_1)}
                    && \quad\eqnote{prior distribution}
                    \\
                    \varepsilon_1 &\sim \Normal*{0, \sigmaeps^2 I_{n_1}}
                    && \quad\eqnote{noise distribution}
                \end{alignedat}
            \end{equation*}
        \end{justification}
        
        We continue similarly with \(H(Y_{1:T} \mid f_{1:T})\).
        \begin{derivation}
            H(Y_{1:T} \mid f_{1:T})
            &= \eqcontrast \eqchange{H(Y_T \mid Y_{1:T-1},f_{1:T}) + H(Y_{1:T-1} \mid f_{1:T})} \label{deriv:lemma-2-5-ps3-aux2-start}
            \\ &= \eqcontrast H(Y_T \mid \eqchange{f_T}) + H(Y_{1:T-1} \mid \eqchange{f_{1:T-1}})
            \\ &= \eqcontrast \eqchange{\frac{1}{2} \log{\det*{2\pi e\sigmaeps^2 I_{n_T}}}} + H(Y_{1:T-1} \mid f_{1:T-1})
            \\ &= \eqcontrast \frac{1}{2} \log*{\eqchange{\paren*{2\pi e\sigmaeps^2}^{n_T}} \det*{I_{n_T}}} + H(Y_{1:T-1} \mid f_{1:T-1})
            \\ &= \eqcontrast \frac{1}{2} \log*{\paren*{2\pi e\sigmaeps^2}^{n_T}} + H(Y_{1:T-1} \mid f_{1:T-1})
            \\ &= \eqcontrast \frac{1}{2} \eqchange{n_T} \log*{2\pi e\sigmaeps^2} \eqchange{+} H(Y_{1:T-1} \mid f_{1:T-1}) \label{deriv:lemma-2-5-ps3-aux2-rec}
            \\ &= \cdots \notag
            \\ &= \eqcontrast \frac{1}{2} \eqchange{\sum_{t=1}^T n_t} \log*{2\pi e\sigmaeps^2} \label{deriv:lemma-2-5-ps3-aux2-end}
        \end{derivation}
        \begin{justification}
            \item by \cref{eq:information-entropy-properties} (property of joint entropy)
            \item by \cref{eq:measured-surface} with \(Y_t\) given \(f_t\) independent of \(Y_{1:t-1}\), and \(f_{1:t-1}\) and \(Y_{1:t-1}\) independent of \(f_t\)
            \item by \cref{lem:information-entropy-gaussian} (entropy of Gaussian distribution) and \cref{eq:measured-surface} (definition noise distribution) with
            \begin{equation*}
                Y_t \mid f_t \sim \Normal*{0, \sigmaeps^2 I_{n_t}}
                \quad\text{ for all } t \geq 1
            \end{equation*}
            \item by property of determinant
            \item by determinant of identity matrix
            \item by property of logarithm
            \item by \crefrange{deriv:lemma-2-5-ps3-aux2-start}{deriv:lemma-2-5-ps3-aux2-rec} applied iteratively to \(H(Y_{1:t} \mid f_{1:t})\)
        \end{justification}

        Finally we can combine the expressions derived in \cref{deriv:lemma-2-5-ps3-aux1-end,deriv:lemma-2-5-ps3-aux2-end}.
        \begin{derivation}
            I(Y_{1:T};f_{1:T}) &= \eqchange{H(Y_{1:T}) - H(Y_{1:T} \mid f_{1:T})}
            \\ &= \eqcontrast \eqchange{\frac{1}{2} \sum_{t=1}^T \log\det*{\sigmaeps^{-2}\Sigma_{t-1}(X_t) + I_{n_t}}}
            \\ &= \eqcontrast \frac{1}{2} \sum_{t=1}^T \log*{\eqchange{\prod_{i=1}^{n_t}\paren*{\sigmaeps^{-2}\lambda_{i,t} + 1}}}
            \\ &= \eqcontrast \frac{1}{2} \sum_{t=1}^T \eqchange{\sum_{i=1}^{n_t}} \log*{\sigmaeps^{-2}\lambda_{i,t} + 1}
        \end{derivation}
        \begin{justification}
            \item by \cref{eq:information-gain-property-conditional} (property of information gain)
            \item by \crefrange{deriv:lemma-2-5-ps3-aux1-start}{deriv:lemma-2-5-ps3-aux1-end} and \crefrange{deriv:lemma-2-5-ps3-aux2-start}{deriv:lemma-2-5-ps3-aux2-end}
            \item by \cref{lem:lemma-aux-det_A_plus_I} instantiated with \(\inst{A}{\sigmaeps^2 \Sigma_{t-1}(X_t)}\) symmetric and hence diagonalizable
            \item by property of logarithm
        \end{justification}
        Observe how the left sum in \cref{deriv:lemma-2-5-ps3-aux1-end} cancels with \cref{deriv:lemma-2-5-ps3-aux2-end}, which corresponds to the ``information'' of the measurement noise. Since the information gain can also be interpreted as the mutual information between \(Y_{1:T}\) and \(f_{1:T}\) as discussed in \cref{ssec:background-infotheory-infogain}, it makes sense that the measurement noise is not part of \(I(Y_{1:T};f_{1:T})\).
    \end{proofstep}
\end{proof}

Showing \cref{lem:lemma-2-5b} is then straightforward.
\begin{proof}
    We instantiate the result from \cref{lem:lemma-2-5a} with \(X_t = \set{\theta_t}\).
    \begin{derivation}
        && \frac{1}{2} \sum_{t=1}^T \sum_{i=1}^{n_t} \sigma_{t-1}(X_{t,i})^2
        &\leq \frac{N_T}{\log*{\sigmaeps^{-2} + 1}} I(Y_{1:T};f_{1:T})
        \\
        \implies\quad
        && \frac{1}{2} \sum_{t=1}^T \sigma_{t-1}(\theta_t)^2
        &\leq \frac{1}{\log*{\sigmaeps^{-2} + 1}} I(\fm(\theta_{1:T});f(\theta_{1:T}))
    \end{derivation}
    \begin{justification}
        \item by \cref{lem:lemma-2-5a}
        \item by \cref{eq:observed-surface,eq:measured-surface} (definition of observed and measured surface) with \(X_t = \set{\theta_t}\) (\textit{assumption})
        \qedhere
    \end{justification}
\end{proof}

\subsection{\fullref{lem:lemma-2-6}}\label{ssec:proofs-lemma-2-6}

The goal of this proof is to upper bound the information capacity
\begin{equation*}
    \gamma_T
    = \sup_{\Phi\subseteq\Domain,\abs{\Phi}=T} \frac{1}{2} \sum_{t=1}^T \log\det*{I_T + \sigmaeps^{-2}K_\varphi}
\end{equation*}
with \(K_\Phi = \brack{k(\varphi,\varphi')}_{\varphi,\varphi'\in\Phi}\) as defined in \cref{def:information-capacity}.

To obtain an upper bound for \(\gamma_T\) when using the periodic \matern{} kernel \(\kpsuminf[M_\nu]\) with \(\nu = n + \frac{1}{2},n\in\Natural\) as defined in \cref{eq:periodic-matern-kernel-infinite-sum}, we make use of the bound provided by \textcite[Corollary 1]{vakili2021information} which is applicable to all kernels. This bound depends on a sufficiently fast decay of eigenvalues \(\set{\lambda_m}_{m=1}^\infty\) from the eigendecomposition of the used kernel function
\begin{equation*}
    k(\varphi,\varphi') = \sum_{m=1}^\infty \lambda_m \phi_m(\varphi)\phi_m(\varphi')
\end{equation*}
with eigenfunctions \(\set{\phi_m}_{m=1}^\infty\). This decomposition is guaranteed to exist for kernels satisfying the conditions of Mercer's theorem \autocite[Theorem 4.2]{rasmussen2005gaussian}. The reason for requiring a fast eigendecay is that the bound needs the tail mass of eigenvalues \(\sum_{m=D+1}^\infty \lambda_m\) to be sufficiently small for arbitrarily large, but fixed \(D\). The precise reasoning can be found in their derivations.

The assumptions made by \textcite[Assumption 1]{vakili2021information} are:
\begin{enumerate}[label=(A\arabic*), ref=Assumption~\arabic*]
    \item \label{item:assump-mercer-kernel} \(k\) is a Mercer kernel (\ie{} satisfies conditions of Mercer's theorem)
    \item \label{item:assump-bounded-kernel} \(\abs{k(\varphi,\varphi')} \leq k_{max}\) for all \(\varphi,\varphi'\in\Domain\) (\ie{} bounded kernel function)
    \item \label{item:assump-bounded-eigenfunctions} \(\abs{\phi_m(\varphi)} \leq \phi_{max}\) for all \(\varphi\in\Domain,m\in\Natural\) (\ie{} bounded eigenfunctions)
\end{enumerate}

\begin{proof}
    We consider the kernel function \(\kpsuminf[M_\nu](r)\) and assume
    \begin{equation}\label{eq:proofs-lemma-2-6-assumption}
        \abs{\kpsuminf[M_\nu](r)} \leq 1
        \quad\text{ for all } r\in\Real
    \end{equation}
    by choosing \(\sigma_f^2 \leq 1\).

    We first show in \cref{ps:lemma-2-6-ps1} that all assumptions above are satisfied for \(\kpsuminf[M_\nu]\).
    In \cref{ps:lemma-2-6-ps2} we derive a polynomial eigendecay for \(\kpsuminf[M_\nu]\) with the help of spectral analysis. This together allows us to derive the desired upper bound.

    \begin{proofstep}\label{ps:lemma-2-6-ps1}
        \begin{itemize}
            \item \cref{item:assump-mercer-kernel} requires \(\kpsuminf[M_\nu]\) to be a continuous, symmetric, positive semi-definite kernel defined on \(\Domain\) with respect to a finite measure \(\mu\). It is trivial to show that continuity and symmetry are satisfied. For the reasoning that positive definiteness is preserved by periodic summation, we refer to \textcite[Section 3]{borovitskiy2020matern}. Since all our kernels are defined on a compact domain \(\Domain = [0,2\pi]\), the Lebesgue measure defined on \(\Domain\) with \(\mu(\Domain) = 2\pi < \infty\) is finite.%
            \footnote{Note that Mercer's theorem does not specifically require a finite Borel measure as stated by \textcite[Theorem 1]{vakili2021information}, but any finite measure suffices \autocite[Theorem 4.2]{rasmussen2005gaussian}.}
            \item \cref{item:assump-bounded-kernel} is satisfied with \(k_{max} = 1\) by \cref{eq:proofs-lemma-2-6-assumption} (\textit{assumption})
            \item \cref{item:assump-bounded-eigenfunctions} is satisfied with \(\phi_{max} = 1\), since for any stationary kernel defined on a compact domain \([a,b]\) the eigenfunctions with respect to the Lebesgue measure are \(\phi_m(x) = \cos*{2\pi\omega_{m-1} x}\) with frequencies \(\omega_m = \frac{m}{b-a}\) and they all satisfy \(\abs*{\cos*{2\pi\omega_{m-1}x}} \geq 1\). More details are given in \cref{remark:stationary-kernel-eigenfunctions}.
        \end{itemize}
    \end{proofstep}
    \begin{proofstep}\label{ps:lemma-2-6-ps2}
        This second proof step is structured in the following way. \textcite{borovitskiy2020matern} provides us with the spectral density \(S(\omega_m)\) of \(\kpsuminf[M_\nu]\) with \(\omega_m=\frac{m}{2\pi}, m\in\Integer\). By Bochner's Theorem \autocite[Theorem 4.1 and Eq. 4.6]{rasmussen2005gaussian} this corresponds to the Fourier coefficients of \(\kpsuminf[M_\nu]\). As detailed in \cref{remark:stationary-kernel-eigenfunctions}, these Fourier coefficients correspond to twice the eigenvalues \(\lambda_m\) of \(\kpsuminf[M_\nu]\) for \(m\geq2\) and the decay of the frequency spectrum corresponds to the kernel function's eigendecay. Finally, this allows us to apply the results given by \textcite{vakili2021information}.
        \begin{derivation}
            && \lambda_m
            &= \eqcontrast \eqchange{2 \cdot S_\nu(\omega_{m-1})} \eqnocontrast \quad\text{ for all } m \geq 2
            \\ && &= \eqcontrast 2 \cdot \eqchange{C_1 \paren*{\frac{2\nu}{l^2} + \omega_{m-1}^2}^{-\paren*{\nu+\frac{1}{2}}}}
            \\ && &= \eqcontrast 2 \cdot C_1 \paren*{\frac{2\nu}{l^2} + \eqchange{\paren*{\frac{m-1}{2\pi}}^{\eqnochange{2}}}}^{-\paren*{\nu+\frac{1}{2}}}
            \\ && &= \eqcontrast \eqchange{C_2} \paren*{\frac{\eqchange{8\pi^2}\nu}{l^2} + \eqchange{(m-1)}^2}^{-\paren*{\nu+\frac{1}{2}}} \notag
            \\ && &\leq \eqcontrast \eqchange{C_3 \cdot m^{-(2\nu+1)}} \label{deriv:lemma-2-6-ps2-eigendecay}
            \\
            \implies\quad
            && \gamma_T
            &\leq \eqcontrast \paren*{\paren*{\frac{\eqchange{C_3}T}{\sigma_\varepsilon^2}}^{\frac{1}{\eqchange{2\nu+1}}} \log*{1+\frac{T}{\sigma_\varepsilon^2}}^{-\frac{1}{\eqchange{2\nu+1}}} + 1} \log*{1+\frac{T}{\sigma_\varepsilon^2}}
            \\ && &= \eqcontrast \eqchange{C_4} \cdot T^{\frac{1}{2\nu+1}} \log*{1+\frac{T}{\sigma_\varepsilon^2}}^{\frac{\eqchange{2\nu}}{2\nu+1}} \eqchange{+ \log*{1+\frac{T}{\sigma_\varepsilon^2}}}
            \\ && &\leq \eqcontrast \eqchange{\bigO*{T^{\frac{1}{2\nu+1}}\log(T)^{\frac{2\nu}{2\nu+1}}}} \notag
        \end{derivation}
        \begin{justification}
            \item by \cref{remark:stationary-kernel-eigenfunctions} (relation between eigenvalues and Fourier coefficients)
            \item by \textcite[Eq. 48]{borovitskiy2020matern} instantiated%
            \footnote{The reason for this instantiation is the same as for \cref{eq:periodic-matern-kernel-infinite-sum-examples}.}
            with \(\inst{n}{\omega_m}\) and \(\inst{\kappa}{\frac{l}{2\pi}}\)
            \\ \(\implies C_1 = \frac{\sigma_f^2}{C_\nu} \frac{\sqrt{2\nu}\sinh*{\frac{\sqrt{2\nu}\pi}{l}}}{\pi l}\)
            \item by arithmetic
            \\ \(\implies C_2 = 2 \cdot C_1 (4\pi^2)^{v+\frac{1}{2}}\)
            \item by \cref{lem:lemma-aux-a_plus_xsqr_leq_1_plus_x_sqr} instantiated with \(\inst{x}{m-1}\), \(\inst{c}{\frac{8\pi^2\nu}{l^2}}\) and \(\inst{\beta}{\nu+\frac{1}{2}}\)
            \\ \(\implies C_3 = C_2 \cdot \paren*{1+\frac{l^2}{8\pi^2\nu}}^{\nu+\frac{1}{2}}\)
            \item by \cref{ps:lemma-2-6-ps1} and \textcite[Definition 1 and Corollary 1]{vakili2021information}%
            \footnote{Note that \textcite{vakili2021information} define \((C_p,\beta_p)\) polynomial eigendecay in Definition 1 as \(\lambda_m \leq C_p \cdot m^{-\beta_p}\) for all \(m\in\Natural\), while they only need the inequality to hold for \(m\geq D+1\) with \(D\) preferably large (see Proof of Corollary 1, first inequality). Hence, it suffices to show the eigendecay bound in \cref{deriv:lemma-2-6-ps2-eigendecay} only for \(m\geq2\), although the same bound for \(m=1\) follows from \(\lambda_1 = S_\nu(\omega_0)\).}
            instantiated with \(\inst{\beta_p}{2\nu+1}\) and \(\inst{C_p}{C_3}\)
            \item by arithmetic
            \\ \(\implies C_4 = \paren*{\frac{C_3}{\sigmaeps^2}}^{\frac{1}{2\nu+1}}
            = \paren*{\frac{\sigma_f^2}{C_\nu\sigmaeps^2} \frac{2\sqrt{2\nu}\sinh*{\frac{\sqrt{2\nu}\pi}{l}}}{\pi l}}^{\frac{1}{2\nu+1}} \paren*{4\pi^2+\frac{l^2}{2^\nu}}^{1/2}\)
            \qedhere
        \end{justification}
    \end{proofstep}
\end{proof}
\begin{remark}\label{remark:stationary-kernel-eigenfunctions}
    Understanding that the eigenfunctions of all stationary kernels are complex exponentials (\ie{} sinusoidal functions) requires a more in-depth treatment of the spectral analysis of kernel functions. We refer to \textcite[Section A.7]{rasmussen2005gaussian} for a preliminary treatment of measure theory, which is helpful for this remark.
    
    Let \(k\colon\Xcal\times\Xcal\to\Real\) be a kernel function defined on a measure space \((\Xcal, \mu)\) with measure \(\mu\). Each kernel function is associated with the integral operator
    \begin{equation}\label{eq:proofs-remark-integral-operator}
        (T_kf)(x) = \int_\Xcal k(x,x') f(x')\d{\mu(x')}
    \end{equation}
    defined with respect to the measure \(\mu\) \autocite[Eq. 4.1]{rasmussen2005gaussian}. A function \(\psi\) satisfying
    \begin{equation}\label{eq:proofs-remark-eigenfunctions}
        (T_k\psi)(x) = \lambda\psi(x)
    \end{equation}
    is called the eigenfunction of \(T_k\) or equivalently of the kernel \(k\) with corresponding eigenvalue \(\lambda\) with respect to \(\mu\) \autocite[Eq. 4.36]{rasmussen2005gaussian}. We assume the eigenfunctions are normalized with respect to \(\mu\) in these sense of \(\int_\Xcal \psi(x)\psi(x)\d{\mu(x)} = 1\).

    Let us consider a stationary kernel \(k\) defined on \(\Xcal = \Real\). By Bochner's theorem \(k\) can be represented as the Fourier transform of some positive finite measure
    \begin{equation}\label{eq:proofs-remark-Bochners-theorem}
        k(x-x') = \int_\Real e^{2\pi\i \omega(x-x')} \d{\mu_k(\omega)}
        ,
    \end{equation}
    where the measure \(\mu_k\) is specific to each kernel \(k\) \autocite[Theorem 4.1]{rasmussen2005gaussian}. Note that due to symmetry \(k(x-x')=k(x'-x)\), one can also write \cref{eq:proofs-remark-Bochners-theorem} with an additional minus in the complex exponential, which more correctly matches the Fourier transform.
    If the corresponding density of \(\mu_k\) exists,%
    \footnote{A function \(p(x)\) is called the density of \(\mu\) if it  satisfies \(\mu(A) = \int_A p(x)dx\) for all \(A\subseteq\Xcal\).}
    it is known as the spectral density \(S(\omega)\) with \(\d{\mu_k(\omega)} = S(\omega)d\omega\) which allows us to write the integral as a Fourier integral
    \begin{equation}\label{eq:proofs-remark-Bochners-theorem-spectral-density}
        k(x-x') = \int_\Real S(\omega) e^{2\pi\i \omega(x-x')} \d{\omega}
        = \int_\Real S(\omega) e^{2\pi\i \omega x} e^{-2\pi\i \omega x'} \d{\omega}
        .
    \end{equation}
    Intuitively, \(S(\omega)\) weights the different frequency components \(\omega\) contained in \(k\).%
    \footnote{In particular, the stationary kernel and its spectral density are Fourier duals, which is also known as the Wiener-Khintchine theorem \textcite[Eq. 4.6]{rasmussen2005gaussian}.}
    This allows us to show that the complex exponentials \(\psi_\omega(x) = e^{2\pi\i \omega x}\) are the eigenfunctions of \(k\) with respect to the Lebesgue measure, for which the corresponding eigenvalues are given by the spectral density \(S(\omega)\).
    \begin{derivation}
            (T_k\psi_\omega)(x)
            &= \eqcontrast \int_\Real k(x,x') \cdot \eqchange{e^{2\pi\i \omega x'}} \d{\mu(x')}
            \\ &= \eqcontrast \int_\Real k(x,x') \cdot e^{2\pi\i \omega x'} \eqchange{\d{x'}}
            \\ &= \eqcontrast \int_\Real \eqchange{\int_\Real S(\omega') e^{2\pi\i \omega' x} e^{-2\pi\i \omega' x'} \d{\omega'}} \cdot e^{2\pi\i \omega x'} \d{x'}
            \\ &= \eqcontrast \eqchange{\int_\Real S(\omega') e^{2\pi\i \omega' x}} \eqchange{\int_\Real} e^{-2\pi\i \omega' x'} e^{2\pi\i \omega x'} \eqchange{\d{x'} \d{\omega'}}
            \\ &= \eqcontrast \int_\Real S(\omega') e^{2\pi\i \omega' x} \int_\Real \eqchange{1 \cdot} e^{-2\pi\i \eqchange{(\omega'-\omega)} x'} \d{x'} \d{\omega'} \notag
            \\ &= \eqcontrast \int_\Real S(\omega') e^{2\pi\i \omega' x} \eqchange{\Fourier{x'\mapsto1}{\omega'-\omega}} \d{\omega'}
            \\ &= \eqcontrast \int_\Real S(\omega') e^{2\pi\i \omega' x} \eqchange{\delta(\omega'-\omega)} \d{\omega'}
            \\ &= \eqcontrast S(\eqchange{\omega}) e^{2\pi\i \eqchange{\omega} x}
            \\ &= \eqcontrast S(\omega) \eqchange{\psi_\omega(x)} \notag
    \end{derivation}
    \begin{justification}
        \item by \cref{eq:proofs-remark-integral-operator} (definition of integral operator)
        \item by integration on Lebesgue measure
        \item by \cref{eq:proofs-remark-Bochners-theorem-spectral-density} (Bochner's theorem with spectral density)
        \item by Fubini's theorem (i.e. possible to swap integrals)
        \item by definition of Fourier transform
        \item by Fourier transform of \(f(x)=1\)
        \item by property of Dirac delta function
    \end{justification}
    
    However, for non-periodic kernel functions on a compact domain \(\Xcal = [a,b] \subset \Real\), which can be periodically extended to \(\Real\), or for periodic kernel functions on \(\Xcal = \Real\) with periodicity \(b-a\)
    the Fourier integral given in \cref{eq:proofs-remark-Bochners-theorem-spectral-density} turns into a Fourier series
    \begin{equation}\label{eq:proofs-remark-Fourier-series-spectral-density}
        k(x-x')
        = \sum_{m\in\Integer} S(\omega_m) e^{2\pi\i \omega_m(x-x')}
        = \sum_{m\in\Integer} S(\omega_m) e^{2\pi\i \omega_m x} e^{-2\pi\i \omega_m x'}
    \end{equation}
    with Fourier coefficients \(S(\omega_m)\). Since the fundamental frequency of the periodic or periodically extended kernel function is \(\omega_1 = \frac{1}{b-a}\), the frequency spectrum of \(k\) is discrete with frequencies \(\omega_m = \frac{m}{b-a}\) with \(m\in\Integer\). We show that the complex exponentials \(\psi_m(x) = e^{2\pi\i\omega_mx}\) and the spectral density \(S(\omega_m)\) again form the eigenfunctions and eigenvalues of \(k\) with respect to the Lebesgue measure, which closely follows the derivation steps above. First note that the set of complex exponentials \(\set{\psi_m\mid m\in\Integer}\) is orthonormal with respect to the Lebesgue measure, since
    \begin{equation}\label{eq:proofs-remark-orthogonality}
        \inner{\psi_i,\psi_j}_\mu
        = \int_\Xcal \psi_i(x)\overline{\psi_j(x)} \d{\mu(x')}
        = \int_\Xcal e^{2\pi\i\omega_ix}e^{-2\pi\i\omega_jx} \dx
        = \delta_{ij}
    \end{equation}
    by case distinction over \(i=j\) and \(i\neq j\). Then it follows that
    \begin{derivation}
        (T_k\psi_j)(x)
        &= \eqcontrast \int_\Xcal k(x,x') \cdot \eqchange{e^{2\pi\i \omega_j x'}} \d{\mu(x')}
        \\ &= \eqcontrast \int_\Xcal k(x,x') \cdot e^{2\pi\i \omega_j x'} \eqchange{\d{x'}}
        \\ &= \eqcontrast \int_\Xcal \eqchange{\sum_{i\in\Integer} S(\omega_i) e^{2\pi\i \omega_i x} e^{-2\pi\i \omega_i x'}} \cdot e^{2\pi\i \omega_j x'} \d{x'}
        \\ &= \eqcontrast \eqchange{\sum_{i\in\Integer} S(\omega_i) e^{2\pi\i \omega_i x}} \eqchange{\int_\Xcal} e^{-2\pi\i \omega_i x'} e^{2\pi\i \omega_j x'} \eqchange{\d{x'}}
        \\ &= \eqcontrast \sum_{i\in\Integer} S(\omega_i) e^{2\pi\i \omega_i x} \eqchange{\delta_{ij}}
        \\ &= \eqcontrast S(\eqchange{\omega_j}) e^{2\pi\i \eqchange{\omega_j} x}
        \\ &= \eqcontrast S(\omega_j) \eqchange{\psi_j(x)} \notag
    \end{derivation}
    \begin{justification}
        \item by \cref{eq:proofs-remark-integral-operator} (definition of integral operator)
        \item by integration on Lebesgue measure
        \item by \cref{eq:proofs-remark-Fourier-series-spectral-density} (Fourier series)
        \item by dominated convergence theorem (i.e. possible to swap integral and sum)
        \item by \cref{eq:proofs-remark-orthogonality} (orthogonality of complex exponentials)
        \item by property of Kronecker delta
    \end{justification}
    In particular, by symmetry of \(k\) the Fourier series in \cref{eq:proofs-remark-Fourier-series-spectral-density}
    can be equivalently written as
    \begin{equation*}
        k(x-x')
        = S(\omega_0) + \sum_{m=1}^\infty 2 S(\omega_m) \cos*{2\pi\omega_m x} \cos*{2\pi\omega_m x'}
    \end{equation*}
    which closely corresponds to the eigendecomposition of Mercer's theorem
    \begin{equation*}
        k(x-x') = \sum_{m=1}^\infty \lambda_m \phi_m(x)\phi_m(x')
    \end{equation*}
    with eigenfunctions \(\phi_m(x) = \cos*{2\pi\omega_{m-1} x}\) and eigenvalues \(\lambda_1 = S(\omega_0)\) and \(\lambda_m = 2 S(\omega_{m-1}),m\geq2\).
\end{remark}

\subsection{\cref{lem:lemma-2-3-greedy-up}}\label{ssec:proofs-lemma-2-3-greedy-up}

The goal of this proof is to show the last mile towards sublinear regret for the specific greedy algorithm using the \nameref{sssec:design-objective-up} objective. The idea is to first square the sum of objective values, which can be upper bounded based on the sum of squared objective values by a specific version of the Cauchy-Schwarz inequality shown in \cref{lem:lemma-aux-cauchy_schwarz}. This allows us to apply \cref{lem:lemma-2-5} to relate the upper bound with the information gain and thereby also with the information capacity. By taking the square root again on both sides, we arrive at our desired result.

This proof is taken from \textcite[Lemma 6-7]{prajapat2022nearoptimal} and adapted to our setting.

\begin{proof}
    We assume that the assumptions for \cref{lem:lemma-2-4,lem:lemma-2-5} are satisfied. Let \(u_t(\varphi)\) and \(l_t(\varphi)\) be defined according to \cref{eq:confidence-bounds-actual} with confidence parameter \(\beta_t\) chosen as in \cref{lem:lemma-2-4}.
    Let \(\theta_{1:T}\) be arbitrary.
    \begin{derivation}
        && \MoveEqLeft[3.6] \sum_{t=1}^T \Fu[UP](\theta_t\mid\theta_{1:t-1})
        = \eqcontrast \sum_{t=1}^T \eqchange{\frac{1}{h^2} \cdot \abs{\PhiS} \paren*{u_t(\theta_t) - l_t(\theta_t)}}
        \\ && &= \eqcontrast \eqchange{\frac{\abs{\PhiS}}{h^2}} \sum_{t=1}^T \paren*{u_t(\theta_t) - l_t(\theta_t)} \notag
        \\ && &= \eqcontrast \frac{\abs{\PhiS}}{h^2} \sum_{t=1}^T \eqchange{2\paren*{\beta_t^{1/2}\sigma_{t-1}(\theta_t) + \frac{1}{t^2}}}
        \\ && &= \eqcontrast \frac{\eqchange{2}\abs{\PhiS}}{h^2} \paren*{\eqchange{\sum_{t=1}^T} \beta_t^{1/2}\sigma_{t-1}(\theta_t) + \eqchange{\sum_{t=1}^T}\frac{1}{t^2}} \notag
        \\ && &\leq \eqcontrast \frac{2\abs{\PhiS}}{h^2} \paren*{\sum_{t=1}^T \beta_t^{1/2}\sigma_{t-1}(\theta_t) + \sum_{t=1}^{\eqchange{\infty}}\frac{1}{t^2}} \notag
        \\ && &= \eqcontrast \frac{2\abs{\PhiS}}{h^2} \paren*{\sum_{t=1}^T \beta_t^{1/2}\sigma_{t-1}(\theta_t) + \eqchange{\frac{\pi}{6}}} \notag
        \\
        \implies\quad
        && \MoveEqLeft[3.6] \eqcontrast \eqchange{\paren*{\eqnochange{\sum_{t=1}^T \Fu[UP](\theta_t\mid\theta_{1:t-1})}}^2}
        = \eqcontrast \eqchange{\paren*{\eqnochange{\frac{2\abs{\PhiS}}{h^2} \paren*{\sum_{t=1}^T \beta_t^{1/2}\sigma_{t-1}(\theta_t) + \frac{\pi}{6}}}}^2} \notag
        \\ && &= \eqcontrast \eqchange{\frac{4\abs{\PhiS}^2}{h^4}} \paren*{\sum_{t=1}^T \beta_t^{1/2}\sigma_{t-1}(\theta_t) + \frac{\pi}{6}}^2 \notag
        \\ && &= \eqcontrast \frac{4\abs{\PhiS}^2}{h^4} \paren*{\eqchange{\frac{3}{2}\paren*{\eqnochange{\sum_{t=1}^T \beta_t^{1/2}\sigma_{t-1}(\theta_t)}}^2} + \eqchange{3\paren*{\eqnochange{\frac{\pi}{6}}}^2}}
        \\ && &= \eqcontrast \frac{4\abs{\PhiS}^2}{h^4} \paren*{\frac{3}{2}\eqchange{T}\sum_{t=1}^T \eqchange{\paren*{\eqnochange{\beta_t^{1/2}\sigma_{t-1}(\theta_t)}}^2} + \eqchange{\frac{\pi^2}{12}}}
        \\ && &= \eqcontrast \frac{4\abs{\PhiS}^2}{h^4} \paren*{\frac{3}{2}T\sum_{t=1}^T \eqchange{\beta_t\sigma_{t-1}(\theta_t)^2} + \frac{\pi^2}{12}} \notag
        \\ && &= \eqcontrast \frac{4\abs{\PhiS}^2}{h^4} \paren*{\frac{3}{2}T\eqchange{\beta_T}\sum_{t=1}^T \sigma_{t-1}(\theta_t)^2 + \frac{\pi^2}{12}}
        \\ && &= \eqcontrast \frac{4\abs{\PhiS}^2}{h^4} \paren*{\eqchange{3}T\beta_T \eqchange{\cdot \frac{1}{2}}\sum_{t=1}^T \sigma_{t-1}(\theta_t)^2 + \frac{\pi^2}{12}} \notag
        \\ && &= \eqcontrast \frac{4\abs{\PhiS}^2}{h^4} \paren*{\eqchange{\frac{\eqnochange{3}}{\log*{\sigmaeps^{-2} + 1}}}T\beta_T \eqchange{I(\fm(\theta_{1:T});f(\theta_{1:T}))} + \frac{\pi^2}{12}}
        \\ && &= \eqcontrast \frac{4\abs{\PhiS}^2}{h^4} \paren*{\frac{3}{\log*{\sigmaeps^{-2} + 1}}T\beta_T\eqchange{\gamma_T} + \frac{\pi^2}{12}}
        \\
        \implies\quad
        && \MoveEqLeft[3.6] \eqcontrast \sum_{t=1}^T \Fu[UP](\theta_t\mid\theta_{1:t-1})
        \leq \eqcontrast \eqchange{\frac{2\abs{\PhiS}}{h^2}} \eqchange{\sqrt{\eqnochange{\frac{3}{\log*{\sigmaeps^{-2} + 1}}T\beta_T\gamma_T + \frac{\pi^2}{12}}}} \notag
        \\ && &\leq \eqcontrast \frac{2\abs{\PhiS}}{h^2} \paren*{\eqchange{\sqrt{\eqnochange{\frac{3}{\log*{\sigmaeps^{-2} + 1}}T\beta_T\gamma_T}} + \sqrt{\eqnochange{\frac{\pi^2}{12}}}}}
        \\ && &= \eqcontrast \frac{2\abs{\PhiS}}{h^2} \sqrt{\frac{3}{\log*{\sigmaeps^{-2} + 1}}}\sqrt{T\beta_T\gamma_T} + \eqchange{\frac{\abs{\PhiS}}{h^2}\frac{\pi}{\sqrt{3}}} \notag
    \end{derivation}
    \begin{justification}
        \item by definition of \nameref{sssec:design-objective-up}
        \item by \cref{eq:confidence-bounds-actual} (refined definition of confidence bounds)
        \item by \cref{lem:lemma-aux-x_plus_a_sqr_leq_xsqr_plus_asqr} instantiated with \(\inst{x}{\sum_{t=1}T\beta_t^{1/2}\sigma_{t-1}(\theta_t)}\) and \(\inst{a}{\frac{\pi}{6}}\) and \(\inst{c}{\frac{3}{2}}\)
        \item by \cref{lem:lemma-aux-cauchy_schwarz} (Cauchy-Schwarz inequality)
        \item by \cref{lem:lemma-2-4} with \(\beta_t\in\bigO*{\log t}\) monotonically increasing
        \item by \cref{lem:lemma-2-5b}
        \item by \cref{def:information-capacity} (definition of information capacity)
        \item since \(\sqrt{x+y} \leq \sqrt{x+2\sqrt{x}\sqrt{y}+y} = \sqrt{x}+\sqrt{y}\) with \(x,y>0\)
        \qedhere
    \end{justification}
\end{proof}

\subsection{\fullref{thm:no-sublinear-greedy-up}}\label{ssec:proofs-theorem-no-sublinear-greedy-up}

The goal of this proof is to show sublinear regret for our specific design choices by bringing all previous results together. Unfortunately, the chosen objective function \(\Fu[UP]\) does not satisfy the necessary requirement \cref{item:req-necessary-bound} and we cannot show sublinear regret for this specific design choice. However, it is still advantageous to conduct part of this proof, since it serves as a good foundation for showing sublinear regret for similar objective functions satisfying \cref{item:req-necessary-bound}.

\begin{proof}
    We make the following design choices:
    \begin{itemize}
        \item Kernel: \(\kpsuminf[M_{\nu}]\) with \(\nu=\frac{5}{2}\) and \(\sigma_f=1\) (see \cref{eq:periodic-matern-kernel-infinite-sum-examples})
        \item Objective: \(\Fu[UP]\) (see \nameref{sssec:design-objective-up})
        \item Algorithm: \(\A(\cdot;\Fu)\) (see \cref{def:greedy-algorithm})
        \item Failure probability: \(\delta\in(0,1)\)
        \item Confidence bounds: \(u_t\) and \(l_t\) as in \cref{eq:confidence-bounds-actual}
        \item Confidence parameter: \(\beta_t\) as in \cref{lem:lemma-2-4}
    \end{itemize}
    
    To be able to apply all previous results, we have to show that the corresponding assumptions are satisfied.
    \begin{itemize}
        \item[\cmark] \cref{lem:lemma-2-1} makes no assumptions.
        \item[\xmark] \cref{lem:lemma-2-2} assumes a reasonable \(\Fu\) and \(\A\) with
        \begin{equation*}
            F(\thetagre_t \mid \theta_{1:t-1}) \leq \Fu(\theta_t \mid \theta_{1:t-1})
            \quad\text{ for all } t \geq 1
        \end{equation*}
        as specified in \cref{eq:proofs-lemma-2-2-assumption}. This is \textit{not} satisfied by our design choice \(\A[UP](\cdot)=\A(\cdot;\Fu[UP])\), since it is not a necessary upper bound (see \cref{item:req-necessary-bound}) as described in \nameref{sssec:design-objective-up}.
        \item[\cmark] \cref{lem:lemma-2-4} assumes surface functions sampled from the Gaussian process are sufficiently well-behaved with
        \begin{equation*}
            \Pr*{\sup_{\varphi\in\Domain} \abs*{\deriv{f}{\varphi}(\varphi)} \leq L} \geq 1 - ae^{-L^2/b^2}
            \quad\text{ for some } a,b > 0
        \end{equation*}
        as specified in \cref{eq:proofs-lemma-2-4-assumption}. Since this is the case for \matern{} kernels with \(\nu > 2\) as we have discussed in \cref{ssec:analysis-tools-confidence}, this assumption is satisfied by our design choice \(\nu = \frac{5}{2}\).
        \item[\cmark] \cref{lem:lemma-2-5} assumes a uniformly bounded kernel function with
        \begin{equation*}
            \abs{k(\varphi,\varphi')} \leq 1
            \quad\text{ for all } \varphi,\varphi'\in\Domain
        \end{equation*}
        as specified in \cref{eq:proofs-lemma-2-5-assumption}. This is satisfied by our design choice \(\sigma_f = 1\).
        \item[\cmark] \cref{lem:lemma-2-6} only holds for \(\kpsuminf[M_\nu](r)\) and assumes a uniformly bounded kernel function with
        \begin{equation*}
            \abs{\kpsuminf[M_\nu](r)} \leq 1
            \quad\text{ for all } r \in \Real
        \end{equation*}
        as specified in \cref{eq:proofs-lemma-2-6-assumption}. This is satisfied by our design choice for \(\kpsuminf[M_\nu](r)\) with \(\sigma_f = 1\).
        \item[\cmark] \cref{lem:lemma-2-3-greedy-up} makes the same assumptions as \cref{lem:lemma-2-4,lem:lemma-2-5} and, in addition, chooses \(u_t\) and \(l_t\) as in \cref{eq:confidence-bounds-actual} and
        \begin{equation*}
            \beta_t = 2\log*{\frac{2\pi^3b}{3}\sqrt{\log*{\frac{2a}{\delta}}}\frac{t^4} {\delta}}
        \end{equation*}
        as specified in \cref{lem:lemma-2-4}. This is satisfied by our design choices for \(u_t\), \(l_t\) and \(\beta_t\).
    \end{itemize}

    We proceed with the proof.
    \begin{derivation}
        R(T)
        &< \eqcontrast \eqchange{R_{ind}(T)}
        \\ &\nleq \eqcontrast \eqchange{\sum_{t=1}^T \Fu[UP](\theta[UP]_t\mid\theta[UP]_{1:t-1})}
        \\ &\leq \eqcontrast \eqchange{\bigO*{\sqrt{T\beta_T\gamma_T}}}
        \\ &\leq \eqcontrast \bigO*{\sqrt{T \cdot \eqchange{\log{T}} \cdot \eqchange{T^{\frac{1}{2\nu+1}}\log(T)^{\frac{2\nu}{2\nu+1}}}}}
        \\ &\leq \eqcontrast \bigO*{T^{\eqchange{\frac{2\nu+2}{4\nu+2}}}\log(T)^{\eqchange{\frac{4\nu+1}{4\nu+2}}}} \notag
    \end{derivation}
    \begin{justification}
        \item by \cref{lem:lemma-2-1}
        \item since assumption for \cref{lem:lemma-2-2} instantiated with \(\inst{\Fu}{\Fu[UP]}\) and \(\inst{\theta_t}{\theta[UP]_t}\) is not satisfied
        \item by \cref{lem:lemma-2-3-greedy-up} instantiated with \(\inst{\theta_t}{\theta[UP]_t}\)
        \item by \cref{lem:lemma-2-4} and \cref{lem:lemma-2-6}
        \qedhere
    \end{justification}
\end{proof}

\subsection{\cref{lem:lemma-2-3-greedy-u}}\label{ssec:proofs-lemma-2-3-greedy-u}
The goal of this proof is to show the last mile towards sublinear regret for the specific greedy algorithm using the \nameref{sssec:design-objective-u} objective. It closely depends on the proof for \cref{lem:lemma-2-3-greedy-up}.

\begin{proof}
    We assume that the assumptions for \cref{lem:lemma-2-4,lem:lemma-2-5} are satisfied. Let \(u_t(\varphi)\) and \(l_t(\varphi)\) be defined according to \cref{eq:confidence-bounds-actual} with confidence parameter \(\beta_t\) chosen as in \cref{lem:lemma-2-4}.
    Let \(\theta_{1:T}\) be arbitrary.

    We first derive the relation between \nameref{sssec:design-objective-u} and \nameref{sssec:design-objective-up}, which then allows us to directly reuse \cref{lem:lemma-2-3-greedy-up}.
    \begin{derivation}
        && \Fu[U](\theta_t\mid\theta_{1:t-1})
        &= \frac{1}{h^2} \cdot \frac{1}{2}\abs*{\PhiS}(u_t(\theta_t)^2 - l_t(\theta_t)^2)
        \\ && &= \eqcontrast \frac{1}{h^2} \cdot \frac{1}{2}\abs*{\PhiS}\eqchange{(u_t(\theta_t) + l_t(\theta_t)) \cdot (u_t(\theta_t) - l_t(\theta_t))} \notag
        \\ && &= \eqcontrast \frac{1}{2}(u_t(\theta_t) + l_t(\theta_t)) \cdot \eqchange{\frac{1}{h^2} \cdot \abs*{\PhiS}(u_t(\theta_t) - l_t(\theta_t))} \notag
        \\ && &= \eqcontrast \frac{1}{2}(u_t(\theta_t) + l_t(\theta_t)) \cdot \eqchange{\Fu[UP](\theta_t\mid\theta_{1:t-1})}
        \\ && &= \eqcontrast \eqchange{\mu_{t-1}(\theta_t)} \cdot \Fu[UP](\theta_t\mid\theta_{1:t-1})
        \\ && &\leq \eqcontrast \eqchange{\dmax} \cdot \Fu[UP](\theta_t\mid\theta_{1:t-1})
        \\
        \implies\quad
        && \eqcontrast \eqchange{\sum_{t=1}^T} \Fu[U](\theta_t\mid\theta_{1:t-1})
        &\leq \eqcontrast \dmax \eqchange{\sum_{t=1}^T} \Fu[UP](\theta_t\mid\theta_{1:t-1}) \notag
        \\ && &\leq \eqcontrast \begin{aligned}[t]
            & \eqchange{\frac{2\eqnochange{\dmax}\abs{\PhiS}}{h^2} \sqrt{\frac{3}{\log*{\sigmaeps^{-2} + 1}}}\sqrt{T\beta_T\gamma_T}}\\
            & \eqchange{+ \frac{\eqnochange{\dmax}\abs{\PhiS}}{h^2}\frac{\pi}{\sqrt{3}}}
        \end{aligned}
    \end{derivation}
    \begin{justification}
        \item by definition of \nameref{sssec:design-objective-u}
        \item by definition of \nameref{sssec:design-objective-up}
        \item by \cref{eq:confidence-bounds-actual} (refined definition of confidence bounds)
        \item by \cref{item:simp-object-bounded}
        \item by \cref{lem:lemma-2-3-greedy-up}
        \qedhere
    \end{justification}
\end{proof}

\subsection{\fullref{thm:sublinear-greedy-u}}\label{ssec:proofs-theorem-sublinear-greedy-u}

The goal of this proof is to show sublinear regret for the specific greedy algorithm using the \nameref{sssec:design-objective-u} objective by bringing all previous results together. It closely follows the structure of \cref{thm:no-sublinear-greedy-up}.

\begin{proof}
    We make the following design choices:
    \begin{itemize}
        \item Kernel: same as \cref{thm:no-sublinear-greedy-up}
        \item Objective: \(\Fu[U]\) (see \nameref{sssec:design-objective-u})
        \item Algorithm: same as \cref{thm:no-sublinear-greedy-up}
        \item Failure probability: same as \cref{thm:no-sublinear-greedy-up}
        \item Confidence bounds: same as \cref{thm:no-sublinear-greedy-up}
        \item Confidence parameter: same as \cref{thm:no-sublinear-greedy-up}
    \end{itemize}
    
    To be able to apply all previous results, we have to show that the corresponding assumptions are satisfied.
    \begin{itemize}
        \item[\cmark] \cref{lem:lemma-2-1} makes no assumptions.
        \item[\cmark] \cref{lem:lemma-2-2} assumes a reasonable \(\Fu\) and \(\A\) with
        \begin{equation*}
            F(\thetagre_t \mid \theta_{1:t-1}) \leq \Fu(\theta_t \mid \theta_{1:t-1})
            \quad\text{ for all } t \geq 1
        \end{equation*}
        as specified in \cref{eq:proofs-lemma-2-2-assumption}.
        It is clear that \nameref{sssec:design-objective-u} provides a necessary upper bound (see \cref{item:req-necessary-bound}) under the assumption that \(f(\theta)\) lies between \(u_t(\theta)\) and \(l_t(\theta)\).
        Since it is only guaranteed that \(f(\theta)\) lies between \(u_t([\theta]_t)\) and \(l_t([\theta]_t)\) with probability at least \(1-\delta\) by \cref{lem:lemma-2-4}, the necessary upper bound is only satisfied for \(\Fu[U]\) together with \(u_t([\theta]_t)\) and \(l_t([\theta]_t)\). We then derive
        \begin{derivation}
            F(\thetagre_t \mid \theta[U]_{1:t-1})
            &\leq \eqcontrast \frac{1}{h^2} \cdot \frac{1}{2}\abs*{\PhiS}\paren*{u_t(\eqchange{[\thetagre_t]_t})^2 - l_t(\eqchange{[\thetagre_t]_t})^2}
            &&\quad\text{\whp{} } 1-\delta
            \\ &\leq \eqcontrast \frac{1}{h^2} \cdot \frac{1}{2}\abs*{\PhiS}\paren*{u_t(\eqchange{\theta[U]_t})^2 - l_t(\eqchange{\theta[U]_t})^2}
            \\ &= \eqcontrast \Fu[U](\theta[U]_t\mid\theta[U]_{1:t-1})
            .
        \end{derivation}
        \begin{justification}
            \item by necessary upper bound of \nameref{sssec:design-objective-u} with \(u_t([\theta]_t)\) and \(l_t([\theta]_t)\)
            \item by \cref{def:greedy-algorithm} (definition of greedy algorithm)
            \item by definition of \nameref{sssec:design-objective-u}
        \end{justification}
        Hence, the assumption of \cref{lem:lemma-2-2} is satisfied by our design choice \(\A[U](\cdot)=\A(\cdot;\Fu[U])\) with probability at least \(1-\delta\).
        \item[\cmark] \cref{lem:lemma-2-4,lem:lemma-2-5,lem:lemma-2-6} only make assumptions on the kernel, which are satisfied by our design choices as it is the case for \cref{thm:no-sublinear-greedy-up}.
        \item[\cmark] \cref{lem:lemma-2-3-greedy-u} only makes assumptions on the confidence bounds and confidence parameter, which are satisfied by our design choices as it is the case for \cref{thm:no-sublinear-greedy-up}.
    \end{itemize}
    
    We proceed with the proof, which closely follows the derivation for \cref{thm:no-sublinear-greedy-up}.
    \begin{derivation}
        R(T)
        &< \eqcontrast \eqchange{R_{ind}(T)}
        \\ &\leq \eqcontrast \eqchange{\sum_{t=1}^T \Fu[U](\theta[U]_t\mid\theta[U]_{1:t-1})}
        && \quad\text{\whp{} } 1-\delta
        \\ &\leq \eqcontrast \eqchange{\bigO*{\sqrt{T\beta_T\gamma_T}}}
        \\ &\leq \eqcontrast \bigO*{\sqrt{T \cdot \eqchange{\log{T}} \cdot \eqchange{T^{\frac{1}{2\nu+1}}\log(T)^{\frac{2\nu}{2\nu+1}}}}}
        \\ &\leq \eqcontrast \bigO*{T^{\eqchange{\frac{2\nu+2}{4\nu+2}}}\log(T)^{\eqchange{\frac{4\nu+1}{4\nu+2}}}} \notag
    \end{derivation}
    \begin{justification}
        \item by \cref{lem:lemma-2-1}
        \item by \cref{lem:lemma-2-2} instantiated with \(\inst{\Fu}{\Fu[U]}\) and \(\inst{\theta_t}{\theta[U]_t}\) and assumption satisfied with probability at least \(1-\delta\)
        \item by \cref{lem:lemma-2-3-greedy-u} instantiated with \(\inst{\theta_t}{\theta[U]_t}\)
        \item by \cref{lem:lemma-2-4} and \cref{lem:lemma-2-6}
        \qedhere
    \end{justification}
\end{proof}

\subsection{\cref{lem:lemma-2-3-greedy-cs}}\label{ssec:proofs-lemma-2-3-greedy-cs}
The goal of this proof is to show the last mile towards sublinear regret for the specific greedy algorithm using the \nameref{sssec:design-objective-cs} objective as redefined in \cref{eq:objective-function-cs-new}. It closely depends on the proof for \cref{lem:lemma-2-3-greedy-u}.

\begin{proof}
    We assume that the assumptions for \cref{lem:lemma-2-4,lem:lemma-2-5} are satisfied. Let \(u_t(\varphi)\) and \(l_t(\varphi)\) be defined according to \cref{eq:confidence-bounds-actual} with confidence parameter \(\beta_t\) chosen as in \cref{lem:lemma-2-4}. 
    Let \(\theta_{1:T}\) be arbitrary.

    We first derive the relation between \nameref{sssec:design-objective-cs} and \nameref{sssec:design-objective-u}, which then allows us to directly reuse \cref{lem:lemma-2-3-greedy-u}.
    \begin{derivation}
        && \Fu[CS](\theta_t\mid\theta_{1:t-1})
        &= \frac{1}{h^2} \sum_{\varphi\in\brack{\PhiS(\theta_t)}_t} \frac{1}{2} \paren*{u_t(\varphi)^2-l_t(\varphi)^2} \frac{\abs{\PhiS}}{\abs{\brack{\PhiS}_t}}
        \\ && &\leq \eqcontrast \frac{1}{h^2} \eqchange{\abs*{\brack{\PhiS}_t}} \eqchange{\max_{\varphi\in\brack{\PhiS(\theta_t)}_t}} \frac{1}{2} \paren*{u_t(\varphi)^2-l_t(\varphi)^2} \frac{\abs{\PhiS}}{\abs{\brack{\PhiS}_t}} \notag
        \\ && &= \eqcontrast \max_{\varphi\in\brack{\PhiS(\theta_t)}_t} \eqchange{\frac{1}{h^2}} \cdot \frac{1}{2} \eqchange{\abs*{\PhiS}} \paren*{u_t(\varphi)^2-l_t(\varphi)^2} \notag
        \\ && &= \eqcontrast \max_{\eqchange{\theta}\in\brack{\PhiS(\theta_t)}_t} \eqchange{\Fu[U](\theta\mid\theta_{1:t-1})}
        \\ && &\leq \eqcontrast \max_{\theta\in\eqchange{\Domain}} \Fu[U](\theta\mid\theta_{1:t-1}) \notag
        \\ && &= \eqcontrast \Fu[U]\paren*{\eqchange{\theta[U]_t}\mid\theta_{1:t-1}}
        \\
        \implies
        && \eqcontrast \eqchange{\sum_{t=1}^T} \Fu[CS](\theta_t\mid\theta_{1:t-1})
        &\leq \eqcontrast \eqchange{\sum_{t=1}^T} \Fu[U](\theta[U]_t\mid\theta_{1:t-1}) \notag
        \\ && &\leq \eqcontrast \begin{aligned}[t]
            & \eqchange{\frac{2\dmax\abs{\PhiS}}{h^2} \sqrt{\frac{3}{\log*{\sigmaeps^{-2} + 1}}}\sqrt{T\beta_T\gamma_T}} \\
            & \eqchange{+ \frac{\dmax\abs{\PhiS}}{h^2}\frac{\pi}{\sqrt{3}}}
        \end{aligned}
    \end{derivation}
    \begin{justification}
        \item by definition of \nameref{sssec:design-objective-cs}
        \item by definition of \nameref{sssec:design-objective-u}
        \item by \cref{def:greedy-algorithm} (definition of greedy algorithm)
        \item by \cref{lem:lemma-2-3-greedy-u} \TODO{REMOVABLE WEAK shouldn't be \(\theta_t\) and \(\theta_{1:t-1}\) be sequentially the same?}
        \qedhere
    \end{justification}
\end{proof}

\subsection{\fullref{thm:sublinear-greedy-cs}}\label{ssec:proofs-theorem-sublinear-greedy-cs}

The goal of this proof is to show sublinear regret for the specific greedy algorithm using the \nameref{sssec:design-objective-cs} objective as redefined in \cref{eq:objective-function-cs-new} by bringing all previous results together. It closely follows the derivation structure for \cref{thm:no-sublinear-greedy-up} and \cref{thm:sublinear-greedy-u}.

\begin{proof}
    We make the following design choices:
    \begin{itemize}
        \item Kernel: same as \cref{thm:no-sublinear-greedy-up}
        \item Objective: \(\Fu[CS]\) (see \cref{eq:objective-function-cs-new})
        \item Algorithm: same as \cref{thm:no-sublinear-greedy-up}
        \item Failure probability: same as \cref{thm:no-sublinear-greedy-up}
        \item Confidence bounds: same as \cref{thm:no-sublinear-greedy-up}
        \item Confidence parameter: same as \cref{thm:no-sublinear-greedy-up}
    \end{itemize}
    
    To be able to apply all previous results, we have to show that the corresponding assumptions are satisfied.
    \begin{itemize}
        \item[\cmark] \cref{lem:lemma-2-1} makes no assumptions.
        \item[\cmark] \cref{lem:lemma-2-2} assumes a reasonable \(\Fu\) and \(\A\) with
        \begin{equation*}
            F(\thetagre_t \mid \theta_{1:t-1}) \leq \Fu(\theta_t \mid \theta_{1:t-1})
            \quad\text{ for all } t \geq 1
        \end{equation*}
        as specified in \cref{eq:proofs-lemma-2-2-assumption}.
        It is clear that \nameref{sssec:design-objective-cs} provides a sufficient upper bound (see \cref{item:req-necessary-bound}) under the assumption that \(f(\varphi)\) lies between \(u_t(\varphi)\) and \(l_t(\varphi)\). With the refined definition of \(\Fu[CS]\) in \cref{eq:objective-function-cs-new} this is guaranteed with probability at least \(1-\delta\) by \cref{lem:lemma-2-4}.
        \TODO{REMOVABLE WEAK \(f(\varphi)\) might still escape in between discretization points?}
        Hence, the assumption of \cref{lem:lemma-2-2} is satisfied by our design choice \(\A[CS](\cdot)=\A(\cdot;\Fu[CS])\) with probability at least \(1-\delta\).
        \item[\cmark] \cref{lem:lemma-2-4,lem:lemma-2-5,lem:lemma-2-6} only make assumptions on the kernel, which are satisfied by our design choices as it is the case for \cref{thm:no-sublinear-greedy-up}.
        \item[\cmark] \cref{lem:lemma-2-3-greedy-u} only makes assumptions on the confidence bounds and confidence parameter, which are satisfied by our design choices as it is the case for \cref{thm:no-sublinear-greedy-up}.
    \end{itemize}
    
    Sublinear regret follows as for \cref{thm:sublinear-greedy-u} together with \cref{lem:lemma-2-3-greedy-cs} and the instantiation \(\inst{\Fu}{\Fu[CS]}\) and \(\inst{\theta}{\theta[CS]}\). We obtain
    \begin{equation*}
        R(T) \leq \bigO*{T^{\frac{2\nu+2}{4\nu+2}}\log(T)^{\frac{4\nu+1}{4\nu+2}}}
    \end{equation*}
    with probability at least \(1-\delta\).
\end{proof}

\subsection{\cref{lem:lemma-2-3-twophase-csu}}\label{ssec:proofs-lemma-2-3-twophase-csu}

The goal of this proof is to show the last mile towards sublinear regret for the specific two-phase algorithm using the redefined \nameref{sssec:design-objective-cs} objective from \cref{eq:objective-function-cs-new} for phase 1 and the objective \nameref{sssec:design-objective-u} for phase 2. It closely depends on the proof for \cref{lem:lemma-2-3-greedy-u}.

\begin{proof}
    We assume that the assumptions for \cref{lem:lemma-2-4,lem:lemma-2-5} are satisfied. Let \(u_t(\varphi)\) and \(l_t(\varphi)\) be defined according to \cref{eq:confidence-bounds-actual} with confidence parameter \(\beta_t\) chosen as in \cref{lem:lemma-2-4}. 

    We first derive the relation between \(\A[\TCSU]\) and \(\A[U]\), which then allows us to directly reuse \cref{lem:lemma-2-3-greedy-u}. Let \(\theta_{1:t-1}\) be arbitrary.
    \begin{derivation}
        && \Fu[U](\theta[\TCSU]_t\mid\theta_{1:t-1})
        &= \max_{\theta\in\PhiS\paren{\theta[CS]_t}} \Fu[U](\theta\mid\theta_{1:t-1})
        \\ && &\leq \eqcontrast \max_{\theta\in\eqchange{\Domain}} \Fu[U](\theta\mid\theta_{1:t-1}) \notag
        \\ && &= \eqcontrast \Fu[U]\paren*{\eqchange{\theta[U]_t}\mid\theta_{1:t-1}}
        \\
        \implies\quad
        && \eqcontrast \eqchange{\sum_{t=1}^T} \Fu[U](\theta[\TCSU]_t\mid\theta_{1:t-1})
        &\leq \eqcontrast \eqchange{\sum_{t=1}^T} \Fu[U](\theta[U]_t \notag\mid\theta_{1:t-1})
        \\ && &\leq \eqcontrast \begin{aligned}[t]
            &\eqchange{\frac{2\dmax\abs{\PhiS}}{h^2} \sqrt{\frac{3}{\log*{\sigmaeps^{-2} + 1}}}\sqrt{T\beta_T\gamma_T}} \\
            &\eqchange{+ \frac{\dmax\abs{\PhiS}}{h^2}\frac{\pi}{\sqrt{3}}}
        \end{aligned}
    \end{derivation}
    \begin{justification}
        \item by \cref{def:twophase-algorithm} (definition of two-phase algorithm)
        \item by \cref{def:greedy-algorithm} (definition of greedy algorithm)
        \item by \cref{lem:lemma-2-3-greedy-u} \TODO{REMOVABLE WEAK shouldn't be \(\theta_t\) and \(\theta_{1:t-1}\) be sequentially the same?}
        \qedhere
    \end{justification}
\end{proof}

\subsection{\fullref{thm:sublinear-twophase-csu}}\label{ssec:proofs-theorem-sublinear-twophase-csu}

The goal of this proof is to show sublinear regret for the specific two-phase algorithm using the redefined \nameref{sssec:design-objective-cs} objective from \cref{eq:objective-function-cs-new} for phase 1 and the objective \nameref{sssec:design-objective-u} for phase 2. It closely follows the derivation structure for \cref{thm:no-sublinear-greedy-up} and \cref{thm:sublinear-greedy-u}.

\begin{proof}
    We make the following design choices:
    \begin{itemize}
        \item Kernel: same as \cref{thm:no-sublinear-greedy-up}
        \item Objective: \(\Fu[CS]\) and \(\Fu[U]\)  (see \nameref{sssec:design-objective-u} and \cref{eq:objective-function-cs-new})
        \item Algorithm: \(\A(\cdot;\Fu[1],\Fu[2])\) (see \cref{def:twophase-algorithm})
        \item Failure probability: same as \cref{thm:no-sublinear-greedy-up}
        \item Confidence bounds: same as \cref{thm:no-sublinear-greedy-up}
        \item Confidence parameter: same as \cref{thm:no-sublinear-greedy-up}
    \end{itemize}
    
    To be able to apply all previous results, we have to show that the corresponding assumptions are satisfied.
    \begin{itemize}
        \item[\cmark] \cref{lem:lemma-2-1} makes no assumptions.
        \item[\cmark] \cref{lem:lemma-2-2} assumes a reasonable \(\Fu\) and \(\A\) with
        \begin{equation*}
            F(\thetagre_t \mid \theta_{1:t-1}) \leq \Fu(\theta_t \mid \theta_{1:t-1})
            \quad\text{ for all } t \geq 1
        \end{equation*}
        as specified in \cref{eq:proofs-lemma-2-2-assumption}.
        Clearly, \nameref{sssec:design-objective-cs} provides a sufficient upper bound (see \cref{item:req-necessary-bound}) under the assumption that \(f(\theta)\) lies between \(u_t(\theta)\) and \(l_t(\theta)\). With the refined definition of \(\Fu[CS]\) in \cref{eq:objective-function-cs-new} this is guaranteed with probability at least \(1-\delta\) by \cref{lem:lemma-2-4}.
        We derive
        \begin{derivation}
            \Fu(\thetagre_t\mid\theta_{1:t-1})
            &\leq \eqcontrast \eqchange{\Fu[CS]}(\thetagre_t\mid\theta_{1:t-1})
            \\ &\leq \eqcontrast \Fu[CS](\eqchange{\theta[CS]_t}\mid\theta_{1:t-1})
            \\ &= \eqcontrast \frac{1}{h^2} \sum_{\varphi\in\brack{\PhiS(\theta[CS]_t)}_t} \frac{1}{2} \paren*{u_t(\varphi)^2-l_t(\varphi)^2} \frac{\abs{\PhiS}}{\abs{\brack{\PhiS}_t}}
            \\ &\leq \eqcontrast \frac{1}{h^2} \eqchange{\abs*{\brack{\PhiS}_t}} \eqchange{\max_{\varphi\in\brack{\PhiS(\eqchange{\theta[CS]_t})}_t}} \frac{1}{2} \paren*{u_t(\varphi)^2-l_t(\varphi)^2} \frac{\abs{\PhiS}}{\abs{\brack{\PhiS}_t}} \notag
            \\ &= \eqcontrast \max_{\varphi\in\brack{\PhiS(\theta[CS]_t)}_t} \eqchange{\frac{1}{h^2}} \cdot \frac{1}{2} \eqchange{\abs*{\PhiS}} \paren*{u_t(\varphi)^2-l_t(\varphi)^2} \notag
            \\ &= \eqcontrast \max_{\eqchange{\theta}\in\brack{\PhiS(\theta[CS]_t)}_t} \eqchange{\Fu[U](\theta\mid\theta_{1:t-1})}
            \\ &= \eqcontrast \Fu[U](\eqchange{\theta[\TCSU]_t}\mid\theta_{1:t-1})
        \end{derivation}
        \begin{justification}
            \item by sufficient upper bound of \nameref{sssec:design-objective-cs} from \cref{eq:objective-function-cs-new}
            \item by phase 1 of \cref{def:twophase-algorithm} (definition of two-phase algorithm)
            \item by definition of \nameref{sssec:design-objective-cs}
            \item by definition of \nameref{sssec:design-objective-u}
            \item by phase 2 of \cref{def:twophase-algorithm} (definition of two-phase algorithm)
        \end{justification}
        \TODO{REMOVABLE WEAK discretization of summation interval not considered, maybe redefine twophase algorithm?}
        Hence, the assumption of \cref{lem:lemma-2-2} is satisfied by our design choice \(\A[\TCSU](\cdot)=\A(\cdot;\Fu[CS],\Fu[U])\) with probability at least \(1-\delta\).
        \item[\cmark] \cref{lem:lemma-2-4,lem:lemma-2-5,lem:lemma-2-6} only make assumptions on the kernel, which are satisfied by our design choices as it is the case for \cref{thm:no-sublinear-greedy-up}.
        \item[\cmark] \cref{lem:lemma-2-3-greedy-u} only makes assumptions on the confidence bounds and confidence parameter, which are satisfied by our design choices as it is the case for \cref{thm:no-sublinear-greedy-up}.
    \end{itemize}
    
    Sublinear regret follows as for \cref{thm:sublinear-greedy-u} together with \cref{lem:lemma-2-3-twophase-csu} and the instantiation \(\inst{\Fu}{\Fu[U]}\) and \(\inst{\theta}{\theta[\TCSU]}\). We obtain
    \begin{equation*}
        R(T) \leq \bigO*{T^{\frac{2\nu+2}{4\nu+2}}\log(T)^{\frac{4\nu+1}{4\nu+2}}}
    \end{equation*}
    with probability at least \(1-\delta\).
\end{proof}


\newpage
\section{Auxiliary Proofs}\label{sec:proofs-auxiliary}

\begin{lemma}\label{lem:lemma-aux-x_leq_c_logx_plus_1}
    \(x \leq \frac{c}{\log(c+1)} \log(x+1)\) for all \(x\in[0,c]\) with \(c \geq 0\).
\end{lemma}
\begin{proof}
    The inequality \(\log(x+1) \leq x\) is commonly known, but we want the reversed inequality
    \begin{equation*}
        x \leq C \cdot \log(x+1)
        \quad\text{ for all } x \in [0,c]
    \end{equation*}
    to hold by scaling the logarithm with an appropriate \(C \geq 1\). We already know that the inequality is satisfied exactly for \(x=0\). If we can show that both sides are also equal for \(x=c\), it is intuitive that the inequality holds for \(x\in[0,c]\) due to concavity of the logarithm. Plugging \(x=c\) into the inequality yields \(c \leq C\cdot \log(c+1)\) and tells us that it is satisfied exactly for \(C = \frac{c}{\log(c+1)}\).
    
    The rigorous proof goes as follows. Let \(c \geq 0\) and \(x\in[0,c]\) arbitrary.
    \begin{derivation}
        && x = (1-\alpha) \cdot 0 + \alpha \cdot c = \alpha \cdot c \in [0,c]
        \quad\with \alpha \in [0,1] \notag
        \span
        \\
        \implies\quad
        && \eqcontrast (1 - \alpha) \eqchange{\log(\eqnochange{0}+1)} + \alpha \eqchange{\log(\eqnochange{c}+1)}
        &\leq \eqcontrast \eqchange{\log(\eqnochange{(1-\alpha) \cdot 0 + \alpha \cdot c} + 1)}
        \\
        \implies\quad
        && \eqcontrast \alpha \log(c+1)
        &\leq \eqcontrast \log(\alpha \cdot c + 1) \notag
        \\
        \implies\quad
        && \eqcontrast \alpha \eqchange{c}
        &\leq \eqcontrast \frac{\eqchange{c}}{\eqchange{\log(c+1)}} \log(\alpha \cdot c + 1)
        \\
        \implies\quad
        && \eqcontrast \eqchange{x}
        &\leq \eqcontrast \frac{c}{\log(c+1)} \log(\eqchange{x} + 1)
    \end{derivation}
    \begin{justification}
        \item by concavity of logarithm
        \item since \(c \geq 0\)
        \item since \(x = \alpha c\)
        \qedhere
    \end{justification}
\end{proof}

\begin{lemma}\label{lem:lemma-aux-A_plus_B_inv_pd}
    \((A+B)^{-1}\) is positive definite with \(A \in \Real^{n,n}\) positive semi-definite and \(B \in \Real^{n,n}\) positive definite.
\end{lemma}
\begin{proof}
    Let \wloog{} \(A \in \Real^{n,n}\) be positive semi-definite and \(B \in \Real^{n,n}\) positive definite.
    \begin{derivation}
        & v^TAv \geq 0 \quad\text{and}\quad v^TBv > 0
        &&\quad\text{ for all } v\in\Real^n\setminus\set{0}
        \\ \implies\quad
        & v^T(A+B)v = v^TAv + v^TBv > 0
        &&\quad\text{ for all } v\in\Real^n\setminus\set{0} \notag
        \\ \implies\quad
        & A+B \text{ positive definite}
        \\ \implies\quad
        & \lambda_i(A+B) > 0
        &&\quad\text{ for all } i\in\set{1,\dots,n}
        \\ \implies\quad
        & \lambda_i\paren*{(A+B)^{-1}} = \frac{1}{\lambda_i(A+B)} > 0
        &&\quad\text{ for all } i\in\set{1,\dots,n}
        \\ \implies\quad
        & (A+B)^{-1} \text{ positive definite}
    \end{derivation}
    \begin{justification}
        \item by definition of positive (semi-)definiteness
        \item by definition of positive definiteness
        \item by property of positive definiteness
        \item by eigenvalues of inverse matrix
        \item by property of positive definiteness
        \qedhere
    \end{justification}
\end{proof}

\begin{lemma}\label{lem:lemma-aux-det_A_plus_I}
    \(\det(A+I) = \prod_{i=1}^n (\lambda_i(A) + 1)\) with \(A \in \Real^{n,n}\) diagonalizable.
\end{lemma}
\begin{proof}
    Let \(A \in \Real^{n,n}\) be a diagonalizable matrix.
    \begin{derivation}
        \det(A+I) &= \eqcontrast \det(\eqchange{V\Lambda V^{-1}} + \eqchange{VV^{-1}})
        \quad\with \Lambda \text{ diagonal}
        \\ &= \eqcontrast \det(\eqchange{V(\Lambda + I)V^{-1}}) \notag
        \\ &= \eqcontrast \eqchange{\det(\eqnochange{V})\det(\eqnochange{\Lambda + I})\det(\eqnochange{V^{-1}})}
        \\ &= \eqcontrast \det(\Lambda + I)
        \\ &= \eqcontrast \eqchange{\prod_{i=1}^n (\lambda_i(A) + 1)}
    \end{derivation}
    \begin{justification}
        \item by eigendecomposition of diagonalizable matrix
        \item by determinant of product
        \item by determinant of inverse matrix
        \item by determinant of diagonal matrix
        \qedhere
    \end{justification} 
\end{proof}

\begin{lemma}\label{lem:lemma-aux-a_plus_xsqr_leq_1_plus_x_sqr}
    \((c+x^2)^{-\beta} \leq \paren*{1+\frac{1}{c}}^\beta (1+x)^{-2\beta}\) for all \(x\neq -1\) with \(c > 0,\beta\geq 0\).
\end{lemma}
\begin{proof}
    The idea of this inequality reduces to
    \begin{equation*}
        c+x^2 \geq C \cdot (1+x)^2
    \end{equation*}
    with \(c>0\), which should hold for some \(C\in\Real\) which scales down the parabola \(\x\)-shifted by \(-1\), such that it lies below the parabola \(\y\)-shifted by \(c\) for all \(x\in\Real\). The goal is to choose \(C\) largest possible. We show that this inequality is satisfied exactly for \(C=\frac{c}{c+1}\).
    
    Let \(c > 0,\beta\geq 0\) and \(x\neq -1\) arbitrary.
    \begin{derivation}
        && c+x^2 - \frac{c}{c+1}(1+x)^2
        &= \eqcontrast \paren*{1\eqchange{-\frac{c}{c+1}}}\eqchange{x^2} \eqchange{- 2\frac{c}{c+1}x} + \paren*{c\eqchange{-\frac{c}{c+1}}} \notag
        \\ && &= \eqcontrast \eqchange{\frac{1}{c+1}}x^2 - \frac{1}{c+1}\eqchange{2c}x + \eqchange{\frac{1}{c+1}c^2} \notag
        \\ && &= \eqcontrast \frac{1}{c+1} \eqchange{(x-c)^2} \notag
        \\ && &\geq \eqcontrast \eqchange{0}
        \\ \implies\quad && \eqcontrast c + x^2
        &\geq \eqcontrast \eqchange{\frac{c}{c+1}(1+x)^2} \notag
        \\ \implies\quad && \eqcontrast (c+x^2)^{\eqchange{-\beta}}
        &\leq \eqcontrast \paren*{\eqchange{1+\frac{1}{c}}}^{\eqchange{\beta}} (1+x)^{\eqchange{-\eqnochange{2}\beta}}
    \end{derivation}
    \begin{justification}
        \item since \(c > 0\)
        \item since \(\beta \geq 0\) and both sides nonzero as
        \begin{align*}
            c > 0 &\implies c+x^2 > 0 \text{ and } \frac{c}{c+1} \geq 0
            \\ x \neq -1 &\implies (1+x)^2 > 0
            \tag*{\qedhere}
        \end{align*}
    \end{justification}
\end{proof}



\begin{lemma}\label{lem:lemma-aux-x_plus_a_sqr_leq_xsqr_plus_asqr}
    \((x+a)^2 \leq cx^2 + \frac{c}{c-1}a^2\) for all \(x\in\Real\) with \(c > 1, a\in\Real\).
\end{lemma}
\begin{proof}
    The idea of this inequality is to show
    \begin{equation*}
        (x+a)^2 \leq C_1 \cdot x^2 + C_2
        ,
    \end{equation*}
    which should hold for some \(C_1\in\Real\) which scales up the parabola \(\y\)-shifted by some \(C_2\in\Real\), such that it lies above the parabola \(\x\)-shifted by \(-a\) for all \(x\in\Real\). The goal is to choose \(C_1\) and \(C_2\) smallest possible. We show that this inequality is satisfied exactly for \(C_1 = c\) and \(C_2 = \frac{c}{c-1}\), where \(c>1\) can be freely chosen and trade-offs the size of the additive and multiplicative factor.

    Let \(c > 1, a\in\Real\) and \(x\in\Real\) arbitrary.
    \begin{derivation}
        && cx^2 + \frac{c}{c-1}a^2 - (x+a)^2
        &= \eqcontrast cx^2 + \frac{c}{c-1}a^2 - \eqchange{(x^2+2ax+a^2)} \notag
        \\ && &= \eqcontrast (c\eqchange{-1})x^2 \eqchange{- 2ax} + \frac{\eqchange{1}}{c-1}a^2 \notag
        \\ && &= \eqcontrast \eqchange{\paren*{\sqrt{c-1}x - \frac{1}{\sqrt{c-1}}a}^2}
        \\ && &\geq \eqcontrast \eqchange{0} \notag
        \\
        \implies\quad
        && (x+a)^2 &\leq cx^2 + \frac{c}{c-1}a^2 \notag
    \end{derivation}
    \begin{justification}
        \item since \(c > 1\)
        \qedhere
    \end{justification}
\end{proof}

\begin{lemma}\label{lem:lemma-aux-cauchy_schwarz}
    \((\sum_{i=1}^n x_i)^2 \leq n \cdot \sum_{i=1}^n x_i^2\) for all \(x\in\Real^n\) and \(n\in\Natural\).
\end{lemma}
\begin{proof}
    Let \(n\in\Natural\) and \(x\in\Real^n\) arbitrary.
    \begin{derivationinline}
        \paren*{\sum_{i=1}^n x_i}^2
        = \eqcontrast \paren*{\sum_{i=1}^n \eqchange{1 \cdot} x_i}^2 \eqnocontrast
        = \eqcontrast \eqchange{\abs{\inner{x,\1_n}}^2} \eqnocontrast
        \tageq{\leq} \eqcontrast \eqchange{\norm{x}^2\norm{\1_n}^2} \eqnocontrast
        = \eqcontrast \eqchange{n} \cdot \eqchange{\sum_{i=1}^n x_i^2} \eqnocontrast
    \end{derivationinline}
    \begin{justification}
        \item by Cauchy-Schwarz inequality
        \qedhere
    \end{justification}
\end{proof}

%% file: B_simulation.tex
\chapter{Simulation Framework}\label{chp:simulation}

The simulation framework formed an integral part of this work to empirically guide the design of algorithms prior their theoretical analysis and then to verify the theoretical results after the analysis. This framework is published open source at
\begin{center}\unskip
    \href{https://github.com/danielyxyang/active_reconstruction}{github.com/danielyxyang/active\_reconstruction}
\end{center}\unskip
and consists of an interactive simulation of different algorithms as well as an interactive environment for running multiple simulation experiments.

In this chapter, we briefly review some parts of the simulation framework which are relevant to this written thesis, but leave the remaining information up to the public codebase.
In \cref{sec:simulation-background} we provide mathematical background related to the geometries of the real and polar world which is used by \cref{sec:design-objective} for the design of objective functions.
In \cref{sec:simulation-objects} we present the set of objects on which we evaluate our algorithms in \cref{sec:experiments-results}.


\section{Mathematical Background}\label{sec:simulation-background}

\newcommand{\xw}{x}
\newcommand{\yw}{y}
\newcommand{\rw}{r}
\newcommand{\varphiw}{\varphi}

\newcommand{\xtilde}{\tilde{x}}
\newcommand{\ytilde}{\tilde{y}}
\newcommand{\varphitilde}{\tilde{\varphi}}
\newcommand{\rtilde}{\tilde{r}}

\newcommand{\xc}{x^{(c)}}
\newcommand{\yc}{y^{(c)}}
\newcommand{\rc}{r^{(c)}}
\newcommand{\varphic}{\varphi^{(c)}}

In this section, we derive closed-form expressions for quantities which characterize the shape of the FOV in the polar world. We first derive a polar function parameterizing the rays cast by the camera, which is used by \cref{eq:fov-boundary}, in \cref{ssec:simulation-background-camray}. Then we derive an expression for the polar angles of the two FOV boundary endpoints, which is used by \cref{eq:sum-interval-simple}, in \cref{ssec:simulation-background-fovendpoint}.

\subsection{Camera Ray Function}\label{ssec:simulation-background-camray}

The goal is to find a polar function defined in the polar world coordinate system, which parameterizes rays cast by the camera at the position \(\theta\) and with an casting angle \(\alpha\) relative to the camera's LOS as seen in \cref{fig:problem-different-fovs}.

The idea is to describe the rays with line equations in the Cartesian world coordinate system and to transform the line equations into polar functions. Note that the slope of a ray cast at angle \(\alpha\) relative to the camera's LOS has angle \(\theta+\alpha\) in the world coordinate system.

Assume \(\theta+\alpha \in_\pi \brack*{-\frac{\pi}{4},\frac{\pi}{4}}\) which means that the ray changes slower in its \(y\)-world coordinate than in its \(x\)-world coordinate. This allows us to use
\begin{equation}
    y = mx + b
    \quad\with m=\tan(\theta+\alpha) \text{ and } b = y_{cam} - mx_{cam}
    .
\end{equation}
After instantiating \(\inst{x}{r(\varphi)\cos(\varphi)}\) and \(\inst{y}{r(\varphi)\sin(\varphi)}\), we solve for \(r(\varphi)\) which gives us the camera ray function
\begin{equation}
    \ray{\theta,\alpha}{\varphi}
    = \frac{b}{\sin(\varphi) - m\cos(\varphi)}
    = \dcam \frac{\sin(\varphi) - \tan(\theta+\alpha)\cos(\varphi)}{\sin(\varphi) - \tan(\theta+\alpha)\cos(\varphi)}
    .
\end{equation}
In case \(\theta+\alpha \notin_\pi \brack*{-\frac{\pi}{4},\frac{\pi}{4}}\), the derivation goes similar with the line equation \(x=my+b\). The final camera ray function is then given as
\begin{equation}\label{eq:camera-ray}
    \ray{\theta,\alpha}{\varphi} \defeq \begin{dcases}
        \dcam  \frac{\sin(\theta)-\tan(\theta+\alpha)\cos(\theta)}{\sin(\varphi)-\tan(\theta+\alpha)\cos(\varphi)},
        & \textstyle \theta + \alpha \in_\pi
        \brack*{-\frac{\pi}{4},\frac{\pi}{4}}
        \\
        \dcam  \frac{\cos(\theta)-\cot(\theta+\alpha)\sin(\theta)}{\cos(\varphi)-\cot(\theta+\alpha)\sin(\varphi)},
        & \text{otherwise}
        .
    \end{dcases}
\end{equation}
Note that the case distinction is only made once for each ray depending on \(\theta+\alpha\) and is only relevant for numerical computation, since both expressions are equivalent in the end.

\subsection{FOV Boundary Endpoint}\label{ssec:simulation-background-fovendpoint}

The goal is to find the polar angles of the two FOV boundary endpoints characterized by the quantities \(\frac{\FOV}{2}\) and \(\DOF\).

To this end, we define the \emph{camera coordinate system} to be centered at the camera's current position and its \(x\)-axis aligned with the camera's LOS, which corresponds to an angle of \(\theta+\pi\) relative to the world coordinate system.
In this coordinate system, the left and right FOV boundary can be parameterized with the parametric functions \(r \mapsto \paren*{\frac{\FOV}{2}, r}\) and \(r \mapsto \paren*{-\frac{\FOV}{2}, r}\) with \(r\in[0,\DOF]\) as the distance to the camera. The goal is to transform the polar coordinates given by these parametric functions into polar coordinates in the world coordinate system.

We first define the following notations for the coordinates relative to the world and camera coordinate system:
\begin{equation*}
    \begin{aligned}
        (\xw, \yw) &\defeq \text{Cartesian world coordinates}\\
        (\varphiw, \rw) &\defeq \text{polar world coordinates}\\
        (\xc, \yc) &\defeq \text{Cartesian camera coordinates}\\
        (\varphic, \rc) &\defeq \text{polar camera coordinates}
    \end{aligned}
\end{equation*}
Assume that the camera is currently located at \(\theta\) with the Cartesian world coordinates given as
\begin{equation}\label{eq:simulation-background-camera}
    \begin{aligned}
        \xw_{cam} &= \dcam \cos(\theta)\\
        \yw_{cam} &= \dcam \sin(\theta)
        .
    \end{aligned}
\end{equation}
Given some fixed point with polar world coordinates \((\rw,\varphiw)\), the \emph{point's Cartesian world coordinates} are
\begin{equation}\label{eq:simulation-background-point}
    \begin{aligned}
        \xw &= \rw \cos(\varphiw)\\
        \yw &= \rw \sin(\varphiw)
        .
    \end{aligned}
\end{equation}
Recall, that the transformation of the world to the camera coordinate system corresponds to shifting the world coordinate system with \((\xw_{cam},\yw_{cam})\) and rotating it with \(\theta+\pi\). Since the actual location \((\xw, \yw)\) of the fixed point must be invariant under this transformation of the coordinate system, we apply the inverse transformation to the coordinates of this point by shifting it with \((-\xw_{cam},-\yw_{cam})\) and rotating it with \(-(\theta+\pi)\). This gives us the \emph{point's Cartesian camera coordinates}
\begin{equation}\label{eq:simulation-background-cam-cartesian}
    \begin{aligned}
        \begin{pmatrix} \xc \\ \yc \end{pmatrix}
        &= R^{-1}(\theta+\pi) \begin{pmatrix}
            \xw - \xw_{cam} \\ \yw - \yw_{cam}
        \end{pmatrix}
        \\ &= \begin{pmatrix}
            -\cos(\theta) & -\sin(\theta) \\
             \sin(\theta) & -\cos(\theta)
        \end{pmatrix} \begin{pmatrix}
            \xw - \xw_{cam} \\ \yw - \yw_{cam}
        \end{pmatrix}
    \end{aligned}
\end{equation}
and plugging in \cref{eq:simulation-background-point,eq:simulation-background-camera} yields
\begin{derivation*}
    \xc
    &= \begin{aligned}[t]
        &-\cos(\theta) \cdot \paren*{\rw\cos(\varphiw) - \dcam\cos(\theta)} \\
        &-\sin(\theta) \cdot \paren*{\rw\sin(\varphiw) - \dcam\sin(\theta)}
    \end{aligned}
    \\ &= \dcam \paren*{\cos(\theta)^2 + \sin(\theta)^2} - \rw\paren*{\cos(\theta)\cos(\varphiw) + \sin(\theta)\sin(\varphiw)}
    \\ &= \dcam - \rw\cos(\theta-\varphiw)
\intertext{and similarly}
    \yc
    &= \begin{aligned}[t]
        &\sin(\theta) \cdot \paren*{\rw\cos(\varphiw) - \dcam\cos(\theta)} \\
        &-\cos(\theta) \cdot \paren*{\rw\sin(\varphiw) - \dcam\sin(\theta)}
    \end{aligned}
    \\ &= \rw\paren*{\sin(\theta)\cos(\varphiw) - \cos(\theta)\sin(\varphiw)}
    \\ &= \rw\sin(\theta-\varphiw)
    .
\end{derivation*}
Finally, we obtain the \emph{point's polar camera coordinates}
\begin{equation}\label{eq:simulation-background-cam-polar}
    \begin{aligned}
        \varphic &= \arctan*{\frac{\yc}{\xc}} 
        = \arctan*{\frac{\rw\sin(\theta-\varphiw)}{\dcam - \rw\cos(\theta-\varphiw)}} \\
        \rc &= \sqrt{\paren*{\xc}^2 + \paren*{\yc}^2}
        .
    \end{aligned}
\end{equation}
Since the camera always faces the world center, \(\varphic\) stays inside \(\paren*{-\frac{\pi}{2},\frac{\pi}{2}}\) and the \(\arctan\) does not need additional handling.

To obtain the reversed transformation from the camera to the world coordinate system, which is what we are interested in, one can reuse the result from \cref{eq:simulation-background-cam-polar}. This result provides us with the transformation of polar coordinates between two different polar coordinate systems under the assumption that the center of the target coordinate system (previously the camera's) is located at angle \(\theta\) relative to the source coordinate system (previously the world's). In addition, the target's \(x\)-axis (previously the camera's LOS) is directed towards the center of the source coordinate system.

By swapping the source and target coordinate system we know that the world center, now our target, is located at \(\theta=0\) relative the camera's coordinate system, which is now our source. In addition, we assume that the \(x\)-axis of the world coordinate system is similarly directed towards the camera. This assumed world coordinate system corresponds to rotating the actual world coordinate system by \(\theta\). The resulting Cartesian and polar coordinates in this assumed world coordinate system are then given by \cref{eq:simulation-background-cam-cartesian,eq:simulation-background-cam-polar} as
\begin{equation}
    \begin{aligned}
        \xtilde &= \dcam - \rc\cos(-\varphic)
        \\ \ytilde &= \rc\sin(-\varphic)
        \\ \varphitilde &= \arctan*{\frac{\ytilde}{\xtilde}}
        \\ \rtilde &= \sqrt{\xtilde^2 + \ytilde^2}
        .
    \end{aligned}
\end{equation}
To obtain the polar coordinates in the actual world coordinate system, we simply add \(\theta\) to the polar angle and keep the radial distance the same. We obtain
\begin{equation}
    \begin{aligned}
        \varphiw &= \theta + \varphitilde
        = \theta - \arctan*{\frac{\rc\sin(\varphic)}{\dcam - \rc\cos(\varphic)}}
        \\ \rw &= \rtilde
        .
    \end{aligned}
\end{equation}

Hence, the endpoints of the FOV boundary with the polar camera coordinates \(\paren*{\frac{\FOV}{2}, \DOF}\) and \(\paren*{-\frac{\FOV}{2}, \DOF}\) have the world polar angles
\begin{equation*}
    \begin{aligned}
        \varphiw_1 &= \theta - \arctan*{\frac{\DOF\sin(\FOV/2)}{\dcam - \DOF\cos(\FOV/2)}}\\
        \varphiw_2
        &= \theta + \arctan*{\frac{\DOF\sin(\FOV/2)}{\dcam - \DOF\cos(\FOV/2)}}
    \end{aligned}
\end{equation*}
as used in \cref{eq:sum-interval-simple} for the endpoints of the simple FOV endpoint summation interval.


\newpage
\section{Set of Evaluation Objects}\label{sec:simulation-objects}

\begin{figure}[h!]
    \centering
    \begin{subfigure}{\linewidth}
        \centering
        \includegraphics[width=0.2\linewidth]{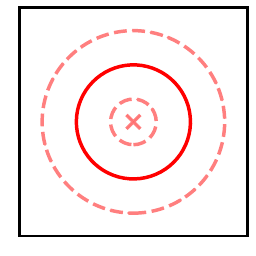}%
        \includegraphics[width=0.2\linewidth]{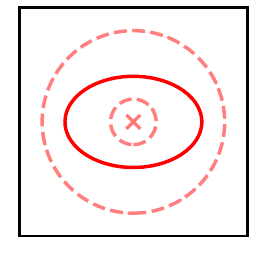}%
        \includegraphics[width=0.2\linewidth]{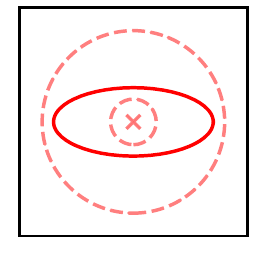}%
        \includegraphics[width=0.2\linewidth]{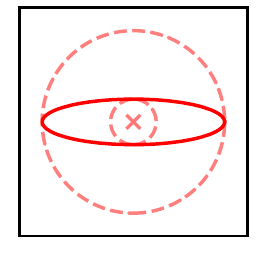}
        \caption{ellipse objects}
    \end{subfigure}
    \par\smallskip
    \begin{subfigure}{\linewidth}
        \centering
        \includegraphics[width=0.2\linewidth]{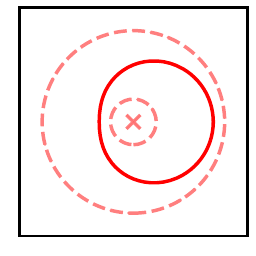}%
        \includegraphics[width=0.2\linewidth]{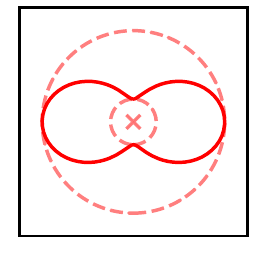}\\
        \includegraphics[width=0.2\linewidth]{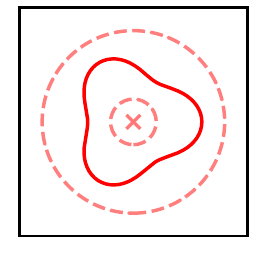}%
        \includegraphics[width=0.2\linewidth]{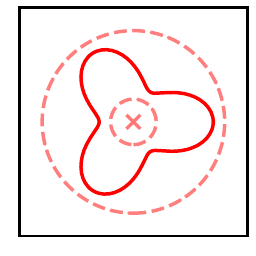}%
        \includegraphics[width=0.2\linewidth]{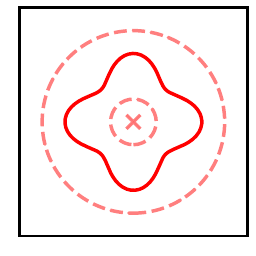}%
        \includegraphics[width=0.2\linewidth]{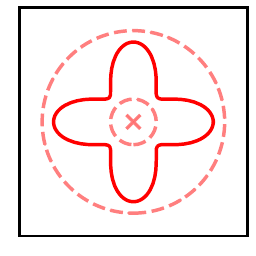}\\
        \includegraphics[width=0.2\linewidth]{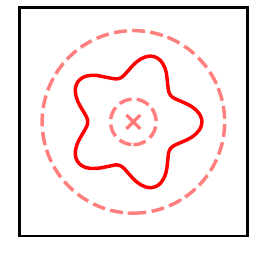}%
        \includegraphics[width=0.2\linewidth]{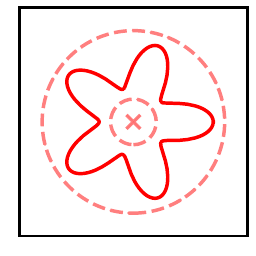}%
        \includegraphics[width=0.2\linewidth]{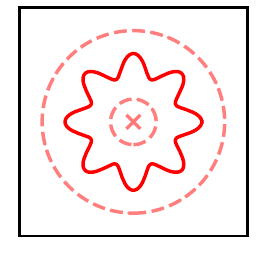}%
        \includegraphics[width=0.2\linewidth]{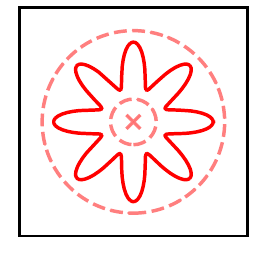}
        \caption{flower objects}
    \end{subfigure}
    \par\smallskip
    \begin{subfigure}{\linewidth}
        \centering
        \includegraphics[width=0.2\linewidth]{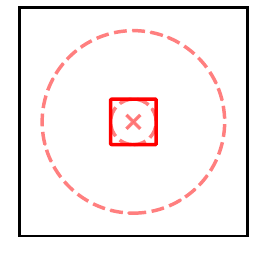}%
        \includegraphics[width=0.2\linewidth]{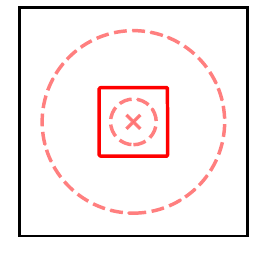}%
        \includegraphics[width=0.2\linewidth]{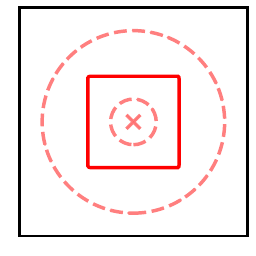}%
        \includegraphics[width=0.2\linewidth]{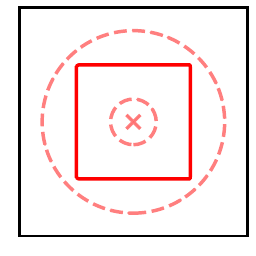}
        \caption{square objects}
    \end{subfigure}
    \par\smallskip
    \begin{subfigure}{\linewidth}
        \centering
        \includegraphics[width=0.2\linewidth]{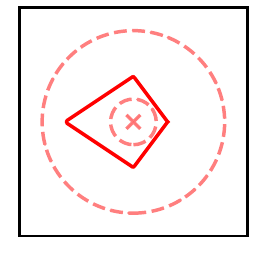}%
        \includegraphics[width=0.2\linewidth]{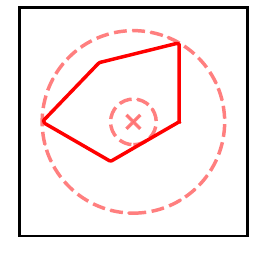}%
        \includegraphics[width=0.2\linewidth]{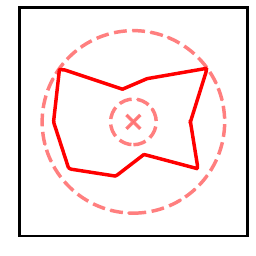}%
        \includegraphics[width=0.2\linewidth]{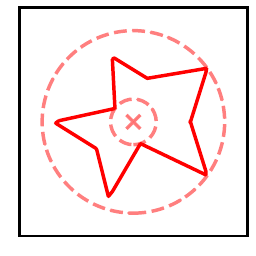}
        \caption{polygon objects}
    \end{subfigure}
    \caption[Set of Evaluation Objects]{
        Set of Evaluation Objects.
        These figures show the complete set of objects which we use in our experiments for evaluating our algorithms in \cref{sec:experiments-results}.
    }
    \label{fig:simulation-experiment-objects}
\end{figure}